%% file: example_paper.tex
\documentclass{article}

\usepackage{microtype}
\usepackage{graphicx}
\usepackage{booktabs} % for professional tables

\input{math_commands.tex}

\usepackage{amsfonts}       % blackboard math symbols
\usepackage{nicefrac}       % compact symbols for 1/2, etc.
\usepackage{enumitem}
\usepackage{titlesec}
\usepackage{times}
\usepackage{latexsym}
\usepackage{dsfont}
\usepackage{subfig}
\usepackage{pgfplots}
\usepackage{url}
\usepackage{booktabs}
\usepackage{multirow}
\usepackage[noend]{algpseudocode}
\usepackage{amsmath,amssymb,amsfonts}
\usepackage{bbm}
\usepackage{balance}
\usepackage{lipsum}
\usepackage{graphicx}
\usepackage{mwe}
\usepackage{sidecap}
\usepackage{wrapfig}
\usepackage{amsthm}
\usepackage[flushleft]{threeparttable}
\usepackage{pifont}% http://ctan.org/pkg/pifont
\usepackage{lipsum}% http://ctan.org/pkg/lipsum
\usepackage{mdframed}% http://ctan.org/pkg/mdframed
\usepackage{xcolor}% http://ctan.org/pkg/xcolor
\newcommand\algname{SPACE}
\newcommand\argFullName{Safe Policy Adaptation with Constrained Exploration}
\newcommand\diam{\mathrm{diam}}
\newtheorem{theorem}{Theorem}[section]

\newtheorem{assumption}{Assumption}
\newtheorem*{assumption*}{Assumption}
\newtheorem{lemma}[theorem]{Lemma}
\providecommand{\ie}{\emph{i.e.,} }
\providecommand{\eg}{\emph{e.g.,} }
\providecommand{\parab}[1]{\noindent\textbf{#1}}

\providecommand{\E}{\mathrm{E}}
\providecommand{\parab}[1]{\noindent\textbf{#1}}

\def\offpolicy{\includegraphics[height=0.8em]{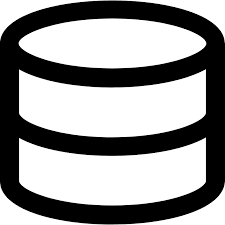}}
\def\onpolicy{\includegraphics[height=0.8em]{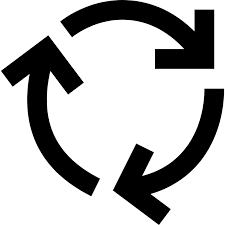}}
\def\demo{\includegraphics[height=0.8em]{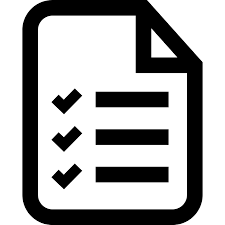}}
\def\cost{\includegraphics[height=0.8em]{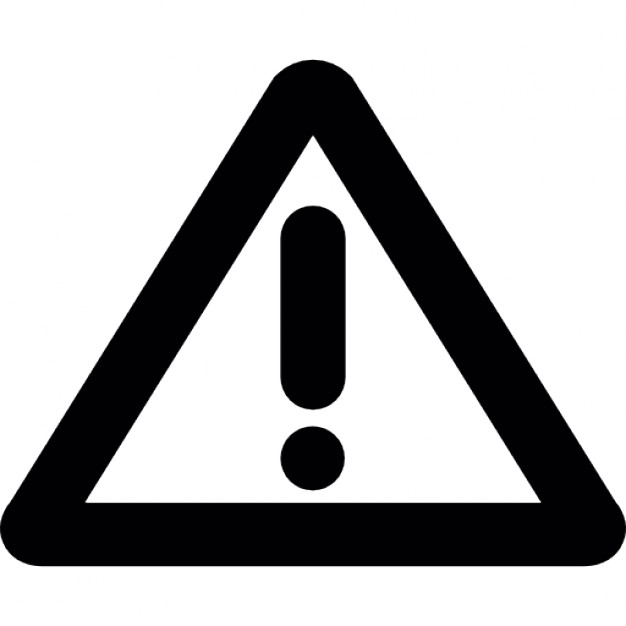}}
\def\reward{\includegraphics[height=0.9em]{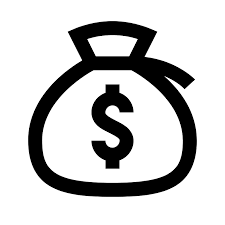}}

\usepackage{hyperref}

\usepackage[accepted]{icml2021}

\icmltitlerunning{Accelerating Safe Reinforcement Learning with Constraint-mismatched Baseline Policies}

\begin{document}

\twocolumn[
\icmltitle{Accelerating Safe Reinforcement Learning \\ with Constraint-mismatched Baseline Policies}

%\icmltitle{Safe Reinforcement Learning with Constraint-mismatched Baseline Policies}



\icmlsetsymbol{equal}{*}

\begin{icmlauthorlist}
\icmlauthor{Tsung-Yen Yang}{to}
\icmlauthor{Justinian Rosca}{goo}
\icmlauthor{Karthik Narasimhan}{to}
\icmlauthor{Peter J. Ramadge}{to}
\end{icmlauthorlist}

\icmlaffiliation{to}{Princeton University}
\icmlaffiliation{goo}{Siemens Corporation, Corporate Technology}

\icmlcorrespondingauthor{Tsung-Yen Yang}{ty3@princeton.edu}

\icmlkeywords{Safe reinforcement learning, learning from expert}

\vskip 0.3in
]



\printAffiliationsAndNotice{}  % leave blank if no need to mention equal contribution

\input{abstract}
\input{introduction}
\input{related.tex}
\input{preliminaries}
\input{model}
\input{implementation}

\input{results}

\input{conclusion}

\section*{Acknowledgements}
The authors would like to thank members of the Princeton NLP Group, the anonymous reviewers, and the area chair for their comments. 
Tsung-Yen Yang thanks Siemens Corporation, Corporate Technology for their support.

\nocite{langley00}

\bibliography{output}
\bibliographystyle{icml2021}

\newpage
\onecolumn
\appendix
\input{appendix_experiment.tex}

\end{document}

%% file: math_commands.tex
\usepackage{amsmath,amsfonts,bm}

\def\eqref#1{equation~\ref{#1}}
\def\1{\bm{1}}

\def\vtheta{{\bm{\theta}}}
\def\va{{\bm{a}}}

\def\vc{{\bm{c}}}

\def\vg{{\bm{g}}}

\def\vu{{\bm{u}}}
\def\vv{{\bm{v}}}

\def\vx{{\bm{x}}}

\def\mF{{\bm{F}}}

\def\mI{{\bm{I}}}

\def\mL{{\bm{L}}}
\def\mM{{\bm{M}}}

\DeclareMathAlphabet{\mathsfit}{\encodingdefault}{\sfdefault}{m}{sl}
\SetMathAlphabet{\mathsfit}{bold}{\encodingdefault}{\sfdefault}{bx}{n}

\newcommand{\E}{\mathbb{E}}

\newcommand{\R}{\mathbb{R}}

\newcommand{\KL}{D_{\mathrm{KL}}}

\DeclareMathOperator*{\argmax}{arg\,max}
\DeclareMathOperator*{\argmin}{arg\,min}

%% file: abstract.tex
% !TEX root = neurips_2019.tex
\begin{abstract}
We consider the problem of reinforcement learning when provided with (1) \textit{a baseline control policy} and (2) \textit{a set of constraints} that the learner must satisfy. The baseline policy can arise from demonstration data or a teacher agent and may provide useful cues for learning, but it might also be sub-optimal for the task at hand, and is not guaranteed to satisfy the specified constraints, which might encode safety, fairness or other application-specific requirements. 
In order to safely learn from baseline policies, we propose an iterative policy optimization algorithm that alternates between maximizing expected return on the task, minimizing distance to the baseline policy, and projecting the policy onto the constraint-satisfying set. We analyze our algorithm theoretically and provide a finite-time convergence guarantee. In our experiments on five different control tasks, our algorithm consistently outperforms several state-of-the-art baselines, achieving 10 times fewer constraint violations and 40\% higher reward on average.
\end{abstract}
%no math, no footnote, no abbreviation in the abstract

%% file: introduction.tex
% !TEX root = neurips_2019.tex
\section{Introduction}
\label{sec:introduction}
Deep reinforcement learning (RL) has achieved impressive results in several domains such as 
%Go
games \citep{mnih2013playing,silver2016mastering} and robotic control 
%tasks 
\citep{levine2016end,rajeswaran2017learning}.
However, in these complex applications, learning policies from scratch often requires tremendous amounts of time and computational power.
%
%This prevents from deploying RL systems in real applications. 
%
%To reduce the number of interactions of the environment, one alternative is to leverage a given baseline policy from either previous applications or a teacher.
To alleviate this issue, one would like to
%alternative is 
leverage a baseline policy available from demonstrations, a teacher or a previous task.
%
%However, 
%in many applications, 
However, the baseline policy may be sub-optimal for the new application and may not be guaranteed to produce actions that satisfy desired constraints 
%the learning agent’s cost constraints with the consideration 
on safety, fairness, or other costs.
For instance, when you drive an unfamiliar vehicle, you do so cautiously to ensure safety, while adapting your driving technique to the vehicle characteristics to improve your `driving reward'. 
%to an autonomous car may take a dangerous maneuver because of learning from an unsafe driving policy.
%
In effect, you (as the agent) gradually adapt a baseline policy (\ie prior driving skill) to avoid violating the constraints (\eg safety) while improving your driving reward (\eg travel time, fuel efficiency). 
%
%In this work, we address the problem of learning control policies under constraints while being guided by the baseline policy.

The problem of safely learning from baseline policies is challenging because directly leveraging the baseline policy, as in DAGGER \citep{ross2011reduction} or GAIL \citep{ho2016generative}, may result in 
%control 
policies that violate the constraints since the baseline is not guaranteed to satisfy them.
To ensure constraint satisfaction, prior work either adds a hyper-parameter weighted copy of the imitation learning (IL) objective (\ie imitating the baseline policy) to the RL objective~\citep{rajeswaran2017learning, gao2018reinforcement, hester2018deep}, or pre-trains a policy with the baseline policy (\eg use a baseline policy as an initial policy) and then fine-tunes it through RL~\citep{mulling2013learning, chernova2014robot}.
However, both approaches do not ensure constraint satisfaction on \textit{every} learning episode, which is an important feature of safe RL.
In addition, the policy initialized by a low entropy baseline policy may never explore.

{
\begin{figure*}[t]
%\vspace{-3mm}
\centering
(a)
{
\includegraphics[width=0.26\linewidth]{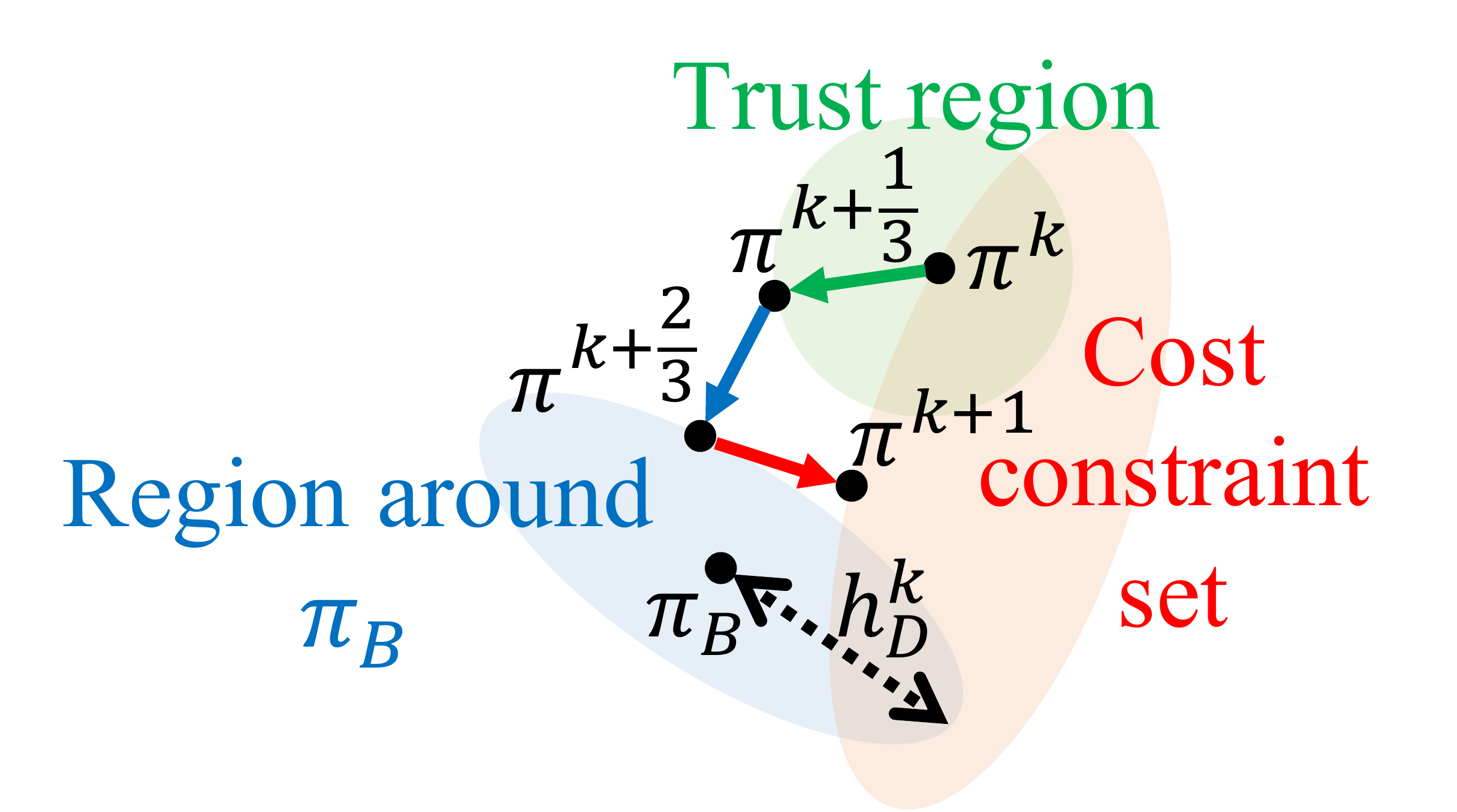}}
(b)
{
\includegraphics[width=0.26\linewidth]{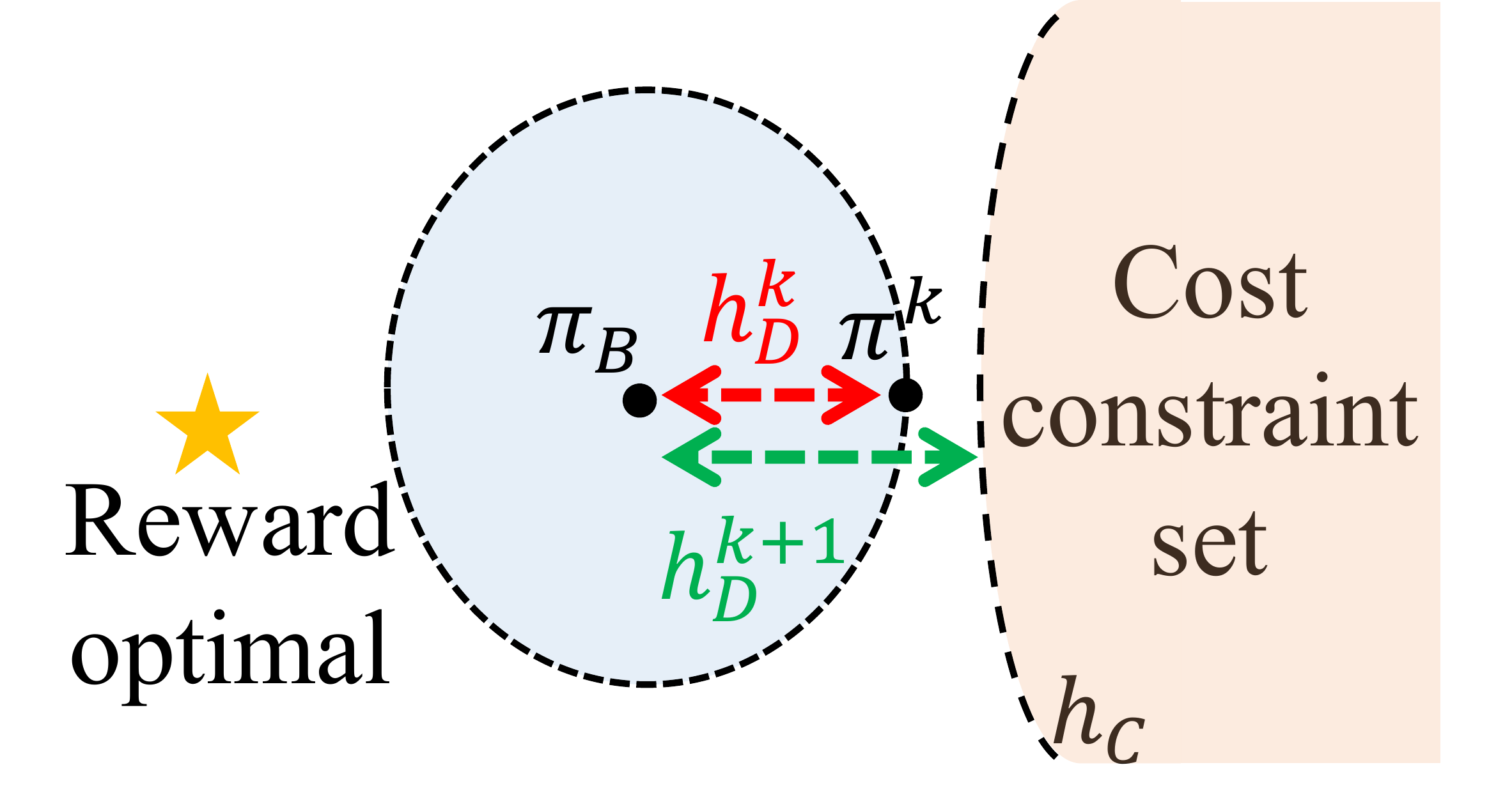}}
(c)
{
\includegraphics[width=0.25\linewidth]{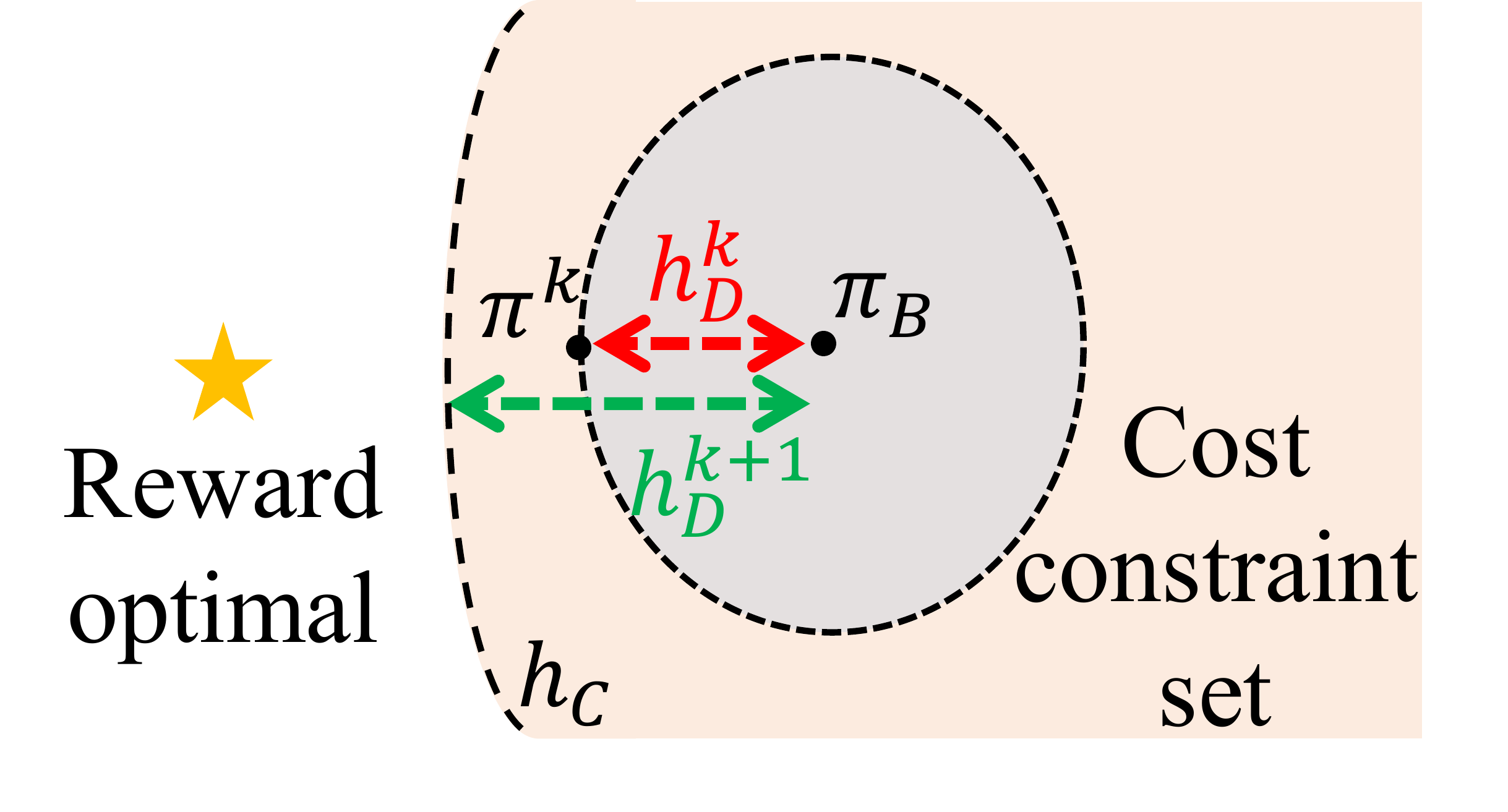}}
\vspace{-2mm}
\caption{\textbf{(a)} Update procedures for \algname. Step 1 (green) % \algname\ 
improves the reward in the trust region. 
Step 2 %, \algname\ 
(blue) projects the policy onto an \textit{adaptable} region around the baseline policy $\pi_B.$
Step 3 %, \algname\ 
(red) projects the policy onto the constraint set.
%
%We control the distance between the learning policy $\pi^k$ and $\pi_B$ iteratively (denoted by $h_D^{k}$).
%
\textbf{(b)} %A case where 
Illustrating when $\pi_B$ is \textit{outside} the constraint set. %Increasing $h_D^{k+1}$ lets the agent to stay inside the cost constraint set.
\textbf{(c)} %A case where
Illustrating when $\pi_B$ is \textit{inside} the constraint set. %Increasing $h_D^{k+1}$ lets the agent to explore more in the cost constraint set.
%The policy achieves higher reward when approaching 
The highest reward is achieved at the yellow star.
$h^k_D$ (the distance between $\pi^k$ and $\pi_B$) is updated to $h^{k+1}_D$ to ensure constraint satisfaction and exploration of the agent.
}
\label{fig:p2cpo}
%\end{mdframed}
\end{figure*}
}

In this work, to learn from the baseline policy while satisfying constraints, we propose an iterative algorithm that performs policy updates in three stages.
The first step updates the policy to maximize expected reward using trust region policy optimization (\eg TRPO \citep{schulman2015trust}).
This can, however, result in a new intermediate policy that is too far from the baseline policy and may not satisfy the constraints.
The second step performs a projection in policy space to control the distance between the current policy and the baseline policy.
%
%This distance is updated each episode depending on the reward improvement and constraint satisfaction of the learning agent.
%
In contrast to the approach that regularizes the standard RL objective with the distance w.r.t. the baseline policy and makes the regularization parameter fade over time, our approach allows the 
learning agent to update the distance when needed.
In addition, this step allows the agent to explore without being overly restricted by the potentially constraint-violating baseline policy.
This also enables the baseline policy to influence the learning even at later iterations without the computational burden of learning a cost function for the baseline policy~\citep{kwon2020humans}. 
The third step ensures constraint satisfaction at every iteration by performing a projection onto the set of policies that satisfy the given constraints.
%
%This ensures recovery from infeasible (\ie constraint-violating) states (\eg due to approximation errors), and eliminates the need for tuning weights for auxiliary cost objective functions~\citep{tessler2018reward}.
%
%In addition, since the projection ensures that the distance between the baseline policy and the learning policy is bounded, we can use this information to understand 
%the location of the baseline policy related to the learning policy.
%
%This information enbles the verifiation of the reward and the cost perofmrnace of the baseline policy. 
%
We call our algorithm \textit{\argFullName}~(\algname).
%(\algname)
%Theoretically, 

This paper's contributions are two-fold. 
\textbf{(1)} We explicitly examine how the baseline policy affects the cost violations of the agent and hence provide a method to safely learn from the baseline policy.
This is done by controlling the distance between the learned policy at iteration $k$ and the baseline policy to ensure both feasibility of the optimization problem and safe exploration by the learning agent (Fig. \ref{fig:p2cpo}(b) and (c)).
Such approach, in contrast to non-adaptable constraint sets and learning a policy from scratch \cite{yang2020projection}, leads to better sample efficiency and hence are more favorable in real applications.
To our knowledge, prior work does not carry out such analysis.
We further provide a finite-time guarantee for the convergence of \algname.
\textbf{(2)} Second, we empirically show that \algname\ can robustly learn from sub-optimal baseline policies in a diverse set of tasks.
These include two Mujoco tasks with safety constraints,
and two real-world traffic management tasks with fairness constraints.
We further show that our algorithm can safely learn from a \textit{human demonstration} driving policy with safety constraints.
In all cases, \algname\ outperforms state-of-the-art safe RL algorithms,
averaging 40\% more reward with 10 times fewer cost violations.
This shows that \algname\ safely and efficiently leverages the baseline policy, and represents a step towards safe deployment of RL in real applications\footnote{Code is available at:~\url{https://sites.google.com/view/spacealgo}}.
%
%\karthik{Say something about how SPACE leverages baseline better compared to vanilla PCPO, using quantitative numbers. This is an important comparison to make clear to readers.}
%

%stick to baseline policy

%% file: related.tex
% !TEX root = neurips_2019.tex
\section{Related Work}
\label{sec:relatedWork}
\parab{Safe RL.}
Learning constraint-satisfying policies has been explored in the context of safe RL~\citep{garcia2015comprehensive,hasanbeig2020cautious,junges2016safety,jansen_et_al:LIPIcs:2020:12815,chow2018lyapunov,bharadhwaj2020conservative,srinivasan2020learning}.
Prior work either uses a conditional-gradient approach~\citep{achiam2017constrained}, adds a weighted copy of the cost objective in the reward function~\citep{tessler2018reward,chow2019lyapunov,fujimoto2019benchmarking,stooke2020responsive}, adds a safety layer to the policy~\citep{dalal2018safe,avni2019run},
or uses the chanced constraints~\citep{fu2018risk,zheng2020constrained}.
In contrast, we use projections to ensure safety.
In addition, \citet{thananjeyan2021recovery} use the sub-optimal demonstration (still safe $\pi_B$) to guide the learning.
They obtain the safe policy by iteratively solving model predictive control.
However, we focus on the model-free setting, which makes it hard to compare to our method.
While \citet{zhang2020cautious, srinivasan2020learning, thananjeyan2020safety}, pre-train a safe policy, they do not focus on how to safely use baseline policies.
Moreover, we do not have two stages of pre-training and fine-tuning.
\citet{yang2020projection} also uses projections to ensure safety--Projection-based Constrained Policy Optimization (PCPO). 
%We leverage a recent state-of-the-art safe RL algorithm--projection-based constrained policy optimization (PCPO) in \citet{yang2020projection} to theoretically and experimentally show how to safely learn from \textit{potentially sub-optimal} baseline policies while satisfying the constraints. 
%
%Here, we exploit the idea of \textit{projections} to ensure constraint satisfaction. 
%
%The key challenge is that the baseline policy needs not satisfy the cost constraint.
%
However, we show that treating learning from the baseline policy as another fixed constraint in PCPO results in cost constraint violations or sub-optimal reward performance.
Instead, our main idea is to have an \textit{adaptable} constraint set that adjusts the distance between the baseline and the learning policies at each iteration with the distance controlled by the learning progress of the agent, \ie the reward improvement and the cost constraint violations.
Such approach ensures exploration and cost satisfaction of the agent.
Please refer to Section \ref{sec:implementation} for the detailed comparison to PCPO.

%
%This analysis allows us to advance towards the practical use of RL in real applications, which PCPO and other algorithms have never done before.
%

%
\parab{Policy optimization with the initial safe set.}
\citet{wachi2020safe,sui2015safe,turchetta2016safe} assume that the initial safe set is given, and the agent explores the environment and verifies the safety function from this initial safe set. In contrast, our assumption is to give a baseline policy to the agent. Both assumptions are reasonable as they provide an initial understanding of the environment.

\parab{Leveraging baseline policies for RL.} 
Prior work has used baseline policies to provide initial information to RL algorithms to reduce or avoid undesirable situations. 
This is done by either: 
initializing the policy with the baseline policy~\citep{driessens2004integrating,smart2000practical,koppejan2011neuroevolutionary,abbeel2010autonomous,gao2018reinforcement,le2019batch,vecerik2017leveraging,jaques2019way}, or
providing a teacher's advice to the agent~\citep{garcia2012safe,quintia2013learning,abel2017agent,zhang2019leveraging}.
However, such works often assume that the baseline policy is constraint-satisfying~\citep{sun2018dual, balakrishna2019policy}.
%
%The potential bias introduced in the baseline policy results in learning constraint-violating policies.
%
In contrast, \algname\ safely leverages the baseline policy without requiring it to satisfy the specified constraints.
\citet{pathak2015greedy,bartocci2011model} also modify the existing known models (policies) based on new conditions in the context of the formal methods.
In contrast, we solve this problem using projections in the policy space.

\parab{Learning from logged demonstration data.} 
%\karthik{Not sure if this section is very relevant? What does evaluation mean here?}
%
To effectively learn from demonstration data given by the baseline policy, \citet{wu2019imitation, brown2019extrapolating, kwon2020humans} assess the demonstration data by either:  
predicting their cost in the new task using generative adversarial networks (GANs)~\citep{goodfellow2014generative}, or
directly learning the cost function of the demonstration data.
%
%These approaches allow the learning agent to identify which actions given by the baseline policy are constraint-satisfying.
%
%However, these approaches require a large number of training samples from the new task.
%
%In addition, the learned cost function is not guaranteed to recover the true one.
%
%This may result in driving the agent to undesirable situations.
%
In contrast, \algname\ controls the distance between the learned and baseline policies to ensure learning improvement.

%% file: preliminaries.tex
\section{Problem Formulation}
\label{sec:preliminaries}

We frame our problem as a constrained Markov Decision Process (CMDP)~\citep{altman1999constrained}, defined as a tuple $<\mathcal{S},\mathcal{A},T,R,C>.$ 
Here $\mathcal{S}$ is the set of states, 
$\mathcal{A}$ is the set of actions, and
%$T:\mathcal{S}\times \mathcal{A}\times \mathcal{S}\rightarrow [0,1]$ is a transition probability.
$T$ specifies the conditional probability $T(s'|s,a)$ that the next state is $s'$ given the current state $s$ and action $a$.
In addition, 
$R:\mathcal{S}\times \mathcal{A}\rightarrow \mathbb{R}$ is a reward function, and
$C:\mathcal{S}\times \mathcal{A}\rightarrow \mathbb{R}$ is a constraint cost function. The reward function encodes the benefit of using action $a$ in state $s,$ while the cost function encodes the corresponding  constraint violation penalty.

A policy is a map from states to probability distributions on $\cal A.$ It specifies that in state $s$ the selected action is drawn from the distribution $\pi(s).$
%$\pi \colon \mathcal{S}\rightarrow \mathcal{P}(\mathcal{A})$ that selects an action $a$ according to a probability distribution $\pi(s)$ over $\cal A.$
The state then transits from $s$ to $s'$ according to the state transition distribution $T(s'|s,a).$
In doing so, a reward $R(s,a)$ is received and a constraint cost $C(s,a)$ is incurred, as outlined above.

Let $\gamma\in (0,1)$ denote a discount factor, and $\tau$ denote the trajectory $\tau = (s_0,a_0,s_1,\cdots)$ induced by a policy $\pi.$
Normally, we seek a policy $\pi$ that maximizes a cumulative discounted reward
%
%%Equation's spacing
\setlength{\abovedisplayskip}{6pt}%
\setlength{\belowdisplayskip}{6pt}%
\setlength{\abovedisplayshortskip}{3pt}%
\setlength{\belowdisplayshortskip}{3pt}%
\setlength{\jot}{0pt}
\begin{align}
J_{R}(\pi)\doteq \E_{\tau\sim\pi}\left[\sum_{t=0}^{\infty}\gamma^{t} R(s_{t},a_{t})\right],\label{eq:p2cpo_problemFormulation_1}
\end{align}
while keeping the cumulative discounted cost below $h_C$
\begin{align}
J_{C}(\pi)\doteq \E_{\tau\sim\pi}\left[\sum_{t=0}^{\infty}\gamma^{t} C(s_{t},a_{t})\right]\leq h_C.\label{eq:p2cpo_problemFormulation_3}
\end{align}
Here we consider an additional objective. 
We are provided with a baseline policy $\pi_B$ 
and at each state $s$ we measure the divergence between $\pi(s)$ and $\pi_B(s).$ For example, this could be the KL-divergence 
$D(s)\doteq D_\mathrm{KL}(\pi(s)\|\pi_{B}(s)).$
We then seek a policy that maximizes Eq. (\ref{eq:p2cpo_problemFormulation_1}), satisfies Eq. (\ref{eq:p2cpo_problemFormulation_3}), and ensures the
discounted divergence 
between the learned and baseline policies is below $h_D$:
\begin{align}
J_{D}(\pi)\doteq \E_{\tau\sim\pi}\left[\sum_{t=0}^{\infty}\gamma^{t} D(s_t)\right]\leq h_D.\label{eq:p2cpo_problemFormulation_2}
\end{align}
We do not assume that %$J_C(\pi_B)\leq h_C.$ So
the baseline policy satisfies the cost constraint. Hence we allow $h_D$ to be adjusted during the learning of $\pi$ to allow for reward improvement and constraint satisfaction.

Let $\mu_t(\cdot|\pi)$ denote the state distribution at time $t$ under policy $\pi.$
The discounted 
state distribution induced by $\pi$ is defined to be $d^{\pi}(s)\doteq(1-\gamma)\sum^{\infty}_{t=0}\gamma^{t}\mu_t(s|\pi).$
%\begin{align}
%    \textstyle 
%d^{\pi}(s)\doteq(1-\gamma)\sum^{\infty}_{t=0}\ga%mma^{t}\mu_t(s|\pi).\nonumber
%\end{align}
%denoted by 
%$$
%d^{\pi}(s)\doteq(1-\gamma)\sum^{\infty}_{t=0}\gamma^{t}\mu(s_t = s|\pi),$$ 
%and $\mu$ is the state distribution of $s$ given $\pi.$
%Finally set
%$h_B$ 
%$$
%{\color{red}
%h_B \doteq \E_{\substack{s\sim d^{\pi_B}}}[D_\mathrm{KL}(\pi(s)\|\pi_B(s))].}
%$$
%The threshold $h_D$ controls the distance between the learned policy and the baseline policy.
%, here $h_B \doteq \E_{\substack{s\sim d^{\pi_B}}}[D_\mathrm{KL}(\pi||\pi_B)[s]],$ 
%and 
%$d^{\pi}$ is the discounted future state distribution induced by $\pi$, 
%denoted by 
%$$
%d^{\pi}(s)\doteq(1-\gamma)\sum^{\infty}_{t=0}\gamma^{t}\mu(s_t %= s|\pi),$$ 
%and $\mu$ is the state distribution of $s$ given $\pi.$
Now bring in the  reward advantage function 
\citep{kakade2002approximately} defined by
\begin{align}
A^{\pi}_{R}(s,a)  \doteq Q^{\pi}_{R}(s,a)-V^{\pi}_{R}(s),\nonumber \nonumber
\end{align}
where $V^{\pi}_{R}(s)  \doteq  \E_{\tau\sim\pi}\left[\sum_{t=0}^{\infty}\gamma^{t} R(s_{t},a_{t})|s_0=s\right]$ is the expected reward from state $s$ under policy $\pi$, and
$
Q^{\pi}_{R}(s,a) \doteq  \E_{\tau\sim\pi}\left[\sum_{t=0}^{\infty}\gamma^{t} R(s_{t},a_{t})|s_0=s,a_0=a\right]\nonumber
$
is the expected reward from state $s$ and initial action $a,$ and thereafter following policy $\pi.$
%
%obtained by $\pi$ given the initial state $s$ and action $a,$ 
%and $V^{\pi}_{R}(s)$
%\doteq\E_{\tau\sim\pi}\big[\sum_{t=0}^{\infty}\gamma^{t}R(s_{t},a_{t})|s_0=s\big]$ is the discounted cumulative reward obtained by $\pi$ given the initial state $s$.
These definitions allow us to express the reward performance of 
one policy $\pi'$ in terms of another
$\pi$:
\begin{align}
J_{R}(\pi') - J_{R}(\pi) = \frac{1}{1-\gamma}\E_{\substack{s\sim d^{\pi'},a\sim \pi'}}[A^{\pi}_{R}(s,a)].\nonumber
\end{align}
%
%where $A^{\pi}_{R}(s,a)$ is the reward advantage function, denoted by %$A^{\pi}_{R}(s,a)\doteq Q^{\pi}_{R}(s,a)-V^{\pi}_{R}(s)$.
%
%Here $Q^{\pi}_{R}(s,a)\doteq\E_{\tau\sim\pi}\big[\sum_{t=0}^{\infty}\gamma^{t} R(s_{t},a_{t})|s_0=s,a_0=a\big]$ is the discounted cumulative reward obtained by $\pi$ given the initial state $s$ and action $a,$ 
%
%and $V^{\pi}_{R}(s)\doteq\E_{\tau\sim\pi}\big[\sum_{t=0}^{\infty}\gamma^{t} R(s_{t},a_{t})|s_0=s\big]$ is the discounted cumulative reward obtained by $\pi$ given the initial state $s$.
%
Similarly, we can define  $A^{\pi}_{D}(s,a)$, $Q^{\pi}_{D}(s,a)$ and $V^{\pi}_{D}(s)$ for the divergence   cost,
and $A^{\pi}_{C}(s,a)$, $Q^{\pi}_{C}(s,a)$ and $V^{\pi}_{C}(s)$ for the constraint cost.
%

%% file: model.tex
% !TEX root = neurips_2019.tex

\section{\argFullName~(\algname)}
\label{sec:model}

%To effectively use the baseline policy without violating the cost constraints, we use
%we develop \algname\ -- 
%a trust region step that performs policy updates to improve the reward, followed by two projections to ensure proximity to the baseline policy and cost constraint satisfaction.

We now describe the proposed iterative algorithm illustrated in Fig.~\ref{fig:p2cpo}. In what follows, $\pi^k$ denotes the learned policy after iteration $k,$ and $M$ denotes a distance measure between policies. For example, $M$ may be the 2-norm of the difference of policy parameters or some average over the states of the KL-divergence of the action policy distributions.

\parab{Step 1.} We perform one step of trust region policy optimization \citep{schulman2015trust}. This
%optimize the reward function by 
maximizes the reward advantage function $A^{\pi}_{R}(s,a)$ 
%subject to a KL-divergence constraint. This constrains the intermediate policy $\pi^{k+\frac{1}{3}}$ 
over a KL-divergence neighborhood
%be within a $\delta$-neighbourhood 
of $\pi^k$:
%(trust region):
\begin{align}
\begin{split}
\pi^{k+\frac{1}{3}}
=\argmax\limits_{\pi}~&\E_{\substack{
s\sim d^{\pi^k},~a\sim \pi}}
[A^{\pi^k}_{R}(s,a)]  \\ 
\quad\text{s.t.}~&\E_{s\sim d^{\pi^{k}}}\big[\KL(\pi(s) \|\pi^{k}(s))\big]\leq \delta.
\end{split}
\label{eq:P2CPO_firstStep}
\end{align}

%This update rule is called Trust Region Policy Optimization (TRPO)~\cite{schulman2015trust}. 
%
%denoted by $\{\pi:\E_{s\sim d^{\pi^{k}}}\big[\KL(\pi||\pi^{k})[s]\big]\leq \delta\}$, 
%
\vspace{-2mm}
\parab{Step 2.} 
%Second, we leverage the baseline policy by 
We project $\pi^{k+\frac{1}{3}}$ onto a region around $\pi_B$ controlled by $h^k_D$ to minimize $M$: 
\begin{align}
\begin{split}
&\pi^{k+\frac{2}{3}}=\argmin\limits_{\pi}~M({\pi,\pi^{k+\frac{1}{3}}}) \\ &\text{s.t.}~J_{D}(\pi^{k})+\frac{1}{1-\gamma}\E_{\substack{s\sim d^{\pi^{k}},~a\sim \pi}}[A^{\pi^k}_{D}(s)]\leq h_D^k.
\label{eq:P2CPO_secondStep}
\end{split}
\end{align}

\vspace{-2mm}
\parab{Step 3.}
%we ensure cost constraint satisfaction by
We project 
%the intermediate policy 
$\pi^{k+\frac{2}{3}}$ onto the set of policies satisfying the cost constraint to minimize $M$: 
\begin{align}
\begin{split}
&\pi^{k+1}=\argmin\limits_{\pi}~M({\pi,\pi^{k+\frac{2}{3}}})\quad  \\
&\text{s.t.}~J_{C}(\pi^{k})+\frac{1}{1-\gamma}\E_{\substack{s\sim d^{\pi^{k}},~a\sim \pi}}[A^{\pi^k}_{C}(s,a)]\leq h_C.
\label{eq:P2CPO_thridStep}
\end{split}
\end{align}
\parab{Remarks.} Since we use a small step size $\delta,$ we can replace the state distribution $d^{\pi}$ with $d^{\pi^k}$ in Eq. (\ref{eq:P2CPO_secondStep}) and (\ref{eq:P2CPO_thridStep}) and hence compute $A^{\pi^k}_D$ and $A^{\pi^k}_C.$
Please see the supplementary material for the derivation of this approximation.
%We iteratively update the bound between the 
%control the distance between the 
%current and the baseline policy.
%to ensure reward improvement and cost constraint satisfaction.
%
%Choosing an appropriate distance $h_D^k$ at step $k$ is challenging.
%This is because that if $\pi^k$ evolves too close to $\pi_{B}$ with the different cost constraints, the agent runs the risk of violating the cost constraints.
%In addition, if $\pi^k$ stays away from $\pi_B,$ the agent could not leverage the baseline policy for fast learning.
%To this end, 

\parab{Control $h_D^k$ in Step 2.} We select $h_D^0$ to be small and gradually increase $h^k_D$ at each iteration to expand the region around $\pi_B.$ 
Specifically, we make $h_D^{k+1} > h_D^k$ if:
\begin{itemize}[itemsep=0pt,topsep=0pt,itemindent=-1mm]
\item [(a)]
$J_{C}(\pi^k)> J_{C}(\pi^{k-1})$: this increase is to ensure a nonempty intersection between the region around $\pi_B$ and the cost constraint set (feasibility). 
See Fig.~\ref{fig:p2cpo}(b).
\item [(b)]
$J_{R}(\pi^k)
<J_{R}(\pi^{k-1})$: this increase gives the next policy more freedom to improve the reward and the cost constraint performance (exploration). See Fig. \ref{fig:p2cpo}(c).
\end{itemize}

It remains to determine how to set the new value of $h_D^{k+1}.$ 
Let $\mathcal{U}_1$ denote the set of policies satisfying the cost constraint, and $\mathcal{U}_2^{k}$ denote the set of policies in the region around $\pi_B$ controlled by $h^k_D.$
%
%This is given in the next lemma.
Then we have the following Lemma.

%\parab{Feasibility:} If 
%%the current value of the cost constraint  is larger than the previous one 
%$J_{C}(\pi^k)$> J_{C}(\pi^{k-1}),$
%we increase %$h^{k+1}_D$ to %ensure %intersection %between the %trust region %around $\pi_B$ and the cost constraint set.
%%This is because that the increase of cost implies that the agent may lie \textit{outside} the cost constraint set given the current $h_D^k$ as shown in Fig. \ref{fig:p2cpo}(b).
%Increasing $h^{k+1}_B$ ensures that there exists a feasible policy that satisfies the cost constraint at step $k+1.$

%\parab{Exploration:} If 
%%the current value of the reward is smaller than the previous one
%$J_{R}(\pi^k)<J_{R}(\pi^{k-1}),$
%we increase $h^{k+1}_D$ to ensure the next policy can
%%in the next step can search a region within the constraint set to achieve 
%improve reward with the cost constraint performance.
%
%This is because that the decrease of reward implies that the agent is \textit{constrained} by the baseline policy that may lie in the low-reward region as shown in Fig. \ref{fig:p2cpo}(c).
%may lie \textit{inside} the cost constraint set given the current $h_D^k$ as shown in Fig. \ref{fig:p2cpo}(c).
%
%Increasing $h^{k+1}_B$ ensures that the learning agent avoids being shackled by the baseline policy.

%
%{\color{red}
\begin{lemma}[\textbf{Updating $h_D$}]
\label{theorem:h_D}
%If $h_D^{k+1}$ at step $k+1$ is increased by
%{\color{blue} Assume $J_C(\pi^k)>J_C(\pi^{k-1})$ or $J_R(\pi^k)<J_R(\pi^{k-1}).$- is this assumption needed?}
%
If at step $k+1$: $h^{k+1}_D \geq \mathcal{O}\big((J_C(\pi^k)-h_C)^2\big)+h_D^k,$
%\begin{align} 
 %   h^{k+1}_D \geq \mathcal{O}\big((J_C(\pi^k)-h_C)^2\big)+h_D^k, \nonumber
%\end{align}
%ensures intersection between the trust region around $\pi_B$ and the cost constraint set (feasibility)
%
%and gives $\pi^{k+1}$ some freedom to improve the reward and the cost constraint performance (exploration).
then $\mathcal{U}_1\cap \mathcal{U}_2^{k+1}\neq \emptyset$ (feasibility) 
%
%and $\exists\pi\in\mathcal{U}_2^{k+1},~\pi\in\partial\mathcal{U}_1$ (exploration).
and~$\mathcal{U}_2^{k+1}\cap\partial\mathcal{U}_1\neq\emptyset$ (exploration).

%ensures {\color{red} the feasibility of the optimization problem and exploration by the learning agent. - make more precise}
%
%then the boundaries of the region around $\pi_B$ the cost constraint set intersect.
%{\color{orange}Here I want to say that if we increase $h_D$ in the next step, then we can ensure the boundaries of these two sets intersect. And if they intersect, then we ensure that the agent is in the constraint set and explores the environment. Please let me know if this makes sense! Thanks!}
\end{lemma}%}
\vspace{-5mm}
\begin{proof}
Proved by Three-point Lemma~\citep{chen1993convergence}.
See the supplementary material for more details.
\end{proof}
\vspace{-2mm}
%
%Lemma \ref{theorem:h_D} shows that the larger the difference between $J_{C}(\pi^k)$ and $h_C,$ the larger $h_D^{k+1}$ is. 
%
%This gives a theoretical justification for controlling $h_D^k$ to exploit the baseline policy for safe learning.  
%
%{\color{red} Where is this proved?} {\color{blue}JY: I included it in the appendix.}
%
%If $J_C(\pi^k)>J_C(\pi^{k-1})$ or $J_R(\pi^k)<J_R(\pi^{k-1}),$ 
%
\parab{Remarks.} Two values are in the big $\mathcal{O}.$
The first value depends on the discounted factor $\gamma,$
and the second value depends on relative distances between $\pi^k,$ $\pi_B,$ and the policy in $\partial\mathcal{U}_1.$
The intuition is that the smaller the distances are, the smaller the update of $h_D^k$ is.
%Hence $h_D^k$ is slightly increased in this case.

Importantly, Lemma \ref{theorem:h_D} ensures that the boundaries of the region around $\pi_B$ determined by $h_D$ and the set of policies satisfying the cost constraint intersect.
%at every iteration.
Note that $h_D$ will become large enough to guarantee feasibility during training.
This \textit{adaptable} constraint set, in contrast to the \textit{fixed} constraint set in PCPO, allows the learning algorithm to 
%stay within (as shown Fig. \ref{fig:p2cpo}(b)) and 
explore policies within
%(as shown Fig. \ref{fig:p2cpo}(c)) within 
the cost constraint set
while still learning from the baseline policy.
Compared to other CMDP approaches, the step of projecting close to $\pi_B$ allows the policy to quickly improve. Compared to behavior cloning, the steps of reward optimization and constraint projection allow the policy to achieve good final performance. 
We examine the importance of updating $h_D$ in Section \ref{sec:experiments}.
%the feasibility of the problem (Fig. \ref{fig:p2cpo}(b)) and the exploration of the learning agent (Fig. \ref{fig:p2cpo}(c)) on every learning episode.
%
%In addition, 
%\algname\ to exploit the baseline policy for safe learning.

%% file: implementation.tex
% !TEX root = neurips_2019.tex
\section{A Theoretical Analysis of \algname}
\label{sec:implementation}
We will implement a policy as a neural network with fixed architecture parameterized by $\vtheta\in\R^n.$
We then learn a policy from the achievable set
$\{\pi(\cdot|\vtheta)\colon \vtheta \in \R^n\}$ by iteratively learning $\vtheta.$ Let $\vtheta^k$ and $\pi^k \doteq \pi(\cdot|\vtheta^k)$ denote the parameter value and the corresponding policy at step $k.$
In this setting, 
it is impractical to solve for the policy updates in Eq.~(\ref{eq:P2CPO_firstStep}), (\ref{eq:P2CPO_secondStep}) and (\ref{eq:P2CPO_thridStep}).
Hence we approximate the reward function and constraints with first order Taylor expansions, and KL-divergence with a second order Taylor expansion. We will need the following derivatives: \\
\textbf{(1)} $\vg^{k}
\doteq\nabla_\vtheta\E_{\substack{s\sim d^{\pi^k},~a\sim \pi}}[A^{\pi^k}_{R}(s,a)],$ \\
\textbf{(2)} $\va^{k}
\doteq\nabla_\vtheta\E_{\substack{s\sim d^{\pi^k},~a\sim \pi}}[A^{\pi^k}_{D}(s)],$ \\
\textbf{(3)} $\vc^{k}
\doteq\nabla_\vtheta\E_{\substack{s\sim d^{\pi^k},~a\sim \pi}}[A^{\pi^k}_{C}(s,a)],$ and \\
\textbf{(4)} $\mF^{k}
\doteq \nabla_{\vtheta}^2 \E_{s\sim d^{\pi^{k}}}\left[
\KL (\pi(s)\|\pi^k(s))\right].$
%\begin{itemize}[itemsep=0pt,topsep=-2pt, itemindent=-4mm]
%
%\item[] $\vg^{k}
%\doteq\nabla_\vtheta\E_{\substack{s\sim d^{\pi^k},~a\sim %\pi}}[A^{\pi^k}_{R}(s,a)]$ (Gradient of the reward advantage %function),
%\item[] $\va^{k}
%\doteq\nabla_\vtheta\E_{\substack{s\sim d^{\pi^k},~a\sim %\pi}}[A^{\pi^k}_{D}(s)]$ (Gradient of the divergence %advantage function),
%\item[] $\vc^{k}
%\doteq\nabla_\vtheta\E_{\substack{s\sim d^{\pi^k},~a\sim %\pi}}[A^{\pi^k}_{C}(s,a)]$ (Gradient of the cost advantage %function),
%\item[] $\mF^{k}
%\doteq \nabla_{\vtheta}^2 \E_{s\sim d^{\pi^{k}}}\left[
%\KL (\pi(s)\|\pi^k(s))\right]$
%(Hessian of the KL-divergence constraint).
%\end{itemize}
{
 \begin{algorithm}[t]
 \centering
   \caption{\algname}
   \label{algo:P2CPO}
   \begin{algorithmic}%[]
       \State Initialize a policy $\pi^0=\pi(\cdot|\vtheta^0)$ and a trajectory  buffer $\mathcal{B}$

       %\State Initialize steps $t_{A}$
       %\State Set $t_{\text{max}}$ and $c$
       \For{$k=0,1,2,\cdots$}
             \State Run $\pi^{k}=\pi(\cdot|\vtheta^{k})$ and store trajectories in $\mathcal{B}$
             %\If{not an $\epsilon$-FOSP}
                 %\State Compute $\vg, \va, \vc, \mF, b$ and $d$ using $\mathcal{B}$
                 \State Obtain $\vtheta^{k+1}$ using the update in Eq. (\ref{eq:P2CPO_final})
                 %\If{$J_{C}(\pi^{k})>J_{C}(\pi^{k-1})$ or      $J_{R}(\pi^{k})<J_{R}(\pi^{k-1})$}
                    \State \textbf{If} $J_{C}(\pi^{k})>J_{C}(\pi^{k-1})$ or      $J_{R}(\pi^{k})<J_{R}(\pi^{k-1})$ 
                    \State \quad\quad Update $h_D^{k+1}$ using \textbf{Lemma~\ref{theorem:h_D}}
                 %\EndIf
             %\Else
                 %\State Compute $\vu,\mF$ and $\nabla^2 f(\vtheta)$ using $\mathcal{D}$
                 %\State Obtain $\vtheta^{k+1}$ using update in Eq. (\ref{eq:P2CPO_final_v2})
             %\EndIf
             \State Empty $\mathcal{B}$
       \EndFor
       %\State \textbf{return} Optimal $\vtheta^{*}$
     %\EndProcedure
   \end{algorithmic} 
\end{algorithm}
}

Each of these derivatives are taken w.r.t.~the neural network parameter and evaluated at $\vtheta^k.$
We also define
%\begin{align}\textstyle
$b^{k}
\doteq J_{D}(\pi^k)-h_D^k,$
and
%\qquad \textrm{and}\qquad
$d^{k}
\doteq J_{C}(\pi^k)-h_C.$ Let $u^{k}
\doteq\sqrt{\frac{2\delta}{{\vg^{k}}^T{\mF^{k}}^{-1}\vg^{k}}},$ and $\mL=\mI$ for the 2-norm projection and $\mL=\mF^{k}$ for the KL-divergence projection.
%
%By approximating Eq.~(\ref{eq:P2CPO_firstStep}), Eq.~(\ref{eq:P2CPO_secondStep}) and Eq.~(\ref{eq:P2CPO_thridStep}),
%We approximate Eq.~(\ref{eq:P2CPO_firstStep}), (\ref{eq:P2CPO_secondStep}) \& (\ref{eq:P2CPO_thridStep}):

\parab{Step 1.} Approximating Eq. (\ref{eq:P2CPO_firstStep}) yields
\begin{align}
\begin{split}
    \vtheta^{k+\frac{1}{3}} = \argmax\limits_{\vtheta}~& {\vg^{k}}^{T}(\vtheta-\vtheta^{k}) \\ \text{s.t.}~&\frac{1}{2}(\vtheta-\vtheta^{k})^{T}\mF^{k}(\vtheta-\vtheta^{k})\leq \delta.
    \label{eq:update1}
    \end{split}
\end{align}
\parab{Step 2 \& 3.}
% Second, if the projections are defined in the parameter space, we directly use the 2-norm projection.
%
%On the other hand, if the projections are defined in the probability space, we use the KL-divergence projection, which is approximated through a second order Taylor expansion.
%
Approximating Eq. (\ref{eq:P2CPO_secondStep}) and (\ref{eq:P2CPO_thridStep}), similarly yields
%at $\pi^{k+\frac{1}{3}}$ and $\pi^{k+\frac{2}{3}}$
%And we approximate the constraints in Eq. (\ref{eq:P2CPO_secondStep}) and (\ref{eq:P2CPO_thridStep}) by a first order Taylor expansion.
\begin{align}
\begin{split}
    \vtheta^{k+\frac{2}{3}}=\argmin\limits_{\vtheta}~&\frac{1}{2}(\vtheta-{\vtheta}^{k+\frac{1}{3}})^{T}\mL(\vtheta-{\vtheta}^{k+\frac{1}{3}})  \\
     \text{s.t.}~&{\va^{ k}}^{T}(\vtheta-\vtheta^{k})+b^{k}\leq 0, 
     \label{eq:update2}
\end{split}
\end{align}
\vspace{-3mm}
\begin{align}
\begin{split}
    \vtheta^{k+1} = \argmin\limits_{\vtheta}~&\frac{1}{2}(\vtheta-{\vtheta}^{k+\frac{2}{3}})^{T}\mL(\vtheta-{\vtheta}^{k+\frac{2}{3}})\\
    \text{s.t.}~&{\vc^{ k}}^{T}(\vtheta-\vtheta^{k})+d^{k}\leq 0,
    \label{eq:update3}
\end{split}
\end{align}
where $\mL=\mI$ for the 2-norm projection and $\mL=\mF^{k}$ for the KL-divergence projection.
We solve these problems using convex programming, then we have ($(\cdot)^{+}$ is $\max(0,\cdot)$)
\begin{align} 
\begin{split}
\vtheta^{k+1} &=\vtheta^{k}+u^{k}{\mF^{k}}^{-1}\vg^{k}\\
&-(\frac{u^{k}{\va^{k}}^T{\mF^{k}}^{-1}\vg^{k}+b^{k}}{{\va^{k}}^T{\mL}^{-1}\va^{k}})^{+}\mL^{-1}\va^{k} \\
&-(\frac{u^{k}{\vc^{k}}^T{\mF^{k}}^{-1}{\vg^{k}}+d^{k}}{{\vc^{k}}^T\mL^{-1}\vc^{k}})^{+}\mL^{-1}\vc^{k}.
\end{split}
\label{eq:P2CPO_final}
\end{align}

Algorithm \ref{algo:P2CPO} shows the corresponding pseudocode.

%However, \algname\ requires to invert $\mF$, which is impractical for huge neural network policies.
%
%Hence we use the conjugate gradient method \citep{schulman2015trust}. 
%
%When \algname\) is in a first-order stationary point (See Definition \ref{definition:sp}), we have the following update which uses the second-order information of the objective function $f$ followed by the trust region update to escape from the stationary point:
%\begin{align}
%    \vu = \argmin\limits_{\vu\in\mathcal{C}_1\cap\mathcal{C}_2}&\quad \vu^{T}\nabla^2 f(\vtheta^k)\vu \nonumber\\
    %\text{s.t.}&\quad\|\vu\|^2\leq 1,\quad{\nabla f(\vtheta^{k})^T}\vu=0, \nonumber
%\end{align}
%and
%\begin{align}
%    \vtheta^{k+1} = \argmin\limits_{\vtheta}&\quad \vu^T(\vtheta-\vtheta^{k}) \nonumber\\
%    \text{s.t.}&\quad\frac{1}{2}(\vtheta-\vtheta^{k})^{T}\mF(\vtheta-\vtheta^{k})\leq \delta. \nonumber
%\end{align}
%Solving these two problems, for each policy update we have
%, we do: 
%
%\begin{align}
%\vtheta^{k+1}=\vtheta^{k}-\sqrt{\frac{2\delta}{\vu^T\mF^{-1}\vu}}\mF^{-1}\vu.
%\label{eq:P2CPO_final_v2}
%\end{align}

%
\parab{Convergence analysis.}
\label{subsec:p2cpoConvergence}
We consider the following simplified problem to provide a convergence guarantee of \algname:
\begin{align}
\label{eq:sp_problem}
\min\limits_{\vtheta\in\mathcal{C}_1\cap\mathcal{C}_2}f(\vtheta),
\end{align}
where $f:\R^n\rightarrow\R$ is a twice continuously differentiable function at every point in a open set $\mathcal{X}\subseteq\R^n,$ and $\mathcal{C}_1\subseteq\mathcal{X}$ and $\mathcal{C}_2\subseteq\mathcal{X}$ are compact convex sets with $\mathcal{C}_1\cap\mathcal{C}_2\neq \emptyset$. 
The function $f$ is the negative reward function of our CMDP, and the two constraint sets represent the cost constraint set and the region around the baseline policy $\pi_B.$ 

%\parab{Notation.} 
For a vector $\vx$, 
let $\|\vx\|$ denote the Euclidean norm. 
For a matrix $\mM$ let $\|\mM\|$ denote the induced matrix 2-norm, and
$\sigma_i(\mM)$ denote the $i$-th largest singular value of $\mM.$
%We will need the following assumptions. 

\goodbreak
\begin{assumption}
\label{as:1}
{\rm We assume:
%\vspace{-4mm}
\begin{itemize}%[itemsep=-0pt]
\item [(1.1)] 
The gradient $\nabla f$ is $L$-Lipschitz continuous over a open set $\mathcal{X}.$ %i.e., for any $\vtheta$ and $\hat{\vtheta}\in\mathcal{X},$ $\| \nabla f(\vtheta) - \nabla f(\hat{\vtheta})\|\leq L\|\vtheta-\hat{\vtheta}\|.$

\item [(1.2)] For some constant $G,$
$\|\nabla f(\vtheta)\|\leq G.$
%The norm of gradient over a open set $\mathcal{X}$ are bounded above by constants $G,$ i.e.,  

%\textbf{(3)} The Hessians $\nabla^2 f(\vtheta)$ are $M_1$-Lipschitz continuous over a open set $\mathcal{X},$ i.e., for any $\vtheta$ and $\hat{\vtheta}\in$ $\mathcal{X},$ we have $\| \nabla^2 f(\vtheta) - \nabla^2 f(\hat{\vtheta})\|\leq M_1\|\vtheta-\hat{\vtheta}\|.$

%\textbf{(4)} The Fisher information matrices $\mF$ (i.e., $\nabla^2 \log f(\vtheta)$) are $M_2$-Lipschitz continuous over a open set $\mathcal{X},$ i.e., for any $\vtheta$ and $\hat{\vtheta}\in$ $\mathcal{X},$ we have $\| \nabla^2 \log f(\vtheta) - \nabla^2 \log f(\hat{\vtheta})\|\leq M_2\|\vtheta-\hat{\vtheta}\|.$

\item [(1.3)] 
For a constant $H,$
$\diam({\mathcal C}_1) \leq H$ and 
$\diam({\mathcal C}_2) \leq H.$

%The diameter of the constraint sets $\mathcal{C}_1$ and $\mathcal{C}_2$ are all bounded above by constants $H,$ i.e., $\max_{\vtheta,\hat{\vtheta}\in\mathcal{C}_1}\|\vtheta-\hat{\vtheta}\|\leq H$ and $\max_{\vtheta,\hat{\vtheta}\in\mathcal{C}_2}\|\vtheta-\hat{\vtheta}\|\leq H.$
\end{itemize}
}
\end{assumption}
\vspace{-3mm}
Assumptions (1.1) and (1.2) ensure that the gradient can not change too rapidly and the norm of the gradient can not be too large.
(1.3) implies that for every iteration, the diameter of the region around $\pi_B$  is bounded above by $H$.

%We will show that \algname\ converges to the following stationary points.
%
We will need a concept of an $\epsilon$-first order stationary point \citep{mokhtari2018escaping}.
For $\epsilon >0,$
%\begin{definition}[]
%\label{definition:sp}
we say that $\vtheta^*\in\mathcal{C}_1\cap\mathcal{C}_2$ an $\epsilon$-first order stationary point ($\epsilon$-FOSP) of Problem (\ref{eq:sp_problem}) under KL-divergence projection 
if 
%Eq. (\ref{eq:sp_cond1}) is satisfied (This definition is from \cite{mokhtari2018escaping}).
\begin{align}
\label{eq:sp_cond1}
{\nabla f(\vtheta^*)^T}(\vtheta-\vtheta^*)\geq -\epsilon,\quad\forall \vtheta \in \mathcal{C}_1\cap\mathcal{C}_2.
\end{align}
%In addition, 
Similarly,
under the 2-norm projection,
$\vtheta^*\in\mathcal{C}_1\cap\mathcal{C}_2$ an 
%$\epsilon$-first order stationary point
$\epsilon$-FOSP of (\ref{eq:sp_problem}) if 
%Eq. (\ref{eq:sp_cond2}) is satisfied.
\begin{align}
\label{eq:sp_cond2}
{\nabla f(\vtheta^*)^T}{\mF^*}(\vtheta-\vtheta^*)\geq -\epsilon,\quad\forall \vtheta \in \mathcal{C}_1\cap\mathcal{C}_2,
\end{align}
%\end{definition}
%\vspace{-3mm}
%Definition \ref{definition:sp} shows that an $\epsilon$-FOSP is either in the boundary or in the interior of the constraint set. 
%Definition \ref{definition:sp} shows 
where $\mF^*\doteq\nabla_{\vtheta}^2 \E_{s\sim d^{\pi^{*}}}\left[\KL (\pi(s)\|\pi^*(s))\right].$ Notice that 
%
%$\mF$ is computed at the one step before reaching stationary points.
%
\algname\ converges to distinct stationary points under the two possible projections
(see the supplementary material). With these assumptions, we have the following Theorem.

\begin{theorem}[\textbf{Finite-Time Convergence Guarantee of \algname}]
\label{theorem:P2CPO_converge}
%\textbf{(a)} 
Under the KL-divergence projection, there exists a %series
sequence $\{\eta^k\}$ such that \algname\ converges to an $\epsilon$-FOSP 
%in Eq. (\ref{eq:sp_cond1}) 
in at most $\mathcal{O}(\epsilon^{-2})$ iterations. 
Moreover, at step $k+1$ 
%the objective value satisfies
\begin{align}
    f(\vtheta^{k+1}) \leq  f(\vtheta^{k})-\frac{L\epsilon^2}{2(G+\frac{H\sigma_1({ \mF^k}
    )}{\eta^k})^2}.
   \label{eq:thm_P2CPO_converge_1}
\end{align}
%\textbf{(b)} 
Similarly, under the 2-norm projection, there exists a %series 
sequence $\{\eta^k\}$ such that \algname\ converges to an $\epsilon$-FOSP 
%in Eq. (\ref{eq:sp_cond2}) 
in at most $\mathcal{O}(\epsilon^{-2})$ iterations. Moreover, at step $k+1$ 
%the objective value satisfies
\begin{align}
    f(\vtheta^{k+1}) \leq  f(\vtheta^{k})-\frac{L\epsilon^2}{2({G\sigma_1(
    {\mF^k}^{-1}
    )+\frac{H}{\eta^k}})^2}.
   \label{eq:thm_P2CPO_converge_2}
\end{align}
%\textbf{(c)} Independently of the projections, \algname\ converges to an $(\epsilon,\nu)$-SOSP at most $\mathcal{O}(\nu^{-1})$ iterations. Moreover, at step $k+1$ the objective value satisfies 
%\begin{align}
%   f(\vtheta^{k+1})\leq f(\vtheta^{k}) - (\frac{1}{2}-\frac{\sqrt{2}}{6})\nu.
 %  \label{eq:thm_P2CPO_converge_3}
%\end{align}
\end{theorem}
\vspace{-5mm}
\begin{proof}
The proof and the sequence $\{\eta^k\}$ are given in the supplementary material.
\end{proof}
\vspace{-4mm}
\begin{figure*}[t]
\vspace{0mm}
\label{fig:tasks}
\centering
\subfloat[Gather
%\cite{achiam2017constrained}\label{subfig:pg}
]
{%
\centering\includegraphics[scale=0.2]{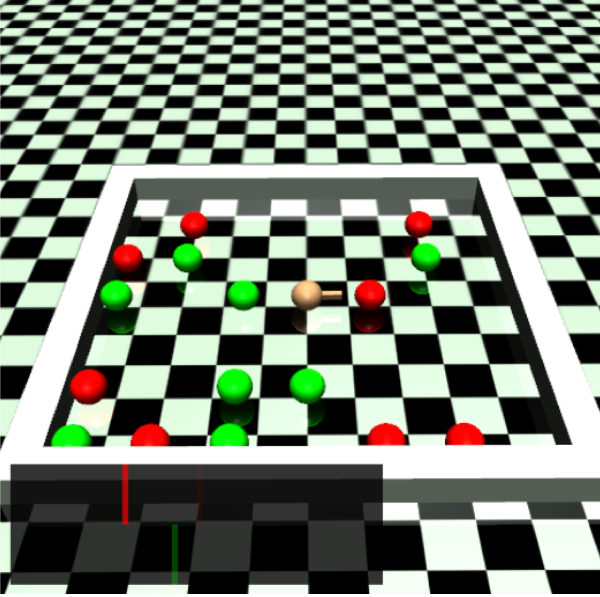}\vspace{-6.0mm}}\hspace{1.0mm} 
\subfloat[Circle
%~\cite{achiam2017constrained}
\label{subfig:pc}
]
{%
\centering\includegraphics[scale=0.2]{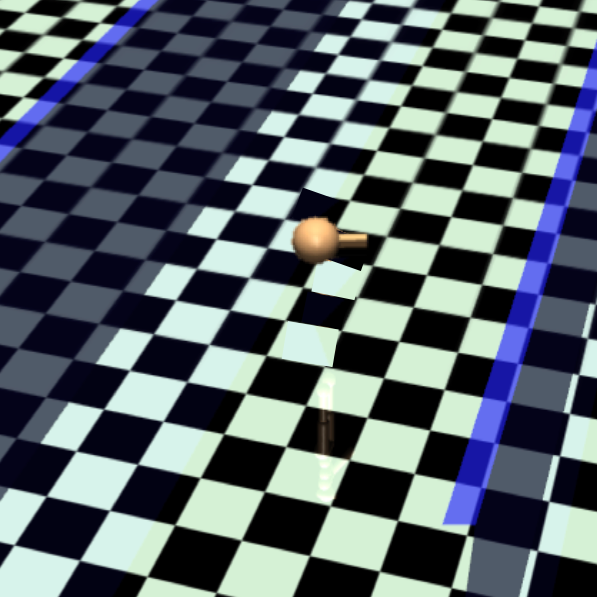}\vspace{-6.0mm}}
\hspace{1.0mm} 
\subfloat[Grid
%~\cite{vinitsky2018benchmarks}
\label{subfig:grid}
]
{%
\centering\includegraphics[scale=0.2]{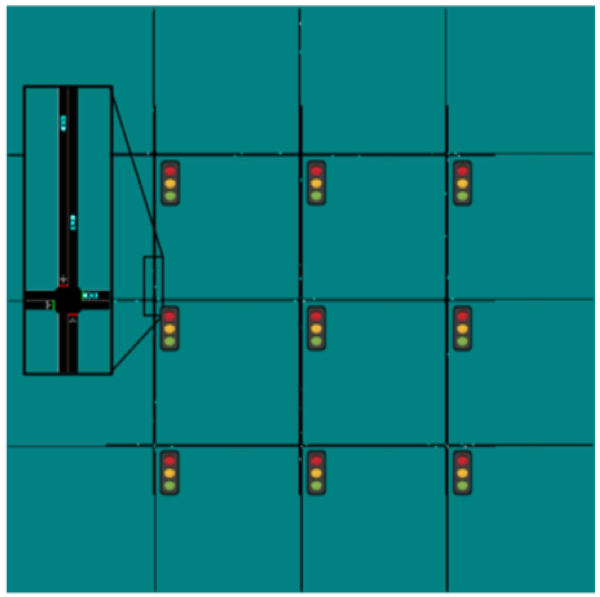}\vspace{-6.0mm}}\hspace{1.0mm}
\subfloat[Bottleneck
%~\cite{vinitsky2018benchmarks}\label{subfig:bn}
]{%
\centering\includegraphics[scale=0.2]{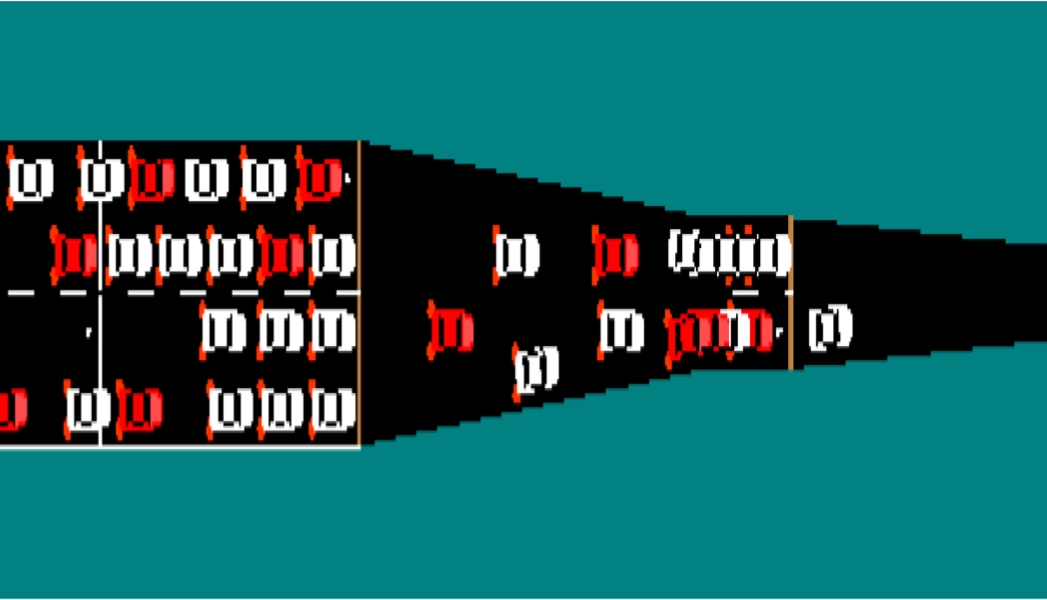}\vspace{-6.0mm}}\hspace{1.0mm}
\subfloat[Car-racing
%~\cite{vinitsky2018benchmarks}\label{subfig:bn}
]{%
\centering\includegraphics[scale=0.2]{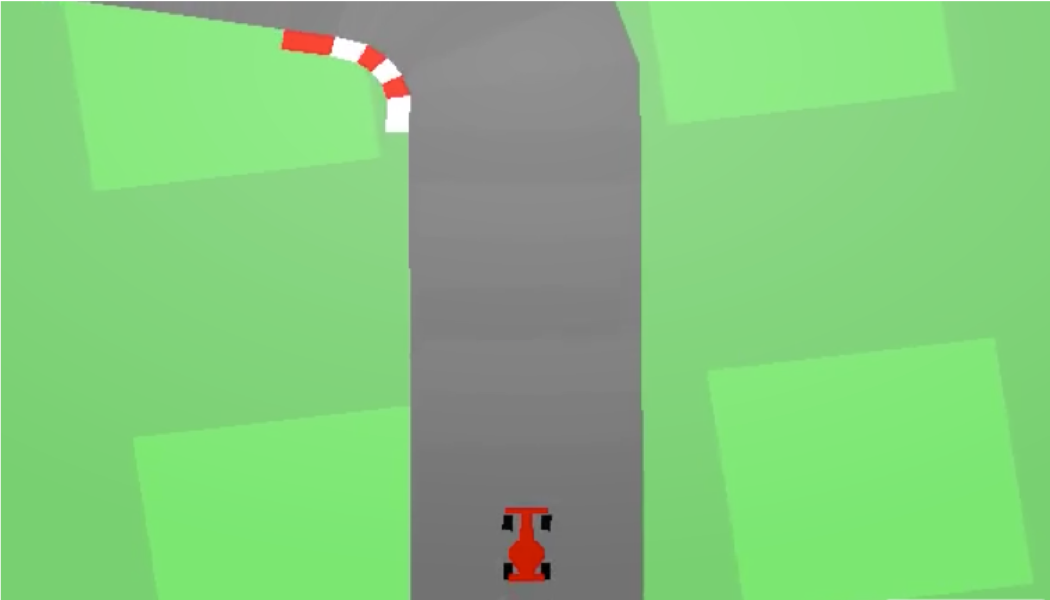}}
\hspace{1.0mm}
\subfloat[Demo.]
{%
\centering\includegraphics[width=27mm,height=20.3mm]{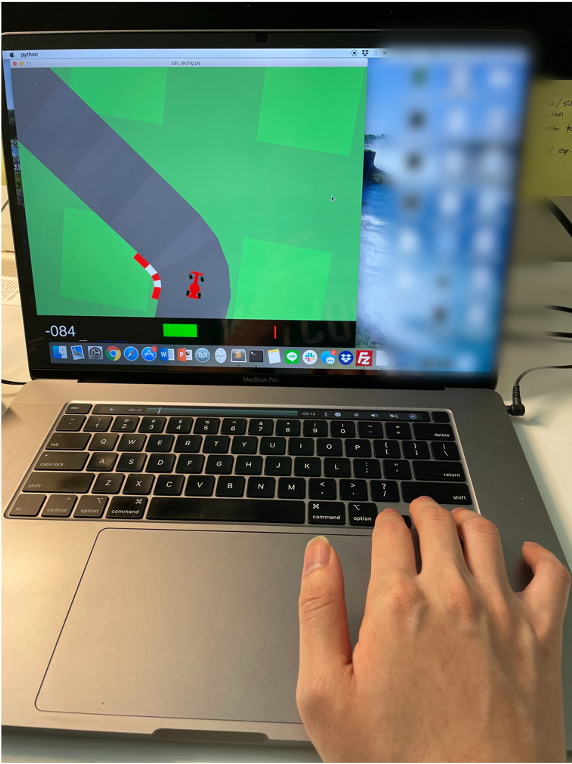}}
\caption{%The gather, circle, grid, bottleneck, and car-racing tasks. 
\textbf{(a)} Gather: 
%the agent is constrained to collect no more than a certain amount of the red poisonous fruit
the agent is rewarded for gathering green apples, but is constrained to collect a limited
%amount
number of red apples~\citep{achiam2017constrained}.
\textbf{(b)} Circle: 
%the agent is constrained to be between the blue walls 
the agent is rewarded for moving in a specified wide circle, but 
is constrained to stay within a safe region smaller than the radius of the circle~\citep{achiam2017constrained}.
\textbf{(c)} Grid: the agent controls the traffic lights in a grid road network and is rewarded for high throughput, but is constrained 
%such as Manhattan. And traffic lights are constrained 
to let lights stay red for at most 7 consecutive seconds~\citep{vinitsky2018benchmarks}.
\textbf{(d)} Bottleneck: the agent controls a set of autonomous vehicles (shown in red) in a traffic merge situation and is rewarded for achieving high throughput,
%such as Oakland-San Francisco Bay Bridge.
but constrained to ensure that human-driven vehicles (shown in white)
%in the traffic are constrained to 
have low speed for no more than 10 seconds~\citep{vinitsky2018benchmarks}.
\textbf{(e)} Car-racing: the agent controls an autonomous vehicle on a race track and is rewarded for driving through as many tiles as possible, but is constrained to use the brakes at most 5 times to encourage a smooth ride~\citep{brockman2016openai}.
\textbf{(f)} A human player plays car-racing with demonstration data logged.
These tasks are to show the applicability of our approach to a diverse set of problems.
%
%The baseline policy in the car-racing task is learned from human demonstrations, whereas the baseline policies in the other tasks are pretrained. %using PCPO \cite{yang2020projection}.
}
\label{fig:tasks}
\end{figure*}

We now make several observations for Theorem \ref{theorem:P2CPO_converge}.

\textbf{(1)} The smaller $H$ is, the greater the decrease in the objective.
%
%Hence increasing $h_D$ 
%in \algname\ can impact reward improvement.
This observation supports the idea of starting with a small value for $h_D$ and increasing it only when needed.

\textbf{(2)} %the singular value of the Fisher information matrix (\ie the local curvature of $f$) affects the decrease of the objective value.
Under the KL-divergence projection, the effect of $\sigma_1(\mF^k)$ is negligible.
This is because in this case $\eta^k$ is proportional to $\sigma_1(\mF^k).$ 
%is determined by the KL-divergence between two consecutive updated policies (see the supplementary material).
%This implies that $\eta^k$ is proportional to $\sigma_1(\mF^k).$ 
%
Hence $\sigma_1(\mF^k)$ does not play a major role in decreasing the objective value.
%
%This observation implies that under the KL-divergence projection, the spectrum of $\mF$ does not play a major role.

\textbf{(3)} Under the 2-norm projection, %$\sigma_1(\mF^{-1})$ (\ie $1/\sigma_n(\mF)$) affects the decrease of objective value. 
the smaller $\sigma_1({\mF^k}^{-1})$ (\ie larger $\sigma_n(\mF^k)$) is, 
%the more the objective value decreases.
the greater the decrease in the objective.
This is because a large $\sigma_n(\mF^k)$ means a large curvature of $f$ in all directions. 
This implies that the 2-norm distance between the pre-projection and post-projection points is small, 
leading to a small deviation from the reward improvement direction after doing projections.

\parab{Comparison to \citet{yang2020projection}.} Our work is inspired by PCPO~\citep{yang2020projection}, which also uses projections to ensure constraint satisfaction during policy learning.
%
% Such principle, in contrast to adding a cost objective in the reward function (\ie soft constraints), ensures constraint satisfaction during training.
%
However, there are a few key differences between our work and PCPO.
\textbf{(1) Algorithm.} 
PCPO does not have the capability to safely exploit a baseline policy, which makes it less sample efficient in cases when we have demonstrations or teacher agents.
In addition, \algname's update dynamically sets distances between policies while PCPO does not--this update is important to effectively and safely learn from the baseline policy.
\textbf{(2) Theory.} Our analysis provides a safety guarantee to ensure the feasibility of the optimization problem while~\citet{yang2020projection} do not.
Merely adding an IL objective in the reward objective of PCPO cannot make the agent learn efficiently, as shown in our experiments (Section~\ref{subsectoin:results}).
In addition, compared to the analysis in \citet{yang2020projection}, Theorem~\ref{theorem:P2CPO_converge} shows the existence of the step size for each iteration. 
%In addition, in PCPO, the bound for 2-norm projection requires the assumption that the maximum eigenvalue of $\mF^k$ is smaller than one. In contrast, we do not require this.
%
\textbf{(3) Problem.} Finally, our work tackles a \textit{different} problem compared to PCPO (which only ensures safety).
We focus on how to safely and efficiently learn from an existing baseline policy, which is more conducive to practical applications of safe RL.

%Hence one may choose either the KL-divergence projection or the 2-norm projection depending on problems.
%one of the projections to achieve better performance depending on problems.
%
%Please see more discussion on supplementary material.
%
%The intuition is that under the Lipschitz assumption, a larger distance between two consecutive updated points would have a larger improvement of the objective value. 
%
%Hence, under the KL-divergence projection the coordinates of the descent direction and the projections are all individually scaled by $\frac{1}{\sigma_i(\mF)}.$ 
%
%This implies that smaller singular values (\ie $\frac{1}{\sigma_i(\mF)}$ is larger) makes two consecutive updated points more distant. 
%
%On the other hand, under the 2-norm projection the coordinates of the descent direction and the projections are different. This implies that larger singular values (\ie $\frac{1}{\sigma_i(\mF)}$ is smaller) make these two directions more orthogonal to each other. Hence two consecutive updated points are more distant. 
%

%Part (c) shows that with the singular vector $\vu$ of the Hessain followed by the trust region update, one can reduce the iteration complexity to a SOSP from $\mathcal{O}(\nu^{-3})$ \cite{mokhtari2018escaping} to $\mathcal{O}(\nu^{-1})$.

%% file: results.tex
% !TEX root = neurips_2019.tex
\section{Experiments}
\label{sec:experiments}
Our experiments study the following three questions: 
\textbf{(1)} How does \algname\ perform compared to other baselines in behavior cloning and safe RL in terms of learning efficiency and constraint satisfaction?
\textbf{(2)} How does \algname\ trained with sub-optimal $\pi_B$ perform (\eg human demonstration)?
\textbf{(3)} How does the step 2 in \algname\ affects the performance?
\subsection{Setup}
\parab{Tasks.} We compare the proposed algorithm with existing approaches on five control tasks: 
three tasks with safety constraints ((a), (b) and (e) in Fig.~\ref{fig:tasks}), 
and two tasks with fairness constraints ((c) and (d) in Fig.~\ref{fig:tasks}).
These tasks are briefly described in the caption of Fig.~\ref{fig:tasks}. 
%
%The first two tasks -- \textit{Gather} and \textit{Circle} -- are Mujoco environments with state space constraints introduced by \cite{achiam2017constrained}. 
%
%The other two tasks -- \textit{Grid} and \textit{Bottleneck} -- are traffic management problems where the agent controls either a traffic light or a fleet of autonomous vehicles introduced by \cite{vinitsky2018benchmarks}. 
%
%
%The Mujoco tasks are used to compare the proposed algorithm with baselines in the typical RL benchmark. 
%
We chose the traffic management tasks since a good control policy can benefit millions of drivers. 
In addition, we chose the car-racing task since a good algorithm should safely learn from baseline human policies.
For all the algorithms, we use neural networks to represent Gaussian policies.
We use the KL-divergence projection in the Mujoco and car-racing tasks, and the 2-norm projection in the traffic management task since it achieves better performance. 
We use a grid-search to select for the hyper-parameters.
See the supplementary material for more experimental details.
\begin{figure*}[t]
\centering
%\rulesep
%{\setlength\arrayrulewidth{0.01pt}
\begin{tabular}[b]{@{}c@{}}%
\textbf{Bottleneck}\\
\vspace{-0mm}
\includegraphics[width=0.28\linewidth]{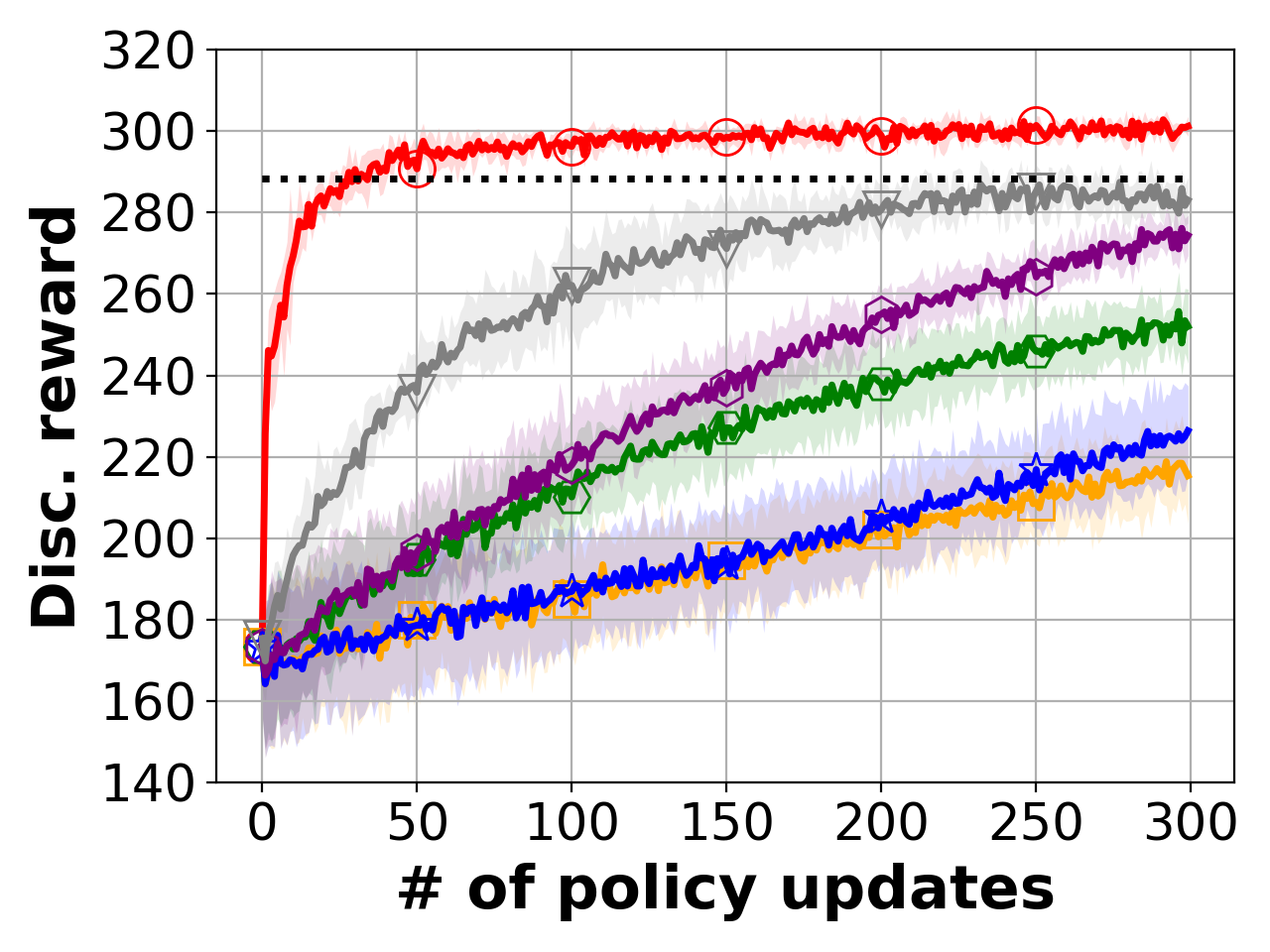}%
\\
\includegraphics[width=0.28\linewidth]{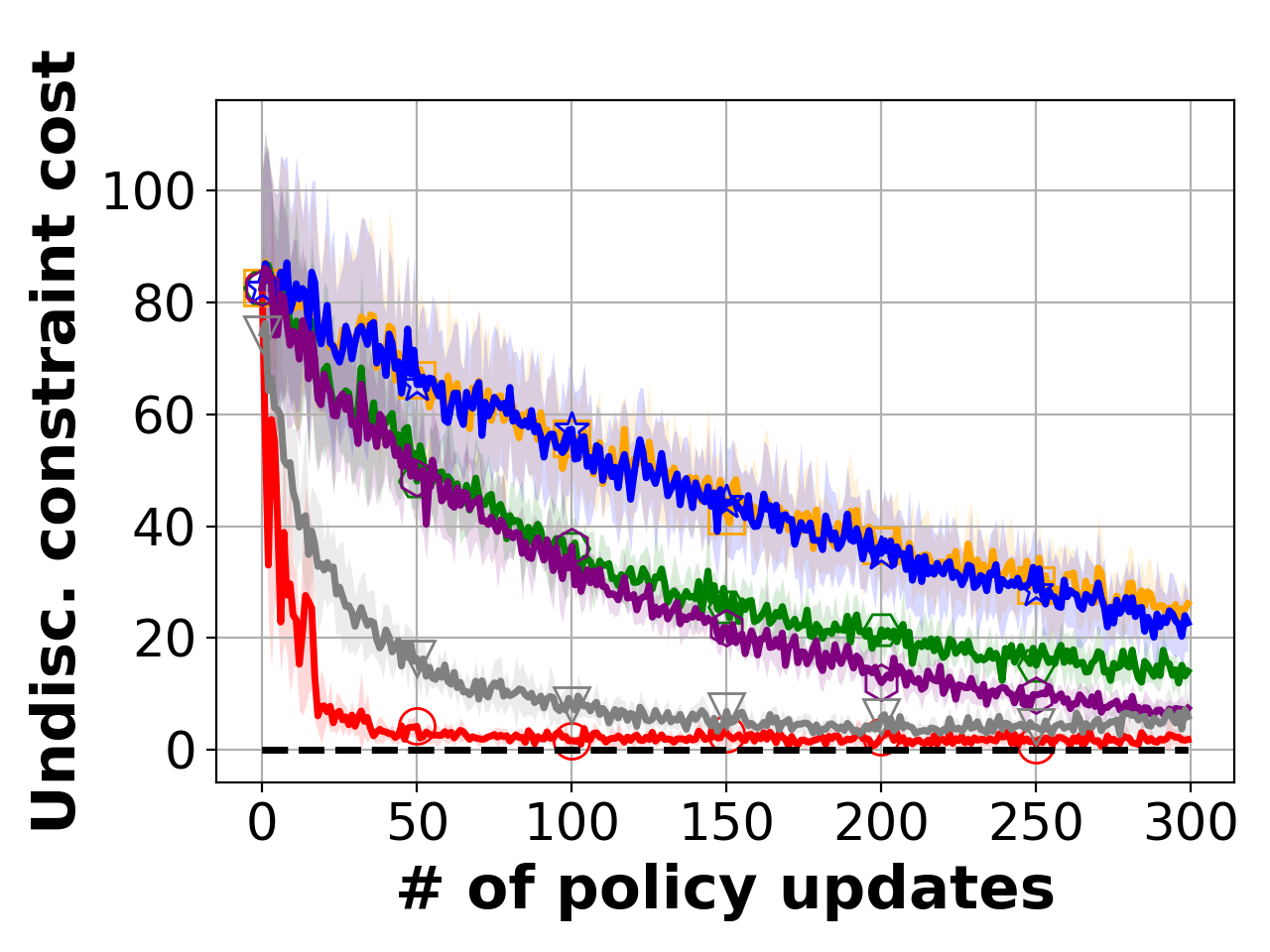}%
\\
\includegraphics[width=0.28\linewidth]{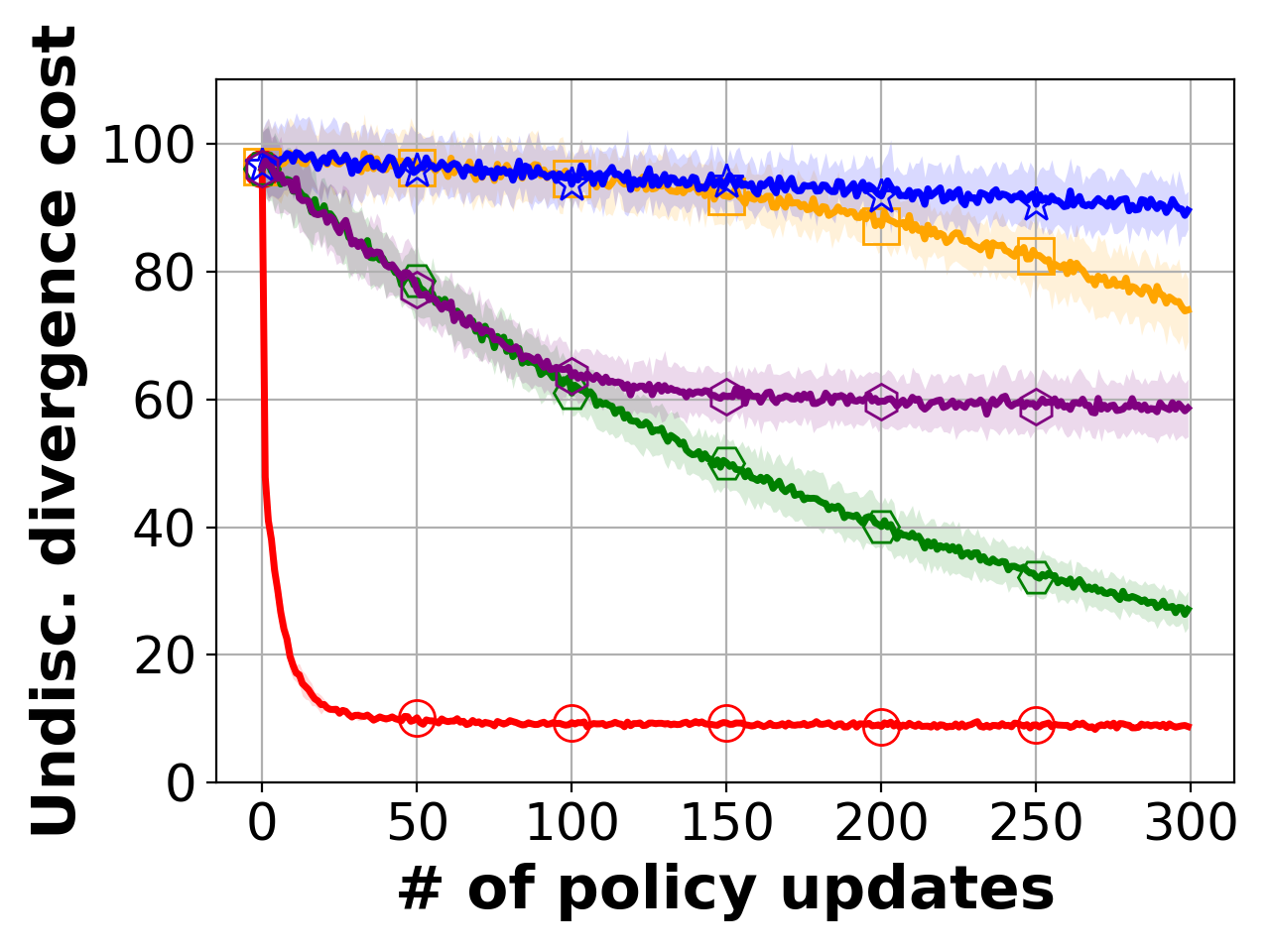}%
\end{tabular}%}
%
%\rulesep
\begin{tabular}[b]{|@{}c@{}|}%
\textbf{Car-racing}\\
\vspace{-0mm}
\includegraphics[width=0.28\linewidth]{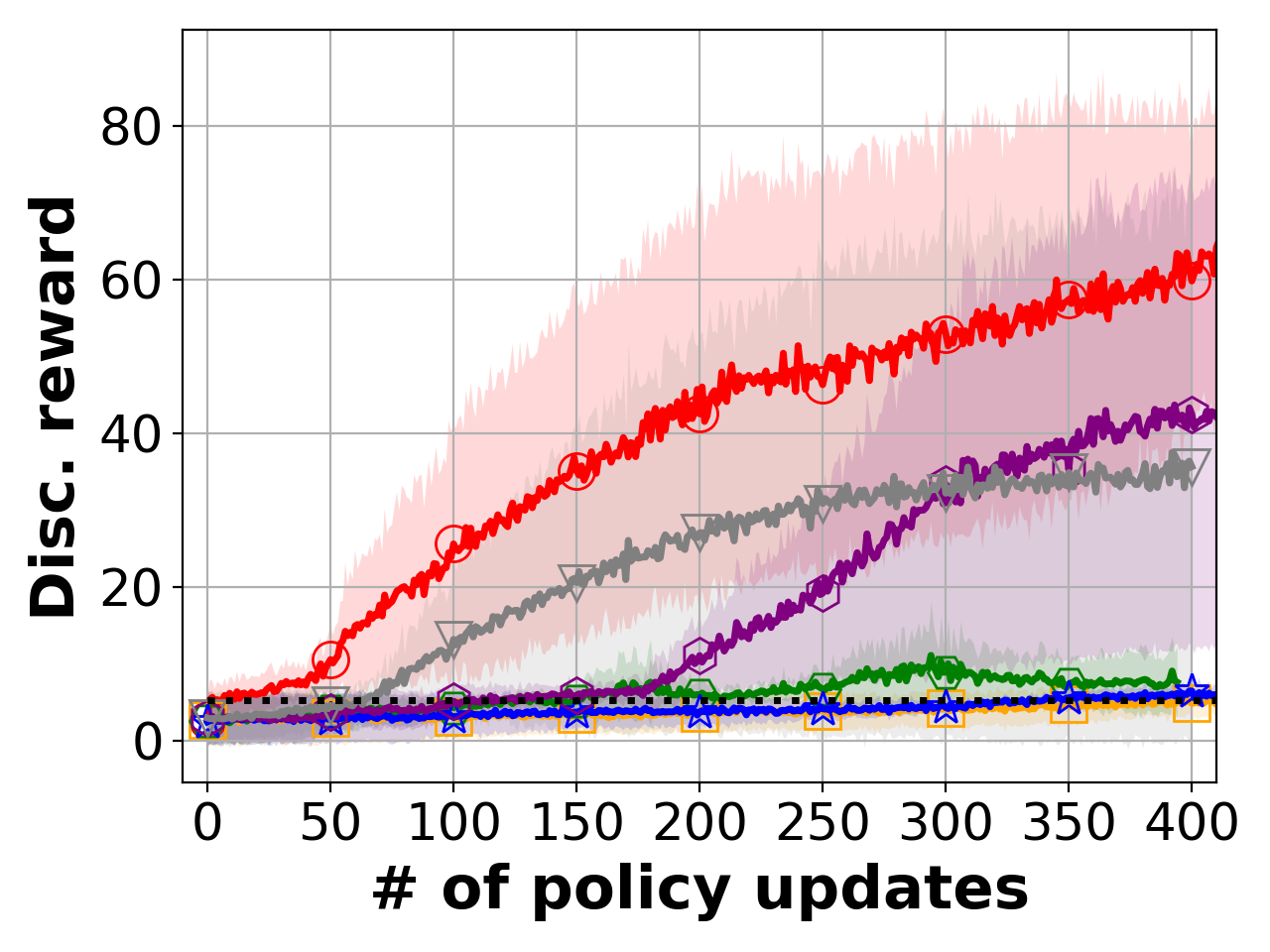}%
\\
\includegraphics[width=0.28\linewidth]{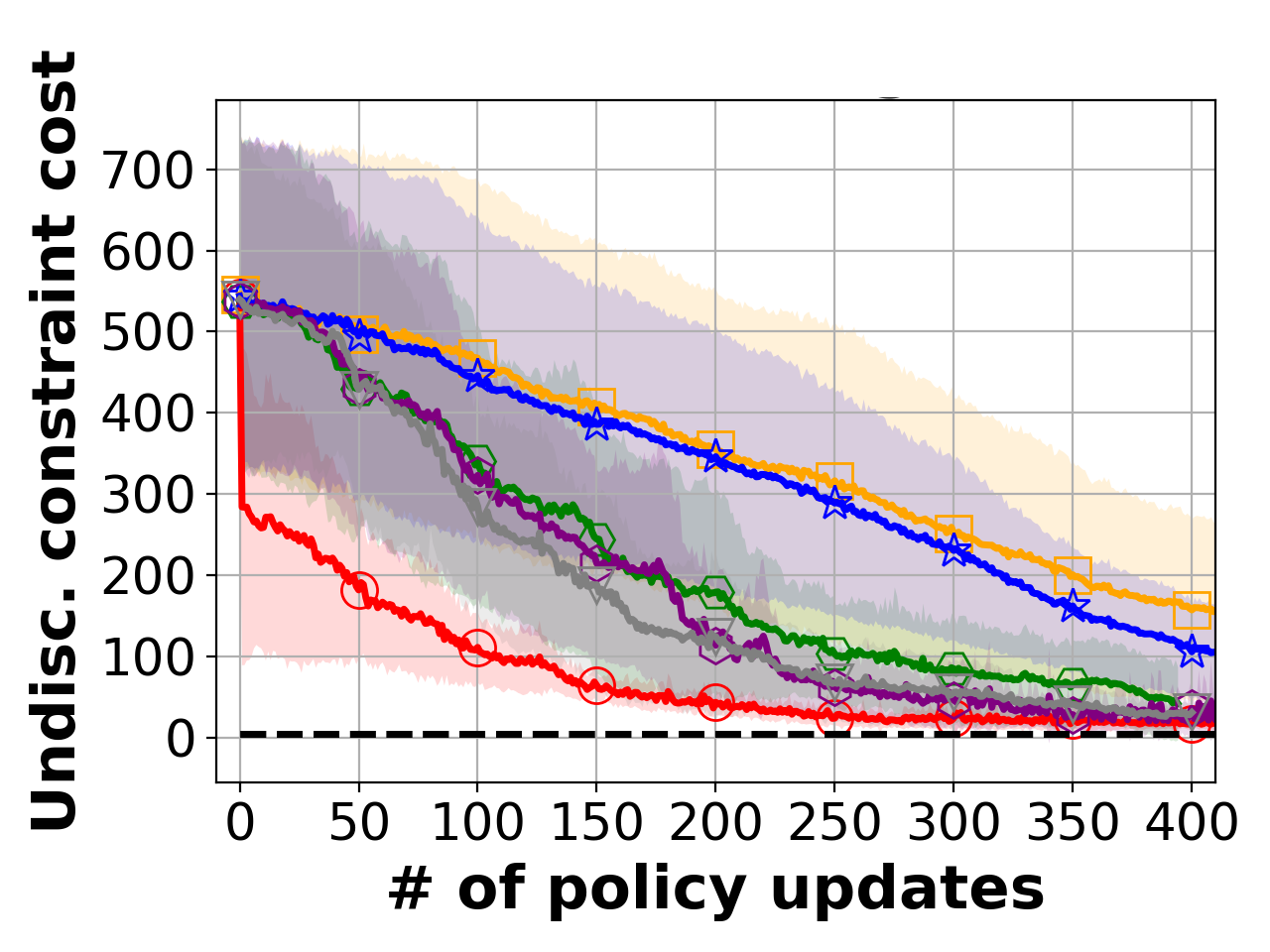}%
\\
\includegraphics[width=0.28\linewidth]{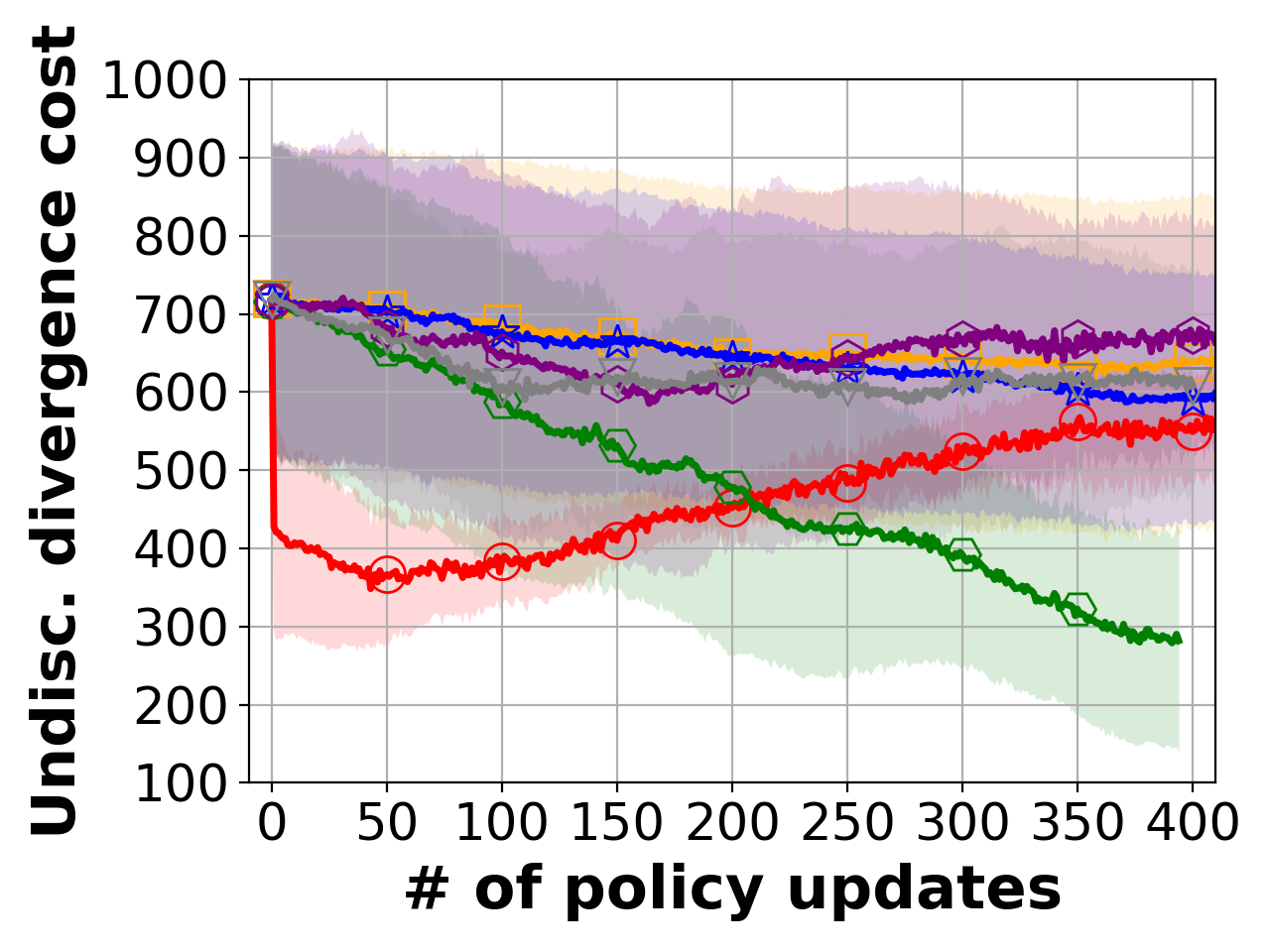}%
\end{tabular}%
\begin{tabular}[b]{@{}c@{}}%
\textbf{Grid}\\
\vspace{-0mm}
\includegraphics[width=0.28\linewidth]{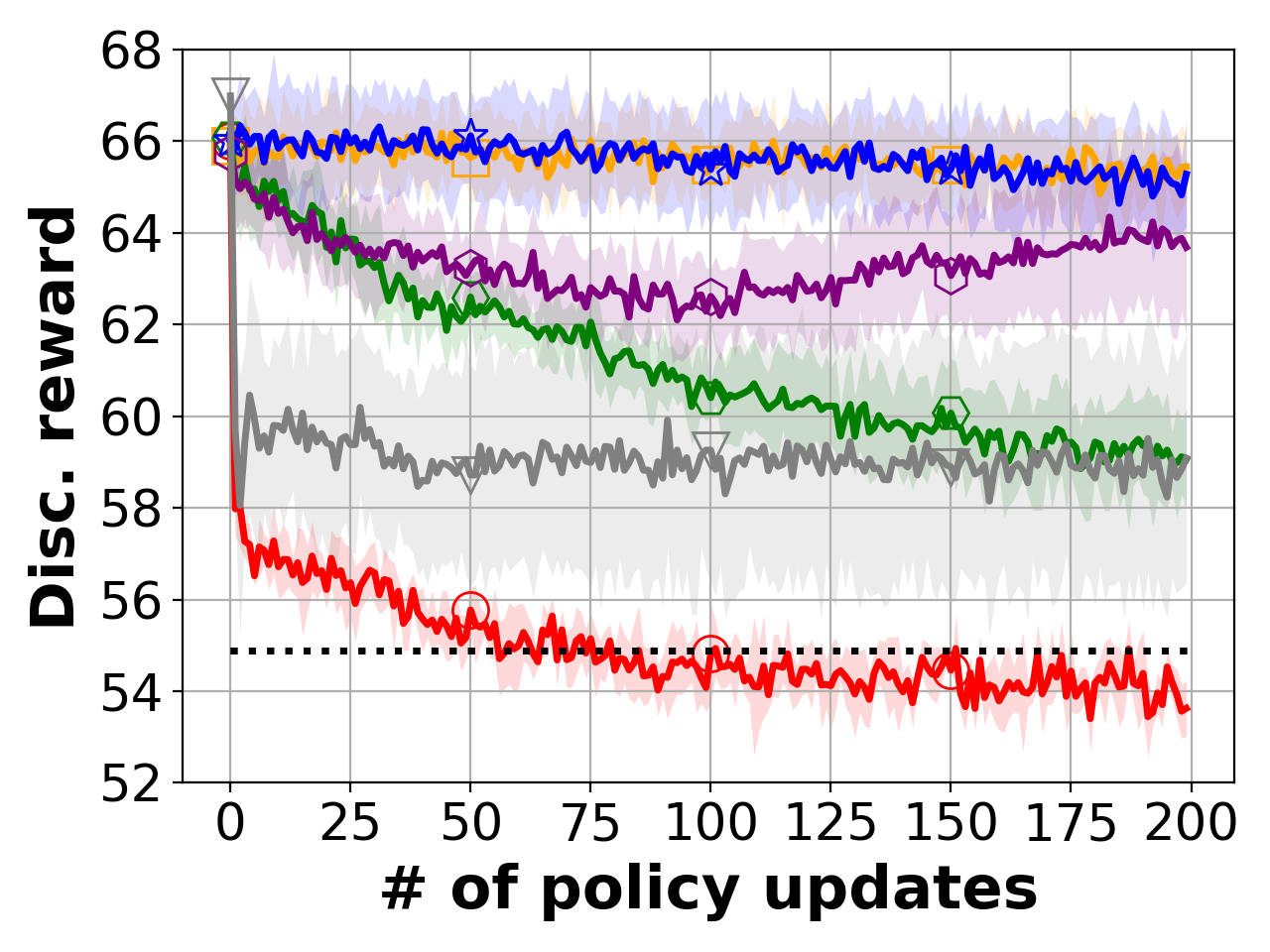}%
\\
\includegraphics[width=0.28\linewidth]{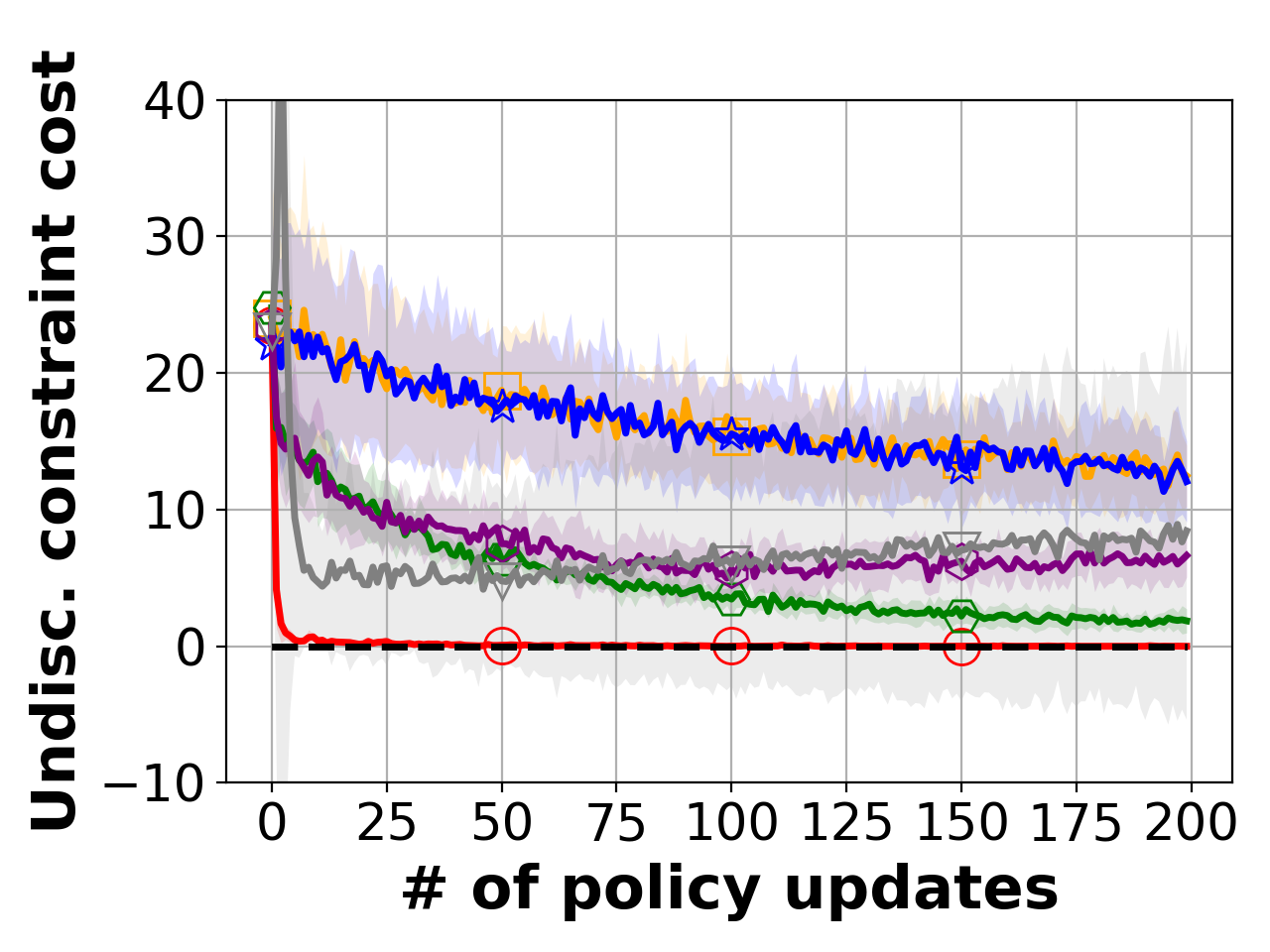}%
\\
\includegraphics[width=0.28\linewidth]{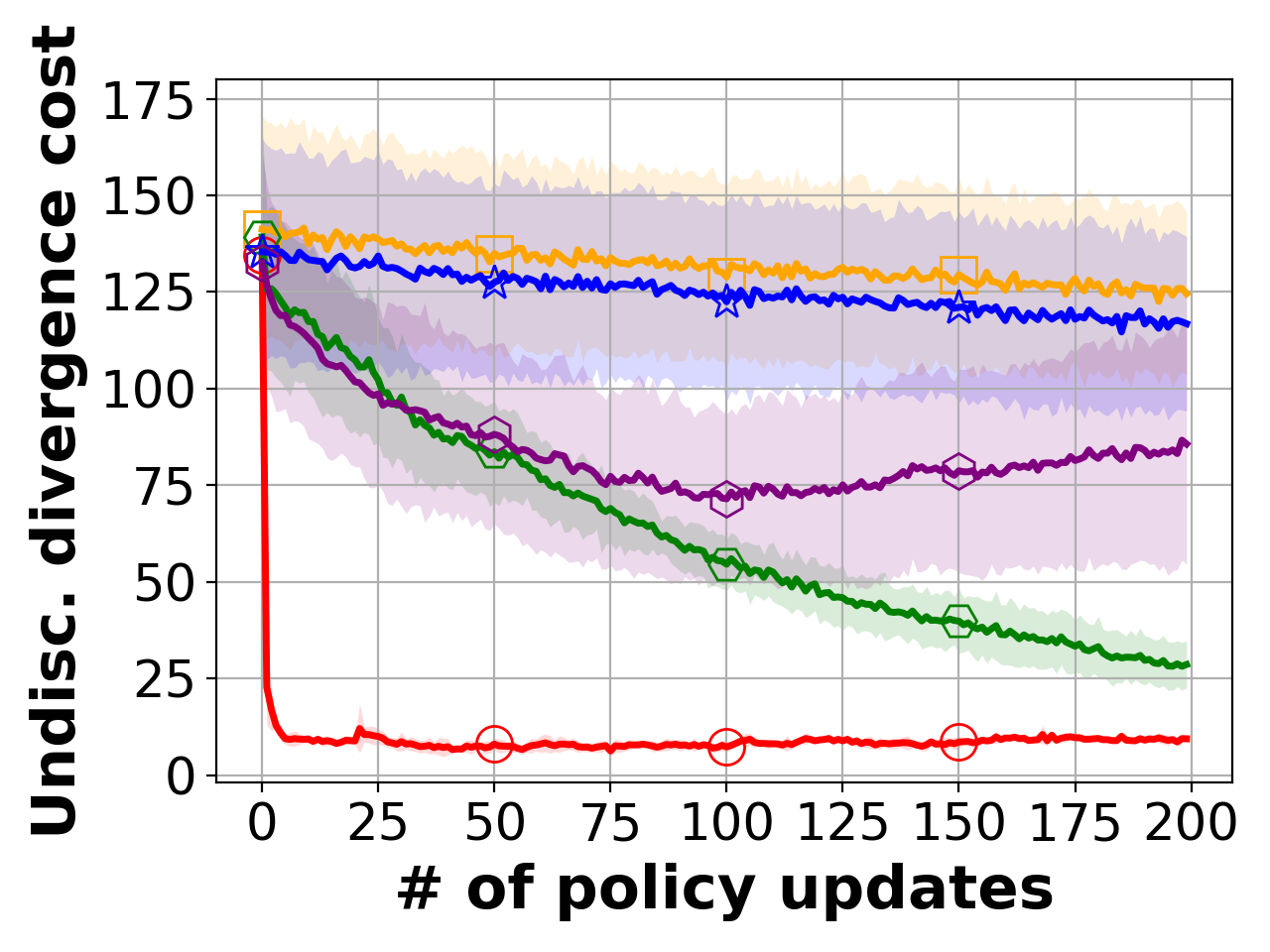}
\end{tabular}
\includegraphics[width=0.95\linewidth]{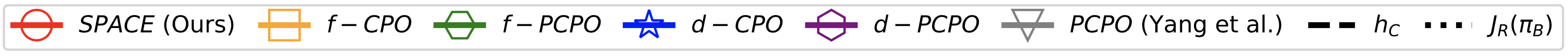}
\caption{
The discounted reward, 
the undiscounted constraint cost,
and the undiscounted divergence cost
over policy updates for the tested algorithms and tasks.
The solid line is the mean and the shaded area is the standard deviation over 5 runs (random seed).
The baseline policies in the grid and bottleneck tasks are $\pi_B^\mathrm{near},$
and the baseline policy in the car-racing task is $\pi_B^\mathrm{human}.$
%
%The curves for baseline oracle, PCPO, indicate the reward and constraint values when the baseline policy is \emph{ignored}. %
%
The black dashed line is the \textbf{cost constraint threshold $h_C.$}
We observe that \algname\ is the only algorithm that satisfies the constraints while achieving superior reward performance.
Although $\pi_B^\mathrm{human}$ has substantially low reward, \algname\ still can learn to improve the reward.
(We show the results in these tasks as representative cases since they are more challenging.
Please read Appendix for more results. Best viewed in color.)
%\textcolor{red}{Peter: in the car-racing task, how about the curve after 400? {\color{blue}Jimmy: After 400, the red and purple one keep improving the reward performance. But the red still outperforms the purple one. The reason to cut in 400 is that I did not have the data for the other baseline. 600 iterations take about 7 days on CS server. Please let me know if the legend and the figure make sense or not!}}
}
\label{fig:overall_performance}
%\end{mdframed}
\end{figure*}

%\textcolor{red}{Karthik: 1 table for the overall performance (grid, bottleneck, car-racing, point gather, and ant circle) in the final iteration; 1 figure for imitating with the safe policy/reward policy; 1 figure for fixed/dynamic hp; 1 figure for making the point that you may want to satisfy different objectives}\textcolor{blue}{~Jimmy: I think showing that the training progress is also important}
%\textcolor{blue}{~Jimmy:Add the d-pcpo, d-cpo and pre-training to the grid and bottleneck task}

\parab{Baseline policies $\pi_B$.} 
%In the gather, circle, grid, and bottleneck tasks, we pre-train the baseline policies using PCPO~\citep{yang2020projection}.
%
To test whether \algname\ can safely and efficiently leverage the baseline policy, %with different cost constraints, 
%whether the proposed algorithm can safely learn from the baseline policy with different or the same cost constraints, 
we consider \textit{three variants} of the baseline policies.

\textbf{(1)} \textit{Sub-optimal} $\pi_B^\mathrm{cost}$ with $J_C(\pi_B^\mathrm{cost})\approx0.$

\textbf{(2)} \textit{Sub-optimal} $\pi_B^\mathrm{reward}$ with $J_C(\pi_B^\mathrm{reward})>h_C.$ 
\textbf{(3)} $\pi_B^\mathrm{near}$ with $J_C(\pi_B^\mathrm{near})\approx h_C$
(\ie the baseline policy has the same cost constraint as the agent, but is not guaranteed to have an optimal reward performance).

These $\pi_B$ have \textit{different} degrees of constraint satisfaction.
This is to examine whether \algname\ can safely learn from sub-optimal $\pi_B.$
%achieve better learning efficiency given the constraint-satisfying $\pi_B$ (\ie $\pi_B^\mathrm{cost}$ and $\pi_B^\mathrm{near}$), and \textbf{(2)} whether \algname\ can safely learn from constraint-violating $\pi_B$ (\ie $\pi_B^\mathrm{reward}$).
%
In addition, in the car-racing task we pre-train $\pi_B$ using an off-policy algorithm (DDPG \citep{lillicrap2015continuous}), which directly learns from \textit{human demonstration data} (Fig.~\ref{fig:tasks}(f)). 
This is to demonstrate that $\pi_B$ may come from a teacher or demonstration data.
This \textit{sub-optimal} human baseline policy is denoted by $\pi_B^\mathrm{human}.$

For the ease of computation, 
we update $h_D$ using $v\cdot(J_C(\pi^k)-h_C)^2+h_D^k$ from Lemma~\ref{theorem:h_D}, with a constant $v>0.$
We found that the performance is not heavily affected by $v$ since we will update $h_D$ at later iteration.
The ablation studies of $v$ can be found in Appendix \ref{subsec:appendix_details}.

\parab{Baseline algorithms.}
Our goal is to study how to \textit{safely} and \textit{efficiently} learn from \textit{sub-optimal} (possibly unsafe) baseline policies.
We compare \algname\ with five baseline methods that combine behavior cloning and safe RL algorithms.
%
%These heuristics are commonly used \citep{rajeswaran2017learning}.

\textbf{(1)} \textit{Fixed-point Constrained Policy Optimization (f-CPO).} In f-CPO, we add the divergence objective in the reward function.
%
%This is a simple extension of CPO \cite{achiam2017constrained}.
%
The weight $\lambda$ is fixed followed by a CPO update (optimize the reward and divergence cost w.r.t.~the trust region and the cost constraints). The f-CPO policy update solves~\cite{achiam2017constrained}:
\begin{align}
        \vtheta^{k+1}=\argmax\limits_{\vtheta}~&(\vg^k+\lambda\va^k)^{T}(\vtheta-\vtheta^{k})\nonumber\\ \text{s.t.}&~\frac{1}{2}(\vtheta-\vtheta^{k})^{T}\mF^k(\vtheta-\vtheta^{k})
        \leq \delta\nonumber\\
        &~{\vc^k}^{T}(\vtheta-\vtheta^{k})+d^k\leq 0. \nonumber
\end{align}
%
%f-CPO adds the divergence objective in the reward function. The lack of weight-adjustment makes it susceptible to suboptimal $\pi_B.$

%
\textbf{(2)} \textit{Fixed-point PCPO (f-PCPO).}  In f-PCPO, we add the divergence objective in the reward function. 
The weight $\lambda$ is fixed followed by a PCPO update (two-step process: optimize the reward and divergence cost, followed by the projection to the safe set).
The f-PCPO policy update solves: 
\begin{align}
&\vtheta^{k+\frac{1}{2}}=\argmax\limits_{\vtheta}~(\vg^k+\lambda\va^k)^{T}(\vtheta-\vtheta^{k})\nonumber\\ &\text{s.t.}~\frac{1}{2}(\vtheta-\vtheta^{k})^{T}\mF^k(\vtheta-\vtheta^{k})\leq \delta,\quad(\text{trust region})\nonumber\\
   &\vtheta^{k+1} = \argmin\limits_{\vtheta}~\frac{1}{2}(\vtheta-{\vtheta}^{k+\frac{1}{2}})^{T}\mL(\vtheta-{\vtheta}^{k+\frac{1}{2}})\nonumber\\
    &\text{s.t.}~{\vc^k}^{T}(\vtheta-\vtheta^{k})+d^k\leq 0.\quad(\text{cost constraint})\nonumber
\end{align}
%
%Note that these to show that they are susceptible to a sub-optimal $\pi_B.$

\textbf{(3)} \textit{Dynamic-point Constrained Policy Optimization (d-CPO).} The d-CPO update solves f-CPO problem with a stateful $\lambda^{k+1}={(\lambda^{k})}^{\beta},$ where $0<\beta<1.$
This is inspired by \citet{rajeswaran2017learning}, in which they have the same weight-scheduling method to adjust $\lambda^k$.

\textbf{(4)} \textit{Dynamic-point PCPO (d-PCPO).} The d-PCPO update solves f-PCPO problem with a stateful $\lambda^{k+1}={(\lambda^{k})}^{\beta},$ where $0<\beta<1.$
For all the experiments and the algorithms, the weight is fixed and it is set to 1.
Note that both d-CPO and d-PCPO regularize the standard RL objective with the distance w.r.t. the baseline policy and make the regularization parameter (\ie $\lambda$) fade over time.
This is a common practice to learn from the baseline policy.
In addition, in many real applications you cannot have access to parameterized $\pi_B$ (\eg neural network policies) or you want to design a policy with different architectures than  $\pi_B.$ 
Hence in our setting, we cannot directly initialize the learning policy with the baseline policy and then fine-tune it.

%\textbf{(5)} \textit{Projection-based CPO (PCPO)}~\citep{yang2020projection}. PCPO is a state-of-the-art algorithm on learning a constraint-satisfying policy.
%
%The PCPO update solves f-PCPO problem with $\lambda=0.$
%
%Note that PCPO \textit{ignores} the baseline policy, making it less sample-efficient.

%\karthik{might be useful to have 1-2 sentences contextualizing these baselines with our method. i.e. what parts are similar, what parts are different.}

%

\begin{figure*}[t]
\centering
\begin{tabular}[b]{@{}c@{}}%
\textbf{Gather}\\
\includegraphics[width=0.246\linewidth]{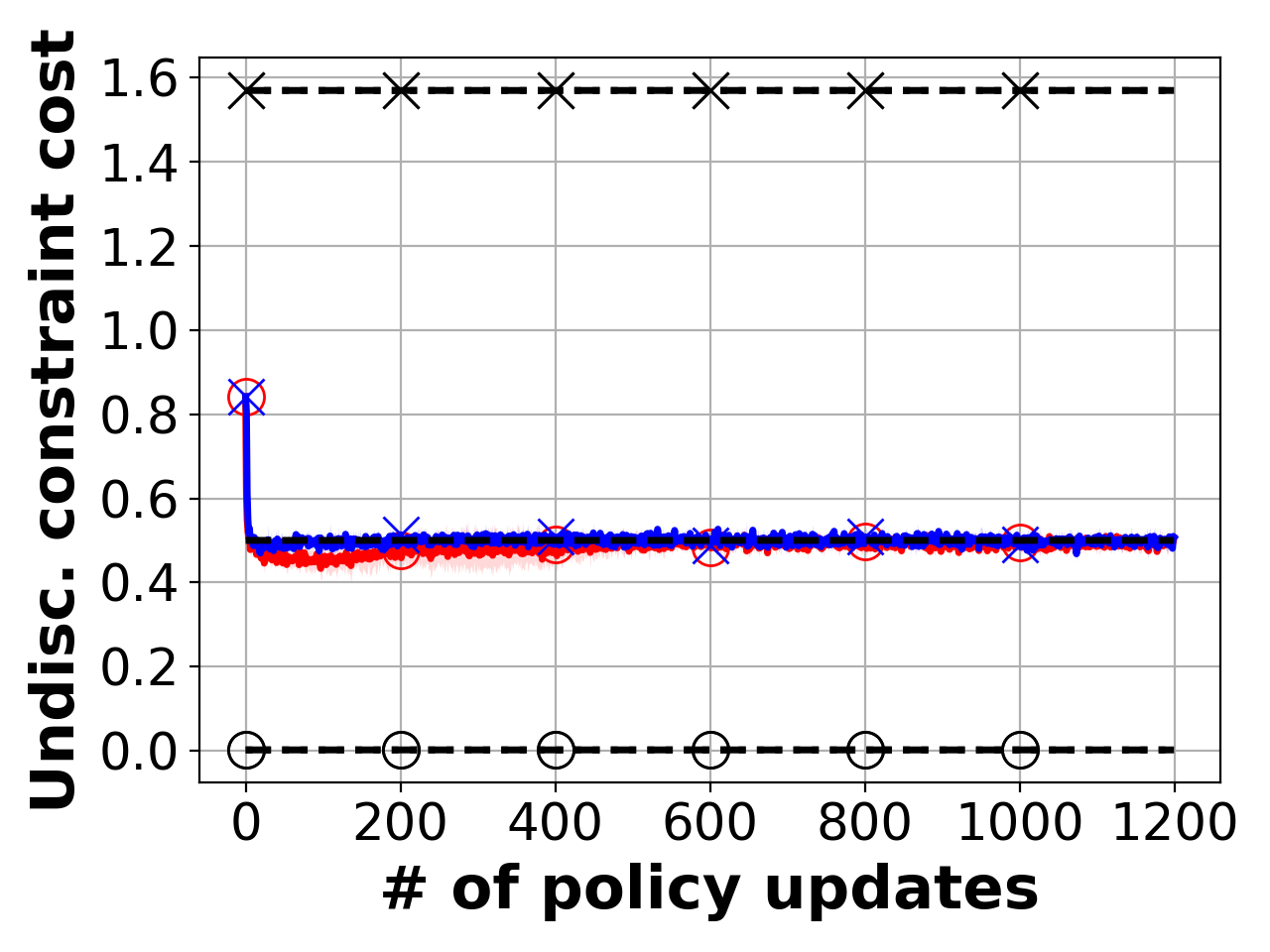}%
\includegraphics[width=0.246\linewidth]{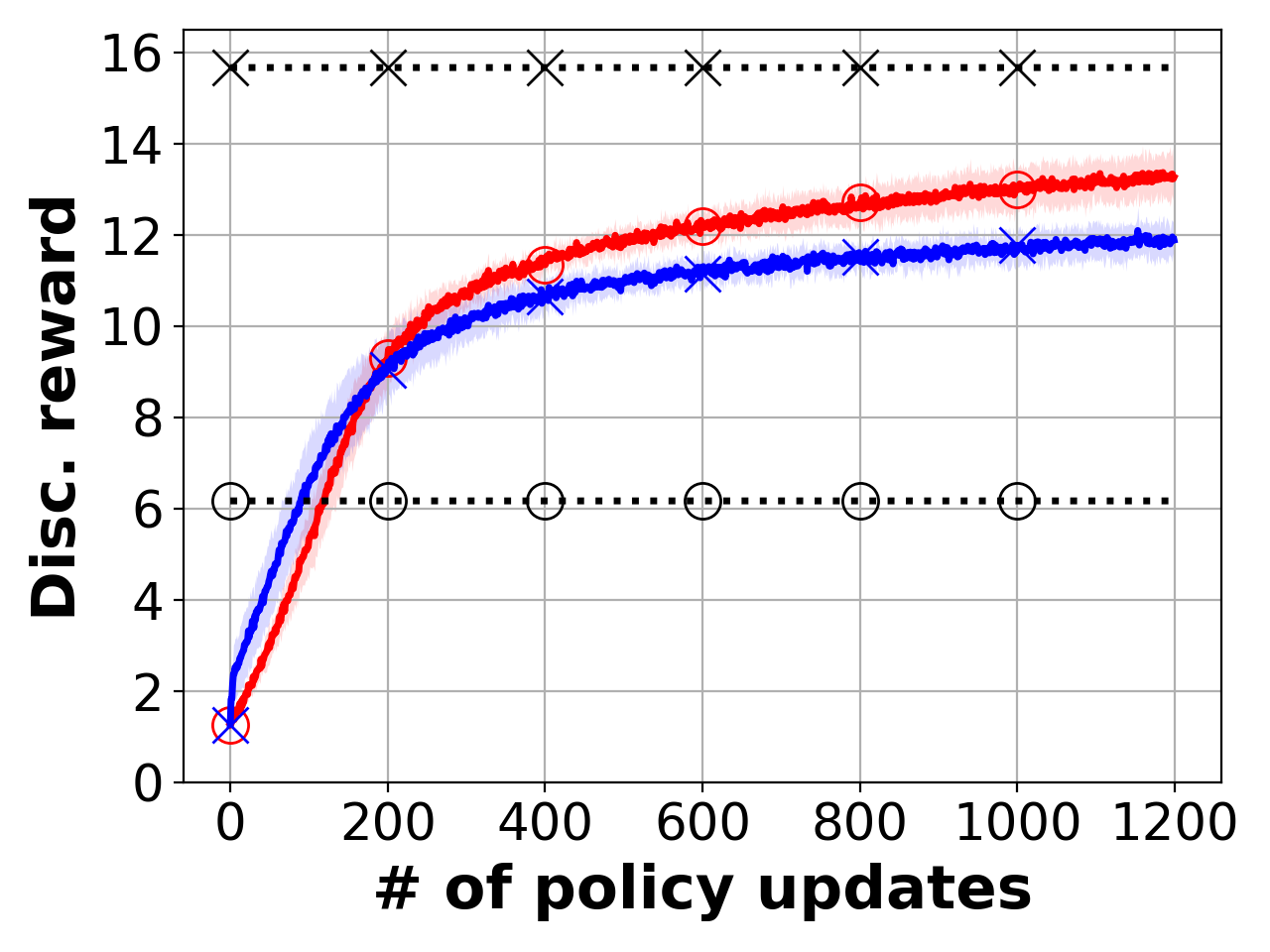}%
\end{tabular}
\begin{tabular}[b]{|@{}c@{}}%
\textbf{Circle}\\
\includegraphics[width=0.246\linewidth]{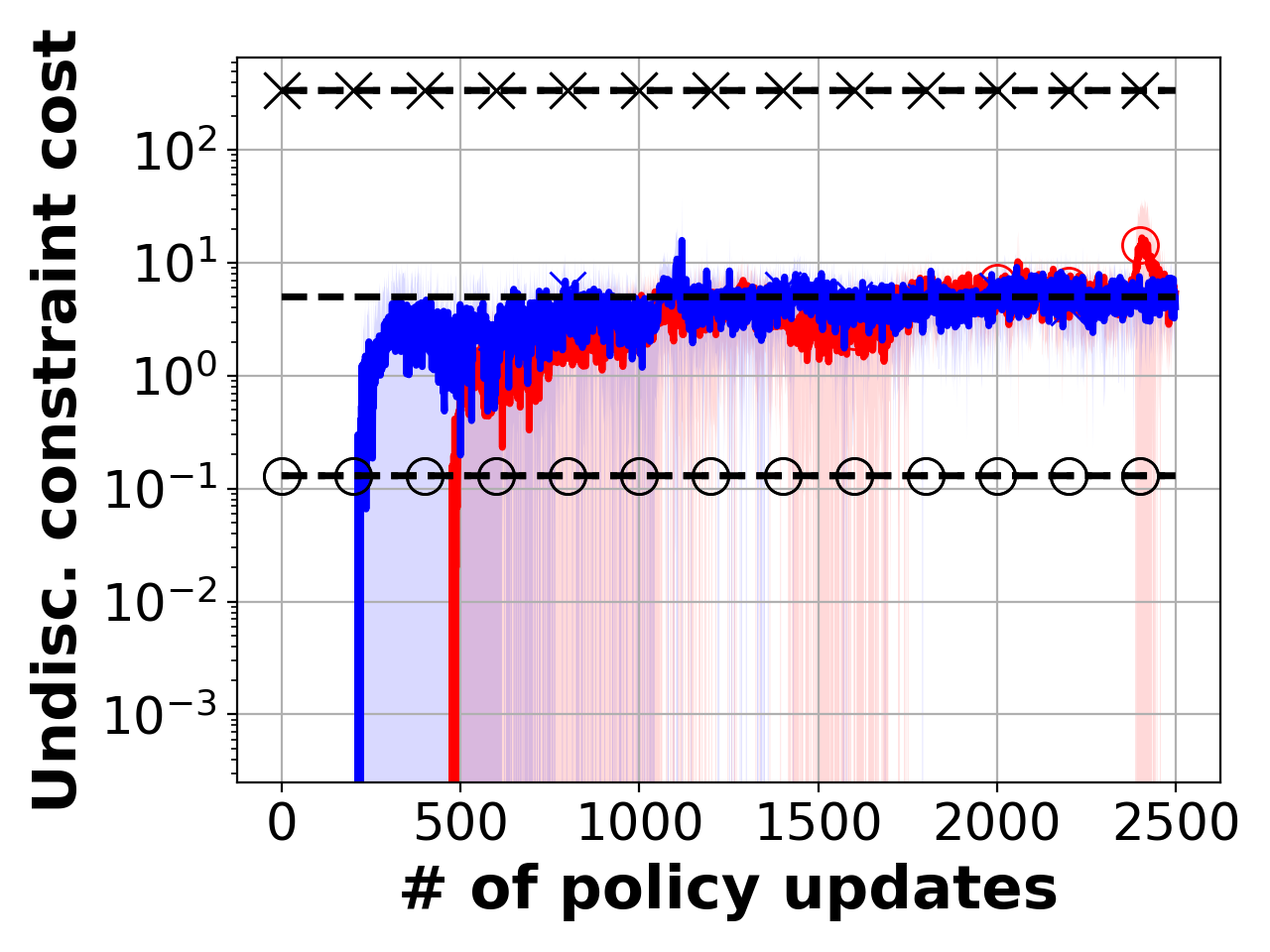}%
\includegraphics[width=0.246\linewidth]{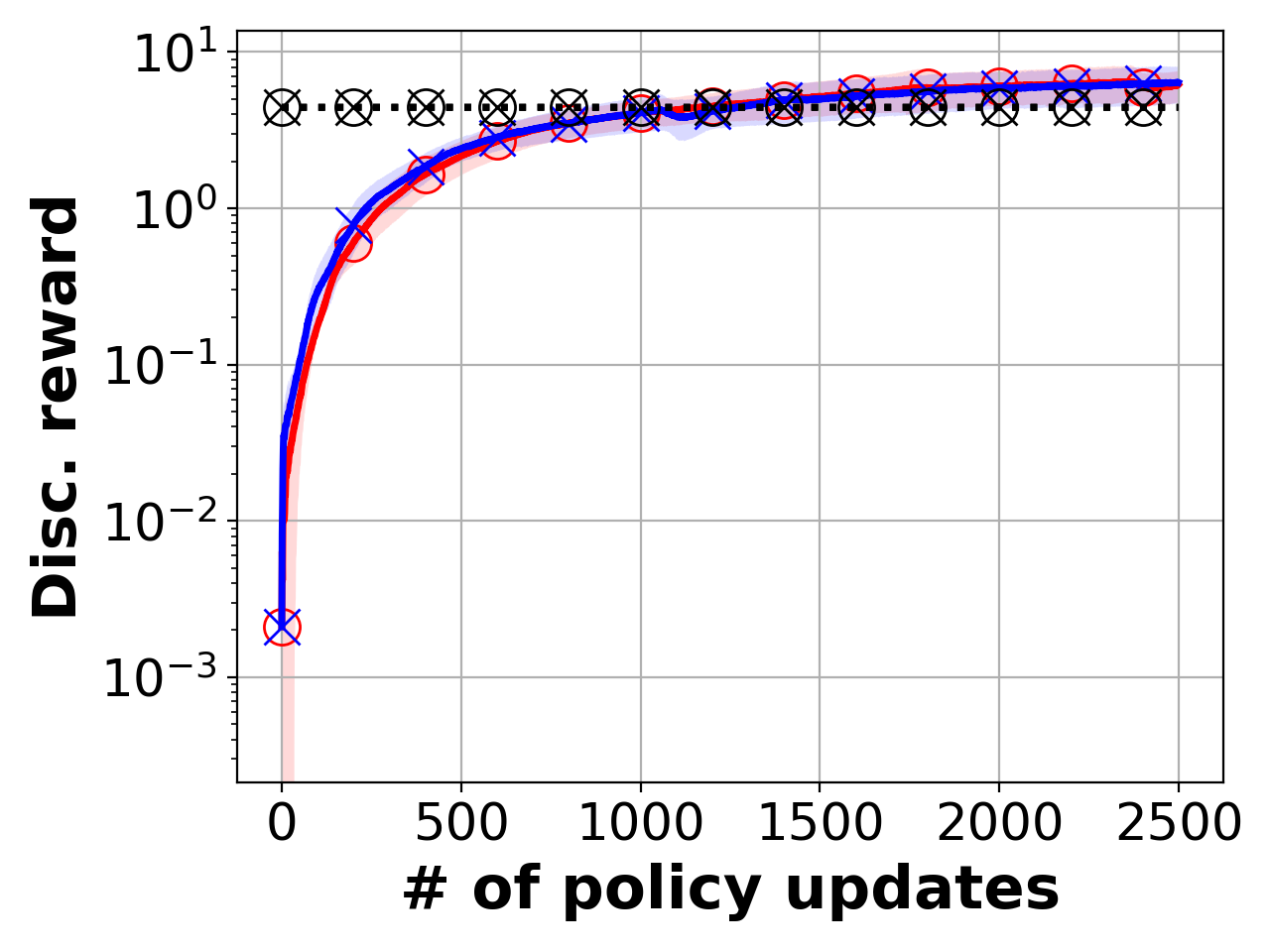}%
\end{tabular}
\includegraphics[width=0.45\linewidth]{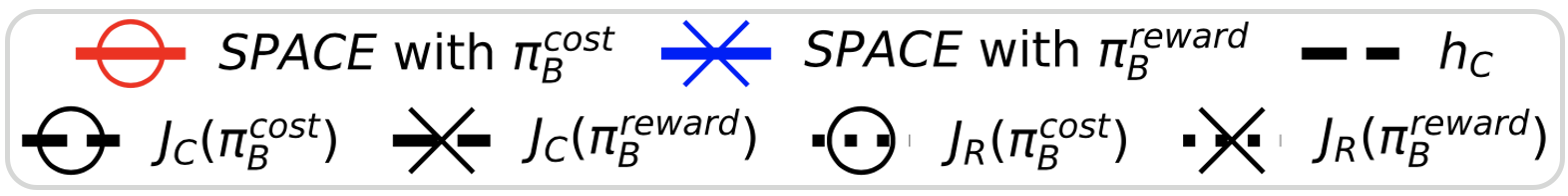}
\caption{
Learning from \textit{sub-optimal} $\pi_B.$
The undiscounted constraint cost and
the discounted reward
over policy updates for the gather and the circle tasks.
The solid line is the mean and the shaded area is the standard deviation over 5 runs.
The black dashed line is the cost constraint threshold $h_C.$
%
%The dashed lines in the cost constraint plot are $J_C(\pi_B^\mathrm{cost})$ (lower) and $J_C(\pi_B^\mathrm{reward})$ (upper).
%
%The dashed lines in the reward plot are $J_R(\pi_B^\mathrm{cost})$ (lower) and $J_R(\pi_B^\mathrm{reward})$ (upper).
%
%The middle dashed line in the cost constraint plot is the cost constraint threshold $h_C$ of the learning agent. %(Best viewed in color.)
%
We observe that \algname\ satisfies the cost constraints even when learning from the sub-optimal $\pi_B.$
}
\label{fig:safevsaggressivePrior}
%\end{mdframed}
\end{figure*}

\begin{figure*}[t]
\centering
\begin{tabular}[b]{@{}c@{}}%
\textbf{Gather}\\
\includegraphics[width=0.246\linewidth]{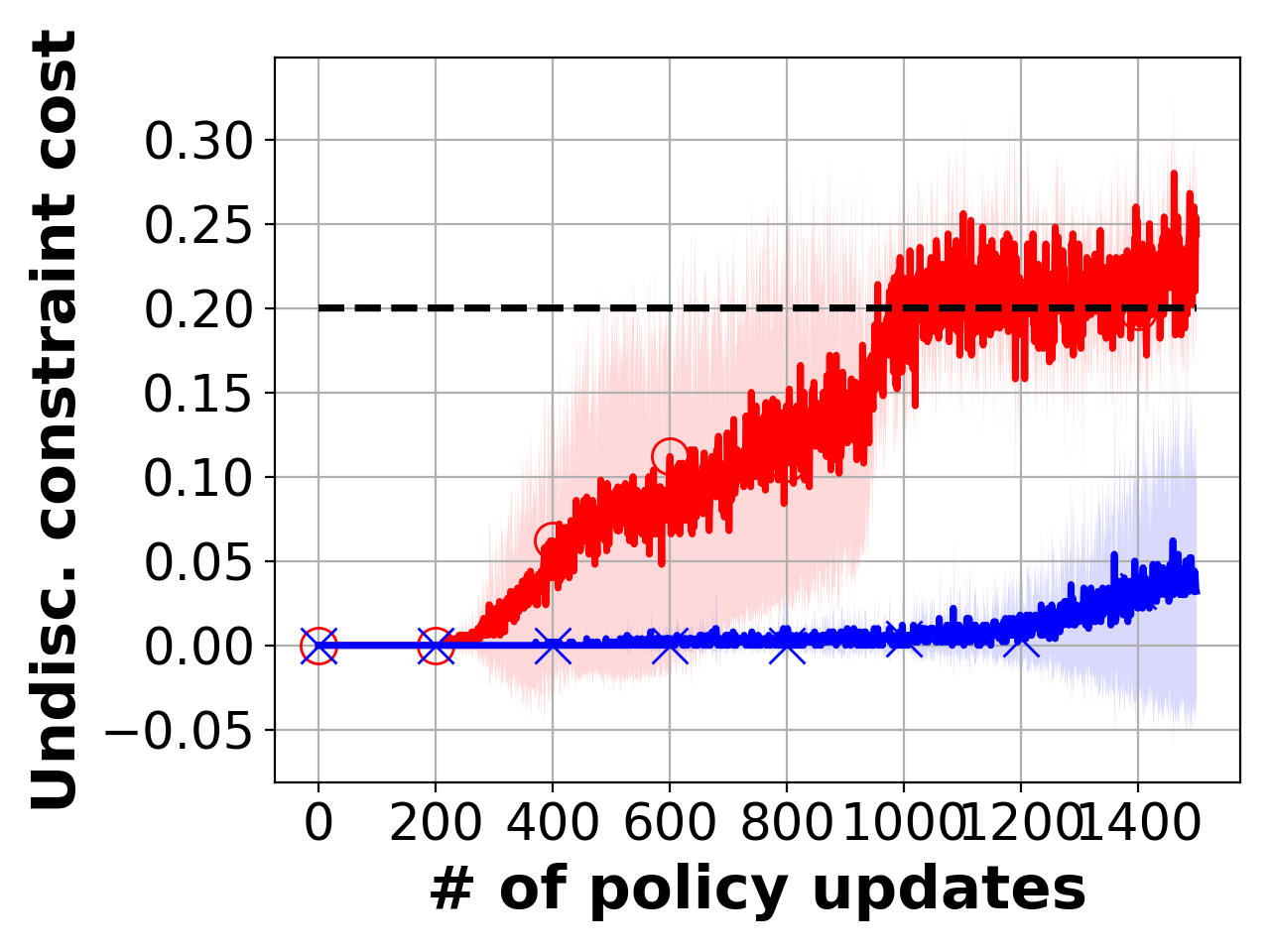}%
\includegraphics[width=0.246\linewidth]{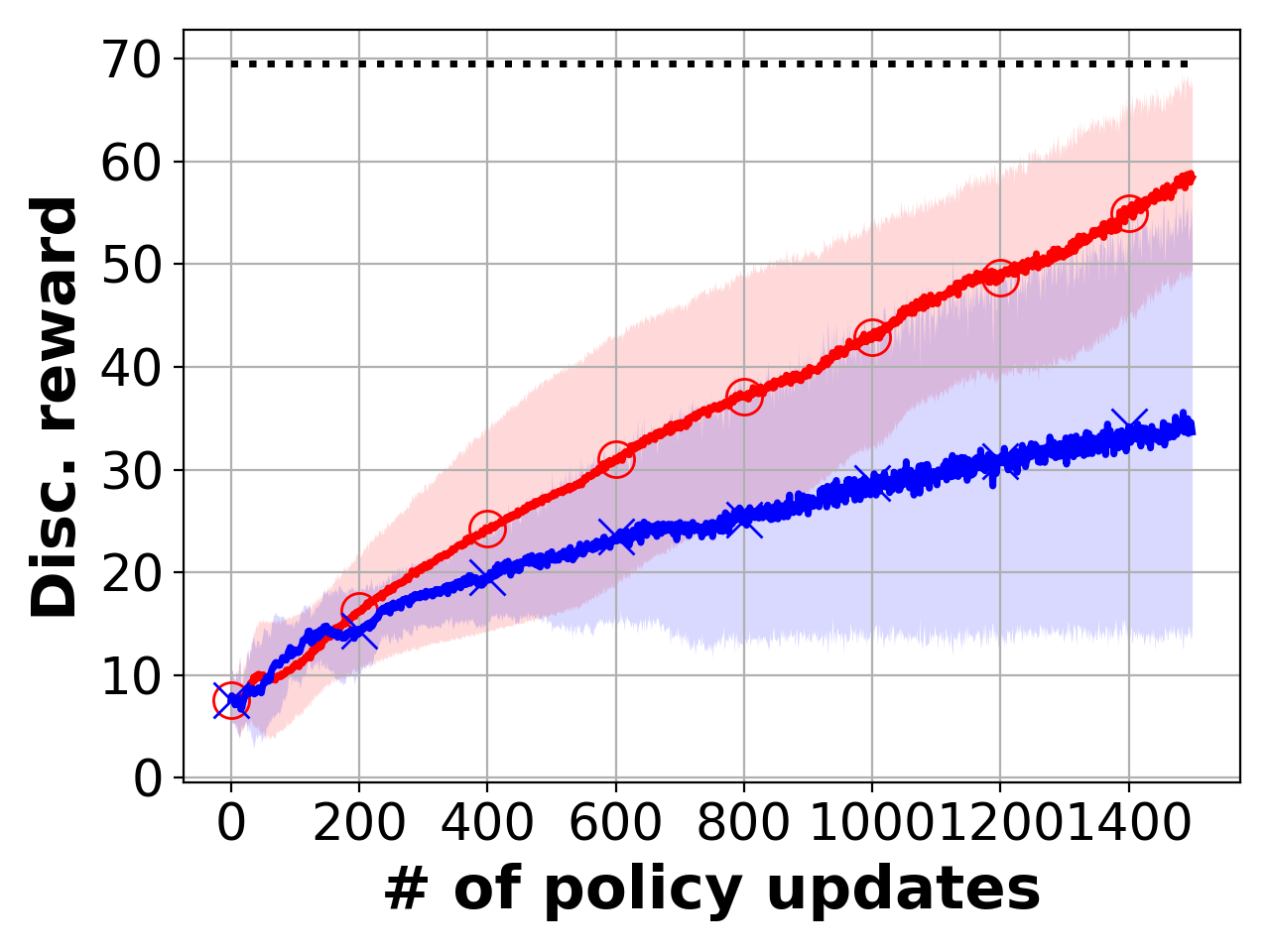}%
\end{tabular}
\begin{tabular}[b]{|@{}c@{}}%
\textbf{Circle}\\
\includegraphics[width=0.246\linewidth]{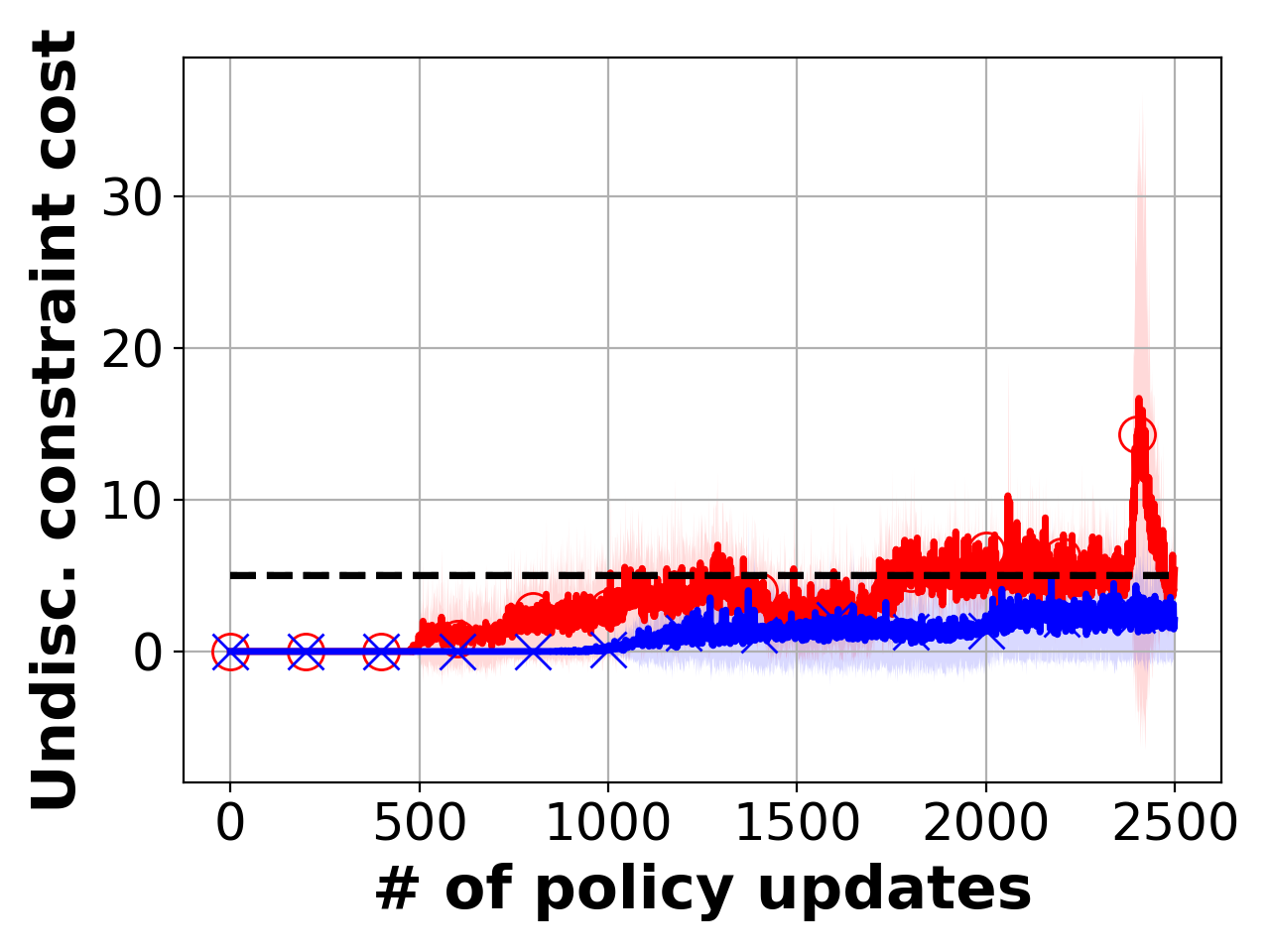}%
\includegraphics[width=0.246\linewidth]{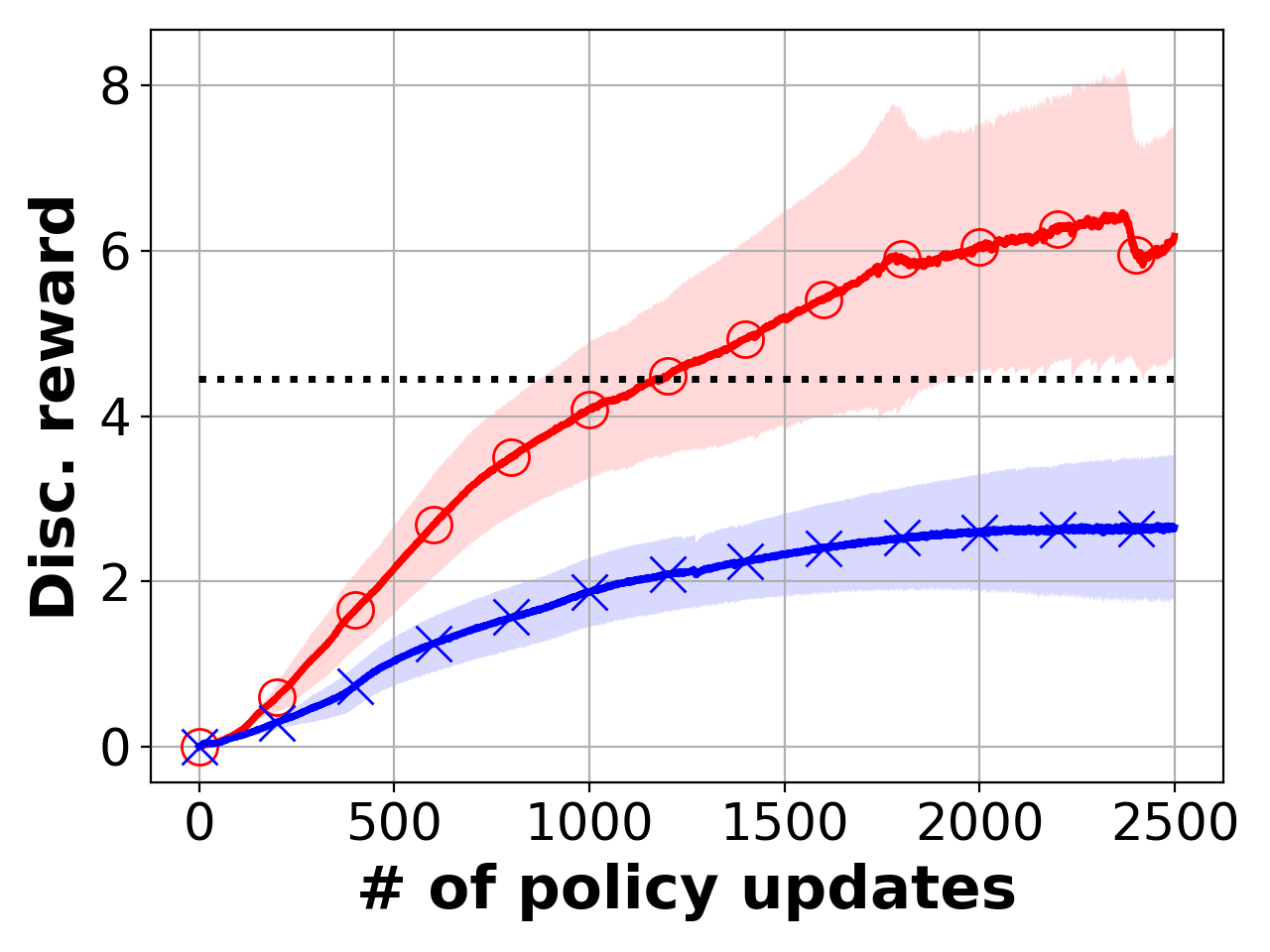}%
\end{tabular}
\includegraphics[width=0.45\linewidth]{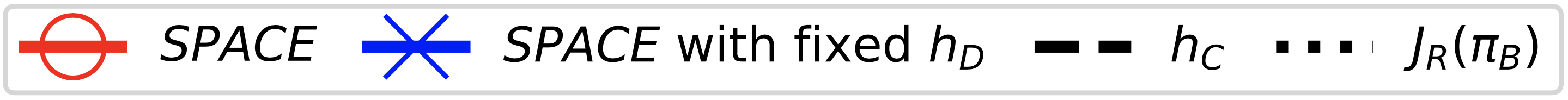}
\caption{
Ablation studies on the fixed $h_D.$ 
The undiscounted constraint cost and
the discounted reward
over policy updates for the gather and the circle tasks.
The solid line is the mean and the shaded area is the standard deviation over 5 runs.
The black dashed line is the cost constraint threshold $h_C.$
%
%The dashed lines in the cost constraint plot are $J_C(\pi_B^\mathrm{cost})$ (lower) and $J_C(\pi_B^\mathrm{reward})$ (upper).
%
%The dashed lines in the reward plot are $J_R(\pi_B^\mathrm{cost})$ (lower) and $J_R(\pi_B^\mathrm{reward})$ (upper).
%
%The middle dashed line in the cost constraint plot is the cost constraint threshold $h_C$ of the learning agent. %(Best viewed in color.)
%
We observe that the update rule is critical for ensuring the learning performance improvement.
}
\label{fig:spacewithfixedhd}
%\end{mdframed}
\end{figure*}

\subsection{Results}
\label{subsectoin:results}
\parab{Overall performance.} 
The learning curves of 
the discounted reward ($J_R(\pi)$), 
the undiscounted constraint cost ($J_C(\pi)$), and
the undiscounted divergence cost ($J_D(\pi)$)
over policy updates are shown for all tested algorithms and tasks in Fig.~\ref{fig:overall_performance}. 
We use ${\pi_{B}^{\mathrm{near}}}$ in bottleneck and grid tasks, and $\pi_{B}^{\mathrm{human}}$ in car-racing task.
Note that $\pi_{B}^{\mathrm{human}}$ from human demonstration is \textit{highly sub-optimal} to the agent (\ie $J_R(\pi_{B}^{\mathrm{human}})$ is small).
The value of the reward is only around 5 as shown in the plot.
It does not solve the task at hand.
%
%( we also test on using $\pi_{B}^{\mathrm{reward}}$ and $\pi_{B}^{\mathrm{cost}}.$
%
%we include them in Appendix).
%
%Overall, we observe that \algname\ consistently improves the reward while effectively leveraging the baseline policy without violating the cost constraint.
%
Overall, we observe that \textbf{(1)} \algname\ achieves at least 2 times faster cost constraint satisfaction in all cases even learning from $\pi_{B}^{\mathrm{human}}.$
%
%In addition, \algname\ achieves better cost satisfaction than PCPO during training.
\textbf{(2)} \algname\ achieves at least 10\% more reward in the bottleneck and car-racing tasks compared to the best baseline,
and \textbf{(3)} \algname\ is the only algorithm that satisfies the cost constraints in all cases.
In contrast, even if f(d)-CPO and f(d)-PCPO (similar to behavior cloning) are provided with good baseline policies $\pi^\mathrm{near}_B,$ they do not learn efficiently due to the conflicting reward and cost objectives.
In addition, PCPO are less sample-efficient, which shows the accelerated learning of \algname.
%
%This is because that \algname\ can dynamically control the distance between the learned policy and the baseline policy depending on the reward and cost performance.
%

For example, in the car-racing task we observe that $J_D(\pi)$ in \algname\ decreases at the initial iteration, but increases in the end.
This implies that the learned policy is guided by the baseline policy $\pi_B^\mathrm{human}$ in the beginning, but use less supervision in the end.
%
%In addition, in the gird and the bottleneck tasks we observe that the value of the divergence cost of \algname\ converges to a smaller value.
%
%This implies that staying close to the baseline policy does not hinder the performance in these tasks.
%
In addition, in the grid task we observe that the final reward of \algname\ is lower than the baseline algorithm.
This is because that \algname\ converges to a policy in the cost constraint set, whereas the baseline algorithms do not find constraint-satisfying policies.
Furthermore, we observe that $J_D(\pi)$ in the traffic tasks decreases throughout the training.
This implies that \algname\ intelligently adjusts $h_D^k$ w.r.t. the performance of $\pi_B$ to achieve safe learning.

%These observations show that \algname\ can robustly ensure constraint satisfaction while achieving fast learning aided by the potentially sub-optimal baseline policy.
%
\parab{f-CPO and f-PCPO.} f-CPO and f-PCPO fail to improve the reward and have more cost violations. Most likely this is due
%This is because 
to persistent supervision from the baseline policies which need not satisfy the cost constraints nor have high reward.
For example, in the car-racing task we observe that the value of the divergence cost decreases throughout the training.
This implies that the learned policy overly evolves to the sub-optimal $\pi_B$ and hence degrades the reward performance. 
%with a safe baseline policy $\pi_B^\mathrm{cost}$ and an aggressive baseline policy $\pi_B^\mathrm{reward}$ in the

%
\parab{d-CPO and d-PCPO.} d-CPO and d-PCPO improve the reward slowly and have more cost violations.
%This is because that 
They do not use projection to quickly learn from $\pi_B$.
For example, in the car-racing task $J_D({\pi})$ in d-CPO and d-PCPO are high compared to \algname\ throughout the training.
This suggests that simply regularizing the RL objective with the faded weight is susceptible to a sub-optimal $\pi_B.$
In contrast to this heuristic, we use Lemma \ref{theorem:h_D} to update $h_D$ when needed, allowing $\pi_B$ to influence the learning of the agent at any iterations depending on the learning progress of the agent.
%(\eg some $\pi_B$ are useful and we should not move away from them as in the traffic tasks).
%

%
Importantly, in our setup the agent does not have any prior knowledge about $\pi_B.$
The agent has to stay close to $\pi_B$ to verify its reward and cost performance.
It is true that $\pi_B$ may be constraint-violating, but it may also provide a useful signal for maximizing the reward.
For example, in the grid task (Fig. \ref{fig:overall_performance}), although $\pi_B$ does not satisfy the cost constraint, it still helps the \algname\ agent (by being close to $\pi_B$) to achieve faster cost satisfaction.
Having demonstrated the overall effectiveness of \algname, our remaining experiments explore \textbf{(1)} \algname's ability to safely learn from sub-optimal polices, and \textbf{(2)} the importance of the update method in Lemma \ref{theorem:h_D}.
For compactness, we restrict our consideration on \algname\ and the Mujoco tasks, which are widely used in RL community.

\iffalse
\begin{wrapfigure}{R}{0.55\textwidth}
\vspace{-5mm}
\centering
%
\begin{tabular}[b]{@{}c@{}}%
\includegraphics[width=0.505\linewidth]{figure/exp_2/NumCost_pg_overallPerformance_v2.png}%
\includegraphics[width=0.505\linewidth]{figure/exp_2/Reward_pg_overallPerformance_v2.png}%
\end{tabular}

\includegraphics[width=0.7\linewidth]{figure/exp_2/legend_2.png}
%
\vspace{-1mm}
\caption{The values of the discounted reward and the undiscounted constraint value over policy updates with a safe baseline policy $\pi_B^\mathrm{cost}$ and an aggressive baseline policy $\pi_B^\mathrm{reward}$ in the point gather task.
%
%
The solid line is the mean and the shaded area is the standard deviation over five runs. 
%
The dashed lines in the reward plot are the rewards of $\pi_B^\mathrm{cost}$ (lower) and $\pi_B^\mathrm{reward}$ (upper).
%
The dashed lines in the cost constraint plot are the cost constraint value of the $\pi_B^\mathrm{cost}$ (lower) and $\pi_B^\mathrm{reward}$ (upper).
%
The middle dashed line in the cost constraint plot is the cost constraint threshold $h_C$ of the learning agent. %(Best viewed in color.)
}
\label{fig:safevsaggressivePrior}
\vspace{-1mm}
%\end{mdframed}
\end{wrapfigure}
\fi

\parab{Sub-optimal $\pi_B^\mathrm{cost}$ and $\pi_B^\mathrm{reward}$.}
Next, we test whether \algname\ can learn from sub-optimal $\pi_B.$
%To show that \algname\ can safely leverage the baseline policy with different cost constraints, we experiment with two baseline policies: $\pi_B^\mathrm{cost}$ and $\pi_B^\mathrm{reward}$. 
%
The learning curves of 
$J_C(\pi)$ and
$J_R(\pi)$
over policy updates are shown for the gather and circle tasks in Fig.~\ref{fig:safevsaggressivePrior}. 
We use two \textit{sub-optimal} $\pi_B$: $\pi_B^\mathrm{cost}$ and $\pi_B^\mathrm{reward},$ and
learning agent's $h_C$ is set to $0.5$
(\ie $\pi_B$ do not solve the task at hand).
We observe that \algname\ robustly satisfies the cost constraints in all cases even when learning from $\pi_B^\mathrm{reward}.$
%
%Please read the figure caption for more detail. 
%shows the learning curves of the discounted reward and the undiscounted cost constraint value in the point gather task with $h_C=5$ for the agent.
%
In addition, we observe that learning guided by $\pi_B^\mathrm{reward}$ achieves faster reward learning efficiency at the \textit{initial} iteration.
This is because $J_R(\pi_B^\mathrm{reward}) > J_R(\pi_B^\mathrm{cost})$ as seen in the reward plot.
Furthermore, we observe that learning guided by $\pi_B^\mathrm{cost}$ achieves faster reward learning efficiency at the \textit{later} iteration.
%
%This is because that in order to simultaneously satisfy being in a region around $\pi_B^\mathrm{reward}$ and in the cost constraint set,
%
%the agent needs extensive back and forth projections onto these sets.
%
This is because by starting in the interior of the cost constraint set (\ie $J_C(\pi_B^\mathrm{cost})\approx0\leq h_C$), the agent can safely exploit the baseline policy.
%
%These observations demonstrate that dynamically updating the distance between policies in \algname\ is important to safely learn from the suboptimal baseline policy.
%
The results show \algname\ enables fast convergence to a constraint-satisfying policy, even if $\pi_B$ does not meet the constraint or does not optimize the reward.

\parab{\algname\ with fixed $h_D.$}
In our final experiments, we investigate the importance of updating $h_D$ when learning from a sub-optimal $\pi_B.$
The learning curves of the $J_C(\pi)$ and $J_R(\pi)$ over policy updates are shown for the gather and circle tasks in Fig. \ref{fig:spacewithfixedhd}.
We observe that \algname\ with fixed $h_D$ converges to less reward.
For example, in the circle task \algname\ with the dynamic $h_D$ achieves 2.3 times more reward.
This shows that $\pi_B$ in this task is highly sub-optimal to the agent and the need of using stateful $h_D^k.$

\begin{figure}[t]
%\vspace{-3mm}
\centering
\includegraphics[width=0.52\linewidth]{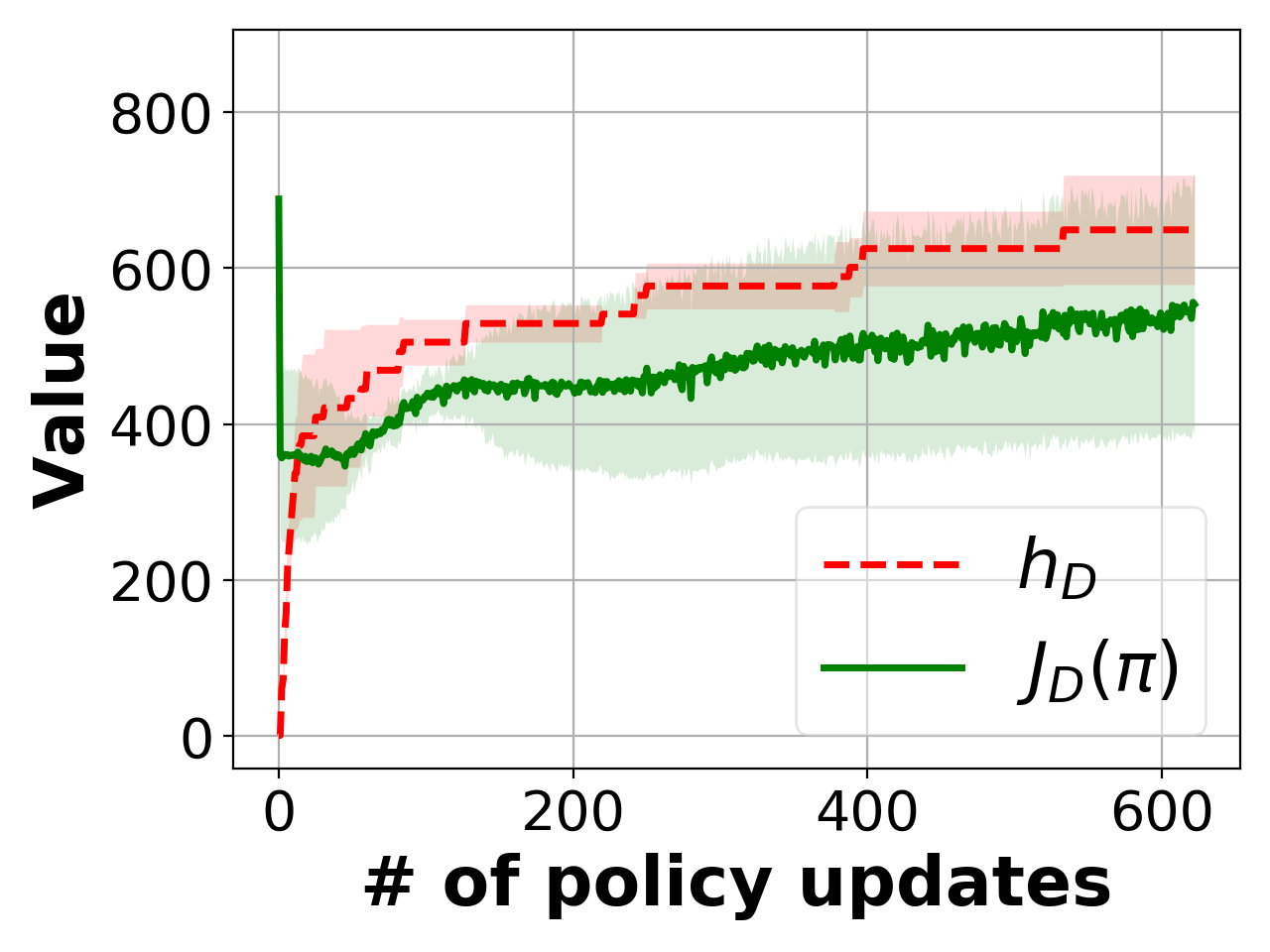}
\caption{The divergence cost $J_D(\pi)$ and the value of $h_D$ over the iterations in the car-racing task.
We see that \algname\ controls $h_D$ to ensure divergence constraint satisfaction.
}
\label{fig:jdhd}
\end{figure}

Moreover, Fig. \ref{fig:jdhd} shows the divergence cost $J_D(\pi)$ and the value of $h_D$ over the iterations in the car-racing task.
We observe that \algname\ gradually increases $h_D$ to improve reward and cost satisfaction performance.

\iffalse 
%
The value of $h_D$ depends on some nontrivial computations. 
%
Numerically, we find that this is hard to compute, so we incrementally increase the value of $h_D$ by setting a constant, \ie $(J_C(\pi^k)-h_C)^2\big)+h_D^k.$
%
Empirically, we found this constant does not affect the performance significantly as shown in Appendix \ref{additional_Experiment} (Fig. \ref{fig:appendix_initialPriorConstraintThreshold}).
%
We leave developing a theorem for this in the future work.
\fi

%
%Although learning guided by $\pi_B^\mathrm{cost}$ improves the reward slowly at first, it actually achieves higher reward in the end.
%
%We leave this as improvement of future work to develop a more robust approach.

%% file: conclusion.tex
% !TEX root = neurips_2019.tex
\vspace{-1mm}
\section{Conclusion}
\label{sec:conclusion}
In this work, we addressed the problem of learning constraint-satisfying policies given potentially sub-optimal baseline policies. 
We explicitly analyzed how to safely learn from the baseline policy, and hence proposed an iterative policy optimization algorithm that alternates between maximizing expected return on the task, minimizing distance to the baseline policy, and projecting the policy onto the constraint-satisfying set.
%
% We provided an effective approach to controlling the distance between the learned and baseline policies given the learning progress of the agent.
%
Our algorithm efficiently learns from a baseline policy as well as human provided demonstration data and achieves superior reward and cost performance compared with state-of-the-art approaches (\ie PCPO).
%
% We further demonstrated that our algorithm is able to safely learn from real human demonstration data. 
%
% These results provide an insight for designing practical RL algorithms that can be deployed in real applications such as personalized robot assistants.

%Although we use the projection to ensure safety in \citet{yang2020projection}, the analysis in Lemma \ref{theorem:h_D} and the particular combination of these components into a safe RL algorithm are, to our knowledge, novel, and the results provide an insight for designing practical RL algorithms in real applications such as robotic assistants.
%

%
No algorithm is without limitations. Future work could improve \algname\ in several ways.
For instance, in Lemma \ref{theorem:h_D}, we do not guarantee that \algname\ will increase $h_D$ enough for the region around the baseline policy to contain the \textit{optimal} policy.
This is challenging since the optimization problem is non-convex.
One possible solution is to rerun \algname\ multiple times and reinitialize $\pi_B$ with the previous learned policy each time. 
One evidence to support this method is that in the bottleneck task (Fig. \ref{fig:overall_performance}), the agent trained with \algname\ outperforms PCPO agent by achieving higher rewards and faster constraint satisfaction.
The PCPO agent here can be seen as the \algname\ agent trained without $\pi_B$. 
And then we train the \algname\ agent with $\pi_B$ from the learned PCPO agent.
This shows that based on the learned policy, we can use \algname\ to improve performance.
%\parab{Multiple $\pi_B$.} One possible idea for learning from multiple $\pi_B$ is to compute the distance to each $\pi_B$. 
%
%Then, select the one with the minimum distance to do the update. 
%
%This ensures that the update for the reward in the first step is less affected by $\pi_B.$ And the analysis we did can be extended.
% In addition, we show that the benefits of step 2 of \algname\ can robustly learn from the potentially sub-optimal $\pi_B.$
%
In addition, it would be interesting to explore using other types of baseline policies such as rule-based policies and see how they impact the learning dynamics of \algname .
%
%This paper provides a first attempt into safe RL with baseline policies.

%% file: appendix_experiment.tex
% !TEX root = neurips_2019.tex
\begin{center}
\Large
\textbf{Supplementary Material for Accelerating Safe Reinforcement Learning \\ with Constraint-mismatched Policies}
\end{center}

\paragraph{Outline.} 
Supplementary material is outlined as follows. 
%
%Section \ref{appendix:prior_bound} provides lower and upper bounds on reward and cost performance when learning from the baseline policy.
%
Section \ref{sec:impact} discusses the impact of the proposed algorithm.
Section \ref{appendix:sec:theorem:h_D} details the proof of updating $h_D$ in Lemma \ref{theorem:h_D}.
Section \ref{appendix:proof_update_rule_1} describes the proof of analytical solution to \algname\ in Eq. (\ref{eq:P2CPO_final}).
Section \ref{appendix:sec:converge} gives the proof of finite-time guarantee of \algname\ in Theorem \ref{theorem:P2CPO_converge} and discuss the difference between the KL-divergence and 2-norm projections.
Section \ref{sec:appendix_experiment} assembles the additional experiment results to provide a detailed examination of the proposed algorithm compared to the baselines. These include:\\
\begin{itemize}
    \item evaluation of the discounted reward versus the cumulative undiscounted constraint cost to demonstrate that \algname\ achieves better reward performance with fewer cost constraint violations,
    \item evaluation of performance of \algname\ guided by baseline policies with different $J_C(\pi_B)$ to demonstrate that \algname\ safely learns from the baseline policies which need not satisfy the cost constraint,
    \item ablation studies of using a fixed $h_D$ in \algname\ to demonstrate the importance of using the dynamic $h_D$ to improve the reward and cost performance,
    \item comparison of \algname\ and other annealing approaches to demonstrate that \algname\ exploits the baseline policy effectively,
    \item comparison of \algname\ under the KL-divergence and the 2-norm projections to demonstrate that they converge to different stationary points,
    \item evaluation of using different initial values of $h^0_D$ to demonstrate that the selection of the initial value does not affect the performance of \algname\ drastically.
\end{itemize}{}
Section \ref{sec:appendix_experiment} also details the environment parameters, the architectures of policies, computational cost, infrastructure for computation and the instructions for executing the code.
Section \ref{appendix:human_policy} provides a procedure for getting a baseline human policy.
Finally, we fill the Machine Learning Reproducibility Checklist in Section \ref{appendix:sec:reproduce}.

\section{Impact of \algname}%(it is like abstract)
\label{sec:impact}
Many autonomous systems such as self-driving cars and autonomous robots are complex.
In order to deal with this complexity,
researchers are increasingly using reinforcement learning in conjunction with imitation learning for designing control policies.
%Leveraging the baseline policy from the previous application that has a similar configuration to the new application prevents from learning the policy from scratch.
%This allows us to deploy reinforcement learning systems on a large scale.
The more we can learn from a previous policy (\eg human demonstration, previous applications), the fewer resources (\eg time, energy, engineering effort, cost) we need to learn a new policy.
The proposed algorithm could be applied in many fields where learning a policy can take advantage of prior applications while providing assurances for the consideration of fairness, safety, or other costs.
For example, in a dialogue system where an agent is intended to converse with a human, the agent should safely learn from human preferences while avoiding producing biased or offensive responses.
In addition, in the self-driving car domain where an agent learns a driving policy, the agent should safely learn from human drivers while avoiding a crash. 
Moreover, in the personalized robotic assistant setting where an agent learns from human demonstration, the agent should carefully imitate humans without damaging itself or causing harm to nearby humans.
These examples highlight the potential impact of the proposed algorithm for accelerating safe reinforcement learning by adapting prior knowledge.
This can open the door to advances in lifelong learning and adaptation of agents to different contexts. 
%
%and dynamic environments, while subject to safety limitations.

%
One deficiency of the proposed algorithm is that the agent still experiments with cost constraint violation when learning control policies.
This is because that any learning-based system needs to experiment with various actions to find a constraint-satisfying policy.
Even though the agent does not violate the safety constraints during the learning phase, any change or perturbation of the environment that was not envisioned at the time of programming or training may lead to a catastrophic failure during run-time.
These systems cannot guarantee that sensor inputs will not induce undesirable consequences, nor can the systems adapt and support safety in situations in which new objectives are created. 
This creates huge concerns in safety-critical applications such as self-driving vehicles and personalized chatbot system.
% 
%Nonetheless, being able to apply knowledge to new situations and learn fast, subject to safety constraints, can support assurance in the operation of the system when the system and environment evolve.
%

%
This raises several questions:
What human-agent communication is needed to bring humans in the loop to increase safety guarantees for the autonomous system? 
How can trust and safety constraints be incorporated into the planning and control processes?
How can one effectively identify unsafe plans of the baseline policy?
We believe this paper will encourage future work to develop rigorous design and analysis tools for continual safety assurance in conjunction with using baseline policies from previous applications. 

\section{Proof of Updating $h_D$ in Lemma \ref{theorem:h_D}}
\label{appendix:sec:theorem:h_D}

\begin{proof}
Based on Theorem 1 in \cite{achiam2017constrained}, for any two policies $\pi$ and $\pi'$ we have
\begin{align}
    &J_{C}(\pi') - J_{C}(\pi)\geq \frac{1}{1-\gamma}\E_{\substack{s\sim d^{\pi}\\ a\sim \pi'}}\Big[A^{\pi}_{C}(s,a)-\frac{2\gamma\epsilon^{\pi'}_{C}}{1-\gamma}\sqrt{\frac{1}{2}\KL(\pi'(s)||\pi(s))}\Big]\nonumber \\ 
    \Rightarrow\quad&
    \frac{2\gamma\epsilon^{\pi'}_C}{(1-\gamma)^2}\E_{\substack{s\sim d^{\pi}}}\Big[\sqrt{\frac{1}{2}\KL(\pi'(s)||\pi(s))}\Big]\geq -J_{C}(\pi')+J_{C}(\pi)+\frac{1}{1-\gamma}\E_{\substack{s\sim d^{\pi}\\ a\sim \pi'}}\Big[A^{\pi}_{C}(s,a)\Big]\nonumber \\ 
    \Rightarrow\quad&\frac{2\gamma\epsilon^{\pi'}_C}{(1-\gamma)^2}\E_{\substack{s\sim d^{\pi}}}\Big[\sqrt{\frac{1}{2}\KL(\pi'(s)||\pi(s))}\Big]\geq -J_{C}(\pi')+J_{C}(\pi)\nonumber \\ 
    \Rightarrow\quad&\frac{\sqrt{2}\gamma\epsilon^{\pi'}_C}{(1-\gamma)^2}\sqrt{\E_{\substack{s\sim d^{\pi}}}\Big[\KL(\pi'(s)||\pi(s))\Big]}\geq -J_{C}(\pi')+J_{C}(\pi)\nonumber \\ 
    \Rightarrow\quad&\E_{\substack{s\sim d^{\pi}}}\Big[\KL(\pi'(s)||\pi(s))\Big]\geq \frac{(1-\gamma)^4(-J_{C}(\pi')+J_{C}(\pi))^2}{2\gamma^2{\epsilon^{\pi'}_C}^2}.
    \label{eq:appendix_B_1}
\end{align}
The fourth inequality follows from Jensen's inequality. 
We then define $\varphi(\pi(s))\doteq\sum_i \pi(a(i)|s)\log \pi(a(i)|s).$
By Three-point Lemma~\citep{chen1993convergence}, for any three policies $\pi, \pi',$ and $\hat{\pi}$ we have
\begin{align}
    \E_{\substack{s\sim d^{\pi}}}\Big[\KL(\pi'(s)||\hat{\pi}(s))\Big]= \E_{\substack{s\sim d^{\pi}}}\Big[\KL(\pi'(s)||{\pi(s)})\Big]+ \E_{\substack{s\sim d^{\pi}}}\Big[\KL(\pi(s)||\hat{\pi}(s))\Big]\nonumber \\ 
    -\E_{s\sim d^{\pi}}\Big[(\nabla\varphi(\hat{\pi}(s))-\nabla\varphi(\pi(s)))^T(\pi'(s)-\pi(s))\Big].
    \label{eq:appendix_B_2}
\end{align}

\begin{figure*}[t]
\vspace{-3mm}
\centering
\subfloat[]{
\includegraphics[width=0.36\linewidth]{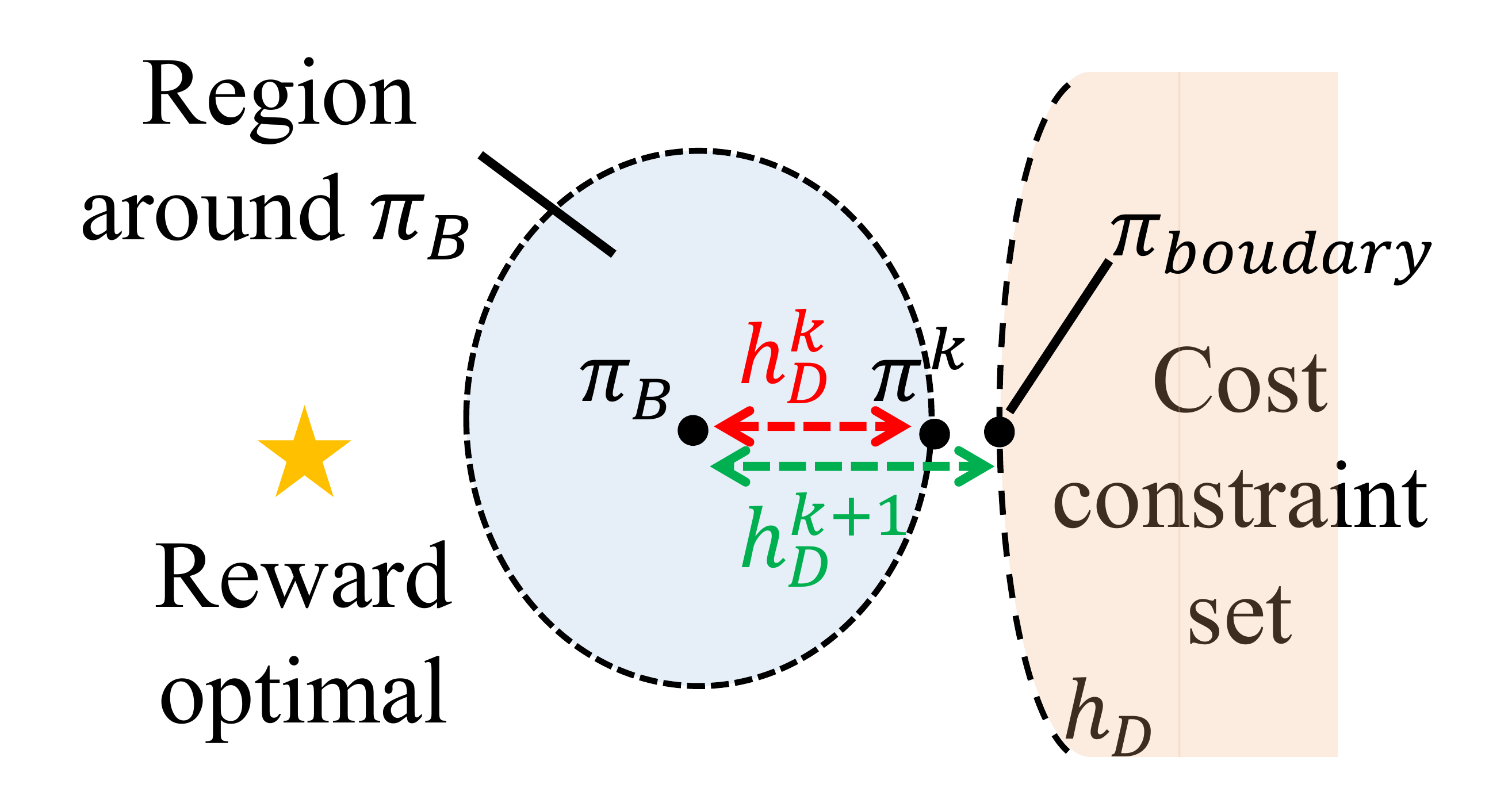}}
\subfloat[]{
\includegraphics[width=0.33\linewidth]{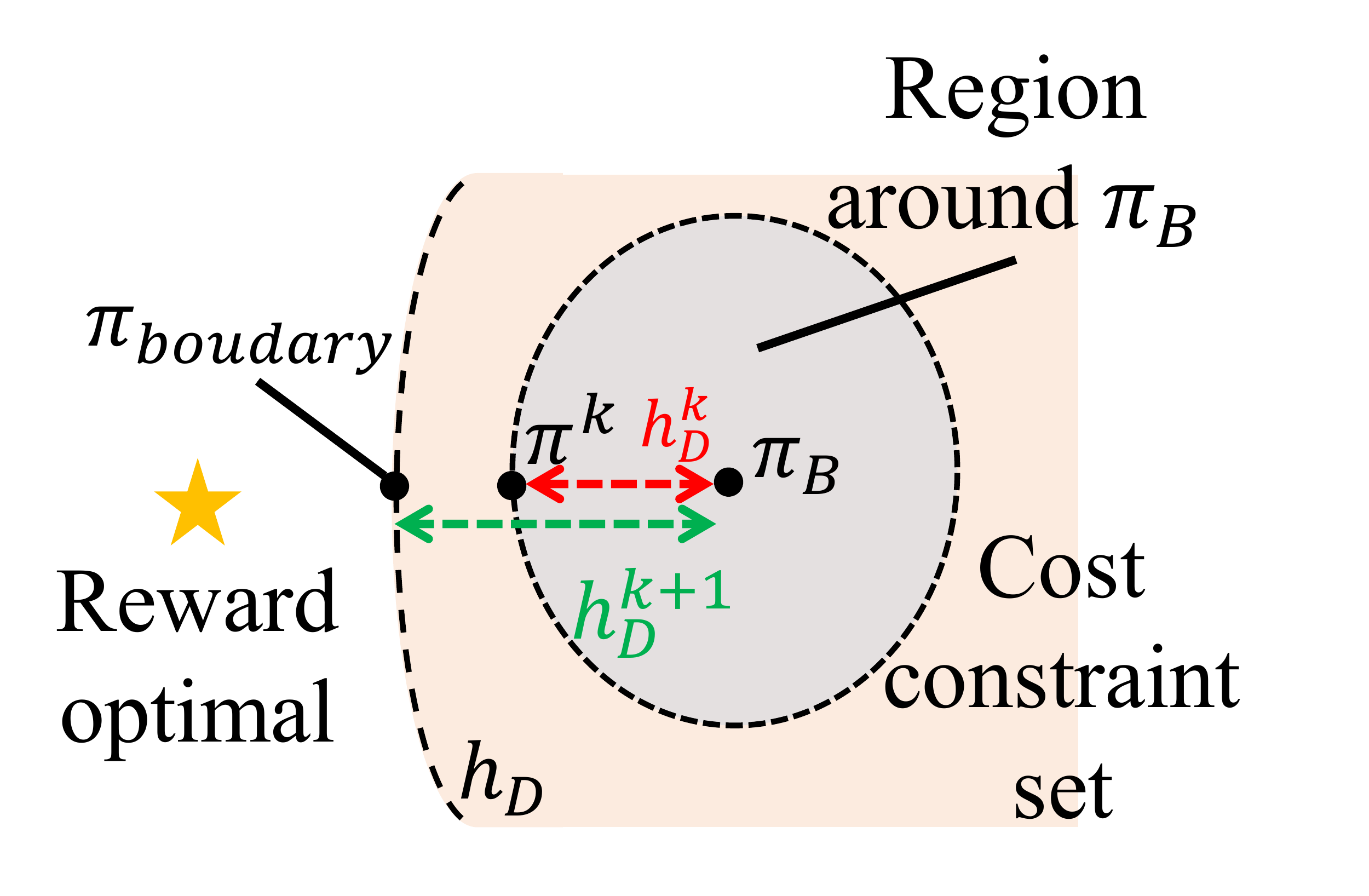}}
\caption{
\textbf{(a)} Illustrating when $\pi_B$ is \textit{outside} the cost constraint set. 
%
%$\pi_{boundary}$ is the policy with $J_C(\pi_{boundary})=h_C.$ 
%
%We aim to bound $h_D^{k+1}$ (\ie the KL-divergence between $\pi_{boundary}$ and $\pi_B$) by using the threshold $h_D^k.$
%
\textbf{(b)} Illustrating when $\pi_B$ is \textit{inside} the cost constraint set. 
$\pi_{boundary}$ is the policy with $J_C(\pi_{boundary})=h_C.$ 
We aim to bound $h_D^{k+1}$ (\ie the KL-divergence between $\pi_{boundary}$ and $\pi_B$) by using $h_D^k.$
}
\vspace{-3mm}
\label{fig:p2cpo_appendix}
\vspace{-3mm}
%\end{mdframed}
\end{figure*}

Let $\pi_{boundary}$ denote a policy satisfying~$J_C(\pi_{boundary})=h_C$ (\ie $\pi_{boundary}$ is in the boundary of the set of the policies which satisfy the cost constraint $J_C(\pi)\leq h_C$).
Let $\pi'=\pi_{boundary}, \hat{\pi}=\pi_B$ and $\pi=\pi^{k}$ in Eq. (\ref{eq:appendix_B_1}) and Eq. (\ref{eq:appendix_B_2}) (this is illustrated in Fig. \ref{fig:p2cpo_appendix}).
Then we have 
\begin{align}
    &~\E_{\substack{s\sim d^{\pi^k}}}\Big[\KL(\pi_{boundary}(s)||\pi_B(s))\Big]-\E_{\substack{s\sim d^{\pi^k}}}\Big[\KL(\pi^{k}(s)||\pi_B(s))\Big]
    \nonumber \\ = &~
     \E_{\substack{s\sim d^{\pi^k}}}\Big[\KL(\pi_{boundary}(s)||{\pi^k(s)})\Big]\nonumber\\
     &\quad\quad\quad-\E_{s\sim d^{\pi^k}}\Big[(\nabla\varphi(\pi_B(s))-\nabla\varphi(\pi^k(s)))^T(\pi_{boundary}(s)-\pi^{k}(s))\Big]\nonumber \\ 
     \geq &~ \frac{(1-\gamma)^4(-J_{C}(\pi_{boundary})+J_{C}(\pi^{k}))^2}{2\gamma^2{\epsilon^{\pi'}_C}^2}\nonumber\\
     &\quad\quad\quad-\E_{s\sim d^{\pi^k}}\Big[(\nabla\varphi(\pi_B(s))-\nabla\varphi(\pi^k(s)))^T(\pi_{boundary}(s)-\pi^{k}(s))\Big]
     \nonumber \\ = &~ \frac{(1-\gamma)^4(-h_C+J_{C}(\pi^k))^2}{2\gamma^2{\epsilon^{\pi'}_C}^2}
     \nonumber\\
     &\quad\quad\quad-\E_{s\sim d^{\pi^k}}\Big[(\nabla\varphi(\pi_B(s))-\nabla\varphi(\pi^k(s)))^T(\pi_{boundary}(s)-\pi^{k}(s))\Big]\nonumber \\ =&~\mathcal{O}\Big(\big(-h_C + J_{C}(\pi^k)\big)^2\Big),\label{eq:bound_for_hc_appendix}
\end{align}
where $J_C(\pi_{boundary})=h_C.$

For the first case in Fig. \ref{fig:p2cpo_appendix}(a), we would like to have $\mathcal{U}_1\cap \mathcal{U}_2^{k+1}\neq \emptyset$~(feasibility). 
For the second case in Fig. \ref{fig:p2cpo_appendix}(b), we would like to have $\mathcal{U}_2^{k+1}\cap\partial\mathcal{U}_1\neq\emptyset$~(exploration).
These implies that the policy in step $k+1$ is $\pi_{boundary}$ which satisfies $\mathcal{U}_1\cap \mathcal{U}_2^{k+1}\neq \emptyset$ and $\mathcal{U}_2^{k+1}\cap\partial\mathcal{U}_1\neq\emptyset.$
Now let $h_D^{k+1}\doteq \E_{\substack{s\sim d^{\pi^k}}}\Big[\KL(\pi_{boundary}(s)||\pi_B(s))\Big]$ and 
$h_D^k \doteq \E_{\substack{s\sim d^{\pi^k}}}[\KL(\pi^k(s)||\pi_B(s))].$ 
Then Eq. \ref{eq:bound_for_hc_appendix} implies 
\begin{align}
   h^{k+1}_D \geq \mathcal{O}\Big(\big(-h_C + J_{C}(\pi^k)\big)^2\Big) + h_D^k. \nonumber
\end{align}
%Note that $\mathcal{O}(\cdot)$ absorbs the error of bounding the inequality $\E_{\substack{s\sim d^{\pi^k}}}\Big[\KL(\pi^{k}(s)||\pi_B(s))\Big]\leq h^k_D.$ 
\end{proof}
Lemma~\ref{theorem:h_D} theoretically ensures $h_D$ is large enough to guarantee feasibility and exploration of the agent.
Note that we do not provide guarantees for finding an optimal policy.
This requires additional assumptions on the objective function (\eg convexity).

In addition, the goal of this paper is to understand and analyze how to effectively exploit a baseline policy in constrained RL.
Without such an analysis,
we are not confident in deploying SPACE in real applications.
Furthermore, the question of safely using baseline policies has a practical potential.
It is less studied by prior work \citep{achiam2017constrained,chow2019lyapunov,tessler2018reward,yang2020projection}.

\section{Proof of Analytical Solution to \algname\ in Eq. (\ref{eq:P2CPO_final})}
\label{appendix:proof_update_rule_1}

We first approximate the three stages in \algname\ using the following approximations.

\parab{Step 1.} 
Approximating Eq. (\ref{eq:P2CPO_firstStep}) yields
\begin{align}\textstyle
    \vtheta^{k+\frac{1}{3}} = \argmax\limits_{\vtheta}~ {\vg^{k}}^{T}(\vtheta-\vtheta^{k})\quad\text{s.t.}&~\frac{1}{2}(\vtheta-\vtheta^{k})^{T}\mF^{k}(\vtheta-\vtheta^{k})\leq \delta.
    \label{eq:update1}
\end{align}
\parab{Step 2 and Step 3.}
% Second, if the projections are defined in the parameter space, we directly use the 2-norm projection.
%
%On the other hand, if the projections are defined in the probability space, we use the KL-divergence projection, which is approximated through a second order Taylor expansion.
%
Approximating Eq. (\ref{eq:P2CPO_secondStep}) and (\ref{eq:P2CPO_thridStep}), similarly yields
%at $\pi^{k+\frac{1}{3}}$ and $\pi^{k+\frac{2}{3}}$
%And we approximate the constraints in Eq. (\ref{eq:P2CPO_secondStep}) and (\ref{eq:P2CPO_thridStep}) by a first order Taylor expansion.
\begin{align}\textstyle
    \vtheta^{k+\frac{2}{3}}=\argmin\limits_{\vtheta}&~\frac{1}{2}(\vtheta-{\vtheta}^{k+\frac{1}{3}})^{T}\mL(\vtheta-{\vtheta}^{k+\frac{1}{3}}) \quad \text{s.t.}~{\va^{ k}}^{T}(\vtheta-\vtheta^{k})+b^{k}\leq 0, \label{eq:update2}\\
    \vtheta^{k+1} = \argmin\limits_{\vtheta}&~\frac{1}{2}(\vtheta-{\vtheta}^{k+\frac{2}{3}})^{T}\mL(\vtheta-{\vtheta}^{k+\frac{2}{3}})\quad\text{s.t.}~{\vc^{ k}}^{T}(\vtheta-\vtheta^{k})+d^{k}\leq 0,
    \label{eq:update3}
\end{align}
where $\mL=\mI$ for the 2-norm projection and $\mL=\mF^{k}$ for the KL-divergence projection.

\begin{proof}
For the first problem in Eq. (\ref{eq:update1}), since $\mF^k$ is the Fisher Information matrix, it is positive semi-definite.
Hence it is a convex program with quadratic inequality constraints.
If the primal problem has a feasible point,
then Slater’s condition is satisfied and strong duality holds. 
Let $\vtheta^{*}$ and $\lambda^*$ denote the solutions to the primal and dual problems, respectively.
In addition, the primal objective function  is continuously differentiable.
Hence the Karush-Kuhn-Tucker (KKT) conditions are necessary and sufficient for the optimality of $\vtheta^{*}$ and $\lambda^*.$
We now form the Lagrangian:
\[
\mathcal{L}(\vtheta,\lambda)=-{\vg^k}^{T}(\vtheta-\vtheta^{k})+\lambda\Big(\frac{1}{2}(\vtheta-\vtheta^{k})^{T}\mF^k(\vtheta-\vtheta^{k})- \delta\Big).
\]
And we have the following KKT conditions:
\begin{align}
   -\vg^k + \lambda^*\mF^k\vtheta^{*}-\lambda^*\mF^k\vtheta^{k}=0~~~~&~~~\nabla_\vtheta\mathcal{L}(\vtheta^{*},\lambda^{*})=0 \label{KKT_1}\\
   \frac{1}{2}(\vtheta^{*}-\vtheta^{k})^{T}\mF^k(\vtheta^{*}-\vtheta^{k})- \delta=0~~~~&~~~\nabla_\lambda\mathcal{L}(\vtheta^{*},\lambda^{*})=0 \label{KKT_2}\\
    \frac{1}{2}(\vtheta^{*}-\vtheta^{k})^{T}\mF^k(\vtheta^{*}-\vtheta^{k})-\delta\leq0~~~~&~~~\text{primal constraints}\label{KKT_3}\\
   \lambda^*\geq0~~~~&~~~\text{dual constraints}\label{KKT_4}\\
   \lambda^*\Big(\frac{1}{2}(\vtheta^{*}-\vtheta^{k})^{T}\mF^k(\vtheta^{*}-\vtheta^{k})-\delta\Big)=0~~~~&~~~\text{complementary slackness}\label{KKT_5}
\end{align}
By Eq.~(\ref{KKT_1}), we have $\vtheta^{*}=\vtheta^k+\frac{1}{\lambda^*}{\mF^k}^{-1}\vg^k.$ 
And by plugging Eq.~(\ref{KKT_1}) into Eq.~(\ref{KKT_2}), 
we have $\lambda^*=\sqrt{\frac{{\vg^k}^T{\mF^k}^{-1}\vg^k}{2\delta}}.$
Hence we have a solution
\begin{align}
\vtheta^{k+\frac{1}{3}}=\vtheta^{*}=\vtheta^k+\sqrt{\frac{2\delta}{{\vg^k}^T{\mF^k}^{-1}\vg^k}}{\mF^k}^{-1}\vg^k, \label{KKT_First}    
\end{align}
which also satisfies Eq.~(\ref{KKT_3}), Eq.~(\ref{KKT_4}), and Eq.~(\ref{KKT_5}).
For the second problem in Eq. (\ref{eq:update2}), we follow the same procedure for the first problem to form the Lagrangian:
\begin{align}
\mathcal{L}(\vtheta,\lambda)=\frac{1}{2}(\vtheta-{\vtheta}^{k+\frac{1}{3}})^{T}\mL(\vtheta-{\vtheta}^{k+\frac{1}{3}})+\lambda({\va^k}^T(\vtheta-\vtheta^k)+b^k). \nonumber
\end{align}
And we have the following KKT conditions: 
\begin{align}
  \mL\vtheta^*-\mL\vtheta^{k+\frac{1}{3}}+\lambda^*\va^k=0~~~~&~~~\nabla_\vtheta\mathcal{L}(\vtheta^{*},\lambda^{*})=0 \label{KKT_6}\\
   {\va^k}^T(\vtheta^*-\vtheta^k)+b^k=0~~~~&~~~\nabla_\lambda\mathcal{L}(\vtheta^{*},\lambda^{*})=0 \label{KKT_7}\\
    {\va^k}^T(\vtheta^*-\vtheta^k)+b^k\leq0~~~~&~~~\text{primal constraints}\label{KKT_8}\\
   \lambda^*\geq0~~~~&~~~\text{dual constraints}\label{KKT_9}\\
   \lambda^*({\va^k}^T(\vtheta^*-\vtheta^k)+b^k)=0~~~~&~~~\text{complementary slackness}\label{KKT_10}
\end{align}
By Eq.~(\ref{KKT_6}), we have $\vtheta^{*}=\vtheta^{k}+\lambda^*\mL^{-1}\va^k.$ 
And by plugging Eq.~(\ref{KKT_6}) into Eq.~(\ref{KKT_7}) and Eq.~(\ref{KKT_9}), 
we have $\lambda^*=\max(0,\frac{{\va^k}^T(\vtheta^{k+\frac{1}{3}}-\vtheta^{k})+b^k}{\va^k\mL^{-1}\va^k}).$
Hence we have a solution
\begin{align}
\vtheta^{k+\frac{2}{3}}=\vtheta^{*}=\vtheta^{k+\frac{1}{3}}-\max(0,\frac{{\va^k}^T(\vtheta^{k+\frac{1}{3}}-\vtheta^k)+b^k}{{\va^k}^T\mL^{-1}{\va^k}^T})\mL^{-1}\va^k,\label{KKT_Second}
\end{align}
which also satisfies Eq.~(\ref{KKT_8}) and Eq.~(\ref{KKT_10}). 
For the third problem in Eq. (\ref{eq:update3}), instead of doing the projection on $\pi^{k+\frac{2}{3}}$ which is the intermediate policy obtained in the second step, we project the policy $\pi^{k+\frac{1}{3}}$ onto the cost constraint.
This allows us to compute the projection without too much computational cost.
We follow the same procedure for the first and second problems to form the Lagrangian:
\begin{align}
\mathcal{L}(\vtheta,\lambda)=\frac{1}{2}(\vtheta-{\vtheta}^{k+\frac{1}{3}})^{T}\mL(\vtheta-{\vtheta}^{k+\frac{1}{3}})+\lambda({\vc^k}^T(\vtheta-\vtheta^k)+d^k). \nonumber
\end{align}
And we have the following KKT conditions: 
\begin{align}
  \mL\vtheta^*-\mL\vtheta^{k+\frac{1}{3}}+\lambda^*\vc^k=0~~~~&~~~\nabla_\vtheta\mathcal{L}(\vtheta^{*},\lambda^{*})=0 \label{KKT_6_C}\\
   {\vc^k}^T(\vtheta^*-\vtheta^k)+d^k=0~~~~&~~~\nabla_\lambda\mathcal{L}(\vtheta^{*},\lambda^{*})=0 \label{KKT_7_C}\\
    {\vc^k}^T(\vtheta^*-\vtheta^k)+d^k\leq0~~~~&~~~\text{primal constraints}\label{KKT_8_C}\\
   \lambda^*\geq0~~~~&~~~\text{dual constraints}\label{KKT_9_C}\\
   \lambda^*({\vc^k}^T(\vtheta^*-\vtheta^k)+d^k)=0~~~~&~~~\text{complementary slackness}\label{KKT_10_C}
\end{align}
By Eq.~(\ref{KKT_6_C}), we have $\vtheta^{*}=\vtheta^{k}+\lambda^*\mL^{-1}\vc^k.$ 
And by plugging Eq.~(\ref{KKT_6_C}) into Eq.~(\ref{KKT_7_C}) and Eq.~(\ref{KKT_9_C}), 
we have $\lambda^*=\max(0,\frac{{\vc^k}^T(\vtheta^{k+\frac{1}{3}}-\vtheta^{k})+d^k}{\vc^k\mL^{-1}\vc^k}).$
Hence we have a solution
\begin{align}
\vtheta^{k+1}=\vtheta^{*}=\vtheta^{k+\frac{1}{3}}-\max(0,\frac{{\vc^k}^T(\vtheta^{k+\frac{1}{3}}-\vtheta^{k})+d^k}{{\vc^k}^T\mL^{-1}{\vc^k}^T})\mL^{-1}\vc^k.\label{KKT_Thrid}
\end{align}

Hence by combining Eq.~(\ref{KKT_First}), Eq.~(\ref{KKT_Second}) and Eq.~(\ref{KKT_Thrid}), we have 
\begin{align}
\vtheta^{k+1}=\vtheta^{k}+\sqrt{\frac{2\delta}{{\vg^k}^T{\mF^k}^{-1}\vg^k}}{\mF^k}^{-1}\vg^k
-&\max(0,\frac{\sqrt{\frac{2\delta}{{\vg^k}^T{\mF^k}^{-1}\vg^k}}{\va^k}^{T}{\mF^k}^{-1}\vg^k+b^k}{{\va^k}^T\mL^{-1}\va^k})\mL^{-1}\va^k \nonumber \\ -&\max(0,\frac{\sqrt{\frac{2\delta}{{\vg^k}^T{\mF^k}^{-1}\vg^k}}{\vc^k}^{T}{\mF^k}^{-1}\vg^k+d^k}{{\vc^k}^T\mL^{-1}\vc^k})\mL^{-1}\vc^k.\nonumber
\end{align}
\end{proof}

\section{Proof of Finite-Time Guarantee of \algname\ in Theorem \ref{theorem:P2CPO_converge}}
\label{appendix:sec:converge}
We now describe the reason for choosing two variants of $\epsilon$-FOSP under two possible projections.
Let $\eta^k_R$ denote the step size for the reward, $\eta^k_D$ denote the step size for the divergence cost, and $\eta^k_C$ denote the step size for the constraint cost.
Without loss of generality, under the KL-divergence projection, at step $k+1$ \algname\ does
\begin{align}
    \vtheta^{k+1}=\vtheta^{k}+\eta^k_R{\mF^{k}}^{-1}\vg^{k}-\eta^k_D{\mF^{k}}^{-1}\va^{k}-\eta^k_C{\mF^{k}}^{-1}\vc^{k}.\nonumber
\end{align}
Similarly, under the 2-norm projection, at step $k+1$ \algname\ does
\begin{align}
    \vtheta^{k+1}=\vtheta^{k}+\eta^k_R{\mF^{k}}\vg^{k}-\eta^k_D\va^{k}-\eta^k_C\vc^{k}.\nonumber
\end{align}
With this definition, we have the following Lemma.
\begin{lemma}[\textbf{Stationary Points for \algname}]
\label{def_fosp}
 Under the KL-divergence projection, \algname\ converges to a stationary point $\vtheta^*$ satisfying
 \begin{align}
     \eta^*_R\vg^*=\eta^*_D\va^*+\eta^*_C\vc^*.\nonumber
 \end{align}
  Under the 2-norm projection, \algname\ converges to a stationary point $\vtheta^*$ satisfying
 \begin{align}
     \eta^*_R\vg^*={\mF^*}(\eta^*_D\va^*+\eta^*_C\vc^*).\nonumber
 \end{align}
\end{lemma}
\begin{proof}
 Under the KL-divergence projection, by using the definition of a stationary point we have
 \begin{align}
     &\vtheta^{*}=\vtheta^{*}+\eta^*_R{\mF^{*}}^{-1}\vg^{*}-\eta^*_D{\mF^{*}}^{-1}\va^{*}-\eta^*_C{\mF^{*}}^{-1}\vc^{*}\nonumber\\
     \Rightarrow\quad& \eta^*_R{\mF^{*}}^{-1}\vg^{*}=\eta^*_D{\mF^{*}}^{-1}\va^{*}+\eta^*_C{\mF^{*}}^{-1}\vc^{*}\nonumber\\
     \Rightarrow\quad&\eta^*_R\vg^{*}=\eta^*_D\va^{*}+\eta^*_C\vc^{*}.\nonumber
 \end{align}
 Under the 2-norm projection, by using the definition of a stationary point we have
  \begin{align}
     &\vtheta^{*}=\vtheta^{*}+\eta^*_R{\mF^{*}}^{-1}\vg^{*}-\eta^*_D\va^{*}-\eta^*_C\vc^{*}\nonumber\\
     \Rightarrow\quad& \eta^*_R{\mF^{*}}^{-1}\vg^{*}=\eta^*_D\va^{*}+\eta^*_C\vc^{*}\nonumber\\
     \Rightarrow\quad&\eta^*_R\vg^{*}={\mF^{*}}(\eta^*_D\va^{*}+\eta^*_C\vc^{*}).\nonumber
 \end{align}
\end{proof}
Hence Lemma \ref{def_fosp} motivates the need for defining two variants of FOSP.

Before proving Theorem \ref{theorem:P2CPO_converge}, we need the following Lemma. 
Define $\mathcal{P}^{\mL}_{\mathcal{C}}(\vtheta)\doteq\argmin\limits_{\vtheta'\in\mathcal{C}}\|\vtheta-\vtheta'\|^2_\mL=\argmin\limits_{\vtheta'\in\mathcal{C}}{(\vtheta-\vtheta')}^T\mL(\vtheta-\vtheta'),$ and $\mL=\mF^k$ under the KL-divergence projection, and $\mL=\mI$ under the 2-norm projection.
\begin{lemma}[\textbf{Contraction of Projection}~\citep{yang2020projection}]
\label{lemma:projecion_appendix}
For any $\vtheta,$ $\vtheta^{*}=\mathcal{P}^{\mL}_{\mathcal{C}}(\vtheta)$ if and only if ${(\vtheta-\vtheta^*)}^T\mL(\vtheta'-\vtheta^*)\leq0, \forall\vtheta'\in\mathcal{C}.$
\end{lemma}
\begin{proof}

$(\Rightarrow)$ Let $\vtheta^{*}=\mathcal{P}^{\mL}_{\mathcal{C}}(\vtheta)$ for a given $\vtheta \not\in\mathcal{C},$ $\vtheta'\in\mathcal{C}$ be such that $\vtheta'\neq\vtheta^*,$ and $\alpha\in(0,1).$ Then we have
\begin{align}
    \|\vtheta-\vtheta^*\|^2_\mL&\leq\|\vtheta-\big(\vtheta^*+\alpha(\vtheta'-\vtheta^*)\big)\|^2_\mL \nonumber\\
    &=\|\vtheta-\vtheta^*\|^2_\mL + \alpha^2\|\vtheta'-\vtheta^*\|^2_\mL-2\alpha(\vtheta-\vtheta^*)^T\mL(\vtheta'-\vtheta^*) \nonumber\\
    \Rightarrow (\vtheta-\vtheta^*)^T\mL(\vtheta'-\vtheta^*)&\leq \frac{\alpha}{2}\|\vtheta'-\vtheta^*\|^2_\mL. \label{eq:appendix_lemmaD1_0}
\end{align}
Since the right hand side of Eq. (\ref{eq:appendix_lemmaD1_0}) can be made arbitrarily small for a given $\alpha$, we have
\[
(\vtheta-\vtheta^*)^T\mL(\vtheta'-\vtheta^*)\leq0, \forall\theta'\in\mathcal{C}.
\]

$(\Leftarrow)$ Let $\vtheta^*\in\mathcal{C}$ be such that $(\vtheta-\vtheta^*)^T\mL(\vtheta'-\vtheta^*)\leq0, \forall\theta'\in\mathcal{C}.$ We show that $\vtheta^*$ must be the optimal solution. Let $\vtheta'\in\mathcal{C}$ and $\vtheta'\neq\vtheta^*.$ Then we have
\begin{align}
    \|\vtheta-\vtheta'\|^2_\mL - \|\vtheta-\vtheta^*\|^2_\mL &=\|\vtheta-\vtheta^*+\vtheta^*-\vtheta'\|^2_\mL -\|\vtheta-\vtheta^*\|^2_\mL \nonumber\\
    &=\|\vtheta-\vtheta^*\|^2_\mL + \|\vtheta'-\vtheta^*\|^2_\mL - 2(\vtheta-\vtheta^*)^T\mL(\vtheta'-\vtheta^*) - \|\vtheta-\vtheta^*\|^2_\mL\nonumber\\
    &>0\nonumber\\
    \Rightarrow \|\vtheta-\vtheta'\|^2_\mL&>\|\vtheta-\vtheta^*\|^2_\mL.\nonumber
\end{align}
Hence, $\vtheta^*$ is the optimal solution to the optimization problem, and $\vtheta^*=\mathcal{P}^{\mL}_{\mathcal{C}}(\vtheta).$
\end{proof}

We now prove Theorem \ref{theorem:P2CPO_converge}. 
Without loss of generality, on each learning episode \algname\ updates the reward followed by the alternation of two projections onto the constraint sets (region around $\pi_B$ and the cost constraint set):
\begin{align}
  &\vtheta^{k+\frac{1}{3}} 
  =\vtheta^k-\eta^{k} \mF^{-1}\nabla f(\vtheta^k),~\vtheta^{k+\frac{2}{3}}=\mathcal{P}_{\mathcal{C}_2}(\vtheta^{k+\frac{1}{3}}),~
    \vtheta^{k+1}=\mathcal{P}_{\mathcal{C}_1}(\vtheta^{k+\frac{2}{3}}),\text{if $\vtheta^{k}\in\mathcal{C}_2,$}\nonumber\\
      &\vtheta^{k+\frac{1}{3}} 
      =\vtheta^k-\eta^{k} \mF^{-1}\nabla f(\vtheta^k),~ \vtheta^{k+\frac{2}{3}}=\mathcal{P}_{\mathcal{C}_1}(\vtheta^{k+\frac{1}{3}}),~
    \vtheta^{k+1}=\mathcal{P}_{\mathcal{C}_2}(\vtheta^{k+\frac{2}{3}}),\text{if $\vtheta^{k}\in\mathcal{C}_1,$}\nonumber
\end{align}
where $\eta^{k}$ is the step size at step $k.$

\begin{proof}
\textbf{\algname\ under the KL-divergence projection converges to an $\epsilon$-FOSP.} Based on Lemma \ref{lemma:projecion_appendix} under the KL-divergence projection, and setting $\vtheta=\vtheta^k-\eta^k{\mF^k}^{-1}\nabla f(\vtheta^k),$ $\vtheta^*=\vtheta^{k+\frac{2}{3}}$ and $\vtheta'=\vtheta^k,$ we have 
\begin{align}
   &(\vtheta^{k}-\vtheta^{k+\frac{2}{3}})^T\mF^k(\vtheta^k-\eta^k{\mF^k}^{-1}\nabla f(\vtheta^k)-\vtheta^{k+\frac{2}{3}})\leq0\nonumber\\
   \Rightarrow\quad&\nabla f(\vtheta^k)^T(\vtheta^{k+\frac{2}{3}}-\vtheta^{k})\leq - \frac{1}{\eta^k}{(\vtheta^{k+\frac{2}{3}}-\vtheta^{k})^T}\mF^k(\vtheta^{k+\frac{2}{3}}-\vtheta^{k}).
   \label{eq:appendix_converge_0}
\end{align}
Based on the $L$-Lipschitz continuity of gradients and Eq. (\ref{eq:appendix_converge_0}), we have 
\begin{align}
f(\vtheta^{k+\frac{2}{3}})& \leq  f(\vtheta^{k})+{\nabla f(\vtheta^k)^T}(\vtheta^{k+\frac{2}{3}}-\vtheta^{k})+\frac{L}{2}{\|\vtheta^{k+\frac{2}{3}}-\vtheta^{k}\|^2} \nonumber \\ 
&\leq f(\vtheta^{k}) - \frac{1}{\eta^k}{(\vtheta^{k+\frac{2}{3}}-\vtheta^{k})^T}\mF^k(\vtheta^{k+\frac{2}{3}}-\vtheta^{k})+\frac{L}{2}{\|\vtheta^{k+\frac{2}{3}}-\vtheta^{k}\|^2}\nonumber \\
&= f(\vtheta^{k})-\frac{L}{2}{\|\vtheta^{k+\frac{2}{3}}-\vtheta^{k}\|^2}-{\nabla f(\vtheta^{k+\frac{2}{3}})^T}(\vtheta^{k+1}-\vtheta^{k+\frac{2}{3}})-\frac{L}{2}{\|\vtheta^{k+1}-\vtheta^{k+\frac{2}{3}}\|^2},
\label{eq:thrm4_1_1}
\end{align}
where the equality follows by setting $\delta$ (\ie the size of the trust region) such that 
\[
\eta^k=\frac{(\vtheta^{k+\frac{2}{3}}-\vtheta^{k})^T\mF^k(\vtheta^{k+\frac{2}{3}}-\vtheta^{k})}{L\|\vtheta^{k+\frac{2}{3}}-\vtheta^{k}\|^2+\nabla f(\vtheta^{k+\frac{2}{3}})^T(\vtheta^{k+1}-\vtheta^{k+\frac{2}{3}})+\frac{L}{2}\|\vtheta^{k+1}-\vtheta^{k+\frac{2}{3}}\|^2}.
\]
Again, based on Lemma \ref{lemma:projecion_appendix}, for $\vtheta\in\mathcal{C}_2$ we have 
\begin{align}
    &{(\vtheta^{k}-\eta^k{\mF^k}^{-1}\nabla f(\vtheta^k)-\vtheta^{k+\frac{2}{3}})}\mF^k(\vtheta-\vtheta^{k+\frac{2}{3}})\leq0 \nonumber \\ 
    \Rightarrow\quad&(-\eta^k{\mF^k}^{-1}\nabla f(\vtheta^k))^T\mF^k(\vtheta-\vtheta^{k+\frac{2}{3}})\leq-(\vtheta^{k}-\vtheta^{k+\frac{2}{3}})^T\mF^k(\vtheta-\vtheta^{k+\frac{2}{3}})\nonumber \\ 
    \Rightarrow\quad&\nabla f(\vtheta^k)^T(\vtheta-\vtheta^{k+\frac{2}{3}}) \geq \frac{1}{\eta^k}(\vtheta^{k}-\vtheta^{k+\frac{2}{3}})^T\mF^k(\vtheta-\vtheta^{k+\frac{2}{3}})\nonumber \\ 
    \Rightarrow\quad& \nabla f(\vtheta^k)^T\vtheta\geq\nabla f(\vtheta^k)^T\vtheta^{k+\frac{2}{3}}+\frac{1}{\eta^k}(\vtheta^k-\vtheta^{k+\frac{2}{3}})^T\mF^k(\vtheta-\vtheta^{k+\frac{2}{3}})\nonumber \\ 
    \Rightarrow\quad&f(\vtheta^k)^T(\vtheta-\vtheta^{k})\geq\nabla f(\vtheta^k)^T(\vtheta^{k+\frac{2}{3}}-\vtheta^{k})+\frac{1}{\eta^k}(\vtheta^k-\vtheta^{k+\frac{2}{3}})^T\mF^k(\vtheta-\vtheta^{k+\frac{2}{3}})\nonumber \\ 
    & \geq -\|\nabla f(\vtheta^k)\|\|\vtheta^{k+\frac{2}{3}}-\vtheta^{k}\|-\frac{1}{\eta^k}\|\vtheta^{k+\frac{2}{3}}-\vtheta^{k}\|\|\mF^k\|\|\vtheta-\vtheta^{k+\frac{2}{3}}\|\nonumber \\ 
    & \geq -\big(G+\frac{D\sigma_1(\mF^k)}{\eta^k}\big)\|\vtheta^{k+\frac{2}{3}}-\vtheta^{k}\|,
    \label{eq:thrm4_1_2}
\end{align}
where in the last two inequalities we use the property of the norm.
Before reaching an $\epsilon$-FOSP, Eq. (\ref{eq:thrm4_1_2}) implies that 
\begin{align}
    &-\epsilon\geq \min\limits_{\vtheta\in\mathcal{C}_2}\nabla f(\vtheta^k)^T(\vtheta-\vtheta^{k})\geq-\big(G+\frac{D\sigma_1(\mF^k)}{\eta^k}\big)\|\vtheta^{k+\frac{2}{3}}-\vtheta^{k}\|\nonumber \\ 
    \Rightarrow\quad&\|\vtheta^{k+\frac{2}{3}}-\vtheta^{k}\|\geq\frac{\epsilon}{G+\frac{D\sigma_1(\mF^k)}{\eta^k}}.
    \label{eq:thrm4_1_3}
\end{align}
Based on Eq. (\ref{eq:thrm4_1_1}) and Eq. (\ref{eq:thrm4_1_3}), we have
\begin{align}
    f(\vtheta^{k+\frac{2}{3}})
    &\leq f(\vtheta^{k})-\frac{L}{2}\|\vtheta^{k+\frac{2}{3}}-\vtheta^{k}\|^2-{\nabla f(\vtheta^{k+\frac{2}{3}})^T}(\vtheta^{k+1}-\vtheta^{k+\frac{2}{3}})-\frac{L}{2}{\|\vtheta^{k+1}-\vtheta^{k+\frac{2}{3}}\|^2}\nonumber \\
     &\leq f(\vtheta^{k})-\frac{L\epsilon^2}{2(G+\frac{D\sigma_1(\mF^k)}{\eta^k})^2}-{\nabla f(\vtheta^{k+\frac{2}{3}})^T}(\vtheta^{k+1}-\vtheta^{k+\frac{2}{3}})-\frac{L}{2}{\|\vtheta^{k+1}-\vtheta^{k+\frac{2}{3}}\|^2}.
     \label{eq:thrm4_1_4}
\end{align}
Based on the $L$-Lipschitz continuity of gradients, for the projection to the constraint set $\mathcal{C}_1$ we have
\begin{align}
    f(\vtheta^{k+1})& \leq  f(\vtheta^{k+\frac{2}{3}})+{\nabla f(\vtheta^{k+\frac{2}{3}})^T}(\vtheta^{k+1}-\vtheta^{k+\frac{2}{3}})+\frac{L}{2}{\|\vtheta^{k+1}-\vtheta^{k+\frac{2}{3}}\|^2}.
    \label{eq:thrm4_1_5}
\end{align}
Combining Eq. (\ref{eq:thrm4_1_4}) with Eq. (\ref{eq:thrm4_1_5}), we have
\begin{align}
    f(\vtheta^{k+1}) \leq  f(\vtheta^{k})-\frac{L\epsilon^2}{2(G+\frac{D\sigma_1(\mF^k)}{\eta^k})^2}.
    \label{eq:thrm4_1_6}
\end{align}
Hence it takes $\mathcal{O}(\epsilon^{-2})$ iterations to reach an $\epsilon$-FOSP.

\textbf{\algname\ under the 2-norm projection converges to an $\epsilon$-FOSP.} Based on Lemma \ref{lemma:projecion_appendix} under the 2-norm projection, and setting $\vtheta=\vtheta^k-\eta^k{\mF^k}^{-1}\nabla f(\vtheta^k),$ $\vtheta^*=\vtheta^{k+\frac{2}{3}}$ and $\vtheta'=\vtheta^k,$ we have
\begin{align}
   &(\vtheta^{k}-\vtheta^{k+\frac{2}{3}})^T(\vtheta^k-\eta^k{\mF^k}^{-1}\nabla f(\vtheta^k)-\vtheta^{k+\frac{2}{3}})\leq0\nonumber\\
   \Rightarrow&({\mF^k}^{-1}\nabla f(\vtheta^k))^T(\vtheta^{k+\frac{2}{3}}-\vtheta^{k})\leq - \frac{1}{\eta^k}{(\vtheta^{k+\frac{2}{3}}-\vtheta^{k})^T}(\vtheta^{k+\frac{2}{3}}-\vtheta^{k}).
   \label{eq:appendix_converge_7}
\end{align}
Based on the $L$-Lipschitz continuity of gradients and Eq. (\ref{eq:appendix_converge_7}), we have 
\begin{align}
f(\vtheta^{k+\frac{2}{3}})& \leq  f(\vtheta^{k})+{\nabla f(\vtheta^k)^T}(\vtheta^{k+\frac{2}{3}}-\vtheta^{k})+\frac{L}{2}{\|\vtheta^{k+\frac{2}{3}}-\vtheta^{k}\|^2} \nonumber \\ 
&\leq f(\vtheta^{k}) + ({\mF^k}^{-1}\nabla f(\vtheta^k))^T(\vtheta^{k+\frac{2}{3}}-\vtheta^{k})+Q+\frac{L}{2}{\|\vtheta^{k+\frac{2}{3}}-\vtheta^{k}\|^2}\nonumber \\
&\leq f(\vtheta^{k}) - \frac{1}{\eta^k}{(\vtheta^{k+\frac{2}{3}}-\vtheta^{k})^T}(\vtheta^{k+\frac{2}{3}}-\vtheta^{k})+Q+\frac{L}{2}{\|\vtheta^{k+\frac{2}{3}}-\vtheta^{k}\|^2}\nonumber \\
&= f(\vtheta^{k})-\frac{L}{2}{\|\vtheta^{k+\frac{2}{3}}-\vtheta^{k}\|^2}-{\nabla f(\vtheta^{k+\frac{2}{3}})^T}(\vtheta^{k+1}-\vtheta^{k+\frac{2}{3}})-\frac{L}{2}{\|\vtheta^{k+1}-\vtheta^{k+\frac{2}{3}}\|^2},
\label{eq:appendix_converge_8}
\end{align}
where $Q:=\nabla f(\vtheta^k)^T(\vtheta^{k+\frac{2}{3}}-\vtheta^{k})-({\mF^k}^{-1}\nabla f(\vtheta^k))^T(\vtheta^{k+\frac{2}{3}}-\vtheta^{k})$, which represents the difference between the gradient and the nature gradient, and the equality follows by setting $\delta$ (\ie the size of the trust region) such that
\[
\eta^k=\frac{\|\vtheta^{k+\frac{2}{3}}-\vtheta^{k}\|^2}{L\|\vtheta^{k+\frac{2}{3}}-\vtheta^{k}\|^2+Q+\nabla f(\vtheta^{k+\frac{2}{3}})^T(\vtheta^{k+1}-\vtheta^{k+\frac{2}{3}})+\frac{L}{2}\|\vtheta^{k+1}-\vtheta^{k+\frac{2}{3}}\|^2}.
\]
Again, based on Lemma \ref{lemma:projecion_appendix}, for $\vtheta\in\mathcal{C}_2$ we have 
\begin{align}
    &{(\vtheta^{k}-\eta^k{\mF^k}^{-1}\nabla f(\vtheta^k)-\vtheta^{k+\frac{2}{3}})}(\vtheta-\vtheta^{k+\frac{2}{3}})\leq0 \nonumber \\ 
    \Rightarrow\quad&(-\eta^k{\mF^k}^{-1}\nabla f(\vtheta^k))^T(\vtheta-\vtheta^{k+\frac{2}{3}})\leq-(\vtheta^{k}-\vtheta^{k+\frac{2}{3}})^T(\vtheta-\vtheta^{k+\frac{2}{3}})\nonumber \\ 
    \Rightarrow\quad&\nabla f(\vtheta^k)^T{\mF^k}^{-1}(\vtheta-\vtheta^{k+\frac{2}{3}}) \geq \frac{1}{\eta^k}(\vtheta^{k}-\vtheta^{k+\frac{2}{3}})^T(\vtheta-\vtheta^{k+\frac{2}{3}})\nonumber \\ 
    \Rightarrow\quad& \nabla f(\vtheta^k)^T{\mF^k}^{-1}\vtheta\geq\nabla f(\vtheta^k)^T{\mF^k}^{-1}\vtheta^{k+\frac{2}{3}}+\frac{1}{\eta^k}(\vtheta^k-\vtheta^{k+\frac{2}{3}})^T(\vtheta-\vtheta^{k+\frac{2}{3}})\nonumber \\ 
    \Rightarrow\quad&\nabla f(\vtheta^k)^T{\mF^k}^{-1}(\vtheta-\vtheta^{k})\geq\nabla f(\vtheta^k)^T{\mF^k}^{-1}(\vtheta^{k+\frac{2}{3}}-\vtheta^{k})+\frac{1}{\eta^k}(\vtheta^k-\vtheta^{k+\frac{2}{3}})^T(\vtheta-\vtheta^{k+\frac{2}{3}})\nonumber \\ 
    & \geq -\|\nabla f(\vtheta^k)\|\|{\mF^k}^{-1}\|\|\vtheta^{k+\frac{2}{3}}-\vtheta^{k}\|-\frac{1}{\eta^k}\|\vtheta^{k+\frac{2}{3}}-\vtheta^{k}\|\|\vtheta-\vtheta^{k+\frac{2}{3}}\|\nonumber \\ 
    & \geq -\big(G\sigma_1({\mF^k}^{-1})+\frac{D}{\eta^k}\big)\|\vtheta^{k+\frac{2}{3}}-\vtheta^{k}\|,
    \label{eq:appendix_converge_10}
\end{align}
where in the last two inequalities we use the property of the norm.
Before reaching an $\epsilon$-FOSP, Eq. (\ref{eq:appendix_converge_10}) implies that 
\begin{align}
    &-\epsilon\geq \min\limits_{\vtheta\in\mathcal{C}_2}\nabla f(\vtheta^k)^T{\mF^k}^{-1}(\vtheta-\vtheta^{k})\geq-\big(G\sigma_1({\mF^k}^{-1})+\frac{D}{\eta^k}\big)\|\vtheta^{k+\frac{2}{3}}-\vtheta^{k}\|\nonumber \\ 
    \Rightarrow\quad&\|\vtheta^{k+\frac{2}{3}}-\vtheta^{k}\|\geq\frac{\epsilon}{\big(G\sigma_1({\mF^k}^{-1})+\frac{D}{\eta^k}\big)}.
    \label{eq:appendix_converge_11}
\end{align}
Based on Eq. (\ref{eq:appendix_converge_8}) and Eq. (\ref{eq:appendix_converge_11}), we have
\begin{align}
    f(\vtheta^{k+\frac{2}{3}})
    &\leq f(\vtheta^{k})-\frac{L}{2}\|\vtheta^{k+\frac{2}{3}}-\vtheta^{k}\|^2-{\nabla f(\vtheta^{k+\frac{2}{3}})^T}(\vtheta^{k+1}-\vtheta^{k+\frac{2}{3}})-\frac{L}{2}{\|\vtheta^{k+1}-\vtheta^{k+\frac{2}{3}}\|^2}\nonumber \\
     &\leq f(\vtheta^{k})-\frac{L\epsilon^2}{2({G\sigma_1({\mF^k}^{-1})+\frac{D}{\eta^k}})^2}-{\nabla f(\vtheta^{k+\frac{2}{3}})^T}(\vtheta^{k+1}-\vtheta^{k+\frac{2}{3}})-\frac{L}{2}{\|\vtheta^{k+1}-\vtheta^{k+\frac{2}{3}}\|^2}.
     \label{eq:appendix_converge_12}
\end{align}
Based on the $L$-Lipschitz continuity of gradients, for the projection to the constraint set $\mathcal{C}_1$ we have
\begin{align}
    f(\vtheta^{k+1})& \leq  f(\vtheta^{k+\frac{2}{3}})+{\nabla f(\vtheta^{k+\frac{2}{3}})^T}(\vtheta^{k+1}-\vtheta^{k+\frac{2}{3}})+\frac{L}{2}{\|\vtheta^{k+1}-\vtheta^{k+\frac{2}{3}}\|^2}.
    \label{eq:appendix_converge_13}
\end{align}
Combining Eq. (\ref{eq:appendix_converge_12}) with Eq. (\ref{eq:appendix_converge_13}), we have
\begin{align}
    f(\vtheta^{k+1}) \leq  f(\vtheta^{k})-\frac{L\epsilon^2}{2({G\sigma_1({\mF^k}^{-1})+\frac{D}{\eta^k}})^2}.
    \label{eq:appendix_converge_14}
\end{align}
Hence it takes $\mathcal{O}(\epsilon^{-2})$ iterations to reach an $\epsilon$-FOSP.
\end{proof}

\paragraph{Comments on Assumption 1.3.}
In the paper, we assume that both the diameters of the cost constraint set ($\mathcal{C}_1$) and the region around $\pi_B$ ($\mathcal{C}_2$) are bounded above by $H.$ 
This implies that given a small value for $h_D,$ the convergence speed is determined by how large the constraint set is.
This allows us to do an analysis for the algorithm. 
In practice, we agree that this assumption is too strong and leave it as a future work for improvement.

\paragraph{Interpretation on Theorem \ref{theorem:P2CPO_converge}.}
We now provide a visualization in Fig. \ref{fig:KL_L2_Proj} under two possible projections.
For each projection,
we consider two possible Fisher information matrices.
Please read the caption for more detail.
%
%Theorem \ref{theorem:P2CPO_converge} implies that a smaller $\sigma_1(\mF^k)$ has more improvement of the objective value.
%
%However, the step size $\eta^k$ is also proportional to the KL-divergence between $\vtheta^k$ and $\vtheta^{k+\frac{2}{3}}.$
%
%This implies that the effect of $\sigma_1(\mF^k)$ is canceled out by $\eta^k.$ 
%
In Fig. \ref{fig:KL_L2_Proj}(a) we observe that since the reward improvement and projection steps use the KL-divergence, the resulting two update points with different $\sigma_1(\mF^k)$ are similar.
In addition, under the 2-norm projection, the larger  $\sigma_n(\mF^k)$ is, the greater the decrease in the objective.
%Theorem \ref{theorem:P2CPO_converge} implies that a larger smallest singular value $\sigma_n(\mF^k)$ has more improvement of the objective value.
%
This is because that a large $\sigma_n(\mF^k)$ implies a large curvature of $f$ in all directions. 
Intuitively, this makes the learning algorithm confident about where to update the policy to decrease the objective value greatly.
Geometrically, a large $\sigma_n(\mF^k)$ makes the 2-norm distance between the pre-projection and post-projection points small, leading to a small deviation from the reward improvement direction.
%
%This observation is also supported by the update procedures in Eq. \ref{eq:P2CPO_final}. 
%
%We observe that a large $\sigma_n(\mF^k)$ makes the coefficients of the projection terms small. 
%
%This implies that $\vtheta^{k+1}$ moves toward the direction that can improve the reward most.
%
This is illustrated in Fig. \ref{fig:KL_L2_Proj}(b).
We observe that since $\mF^k$ determines the curvature of $f$ and the 2-norm projection is used, the updated point with a larger $\sigma_n(\mF^k)$ (red dot) achieves more improvement of the objective value.
These observations imply that the spectrum of the Fisher information matrix does not play a major role in \algname\ under the KL-divergence projection, whereas it affects the decrease of objective value in \algname\ under the 2-norm projection.
Hence we choose either KL-divergence or 2-norm projections depending on the tasks to achieve better performance. 
%
%To further illustrate their difference, we compare these two projections  

\begin{figure*}[t]
\vspace{-3mm}
\centering
\subfloat[\algname\ under the KL-divergence projection]{
\includegraphics[width=0.5\linewidth]{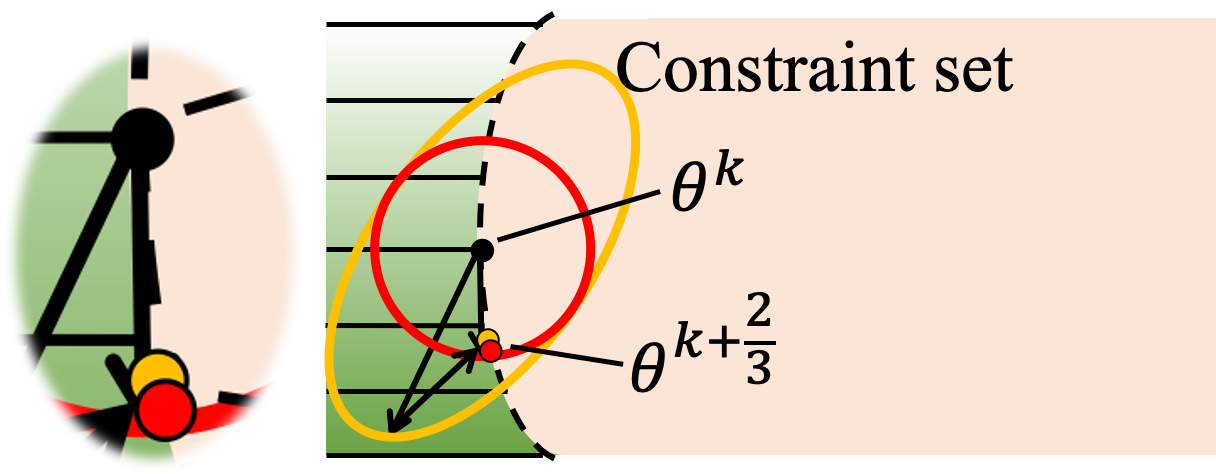}}\hspace{15mm}
\subfloat[\algname\ under the 2-norm projection]{
\includegraphics[width=0.5\linewidth]{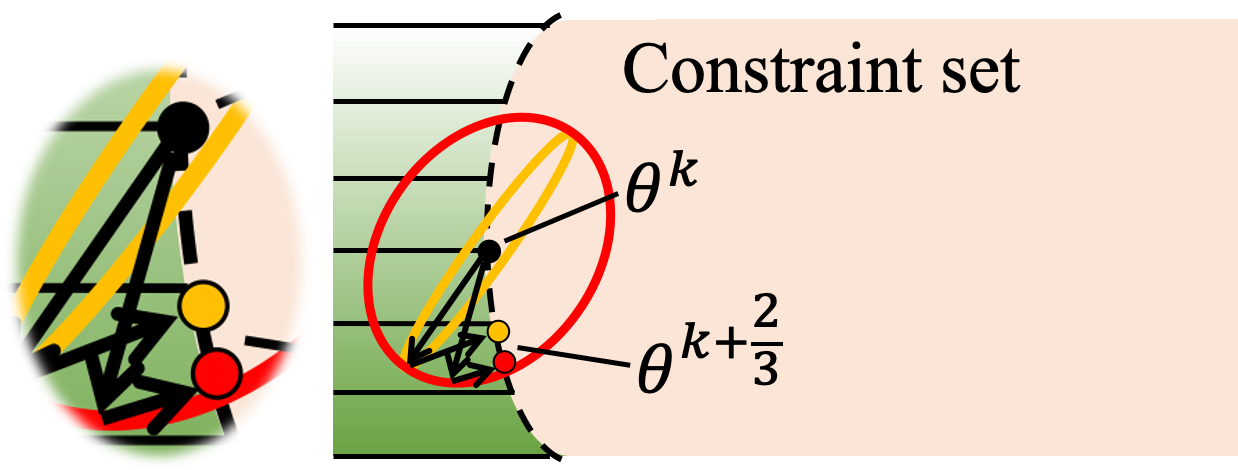}}
\caption{Update procedures for \algname\ under the KL and 2-norm projections with two possible Fisher information matrices.
A lower objective value is achieved at the darker green area.
Red and orange ellipses are $\mF^k$s with two different spectra of singular values.
Red and orange dots are resulting updated points under these two spectra of $\mF^k$s.
\textbf{(a)} A red ellipse has a smaller $\sigma_1(\mF^k)$ and an orange ellipse has a larger $\sigma_1(\mF^k).$ 
Both ellipses have the same $\sigma_n(\mF^k).$ 
The two resulting $\vtheta^{k+\frac{2}{3}}$ are similar.
\textbf{(b)} A red ellipse has a larger $\sigma_n(\mF^k)$ and an orange ellipse has a smaller $\sigma_n(\mF^k).$
Both ellipses have the same $\sigma_1(\mF^k).$ 
$\vtheta^{k+\frac{2}{3}}$ with a larger $\sigma_n(\mF^k)$ (red dot) has greater decrease of the objective value.
}
\label{fig:KL_L2_Proj}
%\end{mdframed}
\end{figure*}

\section{Additional Experiment Results}
\label{sec:appendix_experiment}
\subsection{Implementation Details}
\label{subsec:appendix_details}
\paragraph{Mujoco Task~\citep{achiam2017constrained}.}
In the point circle and ant circle tasks, the reward and cost functions are
\[
R(s) = \frac{\vv^T[-x_2;x_1]}{1+|\|[x_1;x_2]\|-d|},
\]
and
\[
C(s) = \mathbbm{1}[|x_1|>x_\mathrm{lim}],
\]
where $x_1$ and $x_2$ are the coordinates in the plane, $\vv$ is the velocity of the agent, and $d$, $x_\mathrm{lim}$ are environmental parameters that specify the safe area. 
The agent is rewarded for moving fast in a wide circle with radius of $d$, but is constrained to stay within a safe region smaller than the radius of the circle in $x_1$-coordinate $x_\mathrm{lim}\leq d$.
For the point agent, we use $d=5$ and $x_\mathrm{lim}=2.5$; for the ant agent, we use $d=5$ and $x_\mathrm{lim}=1.$
The environment is illustrated in Fig. \ref{fig:circle_env}.

\begin{figure*}[t]
\centering
\includegraphics[scale=0.3]{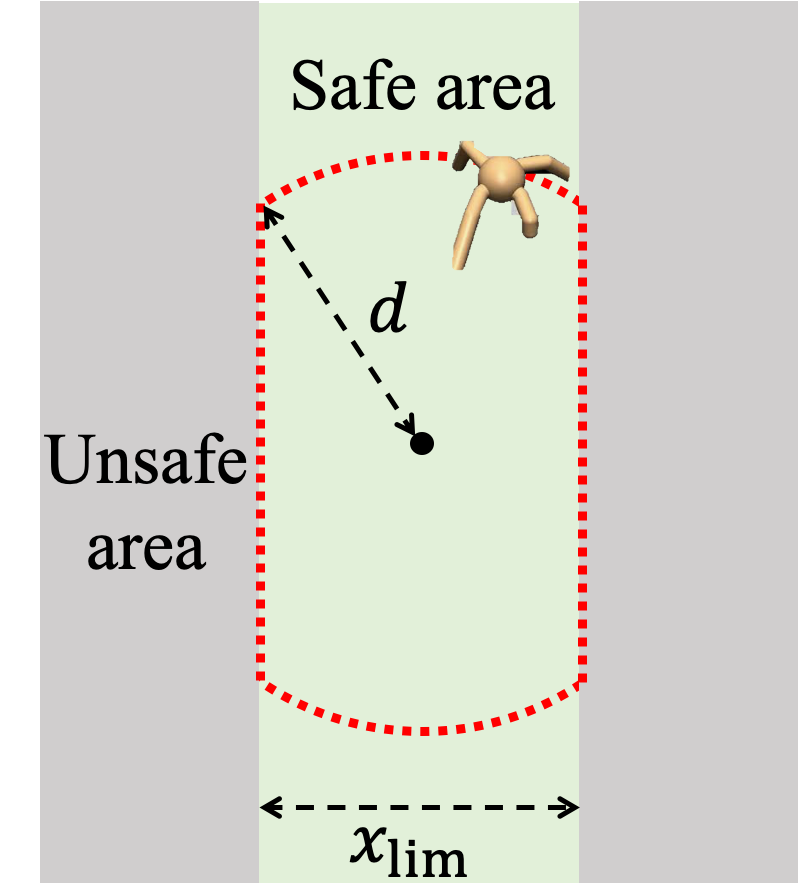}
\caption{
The environment of the circle task (adapted from \cite{achiam2017constrained}). The agent receives the maximum reward while staying in the safe area by following the red dashed line path. 
}
\label{fig:circle_env}
\end{figure*}

In the point gather task, the agent receives a reward of $+10$ for gathering green apples, and a cost of $1$ for gathering red apples.
Two green apples and eight red apples are placed in the environment at the beginning.
In the ant gather task, the agent receives a reward of $+10$ for gathering green apples, and a cost of $1$ for gathering red apples.
The agent also gets a reward of $-10$ for falling down to encourage smooth moving. 
Eight green apples and eight red apples are placed in the environment at the beginning.

For the point and ant agents, the state space consists of the positions, orientations, velocities, and the external forces applied to the torso and joint angles. 
The action space is the force applied to joints.

\paragraph{Traffic Management Task~\citep{vinitsky2018benchmarks}.}
In the grid task, the state space, action space, reward function, and cost function are illustrated as follows.
(1) States: Speed, distance to the intersection, and edge number of each vehicle. 
The edges of the grid are uniquely numbered so the travel direction can be inferred. 
For the traffic lights, we return 0 and 1 corresponding to green or red for each light, 
a number between $[0, t_\mathrm{switch}]$ indicating how long until a switch can occur, 
and 0 and 1 indicating if the light is currently yellow. 
Finally, we return the average density and velocity of each edge.

(2) Actions: A list of numbers $a=[-1, 1]^n$ where $n$ is the number of traffic lights. If $a_i>0$ for traffic light $i$ it switches, otherwise no action is taken.

(3) Reward: The objective of the agent is to achieve high speeds. 
The reward function is
\[
R(s) = \frac{\max(v_\mathrm{target}-\|\vv_\mathrm{target}-\vv\|,0)}{v_\mathrm{target}},
\]
where $v_\mathrm{target}$ is an arbitrary large velocity used to encourage high speeds and $\vv\in\R^k$ is the velocities of $k$ vehicles in the network.

(4) Cost: The objective of the agent is to let lights stay red for at most 7 consecutive seconds.
The cost function is 
\[
C(s) = \sum_{i=1}^{n}\mathbbm{1}[t_{i,\mathrm{red}}>7],
\]
where $t_{i,\mathrm{red}}$ is the consecutive time that the light $i$ is in red.

In the bottleneck task, the state space, action space, reward function, and cost function are illustrated as follows.
(1) States: The states include: the mean positions and velocities of human drivers for each lane for each edge segment,
the mean positions and velocities of the autonomous vehicles on each segment, and 
the outflow of the system in vehicles per/hour over the last 5 seconds.
(2) Actions: For a given edge-segment and a given lane, the action shifts the maximum speed
of all the autonomous vehicles in the segment from their current value. 
By shifting the max-speed to higher or lower values, the system indirectly controls the velocity of the autonomous vehicles.
(3) Reward: The objective of the agent is to maximize the outflow of the whole traffic. 
The reward function is
\[
R(s_t) = \sum_{i=t-\frac{5}{\Delta t}}^{i=t}\frac{n_\mathrm{exit}(i)}{\frac{5}{\Delta t\cdot n_\mathrm{lane}\cdot 500}},
\]
where $n_\mathrm{exit}(i)$ is the number of vehicles that exit the system at time-step $i$, and $n_\mathrm{lane}$ is the number of lanes.
(4) Cost: The objective of the agent is to let the velocities of human drivers have lowspeed for no more than 10 seconds.
The cost function is 
\[
C(s) = \sum_{i=1}^{n_\mathrm{human}}\mathbbm{1}[t_{i,\mathrm{low}}>10],
\]
where $n_\mathrm{human}$ is the number of human drivers, and $t_{i,\mathrm{low}}$ is the consecutive time that the velocity of human driver $i$ is less than 5 m/s.
For more information, please refer to \cite{vinitsky2018benchmarks}. 

\paragraph{Car-racing Task.} In the car-racing task, the state space, action space, reward function, and the cost function are illustrated as follows.

(1) States: It is a high-dimensional space where the state is a $96 \times 96 \times 3$ tensor of raw pixels. Each pixel is in the range of $[0,255].$ 

(2) Actions: The agent has 12 actions in total: $a\in\mathcal{A}=\{(a^\mathrm{steer},a^\mathrm{gas},a^\mathrm{brake})|a^\mathrm{steer}\in\{-1,0,1\},a^\mathrm{gas}\in\{0,1\},a^\mathrm{brake}\in\{0,0.2\}\},$ where $a^\mathrm{steer}$ is the steering angle, $a^\mathrm{gas}$ is the amount of gas applied, and $a^\mathrm{brake}$ is the amount of brake applied.

(3) Reward: In each
episode, we randomly generate the track. The episode is terminated if the agent reaches the maximal step or traverse over 95\% of the track.
The track is discretized into 281 tiles.
The agent receives a reward of $\frac{1000}{281}$ for each tile visited. 
To encourage driving efficiency, the agent receives a penalty of $-1$ per-time step. 

(4) Cost: The cost is to constrain the accumulated number of brakes to encourage a smooth ride.

%
%The results, code, and demonstration videos are available at \url{http://www.anonymous}.
%

%
\paragraph{Architectures and Parameters.}
For the gather and circle tasks we test two distinct agents: 
a point-mass ($S \subseteq \R^{9}, A \subseteq \R^{2}$), 
and an ant robot ($S \subseteq \R^{32}, A \subseteq \R^{8}$).
The agent in the grid task is $S \subseteq \R^{156}, A \subseteq \R^{4},$ and the agent in the bottleneck task is $S \subseteq \R^{141}, A \subseteq \R^{20}.$
Finally, the agent in the car-racing task is $S \subseteq \R^{96\times 96\times3}, A \subseteq \R^{3}.$

For the simulations in the gather and circle tasks, we use a neural network with two hidden layers of size (64, 32) to represent Gaussian policies. And we use the KL-divergence projection.
For the simulations in the grid and bottleneck tasks, we use a neural network with two hidden layers of size (16, 16) and (50, 25) to represent Gaussian policies, respectively. And we use the 2-norm projection.
For the simulation in the car-racing task, we use a convolutional neural network with two convolutional operators of size 24 and 12 followed by a dense layer of size (32, 16) to represent a Gaussian policy. And we use the KL-divergence projection.
The choice of the projections depends on the task itself, we report the best performance among two projections.
We use $\mathrm{tanh}$ as an activation function for all the neural network policies.
In the experiments, since the step size is small, we reuse the Fisher information matrix of the reward improvement step in the KL-divergence projection step to reduce the computational cost.
We use GAE-$\lambda$ approach \citep{schulman2015high} to estimate $A^\pi_{R}(s,a),$ $A^\pi_{C}(s,a),$ and $A^\pi_{D}(s).$
For the simulations in the gather, circle, and car-racing tasks, we use neural network baselines with the same architecture and activation functions as the policy networks.
For the simulations in the grid and bottleneck tasks, we use linear baselines.
The hyperparameters of all algorithms and all tasks are in Table \ref{tab:parab}. 

\begin{table*}[t]
\centering
\vspace{0.0in}
\scalebox{0.9}{
\begin{tabular}{cccccccc}
\toprule
Parameter                                      & PC & PG & AC & AG & Gr & BN & CR\\  \hline
\multirow{1}{*}{Reward dis. factor~$\gamma$}      & 0.995 & 0.995 & 0.995 & 0.995 & 0.999 & 0.999 & 0.990 \\
\multirow{1}{*}{Constraint cost dis. factor~$\gamma_{C}$}      & 1.0  & 1.0  & 1.0  & 1.0  & 1.0  & 1.0  & 1.0 \\
\multirow{1}{*}{Divergence cost dis. factor~$\gamma_{D}$}      & 1.0 & 1.0 & 1.0 & 1.0 & 1.0  & 1.0 & 1.0 \\
\multirow{1}{*}{step size~$\delta$}             & $10^{-4}$ & $10^{-4}$ & $10^{-4}$ & $10^{-4}$ & $10^{-4}$ & $10^{-4}$ & $5\times10^{-4}$ \\
\multirow{1}{*}{$\lambda^\mathrm{GAE}_{R}$}    & 0.95 & 0.95 & 0.95 & 0.95 & 0.97 & 0.97 & 0.95 \\
\multirow{1}{*}{$\lambda^\mathrm{GAE}_{C}$}    & 1.0 & 1.0 & 0.5 & 0.5 & 0.5 & 1.0 & 1.0 \\
\multirow{1}{*}{$\lambda^\mathrm{GAE}_{D}$}    & 0.95 & 0.95 & 0.95 & 0.95 & 0.90 &0.90 & 0.95 \\
\multirow{1}{*}{Batch size}                    & 50,000 & 50,000 & 100,000 & 100,000 & 10,000 & 25,000 & 10,000 \\
\multirow{1}{*}{Rollout length}                & 50 & 15 & 500 & 500 & 400 & 500 & 1000 \\
\multirow{1}{*}{Constraint cost threshold~$h_C$} & 5 & 0.5 & 5 & 0.2 & 0 & 0 & 5 \\
\multirow{1}{*}{Divergence cost threshold~$h_D^0$} & 5 & 3 & 5 & 3 & 10 & 10 & 5 \\
\multirow{1}{*}{Number of policy updates} & 1,000 & 1,200 & 2,500 & 1,500 & 200 & 300 & 600 \\
\bottomrule
\end{tabular}}
\caption{\label{tab:parab}Parameters used in all tasks. (PC: point circle, PG: point gather, AC: ant circle, AG: ant gather, Gr: grid, BN: bottleneck, and CR: car-racing tasks)}
\end{table*}

We conduct the experiments on three separate machines: machine A has an Intel Core i7-4770HQ CPU, machine B has an Intel Core i7-6850K CPU, and machine C has an Intel Xeon X5675 CPU. We report real-time (\ie wall-clock time) in seconds for one policy update for all tested algorithms and tasks in Table \ref{tab:time}.
We observe that \algname\ has the same computational time as the other baselines.

\begin{table*}[t]
\centering
\scalebox{0.8}{
\begin{tabular}{l*{12}{c}r} 
\toprule
&\multicolumn{2}{c}{PCPO}&\multicolumn{2}{c}{\algname\ (Ours)}&\multicolumn{2}{c}{f-PCPO}&\multicolumn{2}{c}{f-CPO}&\multicolumn{2}{c}{d-PCPO}&\multicolumn{2}{c}{d-CPO}\\

\cmidrule(lr){2-3} \cmidrule(lr){4-5}\cmidrule(lr){6-7}\cmidrule(lr){8-9}\cmidrule(lr){10-11}\cmidrule(lr){12-13}

& M/C & Time & M/C & Time & M/C & Time & M/C & Time & M/C & Time & M/C & Time \\ \hline
PG  & B & 22.14 & B & 25.2 & B & 31.9 & B & 25.5   & B & 32.8 & B &32.6\\
PC  & B & 35.1 & B & 51.2 & B & 48.4 & B & 49.4  & B & 55.5 & B & 55.9\\
AG  & B & 386.9 & B & 110.5 & C & 268.6 & C & 235.1  & B & 138.2 & B & 187.5 \\
AC  & B & 148.9 & B & 94.0 & C & 222.6 & C & 214.6   & B & 177.4 & B & 151.2\\
Gr  & A & 105.3 & A & 91.4 & A & 88.2 & A & 58.7  & A & 116.8 & A & 115.3 \\
BN  & A & 257.7 & A & 181.1 & A & 162.9 & A & 161.6   & A & 259.3 & A & 275.6\\
CR  & C & 993.5 & C & 971.6 & C & 1078.3 & C & 940.1  & C & 1000.4 & C & 981.0 \\
\bottomrule
\end{tabular}}
\caption{\label{tab:time}Real-time in seconds for one policy update for all tested algorithms and tasks. (PC: point circle, PG: point gather, AC: ant circle, AG: ant gather, Gr: grid, BN: bottleneck, and CR: car-racing tasks)}
\end{table*}

For the most intensive task, \ie the car-racing task, the memory usage is 6.28GB. 
The experiments are implemented in rllab~\citep{duan2016benchmarking}, 
a tool for developing RL algorithms. 
We provide the link to the code: \url{https://sites.google.com/view/spacealgo}.

\paragraph{Comments on the rationale behind when to increase $h_D$.}
The update method of $h_D$ is empirically designed to ensure that the value of the cost does not increase (\ie $J_C({\pi^k})\leq J_C(\pi^{k-1})$) and the reward keeps improving (\ie $J_R(\pi^k)\geq  J_R(\pi^{k-1})$) after learning from $\pi_B$.
Lemma \ref{theorem:h_D} theoretically ensures $h_D$ is large enough to guarantee feasibility and exploration of the agent.

\paragraph{Implementation of Updating $h_D^k$.}
Lemma \ref{theorem:h_D} shows that $h^{k+1}_D$ should be increased at least by $\mathcal{O}\big((J_{C}(\pi^k)-h_C)^2\big)+h_D^k$ if $J_C(\pi^k)>J_C(\pi^{k-1})$ or $J_R(\pi^k)<J_R(\pi^{k-1})$ at step $k$. 
We now provide the practical implementation.
For each policy update we check the above conditions.
If one of the conditions satisfies, we increase $h_D^{k+1}$ by setting the constant to $10$, \ie $10\cdot(J_{C}(\pi^k)-h_C)^2+h_D^k.$
In practice, we find that the performance of \algname\ is not affected by the selection of the constant.
Note that we could still compute the exact value of $h_D^{k+1}$ as shown in the proof of Lemma \ref{theorem:h_D}.
However, this incurs the computational cost.

\paragraph{Comments on learning from multiple baseline policies $\pi_B$.}
In our setting, we use one $\pi_B$. 
This allows us to do theoretical analysis.
One possible idea for learning from multiple $\pi_B$ is to compute the distance to each $\pi_B$.
Then, select the one with the minimum distance to do the update.
This ensures that the update for the reward in the first step is less affected by $\pi_B.$
And the analysis we did can be extended.
We leave it as future work for developing this.

\paragraph{Comments on refining the PCPO agent's policy~\citep{yang2020projection} directly.}
Fine-tuning the pre-trained policy directly might result in lower reward and cost violations.
This is because that the pre-trained policy has a low entropy and it does not explore.
We empirically observe that the agent pre-trained with the baseline policy yields less reward in the new task (\ie different cost constraint thresholds $h_C$) as illustrated in Section \ref{additional_Experiment}.
In contrast, the \algname\ agent simultaneously learns from the baseline policy while ensuring the policy entropy is high enough to explore the environment.

\parab{Comments on the feasibility of getting safe baseline policies.}
In many real-world applications such as drones, we can obtain baseline policies modeled from the first principle physics, or pre-train baseline policies in the constrained and safe environment, or use rule-based baseline policies. Importantly, we do not assume the baseline has to be a ``safe policy'' -- it can be a heuristic that ignores safety constraints. This is one of the main motivations for our algorithm: to utilize priors from the baseline which may be unsafe, but guarantee the safety of the newly learned algorithm according to the provided constraints.

\paragraph{Instructions for Reproducibility.}
We now provide the instructions for reproducing the results.
First install the libraries for python3 such as numpy, scipy.
To run the Mujoco experiments, get the licence from \url{https://www.roboti.us/license.html}.
To run the traffic management experiments, install FLOW simulator from \url{https://flow.readthedocs.io/en/latest/}.
To run the car-racing experiments, install OpenAI Gym from \url{https://github.com/openai/gym}.
Our implementation is based on the environment from \cite{achiam2017constrained}, please download the code from \url{https://github.com/jachiam/cpo}. 
The code is based on rllab \citep{duan2016benchmarking}, install the relevant packages such as theano (\url{http://deeplearning.net/software/theano/}).
Then, download \algname\ code from \url{https://sites.google.com/view/spaceneurips} and place the codes on the designated folder instructed by Readme.txt on the main folder.
Finally, go to the example folder and execute the code using python command.

\begin{figure*}[t]
\vspace{-3mm}
\centering
\subfloat[Bottleneck\label{subfig:bn}]{\begin{tabular}[b]{@{}c@{}}%
\includegraphics[width=0.33\linewidth]{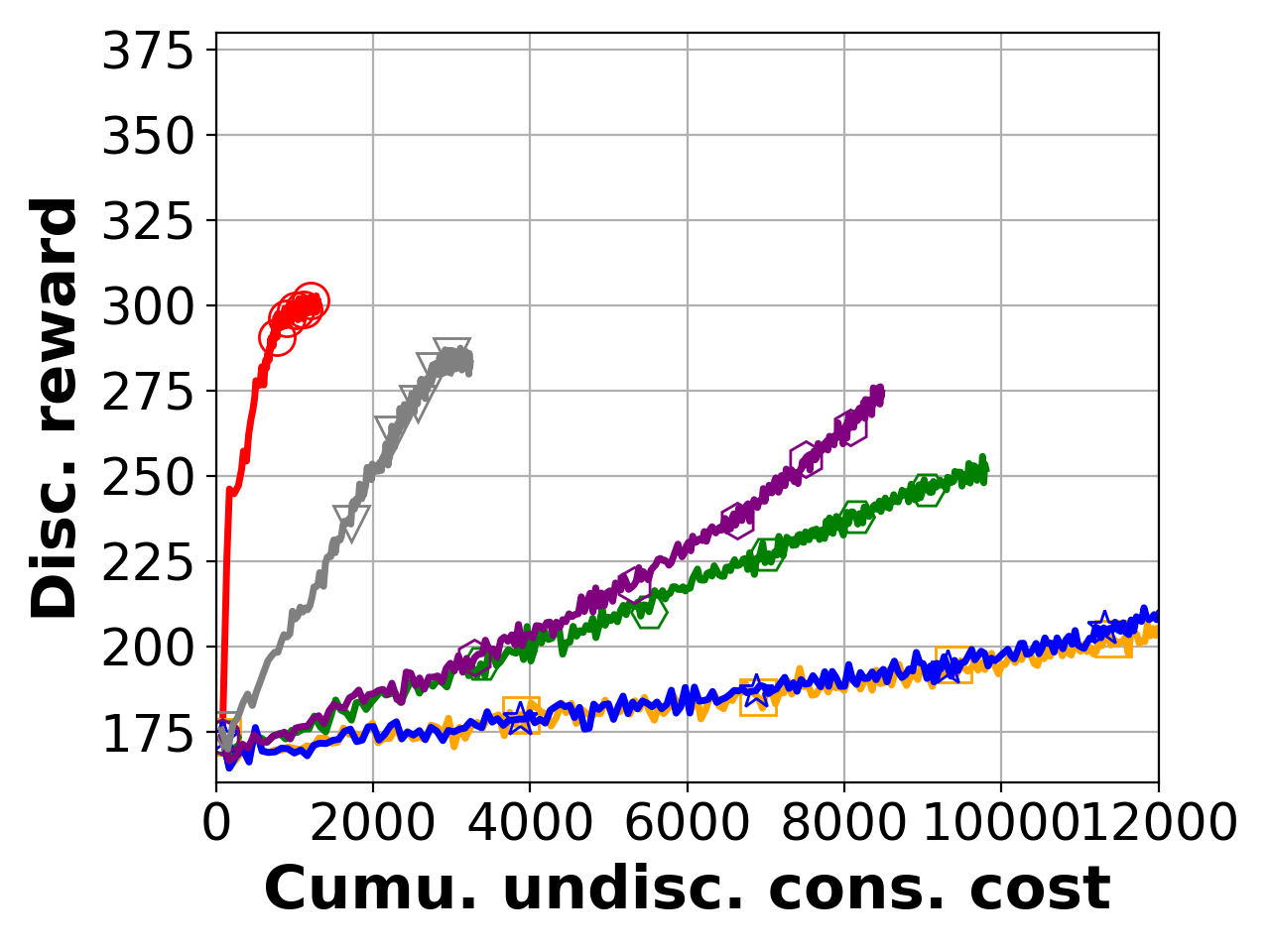}%
\end{tabular}}%
\subfloat[Car-racing\label{subfig:cr}]{\begin{tabular}[b]{@{}c@{}}%
\includegraphics[width=0.33\linewidth]{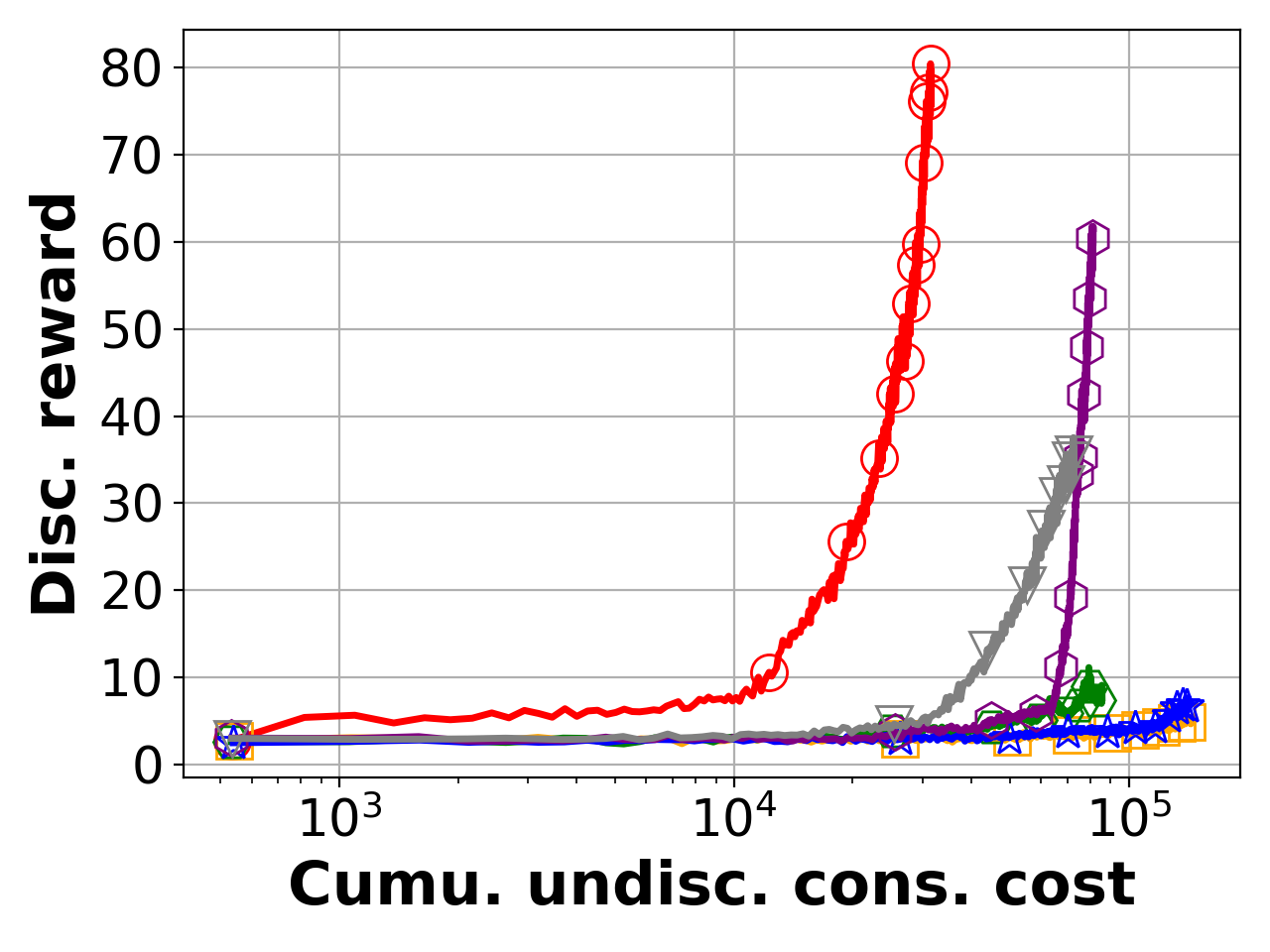}%
\end{tabular}}%
\subfloat[Grid\label{subfig:grid}]{\begin{tabular}[b]{@{}c@{}}%
\includegraphics[width=0.33\linewidth]{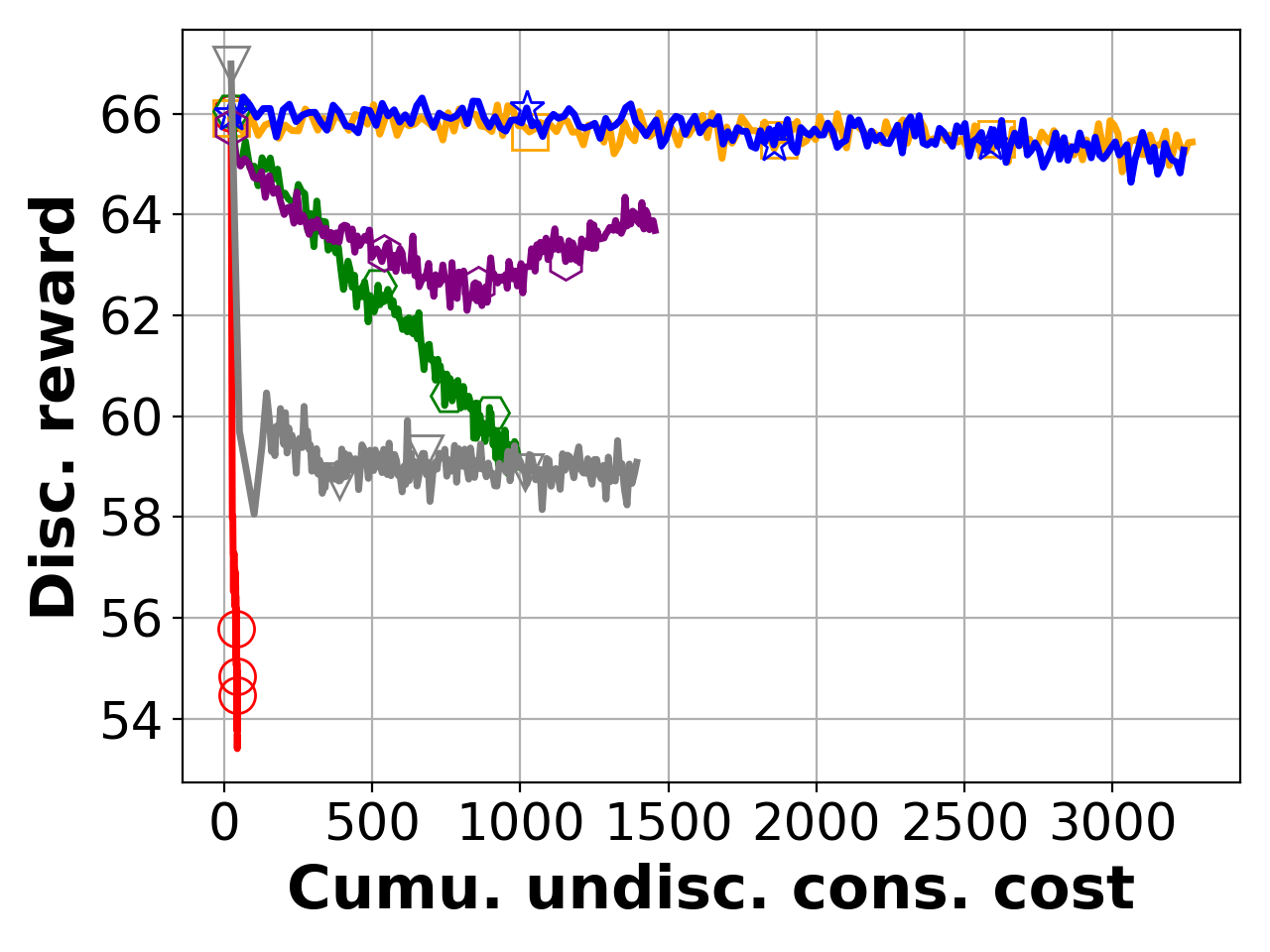}%
\end{tabular}}%
\iffalse
\subfloat[Point gather\label{subfig:pg}]{\begin{tabular}[b]{@{}c@{}}%
\includegraphics[width=0.33\linewidth]{figure/exp_2/RewardvsCost_pg_SafevsAggressive.png}%
\\
\includegraphics[width=0.33\linewidth]{figure/exp_2/NumDemo_pg_SafevsAggressive.png}%
\end{tabular}}%
\subfloat[Ant circle\label{subfig:ac}]{\begin{tabular}[b]{@{}c@{}}%
\includegraphics[width=0.33\linewidth]{figure/exp_2/RewardvsCost_ac_Dynamic_hp_and_fixed_hp.png}%
\\
\includegraphics[width=0.33\linewidth]{figure/exp_2/NumDemo_ac_Dynamic_hp_and_fixed_hp.png}%
\end{tabular}}%
\fi
\vspace{+1mm}

\includegraphics[width=0.9\linewidth]{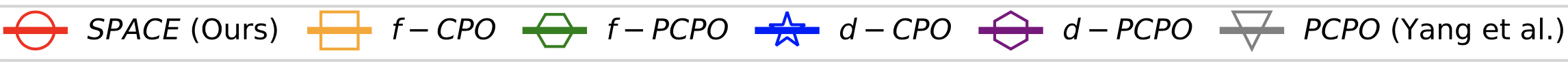}
\vspace{-2mm}

\caption{
The discounted reward vs. the cumulative undiscounted constraint cost
over policy updates for the tested algorithms and tasks.
The solid line is the mean over 5 runs.
\algname\ achieves the same reward performance with fewer cost constraint violations in all cases.
%The baseline policies in the gird and bottleneck tasks are $\pi_B^\mathrm{same},$
%
%and the baseline policy in the car-racing task is $\pi_B^\mathrm{human}.$
%
%The values of the discounted reward, the undiscounted cost constraint value, and the undiscounted prior constraint value under two different priors over policy updates for the tested algorithms and task pairs. 
%
%The solid line is the mean and the shaded area is the standard deviation over five runs. 
%
%The dashed lines in the reward plot are the rewards of $\pi_B^\mathrm{cost}$ (lower) and $\pi_B^\mathrm{reward}$ (upper).
%
%The dashed lines in the cost constraint plot are the cost constraint value of the $\pi_B^\mathrm{cost}$ (lower) and $\pi_B^\mathrm{reward}$ (upper).
%
%The middle dashed line in the cost constraint plot is the cost constraint threshold $h_C$ of the agent.
%
%We observe that \algname\ with a safe prior $\pi^\mathrm{cost}_B$ achieves better reward performance compared to the one with an aggressive prior $\pi^\mathrm{reward}_B.$
%
(Best viewed in color.)
}
\label{fig:reward_vs_cost}
\vspace{-3mm}
%\end{mdframed}
\end{figure*}

\subsection{Experiment Results}
\label{additional_Experiment}

\paragraph{Baseline policies.}
We pre-train the baseline policies using a safe RL algorithm. 
Here we also consider three types of baseline policies:
\textbf{(1)} \textit{suboptimal} $\pi_B^\mathrm{cost}$ with $J_C(\pi_B^\mathrm{cost})\approx0,$
\textbf{(2)} \textit{suboptimal} $\pi_B^\mathrm{reward}$ with $J_C(\pi_B^\mathrm{reward})>h_C,$ 
and 
\textbf{(3)} $\pi_B^\mathrm{near}$ with $J_C(\pi_B^\mathrm{near})\approx h_C$
Note that these $\pi_B$ have different degrees of constraint satisfaction.

\paragraph{The Discounted Reward vs. the Cumulative Undiscounted Constraint Cost (see Fig.~\ref{fig:reward_vs_cost}).}
To show that \algname\ achieves higher reward under the same cost constraint violations (\ie learning a constraint-satisfying policy without violating the cost constraint a lot), we examine the discounted reward versus the \textit{cumulative} undiscounted constraint cost.
The learning curves of the discounted reward versus the cumulative undiscounted constraint cost are shown for all tested algorithms and tasks in Fig. \ref{fig:reward_vs_cost}.
We observe that in these tasks under the same value of the reward, \algname\ outperforms the baselines significantly with fewer cost constraint violations. 
For example, in the car-racing task \algname\ achieves 3 times fewer cost constraint violations at the reward value of 40 compared to the best baseline -- PCPO.
This implies that \algname\ effectively leverages the baseline policy while ensuring the constraint satisfaction.
In contrast, without the supervision of the baseline policy, PCPO requires much more constraint violations to achieve the same reward performance as \algname.
In addition, although the fixed-point and the dynamic-point approaches use the supervision of the baseline policy, the lack of the projection step makes them less efficient in learning a constraint-satisfying policy.

\begin{figure*}[t]
\vspace{-3mm}
\centering
\subfloat[Point gather\label{subfig:grid}]{\begin{tabular}[b]{@{}c@{}}%
\includegraphics[width=0.33\linewidth]{figure/exp_2/NumCost_pg_overallPerformance_v2.png}%
\includegraphics[width=0.33\linewidth]{figure/exp_2/Reward_pg_overallPerformance_v2.png}%
\includegraphics[width=0.33\linewidth]{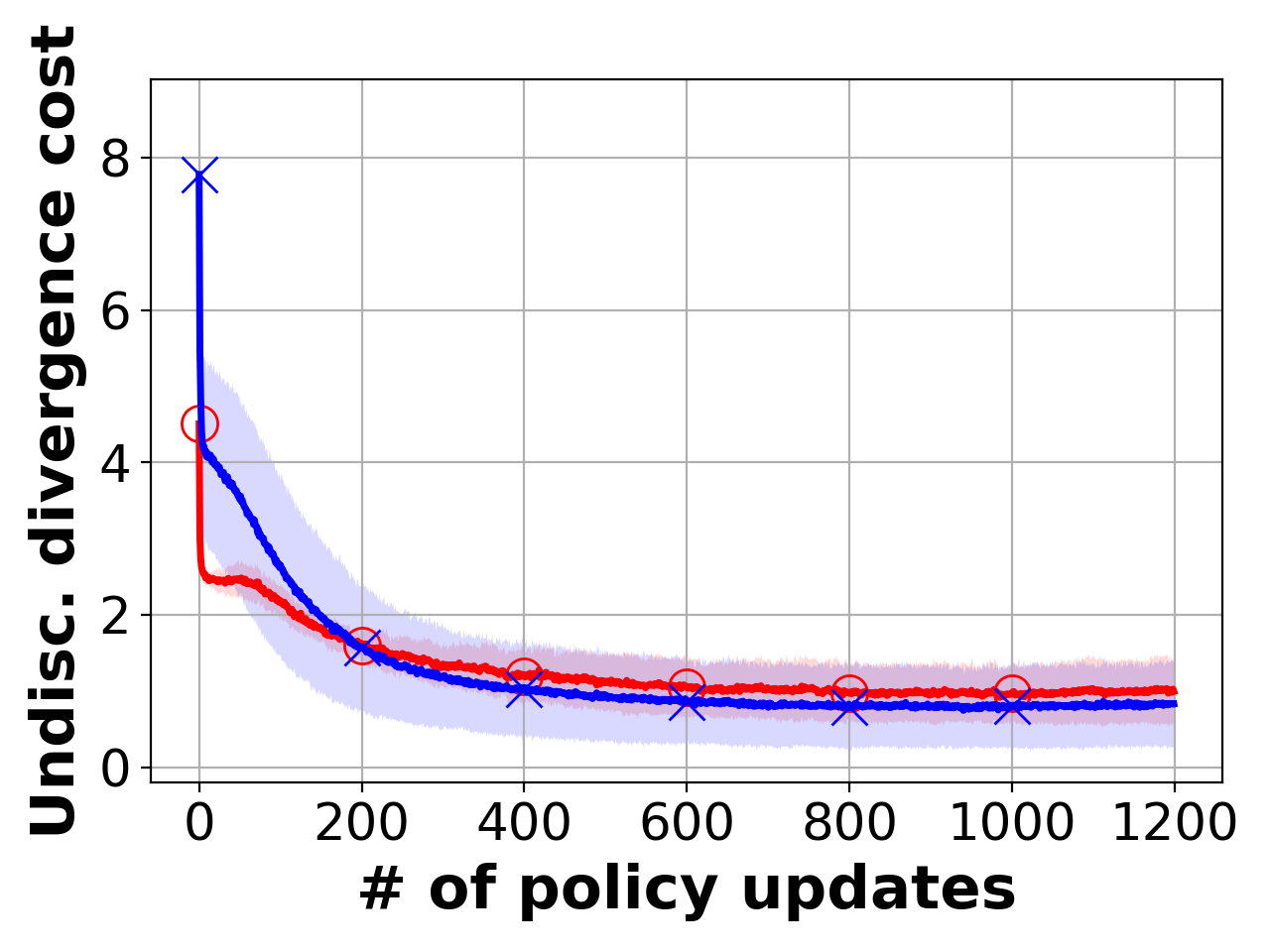}%
\end{tabular}}%

\subfloat[Point circle\label{subfig:grid}]{\begin{tabular}[b]{@{}c@{}}%
\includegraphics[width=0.33\linewidth]{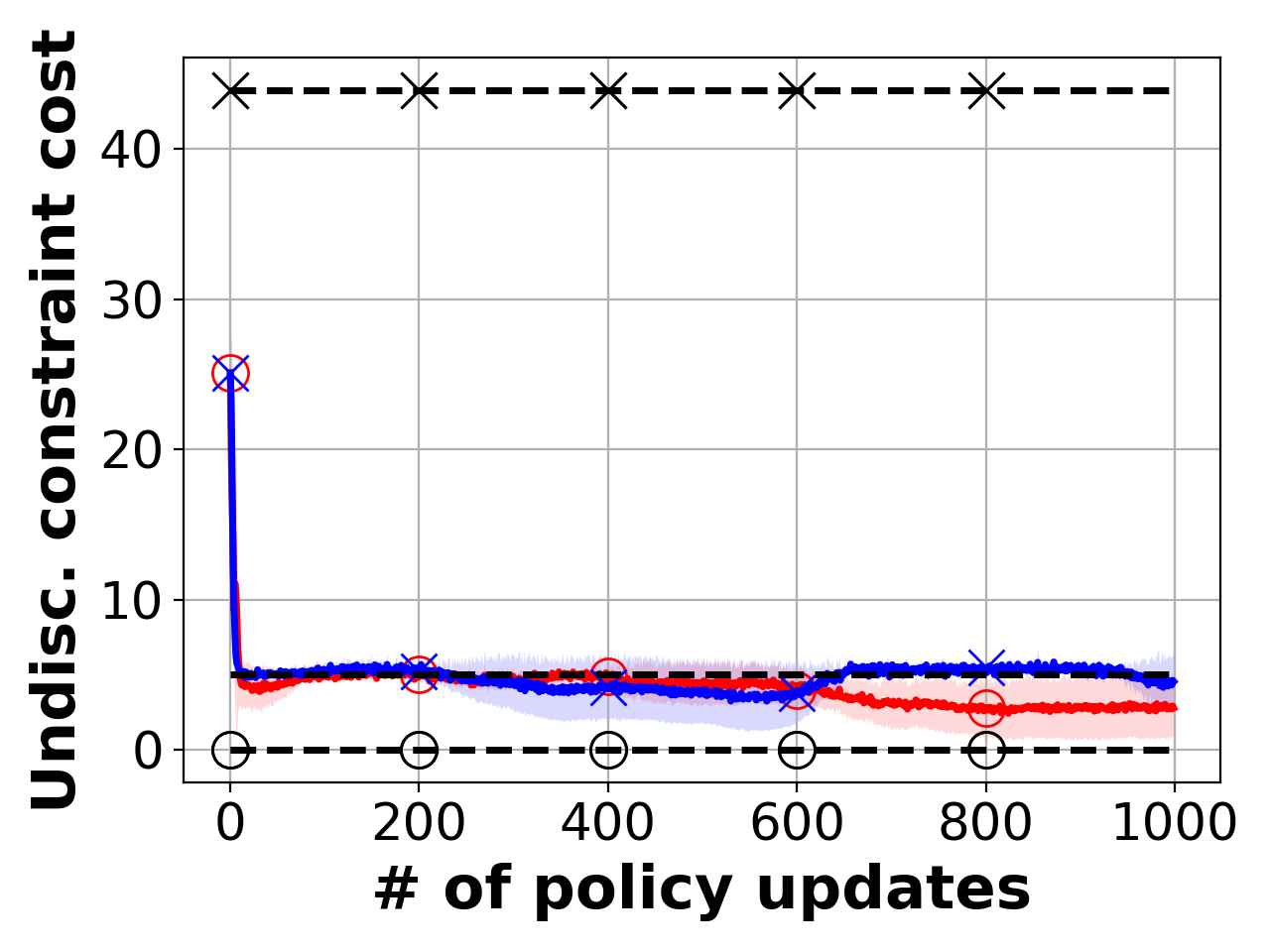}%
\includegraphics[width=0.33\linewidth]{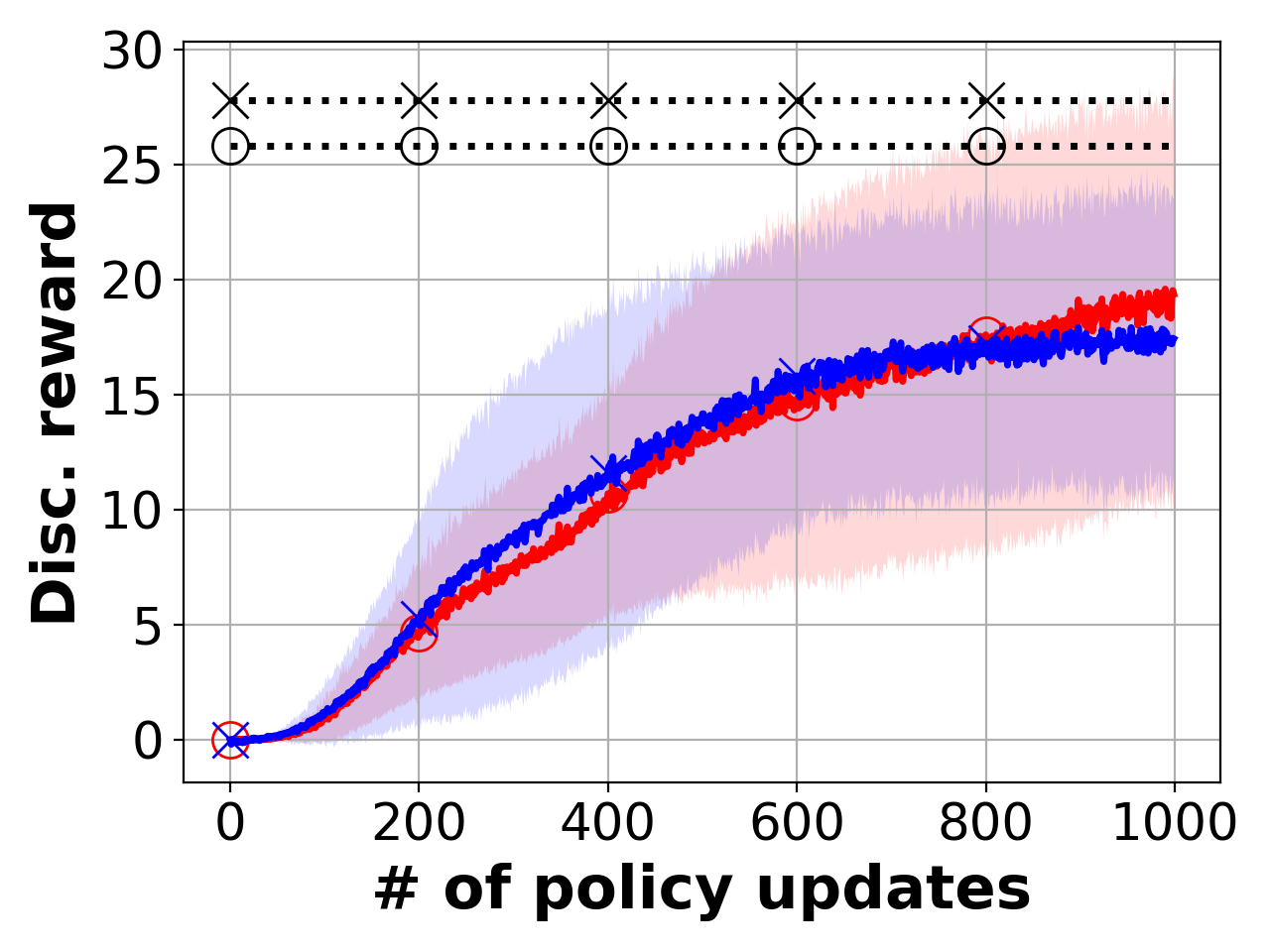}%
\includegraphics[width=0.33\linewidth]{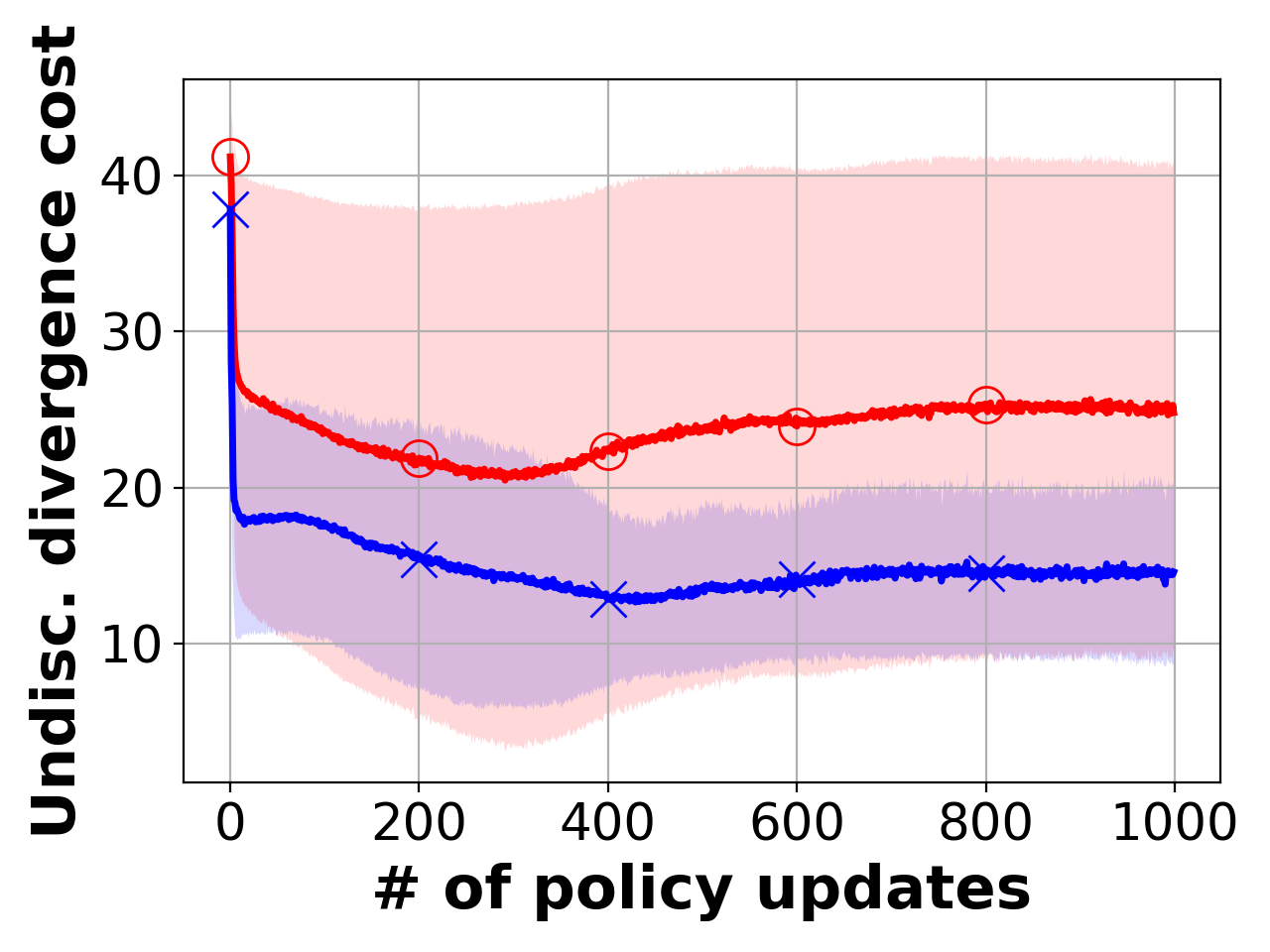}%
\end{tabular}}%

\subfloat[Ant gather\label{subfig:bn}]{\begin{tabular}[b]{@{}c@{}}%
\includegraphics[width=0.33\linewidth]{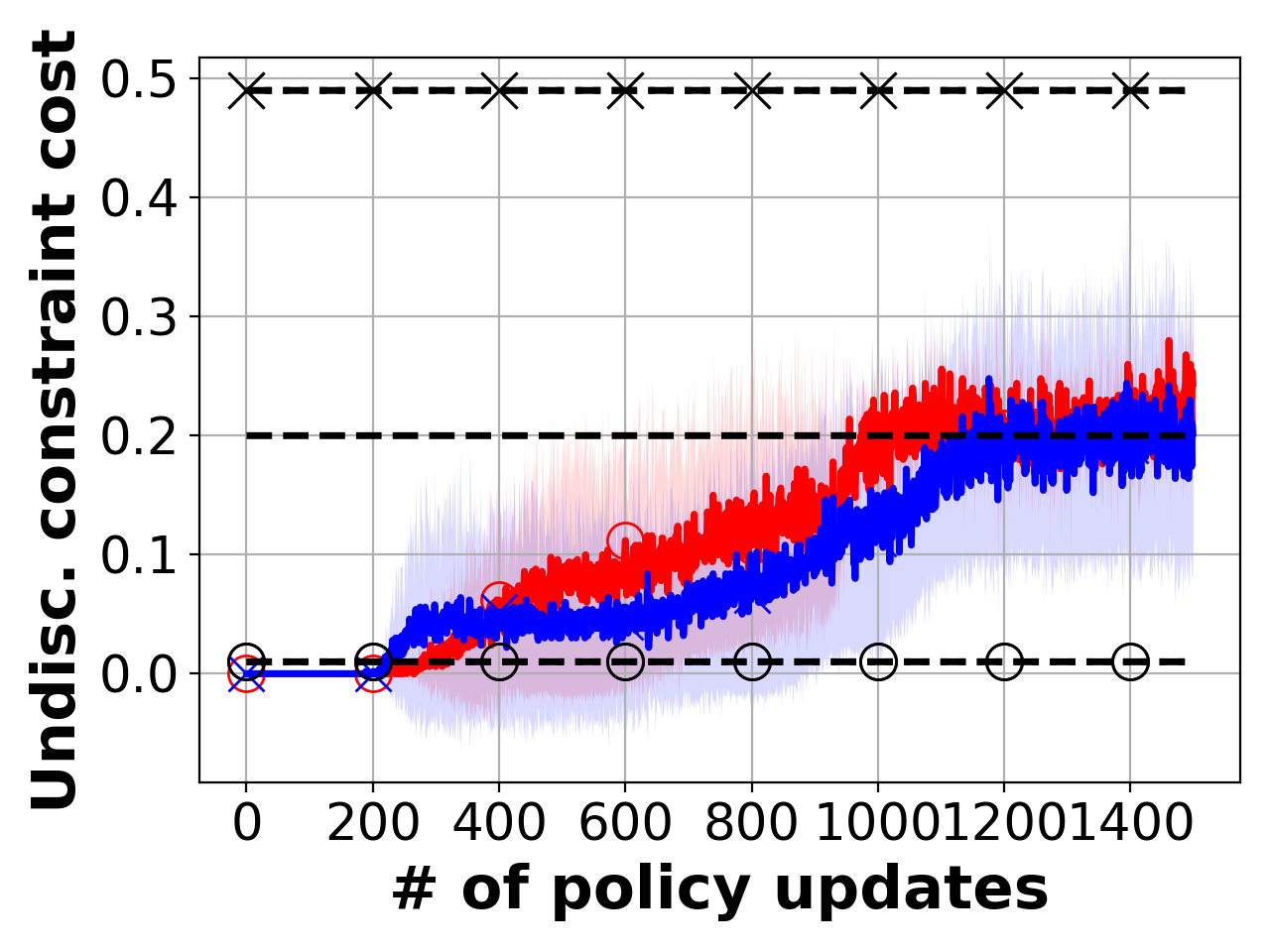}%
\includegraphics[width=0.33\linewidth]{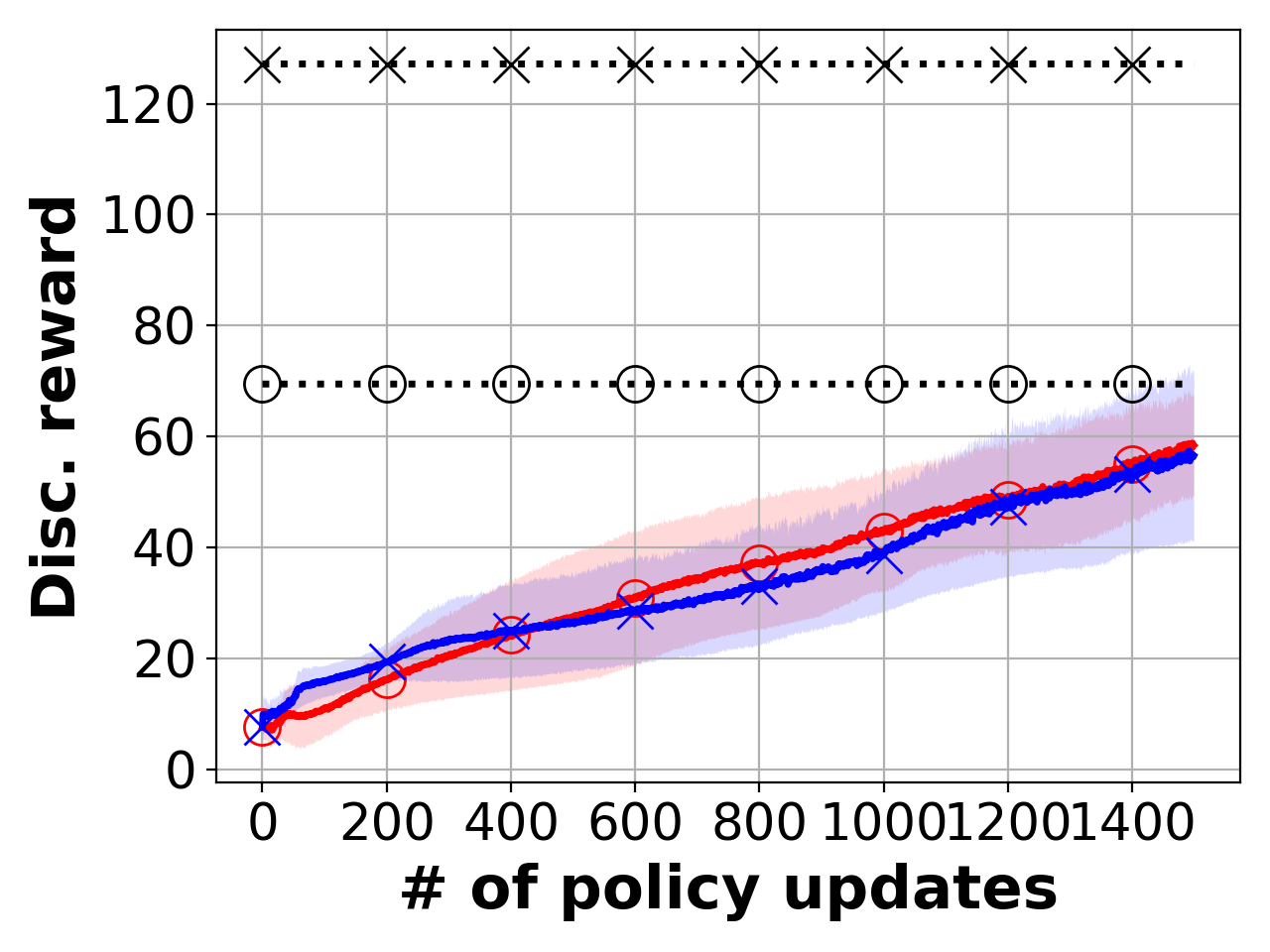}%
\includegraphics[width=0.33\linewidth]{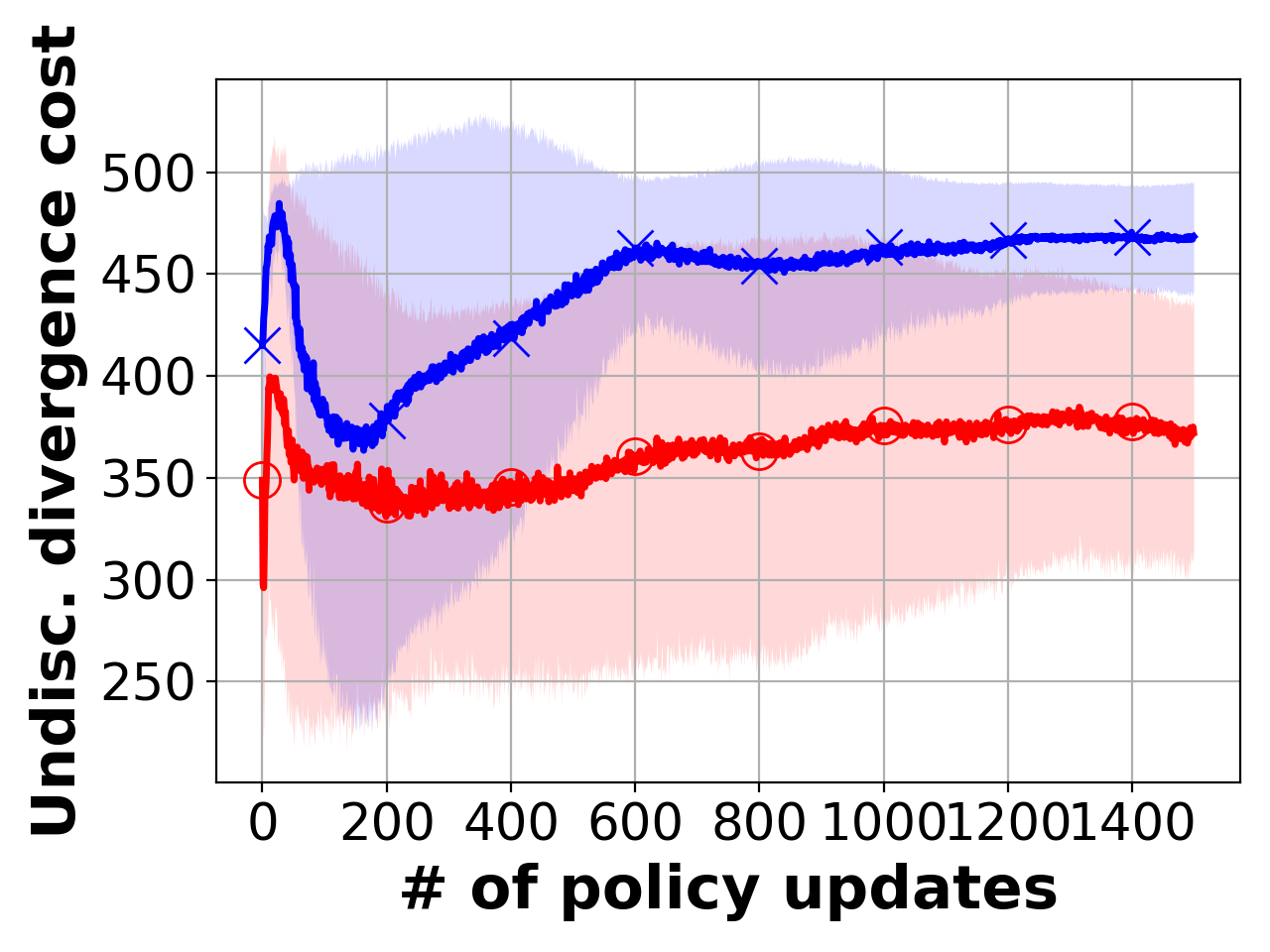}%
\end{tabular}}%

\subfloat[Ant circle 
%(Note that the y-axis of the reward and cost constraint plots is in log-scale. The reward values of two priors are roughly the same but with very different cost values.)
\label{subfig:cr}]{\begin{tabular}[b]{@{}c@{}}%
\includegraphics[width=0.33\linewidth]{figure/exp_2/NumCost_ac_overallPerformance_v2.png}%
\includegraphics[width=0.33\linewidth]{figure/exp_2/Reward_ac_overallPerformance_v2.png}%
\includegraphics[width=0.33\linewidth]{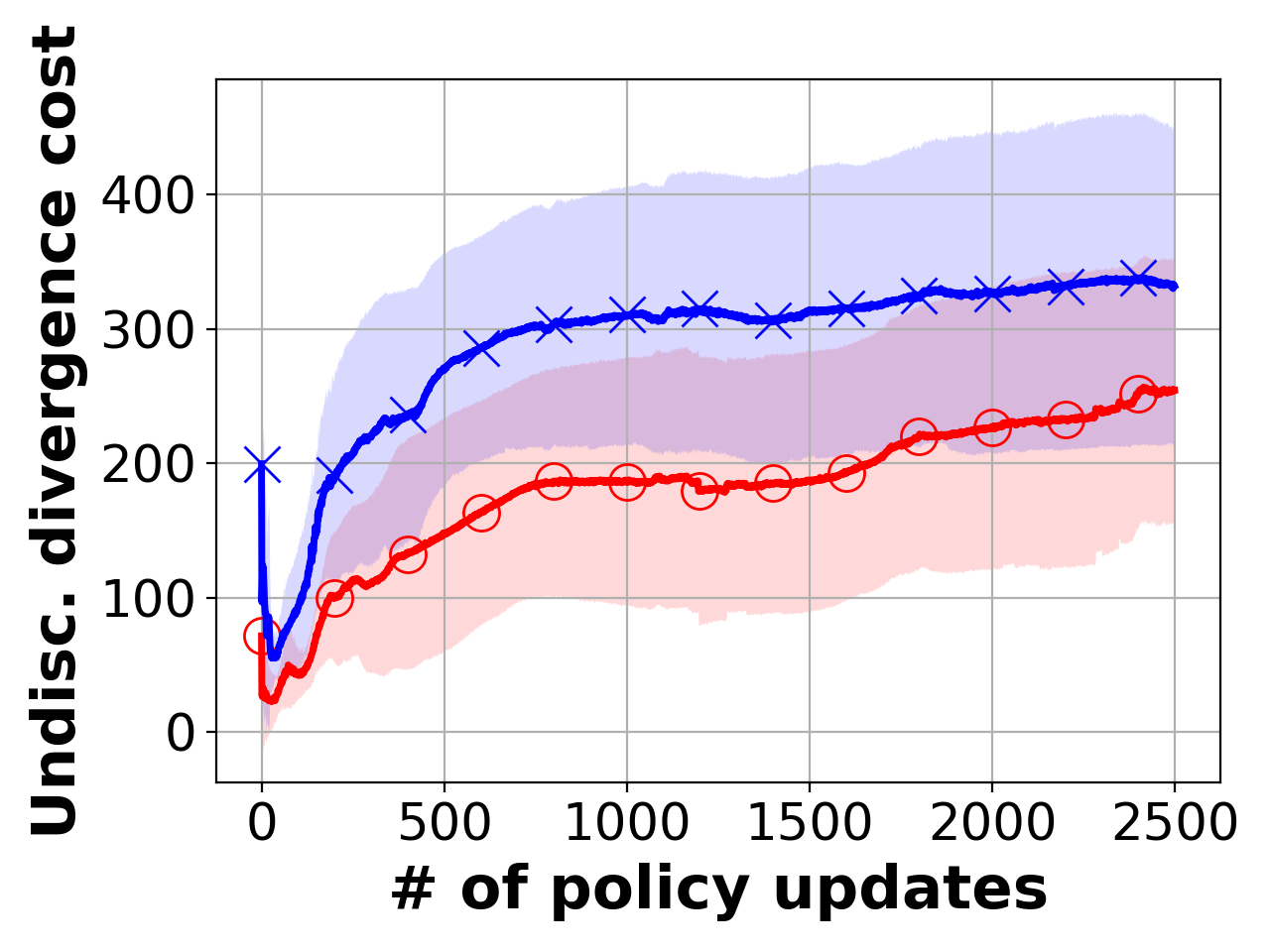}%
\end{tabular}}%

\vspace{+1mm}

\includegraphics[width=0.9\linewidth]{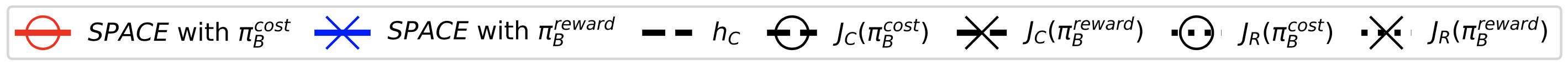}
\vspace{-2mm}

\caption{
The undiscounted constraint cost,
the discounted reward, and
the undiscounted divergence cost
over policy updates for the tested algorithms and tasks.
The solid line is the mean and the shaded area is the standard deviation over 5 runs.
\algname\ ensures cost constraint satisfaction guided by the baseline policy which need not satisfy the cost constraint.
%
%The baseline policies in the gird and bottleneck tasks are $\pi_B^\mathrm{same},$
%
%and the baseline policy in the car-racing task is $\pi_B^\mathrm{human}.$
%
%The values of the discounted reward, the undiscounted cost constraint value, and the undiscounted prior constraint value under two different priors over policy updates for the tested algorithms and task pairs. 
%
%The solid line is the mean and the shaded area is the standard deviation over five runs. 
%
%The dashed lines in the reward plot are the rewards of $\pi_B^\mathrm{cost}$ (lower) and $\pi_B^\mathrm{reward}$ (upper).
%
%The dashed lines in the cost constraint plot are the cost constraint value of the $\pi_B^\mathrm{cost}$ (lower) and $\pi_B^\mathrm{reward}$ (upper).
%
%The middle dashed line in the cost constraint plot is the cost constraint threshold $h_C$ of the agent.
%
%We observe that \algname\ with a safe prior $\pi^\mathrm{cost}_B$ achieves better reward performance compared to the one with an aggressive prior $\pi^\mathrm{reward}_B.$
%
(Best viewed in color.)
}
\label{fig:appendix_safeAggressivePriors}
\vspace{-5mm}
%\end{mdframed}
\end{figure*}

\paragraph{Comparison of Baseline Policies (see Fig.~\ref{fig:appendix_safeAggressivePriors}).}
To examine whether \algname\ can safely learn from the baseline policy which need not satisfy the cost constraint,
we consider two baseline policies: $\pi^\mathrm{cost}_B$ and $\pi^\mathrm{reward}_B.$
The learning curves of the undiscounted constraint cost, the discounted reward, and the undiscounted divergence cost with two possible baselines over policy updates are shown for all tested algorithms and tasks in Fig. \ref{fig:appendix_safeAggressivePriors}.
We observe that in the point gather and point circle tasks, the initial values of the cost are larger than $h_C$ (\ie $J_C(\pi^0)>h_C$).
%
%This implies that the initial policies are outside the cost constraint sets.
%
Using $\pi^\mathrm{cost}_B$ allows the learning algorithm to quickly satisfy the cost without doing the extensive projection onto the cost constraint set.
For example, in the point circle task we observe that learning guided by $\pi^\mathrm{cost}_B$ quickly satisfies the cost constraint.
In addition, we observe that in the ant gather and ant circle tasks, the initial values of the cost are smaller than $h_C$ (\ie $J_C(\pi^0)<h_C$).
%
%This implies that the initial policies are inside the cost constraint sets.
%
Intuitively, we would expect that using $\pi^\mathrm{reward}_B$ allows the agent to quickly improve the reward since the agent already satisfies the cost constraint in the beginning.
In the ant gather task we observe that \algname\ guided by $\pi^\mathrm{reward}_B$ does improve the reward more quickly at around 200 iteration.
However, we observe that the agent guided by the both baseline policies achieve the same final reward performance in the ant gather and ant circle tasks.
The reason is that using dynamic $h_D$ allows the agent to stay away from the baseline policy. 
This makes the baseline policy less influential in the end. 
As a result, the reward improvement mostly comes from the reward improvement step of \algname\ if the agent starts in the interior of the cost constraint set (\ie $J_C(\pi^0)\leq h_C$).

\paragraph{Fixed $h_D$ (see Fig. \ref{fig:appendix_priorConstraintThreshold}).} 
To understand the effect of using dynamic $h_D^k$ when learning from a sub-optimal baseline policy,
we compare the performance of \algname\ with and without adjusting $h_D.$ 
The learning curves of the undiscounted constraint cost, the discounted reward, and the undiscounted divergence cost over policy updates are shown for all tested algorithms and tasks in Fig. \ref{fig:appendix_priorConstraintThreshold}.
We observe that \algname\ with fixed $h_D$ converges to less reward.
For example, in the ant circle task \algname\ with the dynamic $h_D$ achieves 2.3 times more reward.
The value of the divergence cost in the ant circle task shows that staying away from the baseline policy achieves more reward.
This implies that the baseline policy in the ant circle task is highly sub-optimal to the agent.
In addition, we observe that in some tasks the dynamic $h_D$ does not have much effect on the reward performance.
For example, in the point gather task \algname\ achieves the same reward performance. 
The values of the divergence cost in the point gather task decrease throughout the training.
These observations imply that the update scheme of $h_D$ is critical for some tasks.
%
%In contrast, the fixed-point and dynamic-point approaches do not provide this kind of freedom, and hence converge to less reward.

\begin{figure*}[t]
\vspace{-3mm}
\centering
\subfloat[Point gather\label{subfig:cr}]{\begin{tabular}[b]{@{}c@{}}%
\vspace{-3mm}
\includegraphics[width=0.33\linewidth]{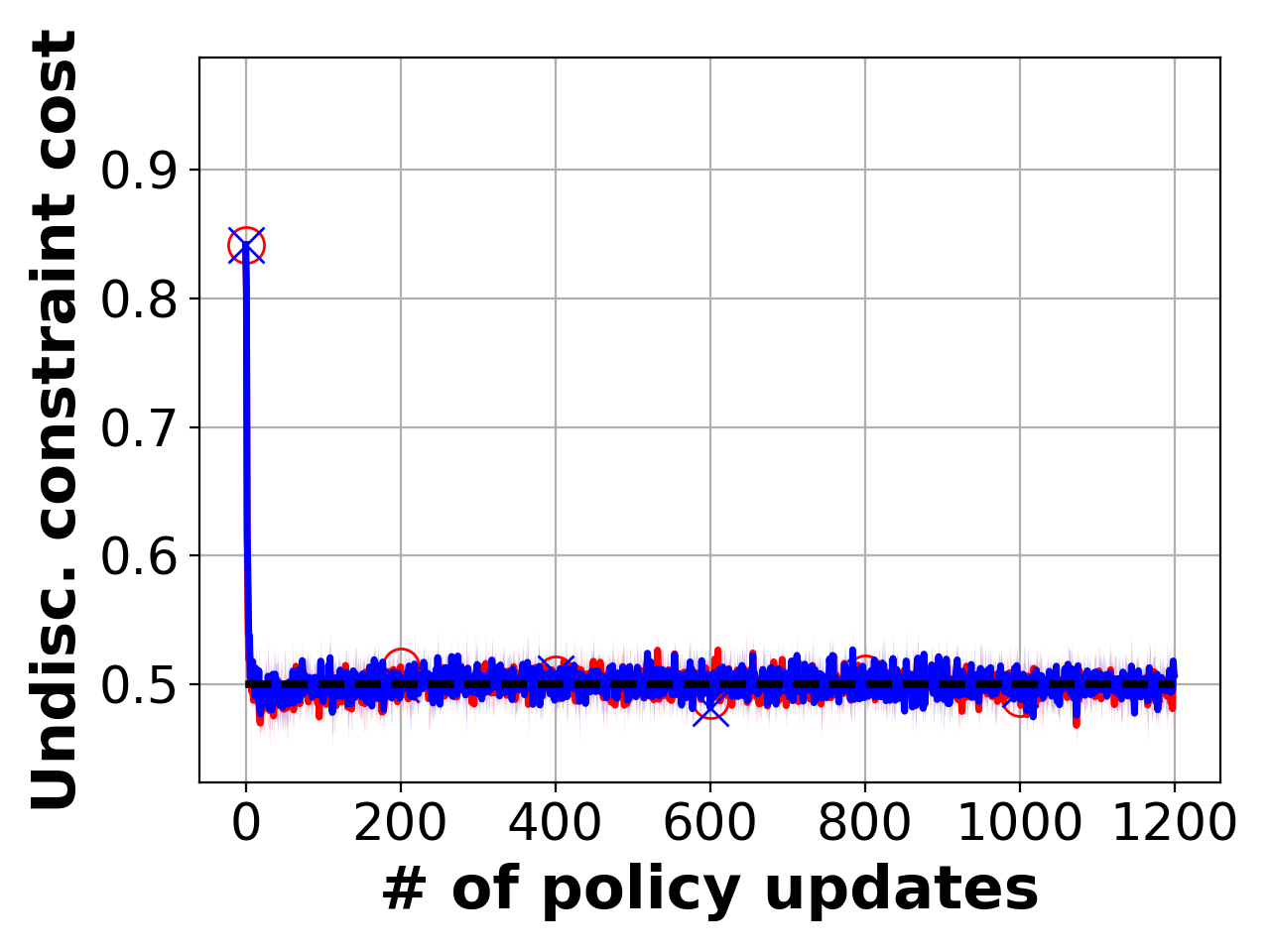}%
\includegraphics[width=0.33\linewidth]{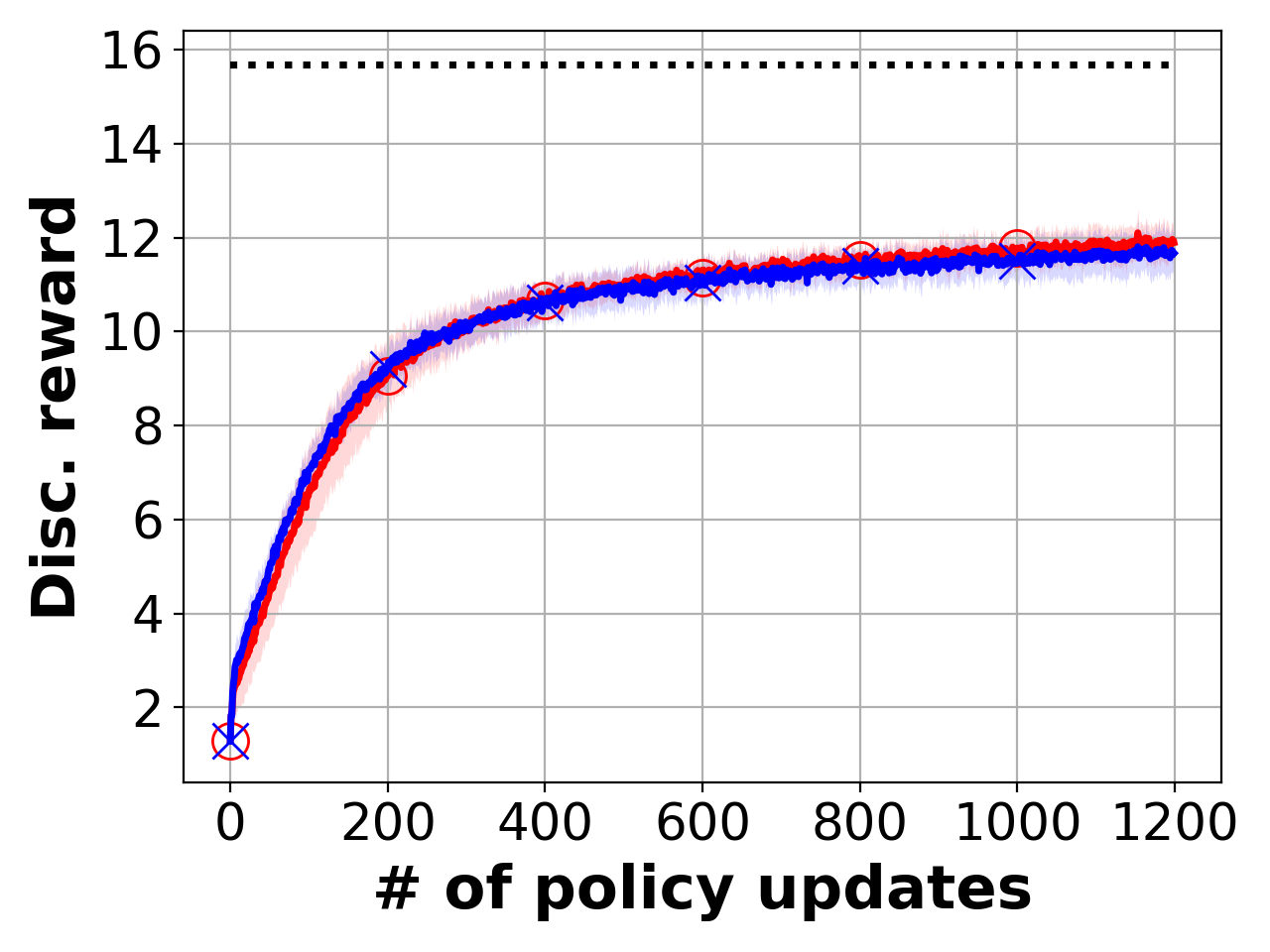}%
\includegraphics[width=0.33\linewidth]{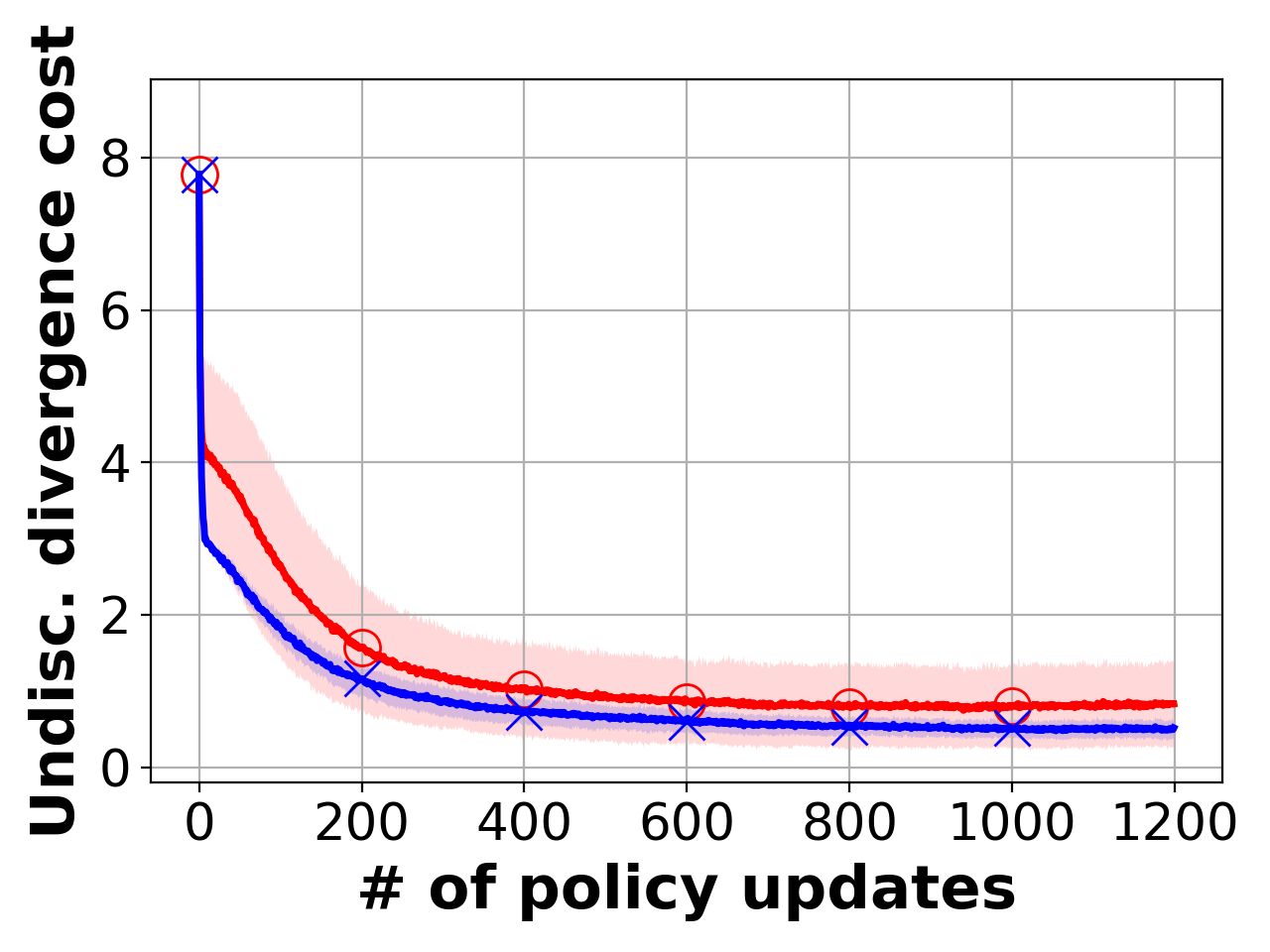}%
\end{tabular}}%
\vspace{-3mm}

\subfloat[Point circle\label{subfig:cr}]{\begin{tabular}[b]{@{}c@{}}%
\vspace{-3mm}
\includegraphics[width=0.33\linewidth]{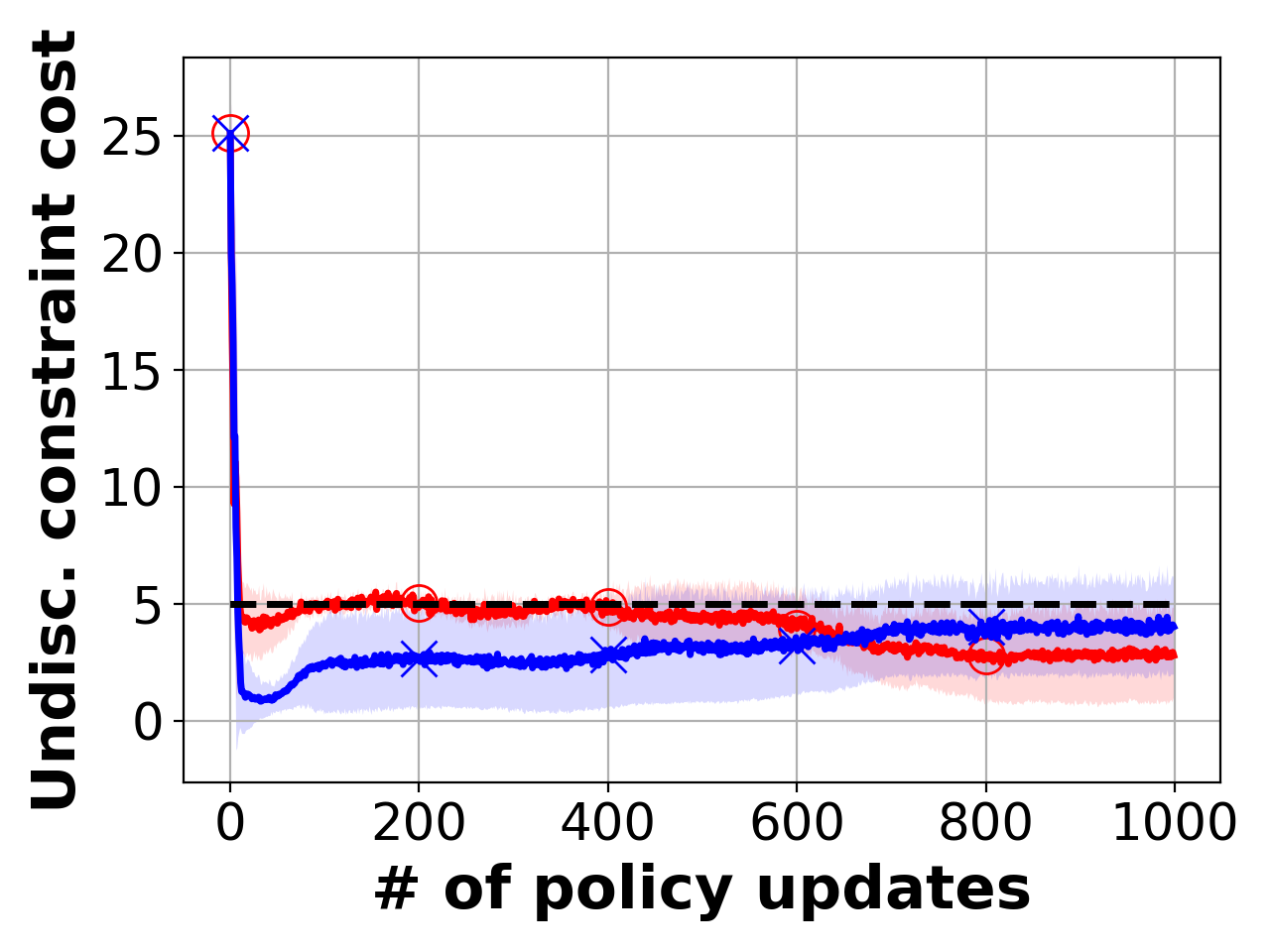}%
\includegraphics[width=0.33\linewidth]{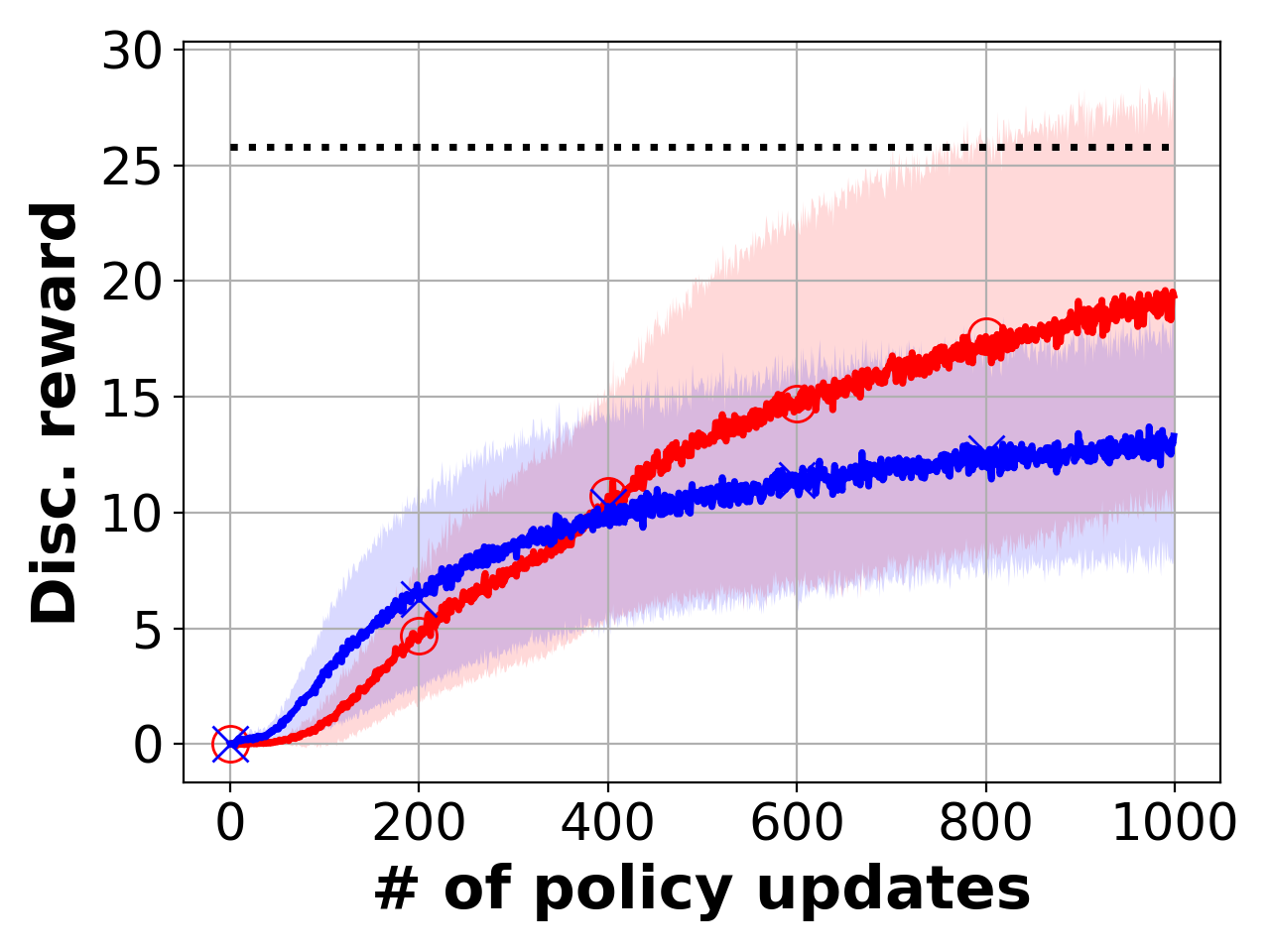}%
\includegraphics[width=0.33\linewidth]{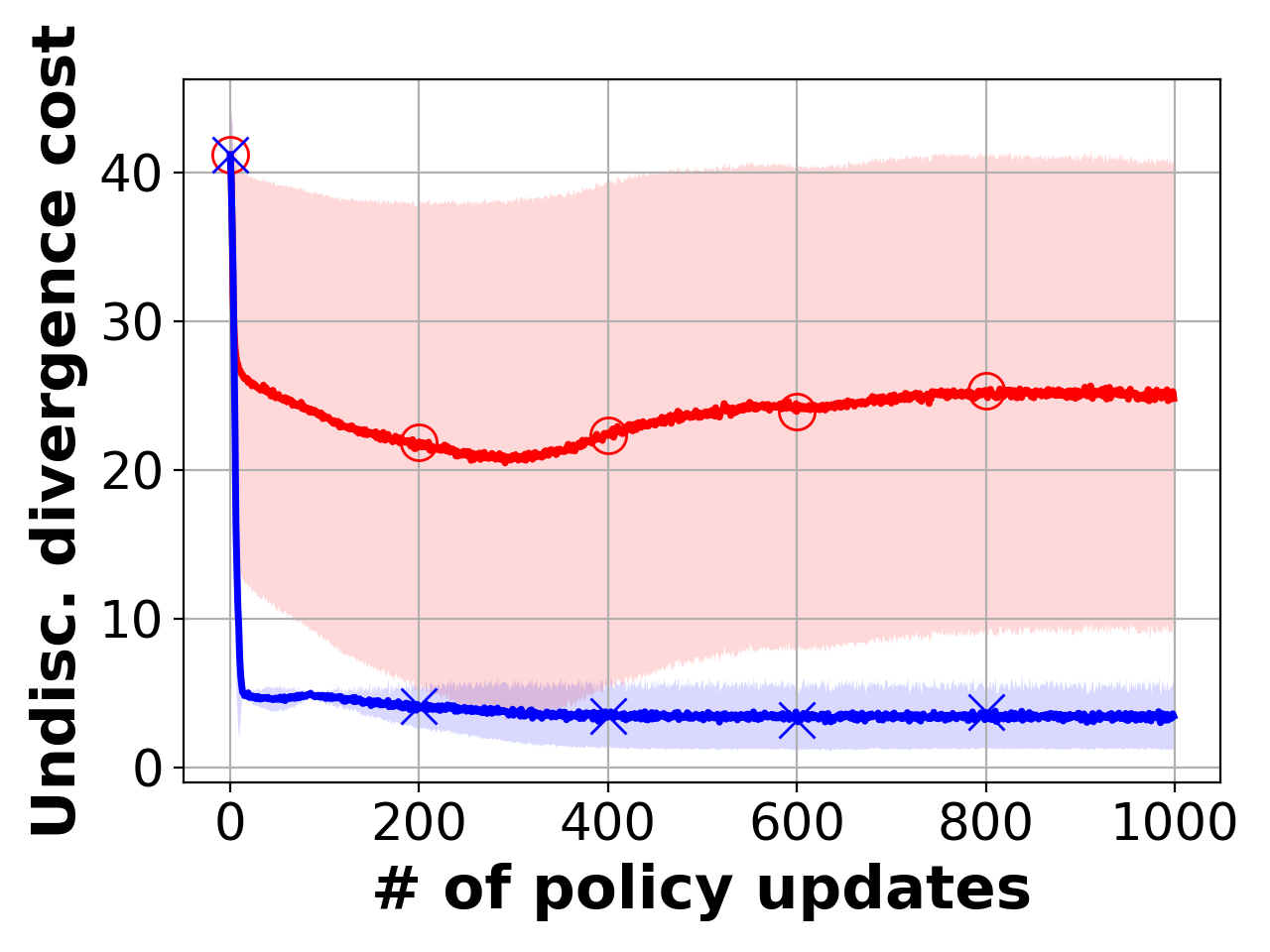}%
\end{tabular}}%
\vspace{-3mm}

\subfloat[Ant gather\label{subfig:cr}]{\begin{tabular}[b]{@{}c@{}}%
\vspace{-3mm}
\includegraphics[width=0.33\linewidth]{figure/exp_2/NumCost_ag_overallPerformance_v2_fixed_dynamic_hp.png}%
\includegraphics[width=0.33\linewidth]{figure/exp_2/Reward_ag_overallPerformance_v2_fixed_dynamic_hp.png}%
\includegraphics[width=0.33\linewidth]{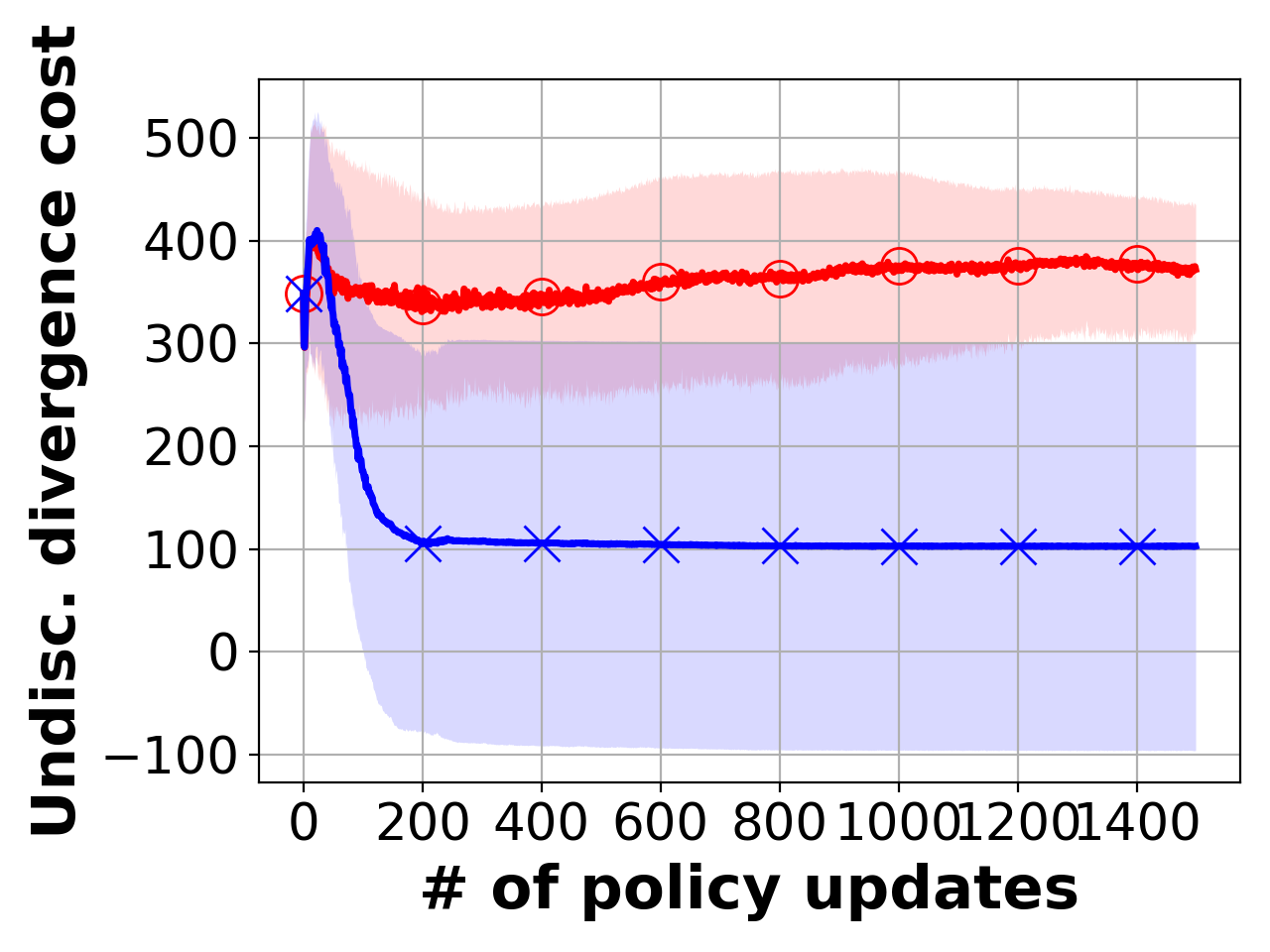}%
\end{tabular}}%
\vspace{-3mm}

\subfloat[Ant circle\label{subfig:cr}]{
\begin{tabular}[b]{@{}c@{}}%
\vspace{-3mm}
\includegraphics[width=0.33\linewidth]{figure/exp_2/NumCost_ac_overallPerformance_v2_fixed_dynamic_hp.png}%
\includegraphics[width=0.33\linewidth]{figure/exp_2/Reward_ac_overallPerformance_v2_fixed_dynamic_hp.png}%
\includegraphics[width=0.33\linewidth]{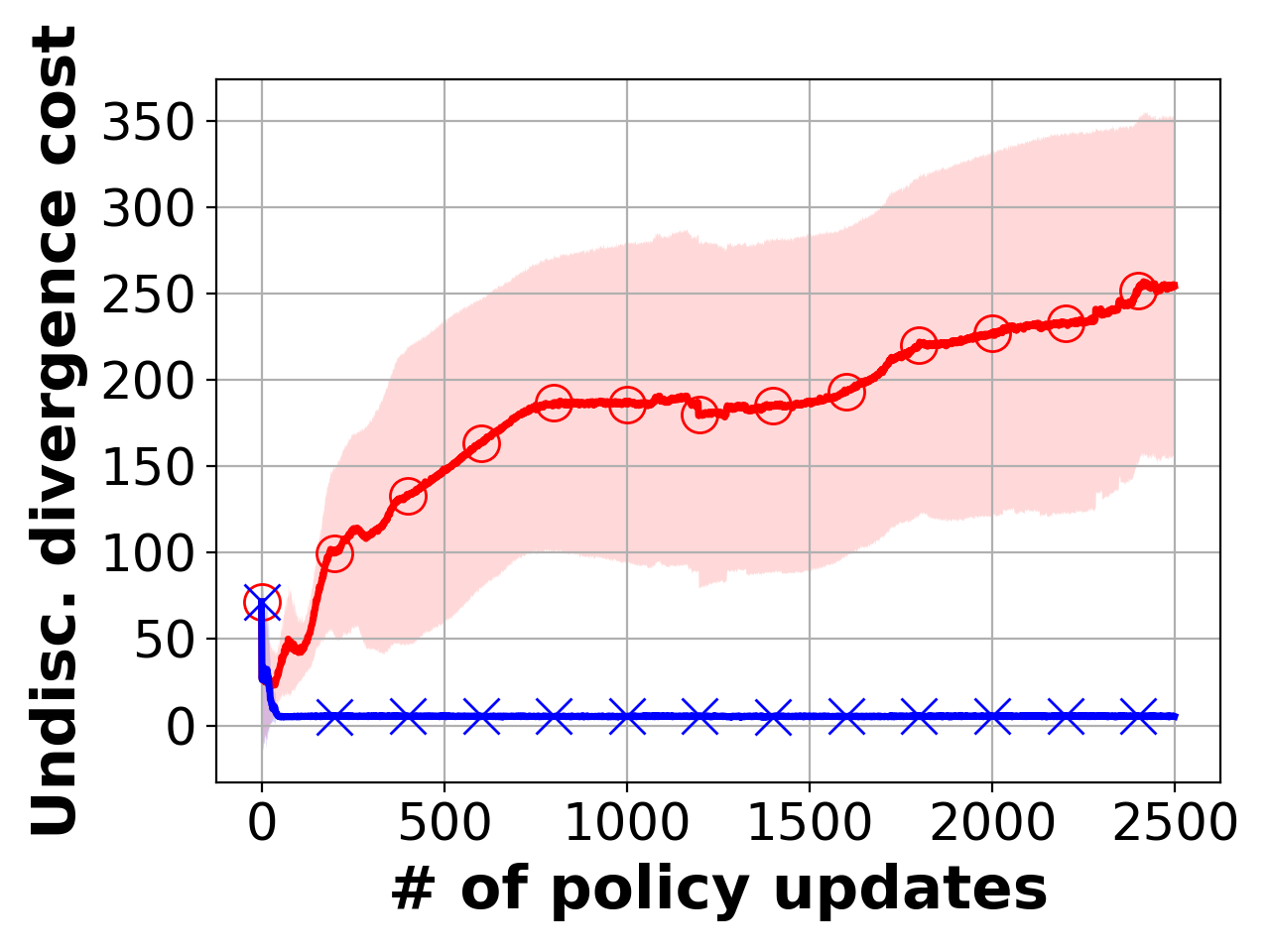}%
\end{tabular}}%
\vspace{-3mm}

\subfloat[Car-racing\label{subfig:cr}]{\begin{tabular}[b]{@{}c@{}}%
\vspace{-3mm}
\includegraphics[width=0.33\linewidth]{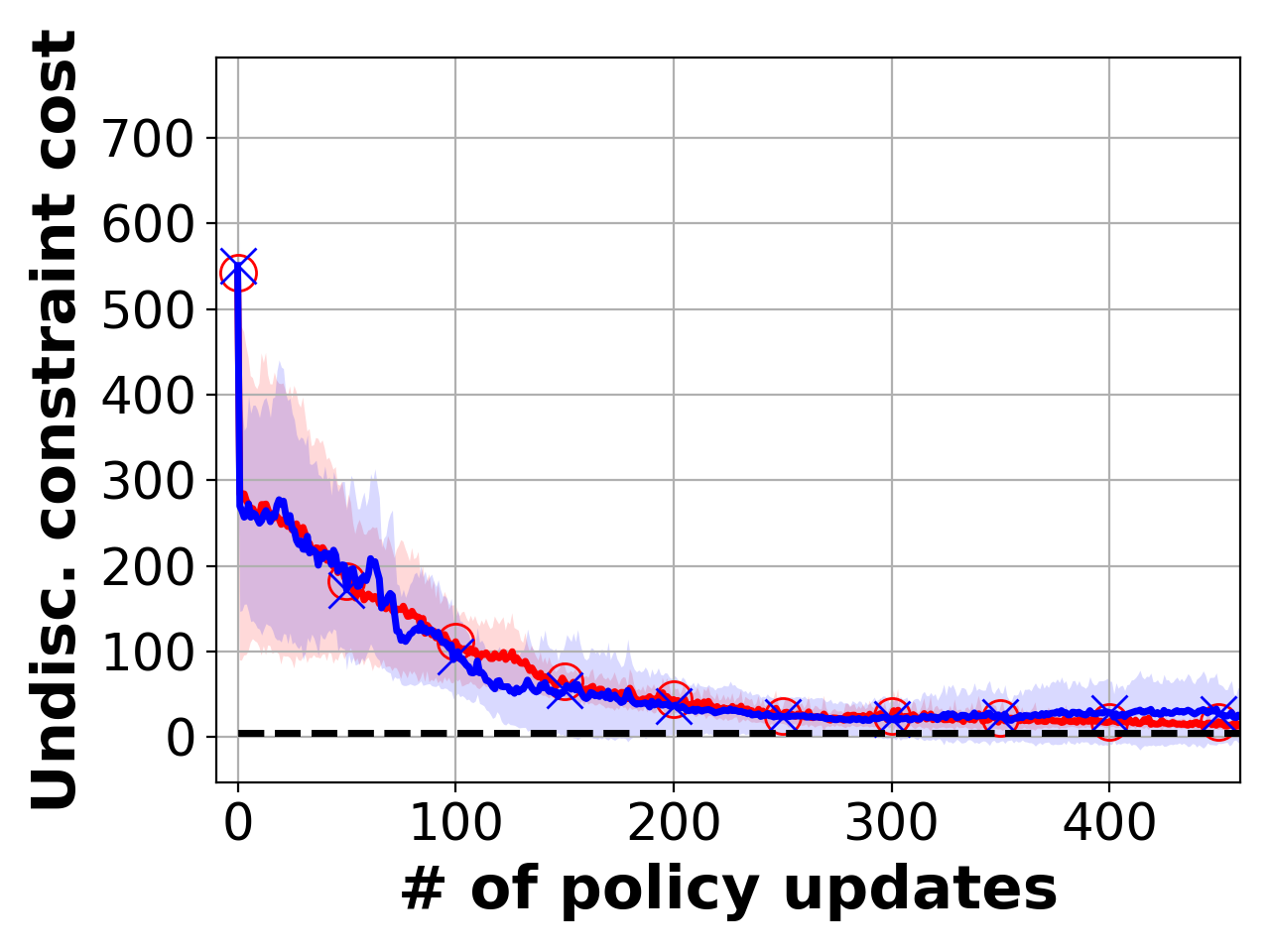}%
\includegraphics[width=0.33\linewidth]{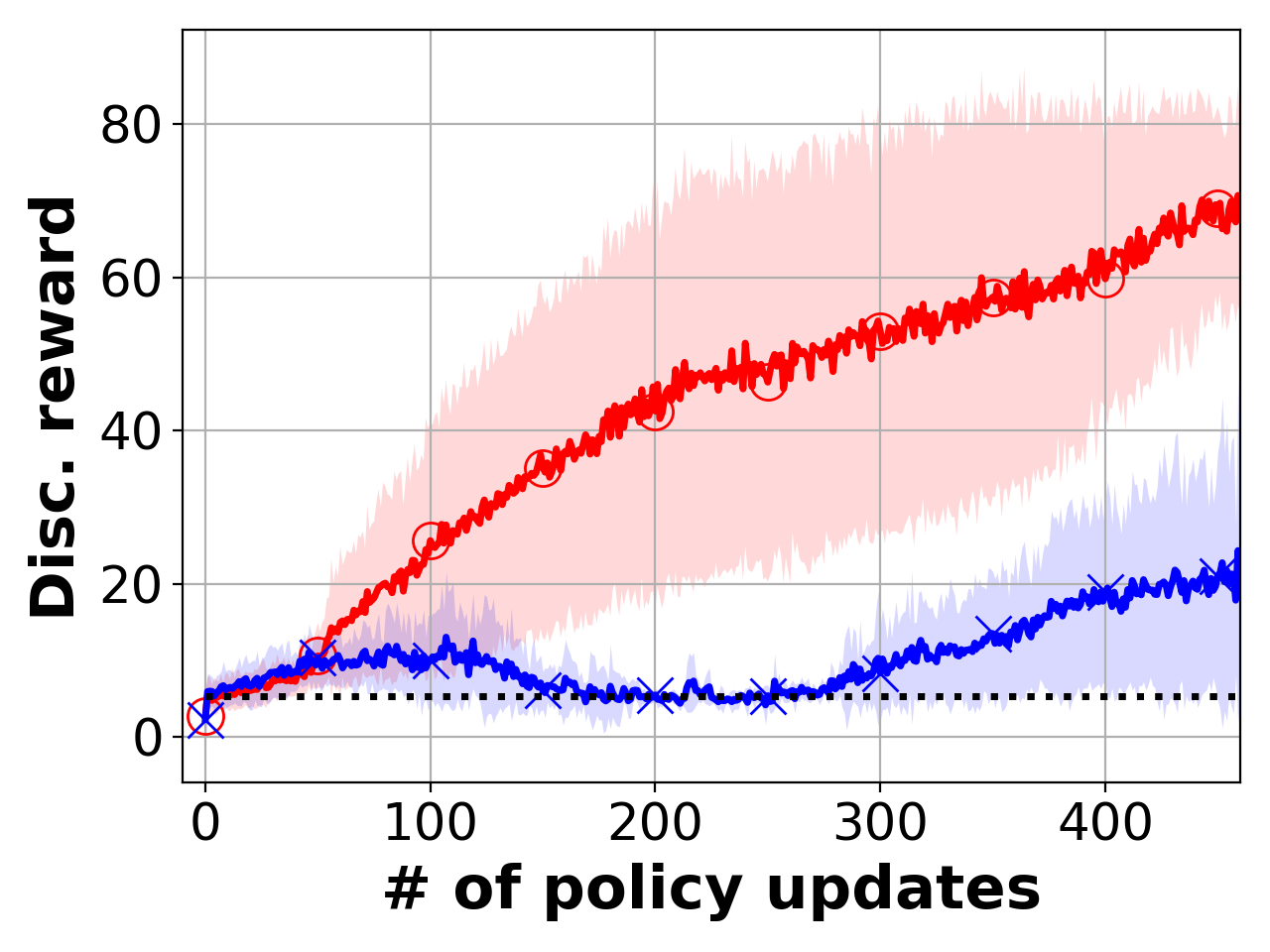}%
\includegraphics[width=0.33\linewidth]{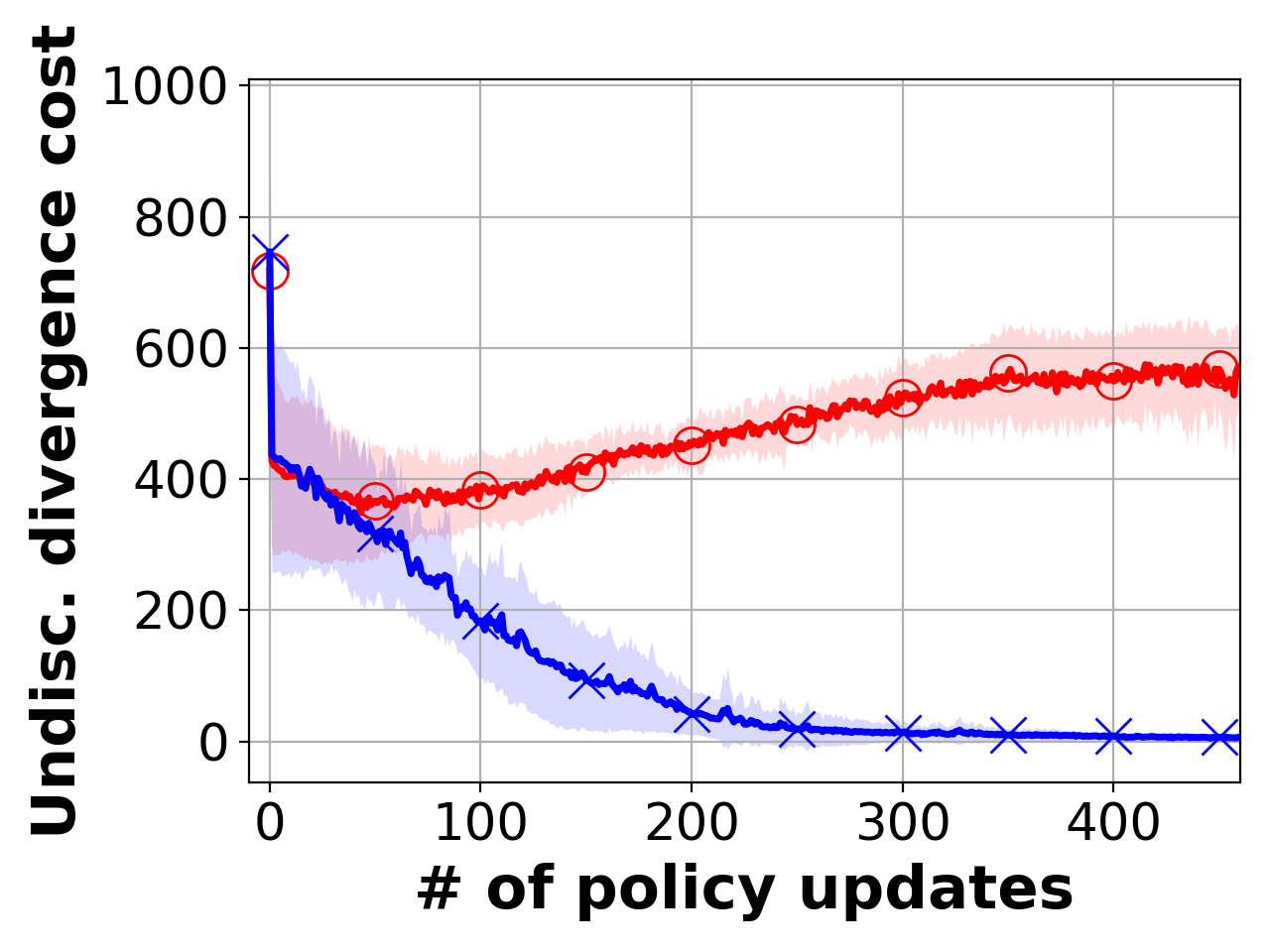}%
\end{tabular}}%

\vspace{-1mm}

\includegraphics[width=0.55\linewidth]{figure/exp_2/legend_fixhD.png}
\vspace{-4mm}

\caption{
%The values of the discounted reward, the undiscounted constraint value, and the undiscounted prior constraint value over policy updates with and without the relaxation step for the tested algorithms and task pairs.
The undiscounted constraint cost,
the discounted reward, and
the undiscounted divergence cost
over policy updates for the tested algorithms and tasks.
The solid line is the mean and the shaded area is the standard deviation over 5 runs.
\algname\ with the dynamic $h_D$ achieves higher reward.
(Best viewed in color.)
}
\label{fig:appendix_priorConstraintThreshold}
\vspace{-3mm}
%\end{mdframed}
\end{figure*}

\paragraph{Comparison of \algname\ vs. d-CPO, d-PCPO and the Pre-training Approach (see Fig.~\ref{fig:appendix_pretrainingPrior}).} 
To show that \algname\ is effective in using the supervision of the baseline policy, we compare the performance of \algname\ to the dynamic-point and the pre-training approaches.
In the pre-training approach, the agent first performs the trust region update with the objective function being the divergence cost.
Once the agent has the same reward performance as the baseline policy (\ie $J_R(\pi^k)\approx J_R(\pi_B)$ for some $k$), the agent performs the trust region update with the objective function being the reward function.
The learning curves of the undiscounted constraint cost, the discounted reward, and the undiscounted divergence cost over policy updates are shown for all tested algorithms and tasks in Fig. \ref{fig:appendix_pretrainingPrior}.
We observe that \algname\ achieves better reward performance compared to the pre-training approach in all tasks.
For example, in the point circle, ant gather and ant circle tasks the pre-training approach seldom improves the reward but all satisfies the cost constraint.
This implies that the baseline policies in these tasks are highly sub-optimal in terms of reward performance.
In contrast, \algname\ prevents the agent from converging to a poor policy.
In addition, we observe that in the point gather task the pre-training approach achieves the same reward performance as the baseline policy, whereas \algname\ has a better reward performance compared to the baseline policy.
The pre-training approach does not keep improving the reward after learning from the baseline policy.
This is because that after pre-training with the baseline policy, the entropy of the learned policy is small.
This prevents the agent from trying new actions which may lead to better reward performance.
This implies that pre-training approach may hinder the exploration of the learning agent on the new environment.
Furthermore, in the car-racing task we observe that using pre-training approach achieves the same reward performance as \algname\ but improves reward slowly, and the pre-training approach has more cost constraint violations than \algname.
This implies that jointly using reinforcement learning and the supervision of the baseline policy achieve better reward and cost performance.
For d-CPO and d-PCPO, in the point and ant tasks we observe that both approaches have comparable or silently better reward and cost performance compared to \algname.
However, in the car-racing task we observe that d-CPO cannot improve the reward due to a slow update procedure for satisfying the cost constraint, whereas d-PCPO has a better reward performance.
These observations imply that the projection steps in \algname\ allow the learning agent to effectively and robustly learn from the baseline policy.

\begin{figure*}[t]
\vspace{-3mm}
\centering
\subfloat[Point gather\label{subfig:cr}]{\begin{tabular}[b]{@{}c@{}}%
\vspace{-3mm}
\includegraphics[width=0.33\linewidth]{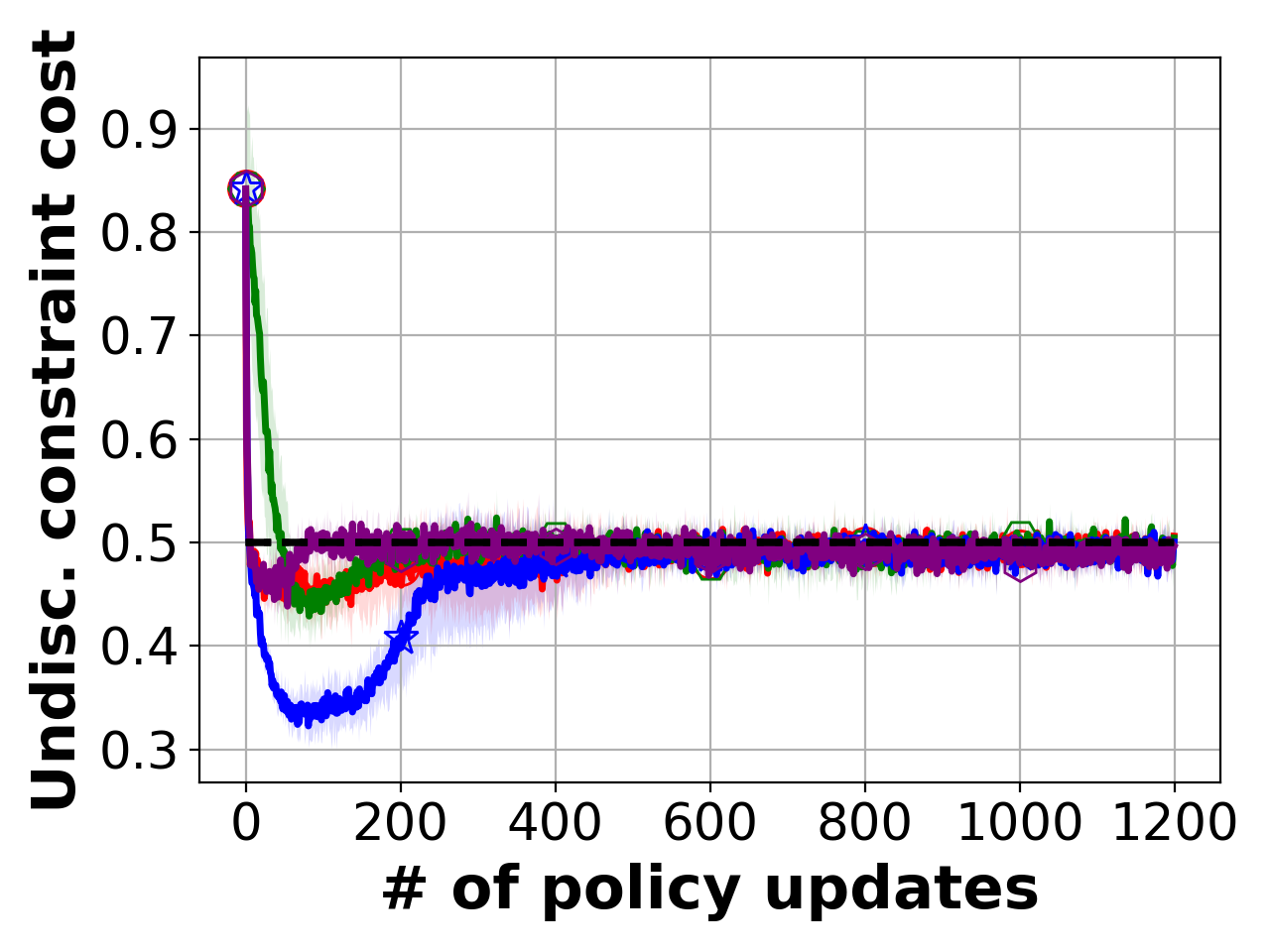}%
\includegraphics[width=0.33\linewidth]{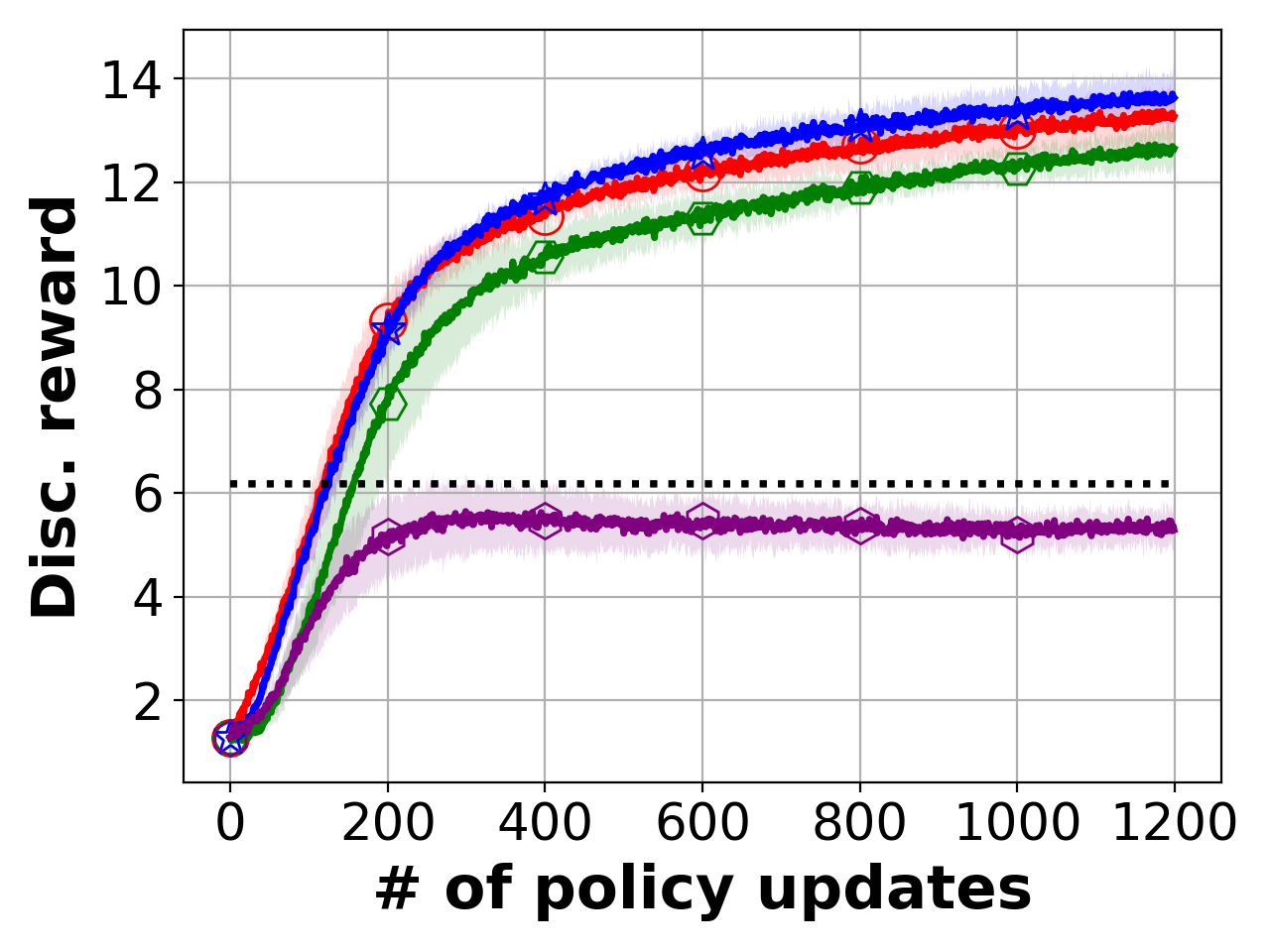}%
\includegraphics[width=0.33\linewidth]{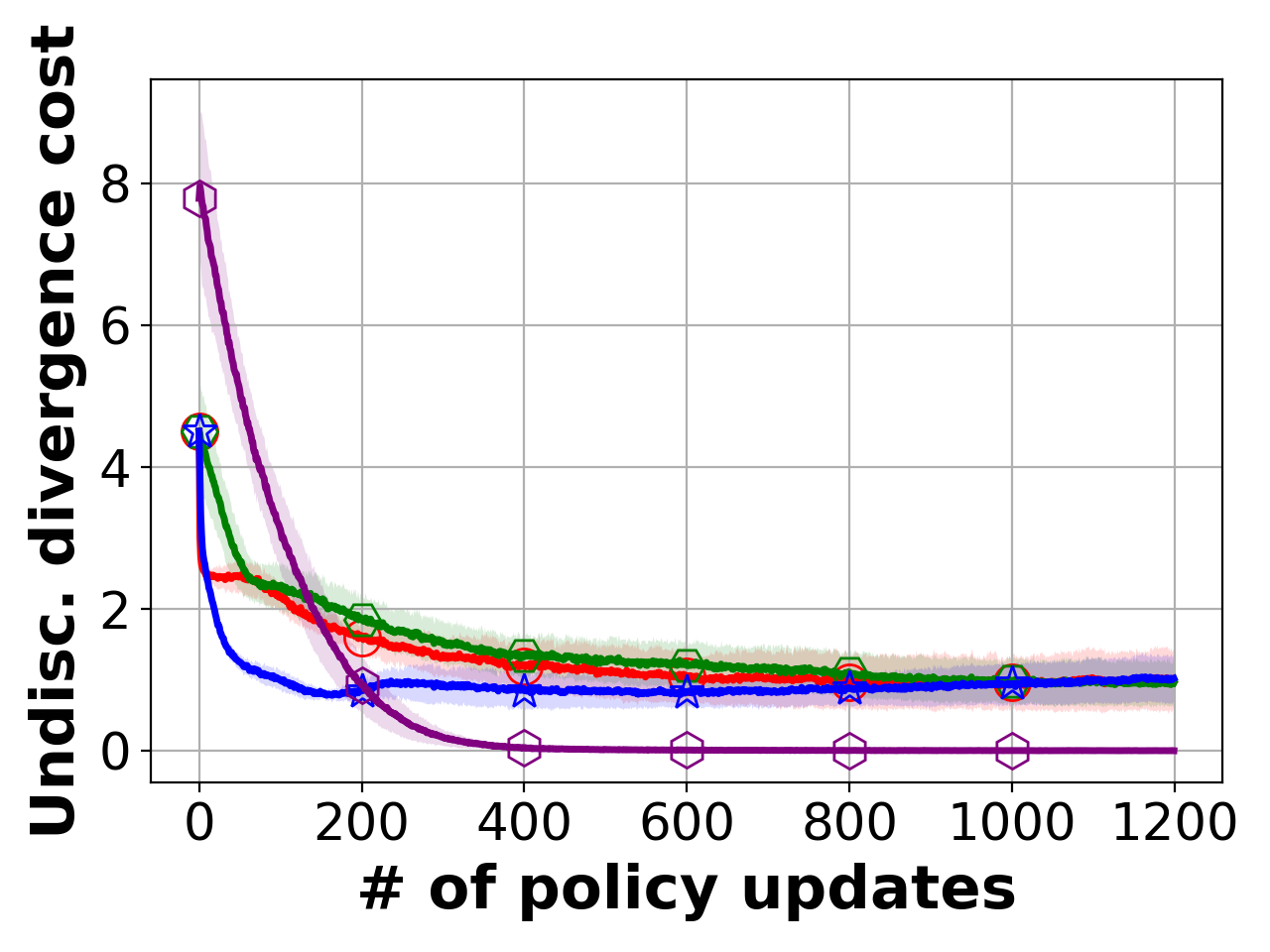}%
\end{tabular}}%
\vspace{-3mm}

\subfloat[Point circle\label{subfig:cr}]{\begin{tabular}[b]{@{}c@{}}%
\vspace{-3mm}
\includegraphics[width=0.33\linewidth]{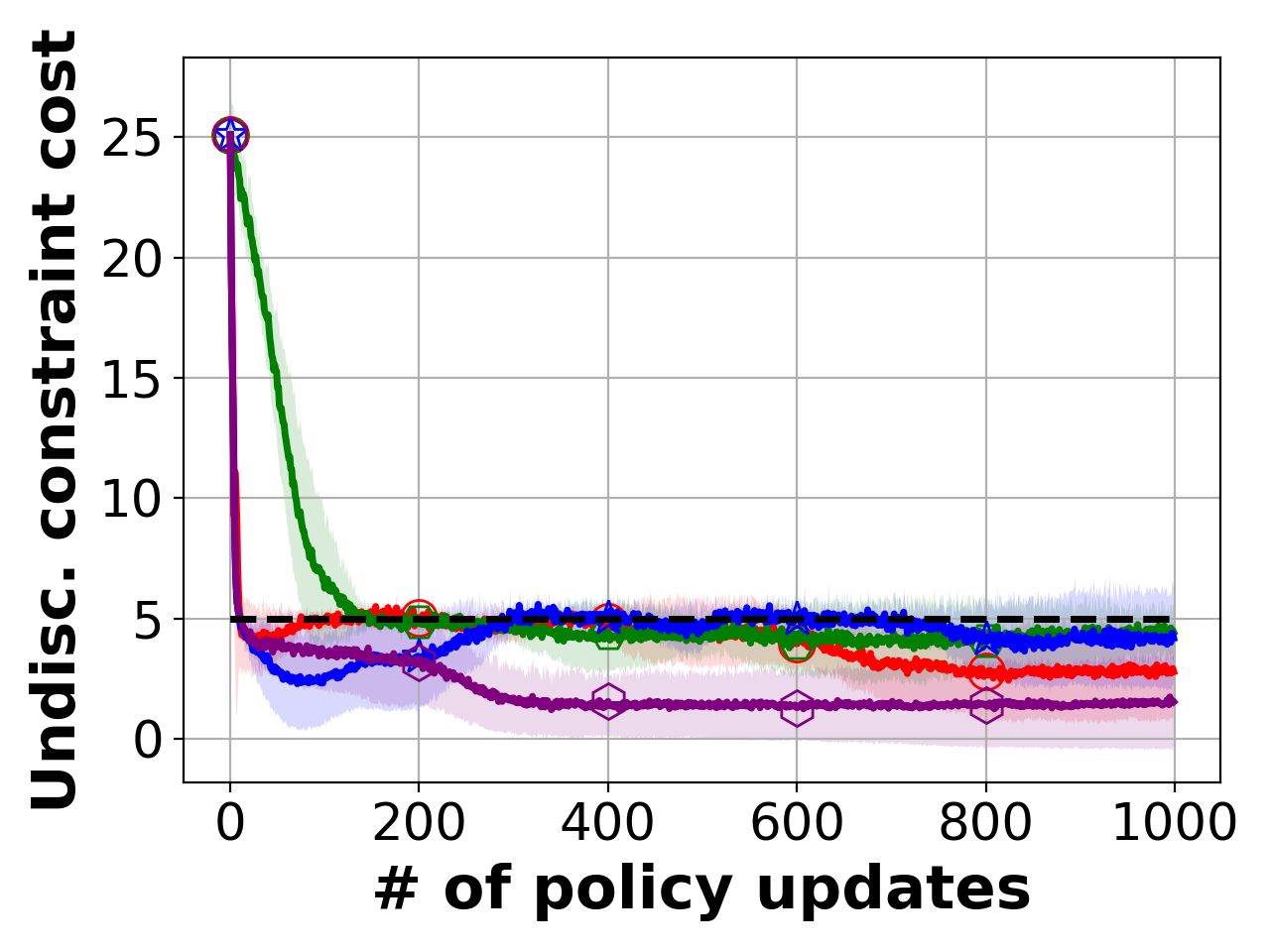}%
\includegraphics[width=0.33\linewidth]{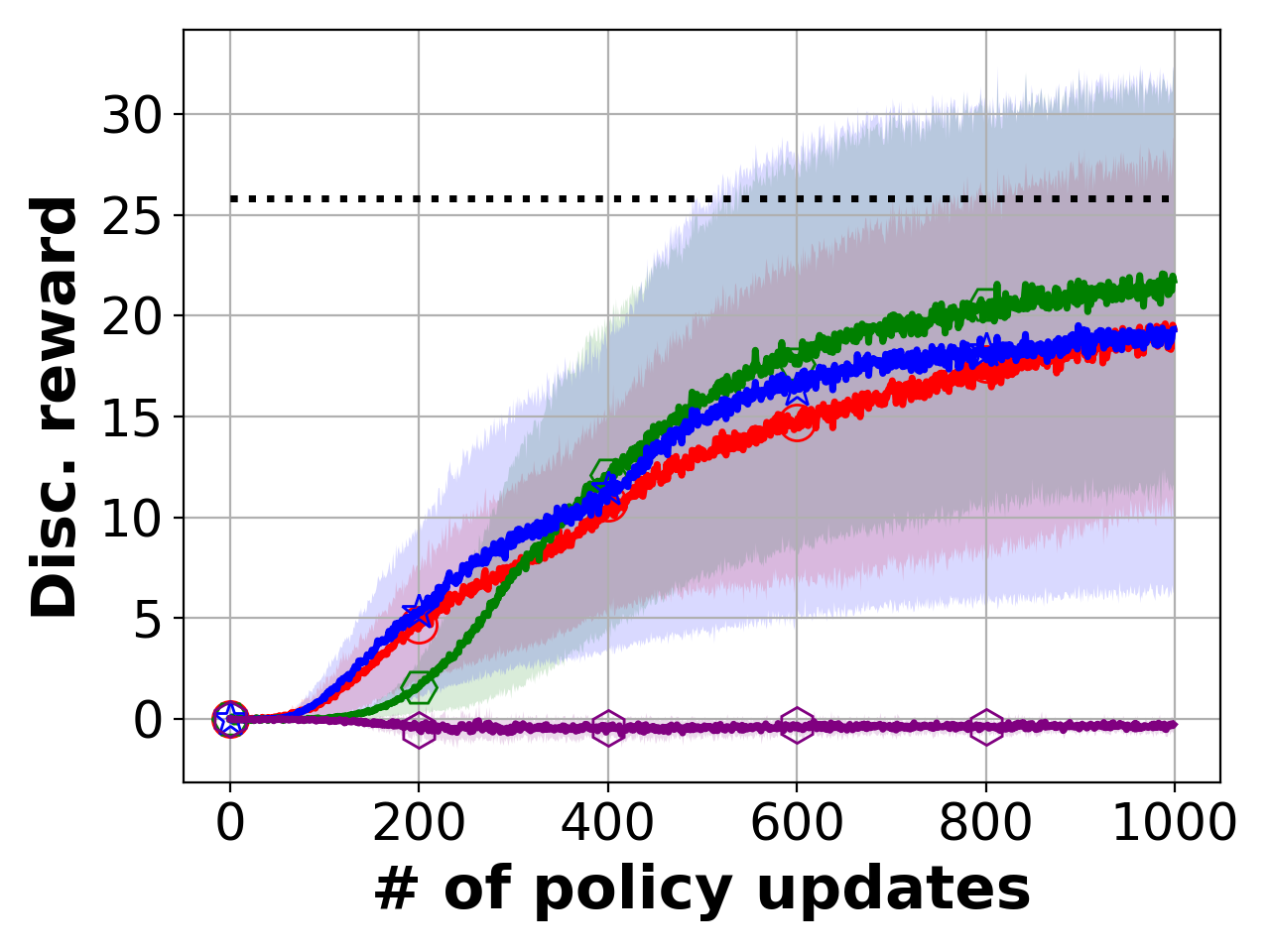}%
\includegraphics[width=0.33\linewidth]{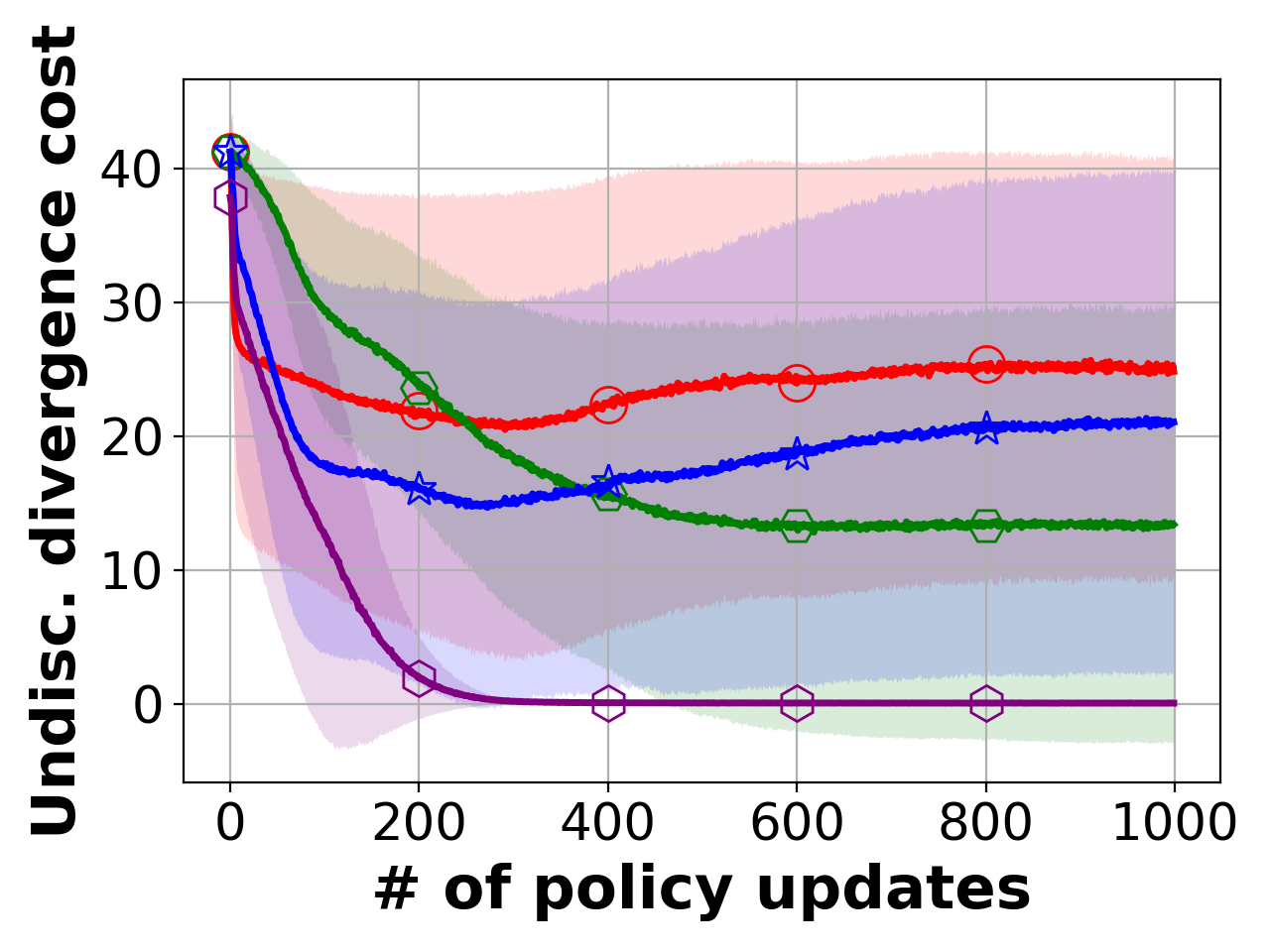}%
\end{tabular}}%
\vspace{-3mm}

\subfloat[Ant gather\label{subfig:cr}]{\begin{tabular}[b]{@{}c@{}}%
\vspace{-3mm}
\includegraphics[width=0.33\linewidth]{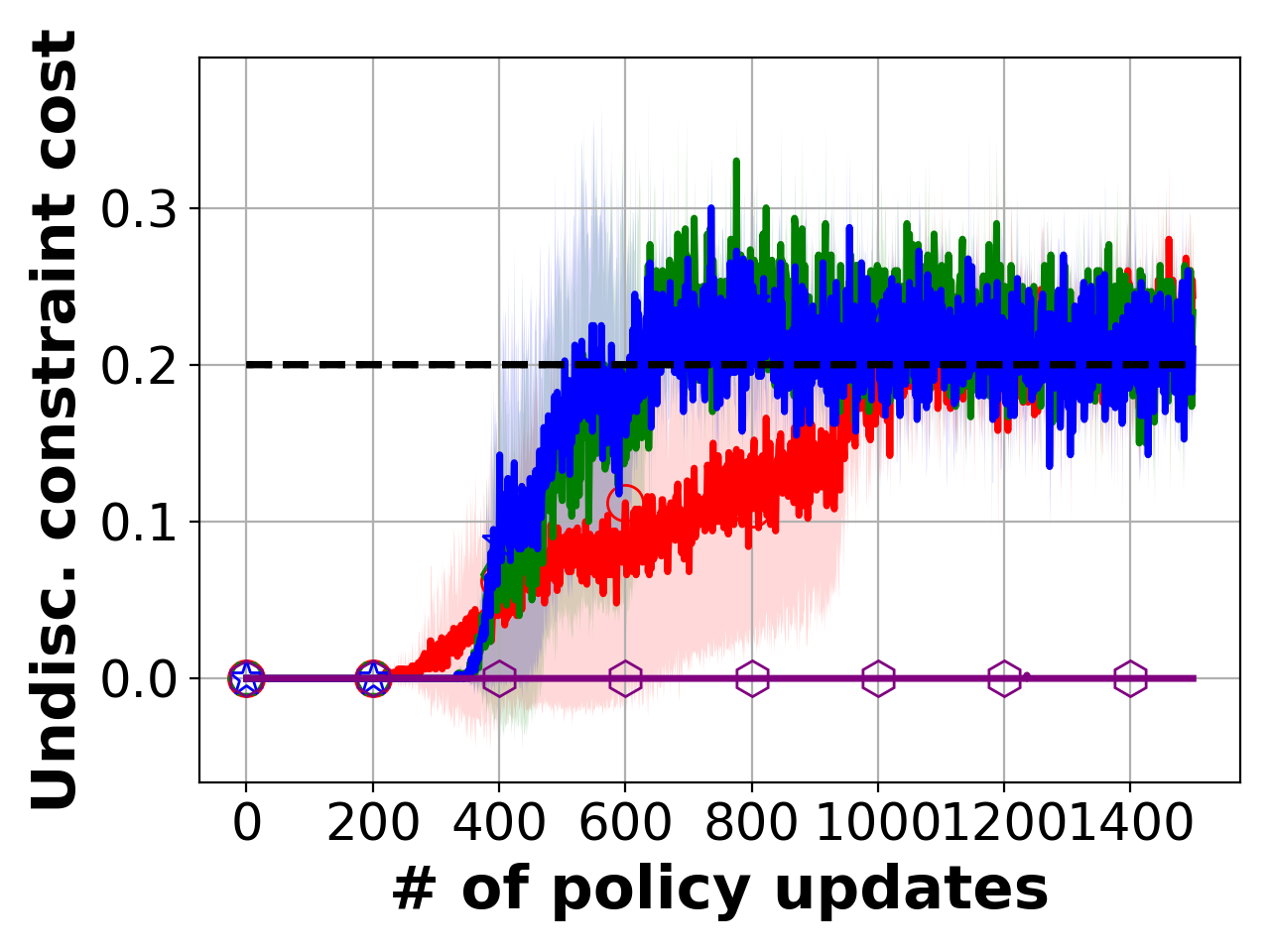}%
\includegraphics[width=0.33\linewidth]{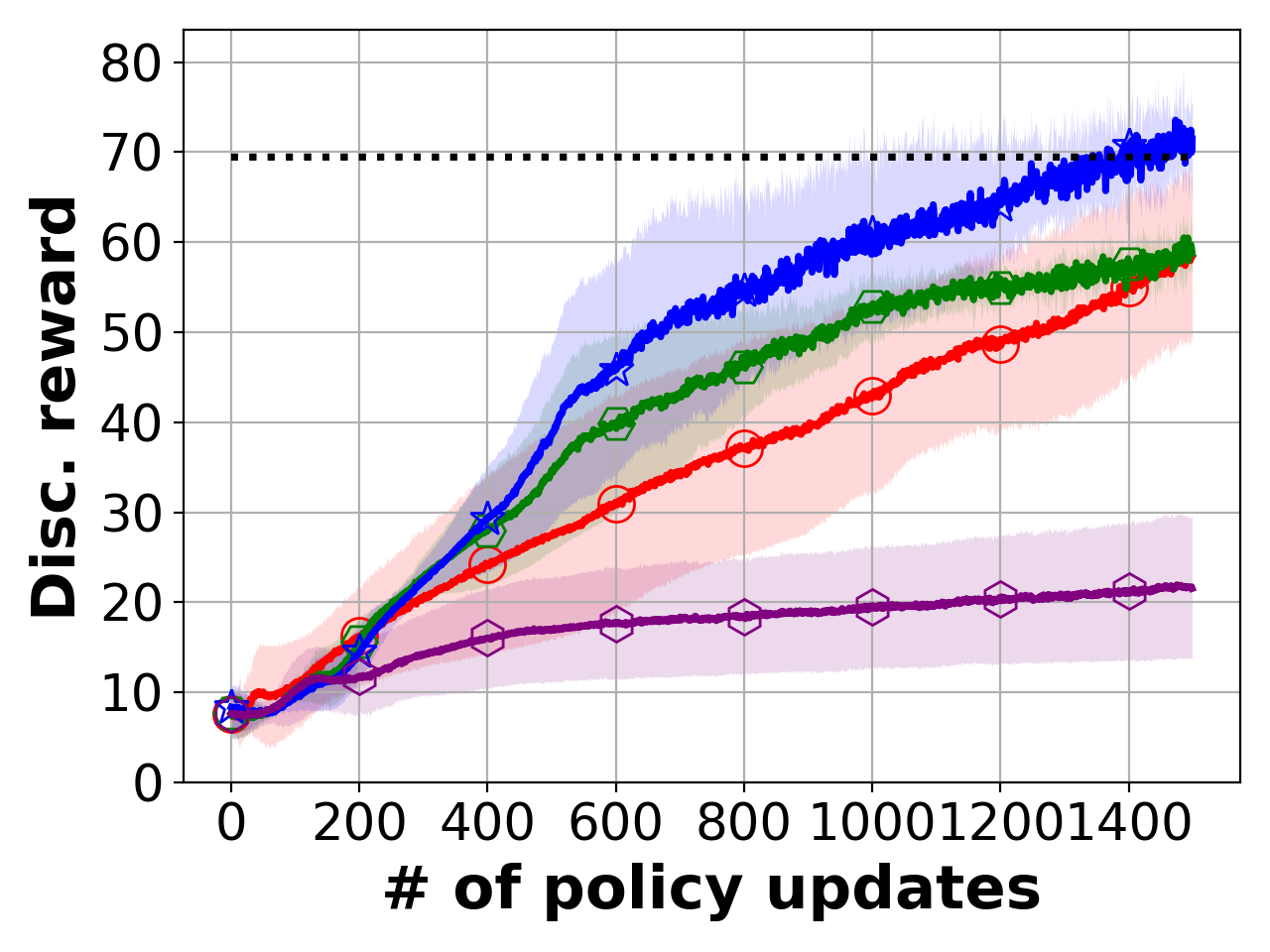}%
\includegraphics[width=0.33\linewidth]{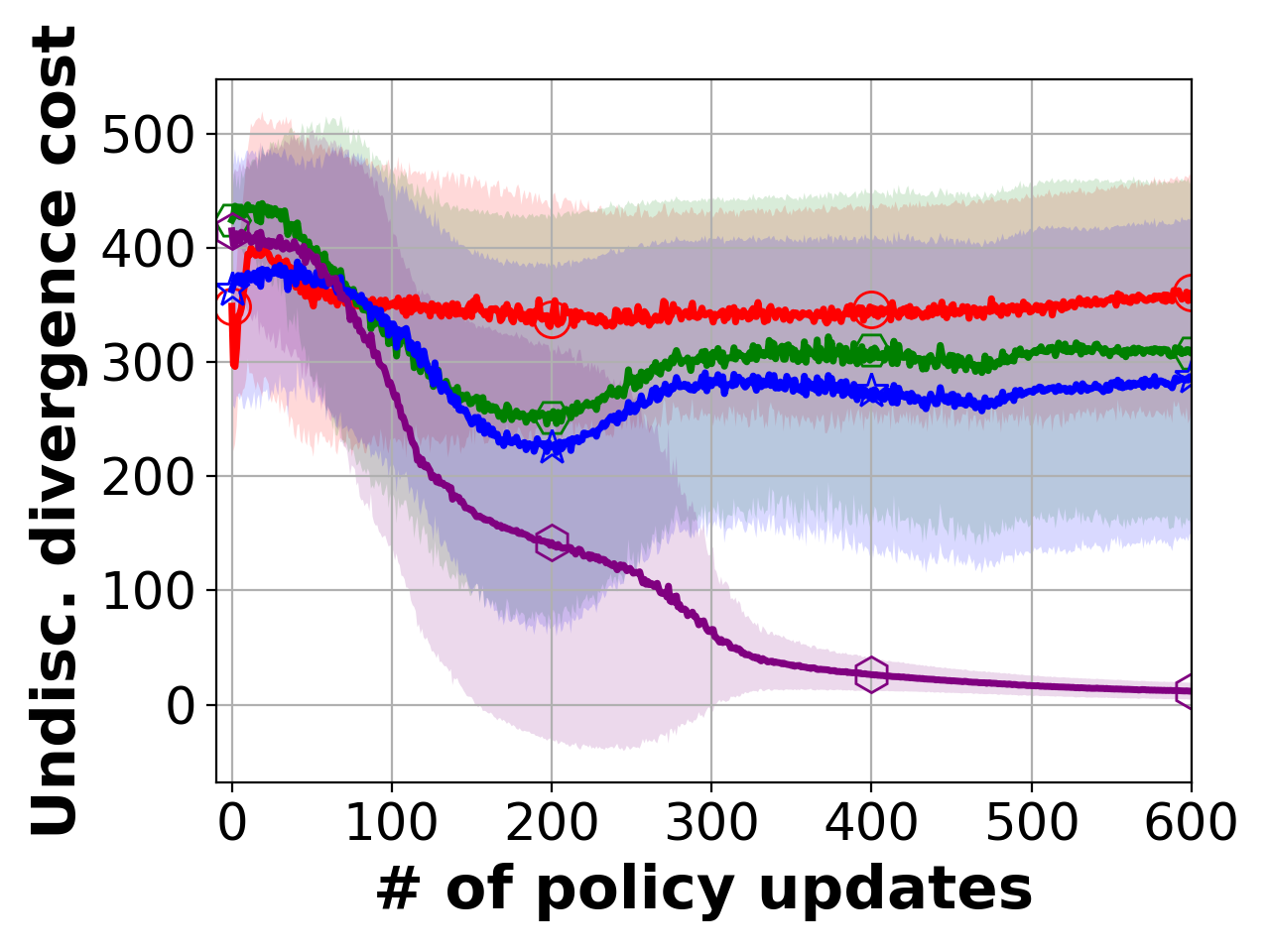}%
\end{tabular}}%
\vspace{-3mm}

\subfloat[Ant circle\label{subfig:cr}]{\begin{tabular}[b]{@{}c@{}}%
\vspace{-3mm}
\includegraphics[width=0.33\linewidth]{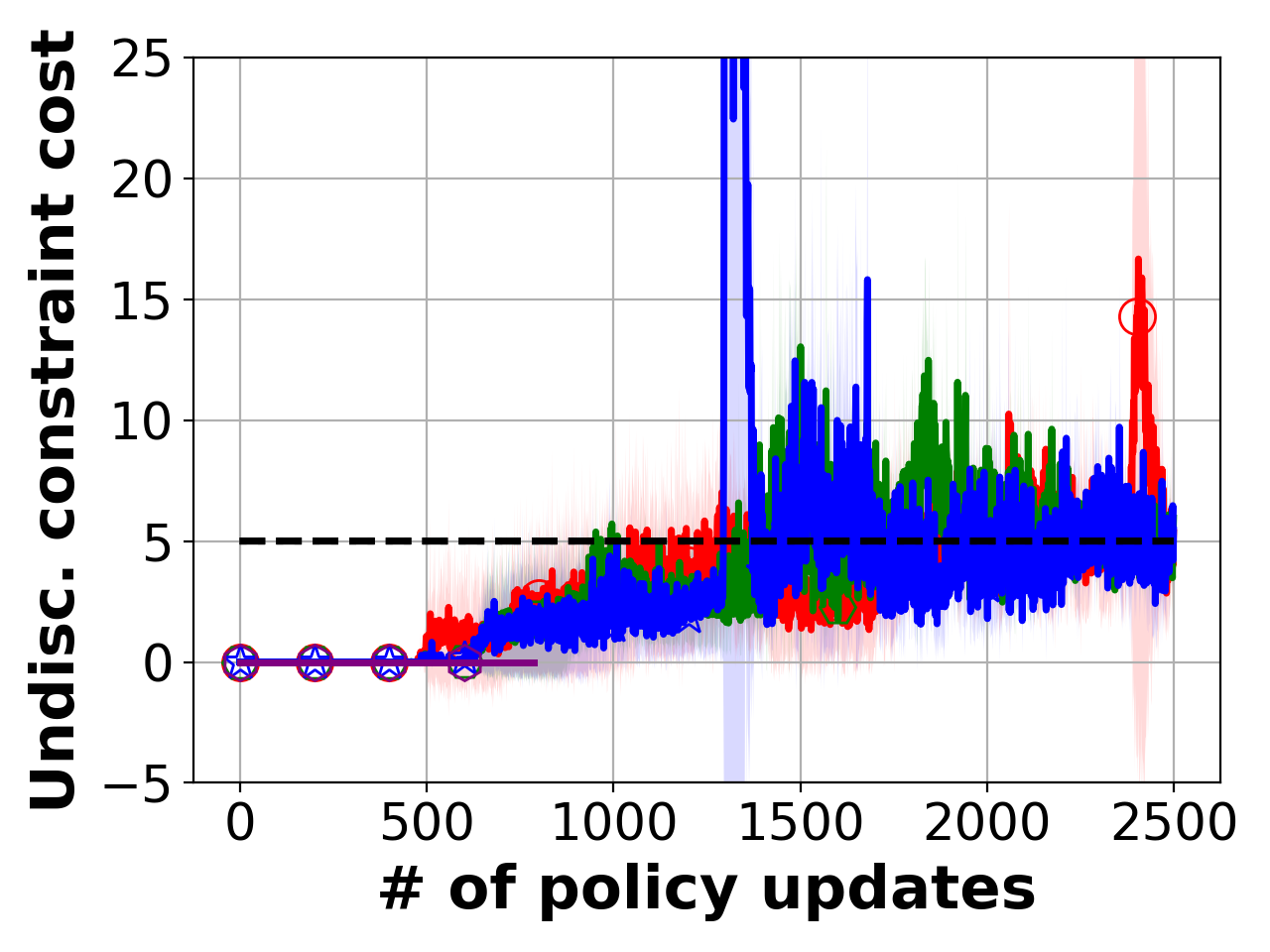}%
\includegraphics[width=0.33\linewidth]{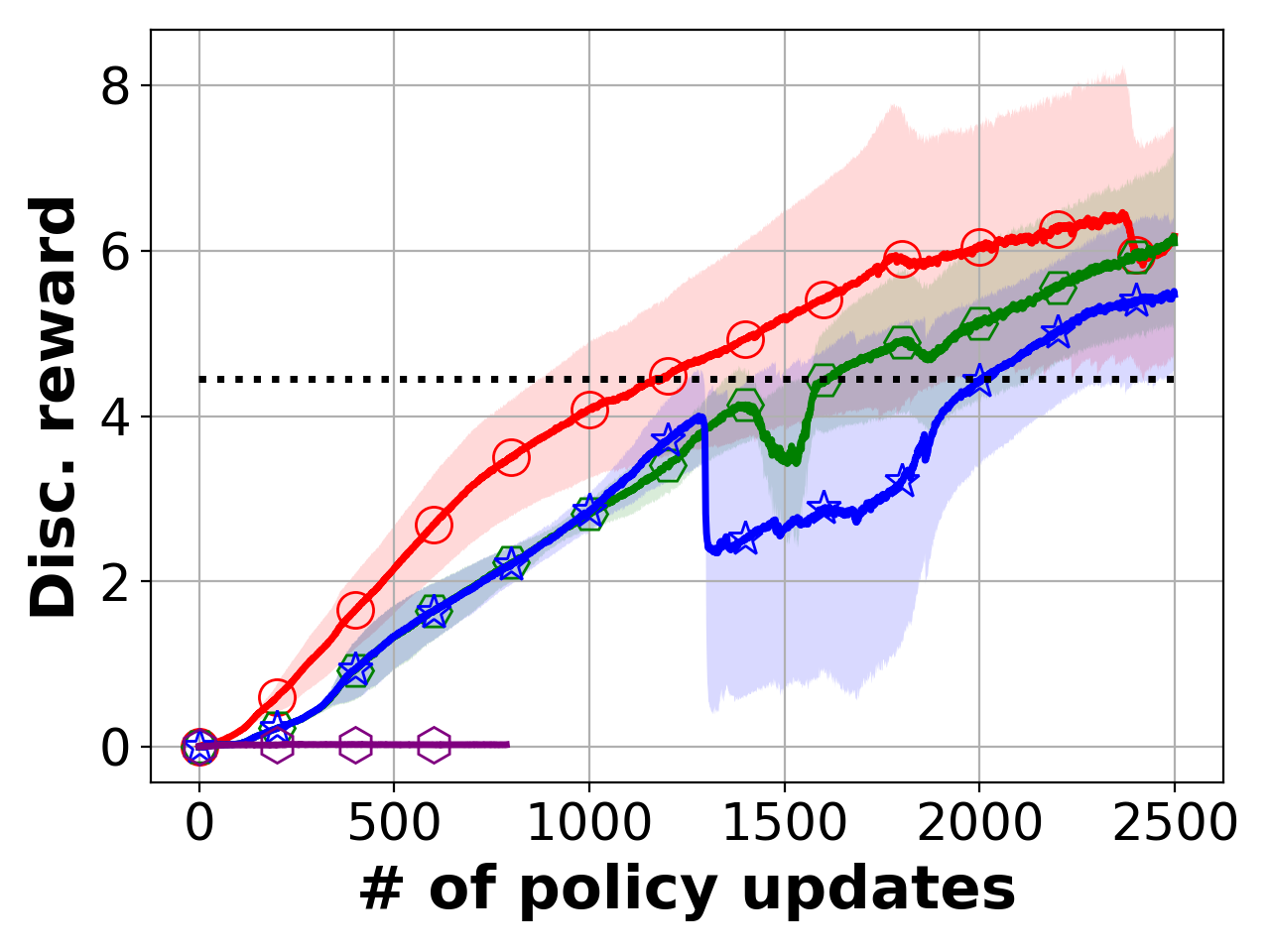}%
\includegraphics[width=0.33\linewidth]{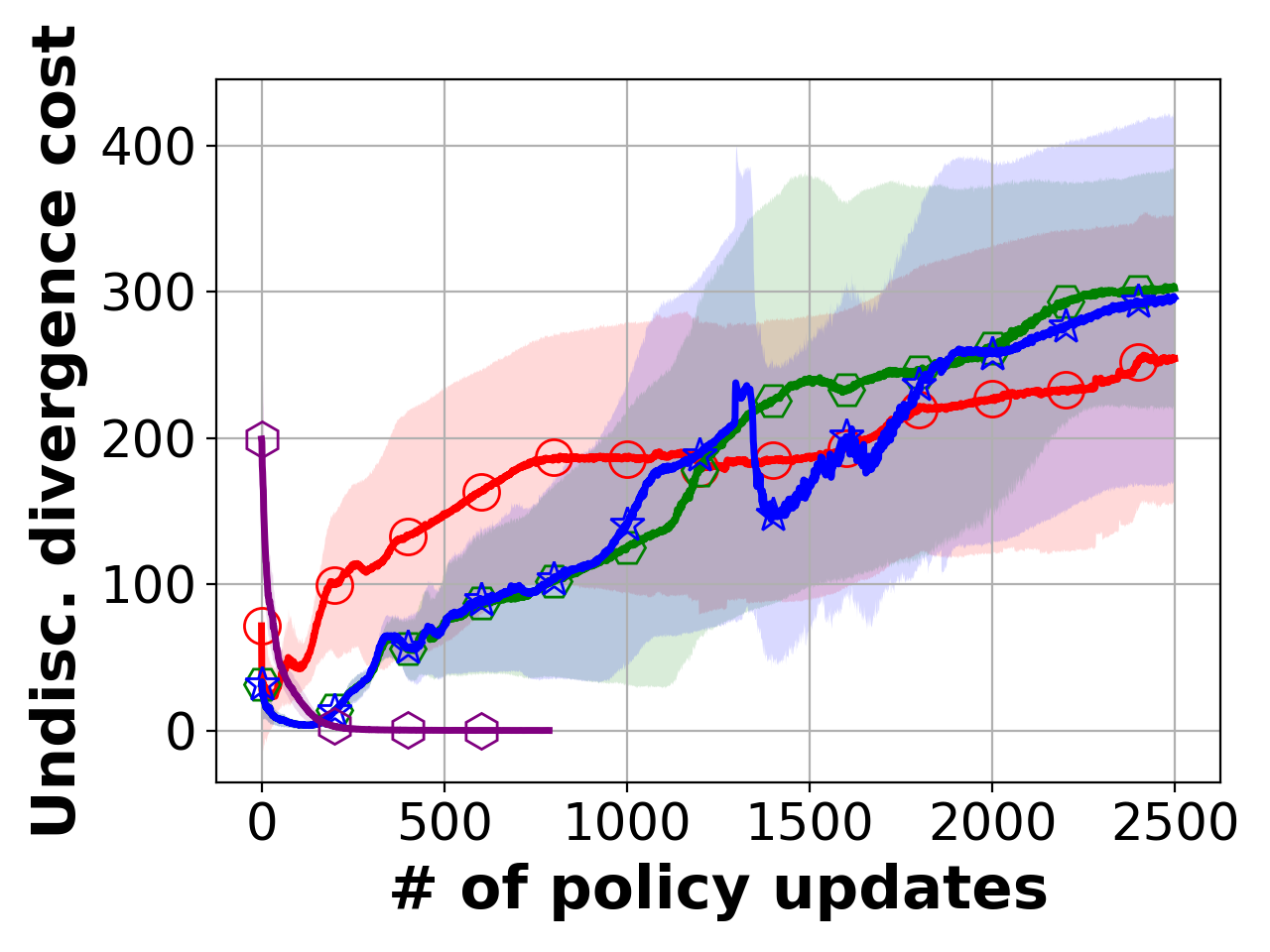}%
\end{tabular}}%
\vspace{-3mm}

\subfloat[Car-racing\label{subfig:cr}]{\begin{tabular}[b]{@{}c@{}}%
\vspace{-3mm}
\includegraphics[width=0.33\linewidth]{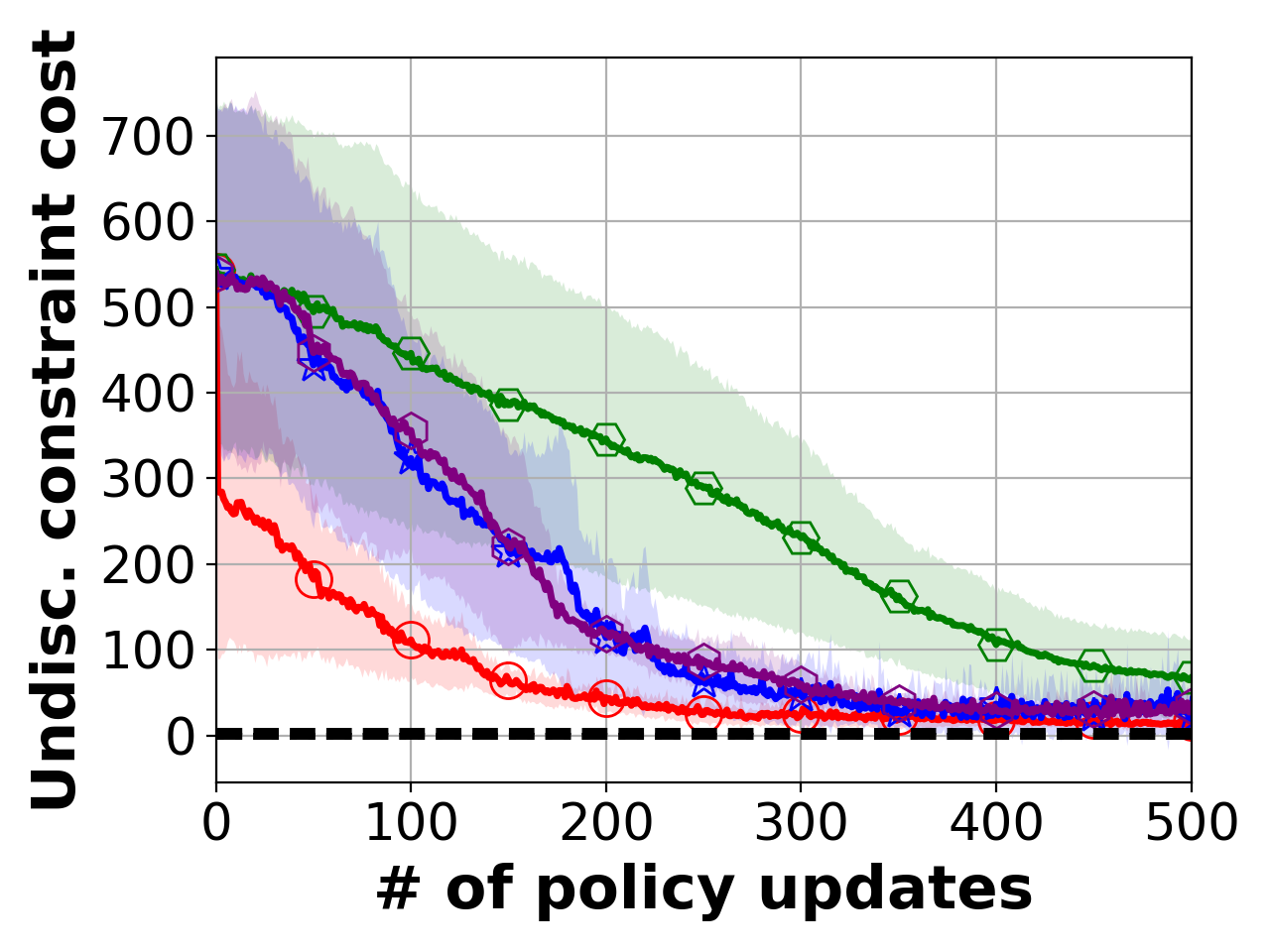}%
\includegraphics[width=0.33\linewidth]{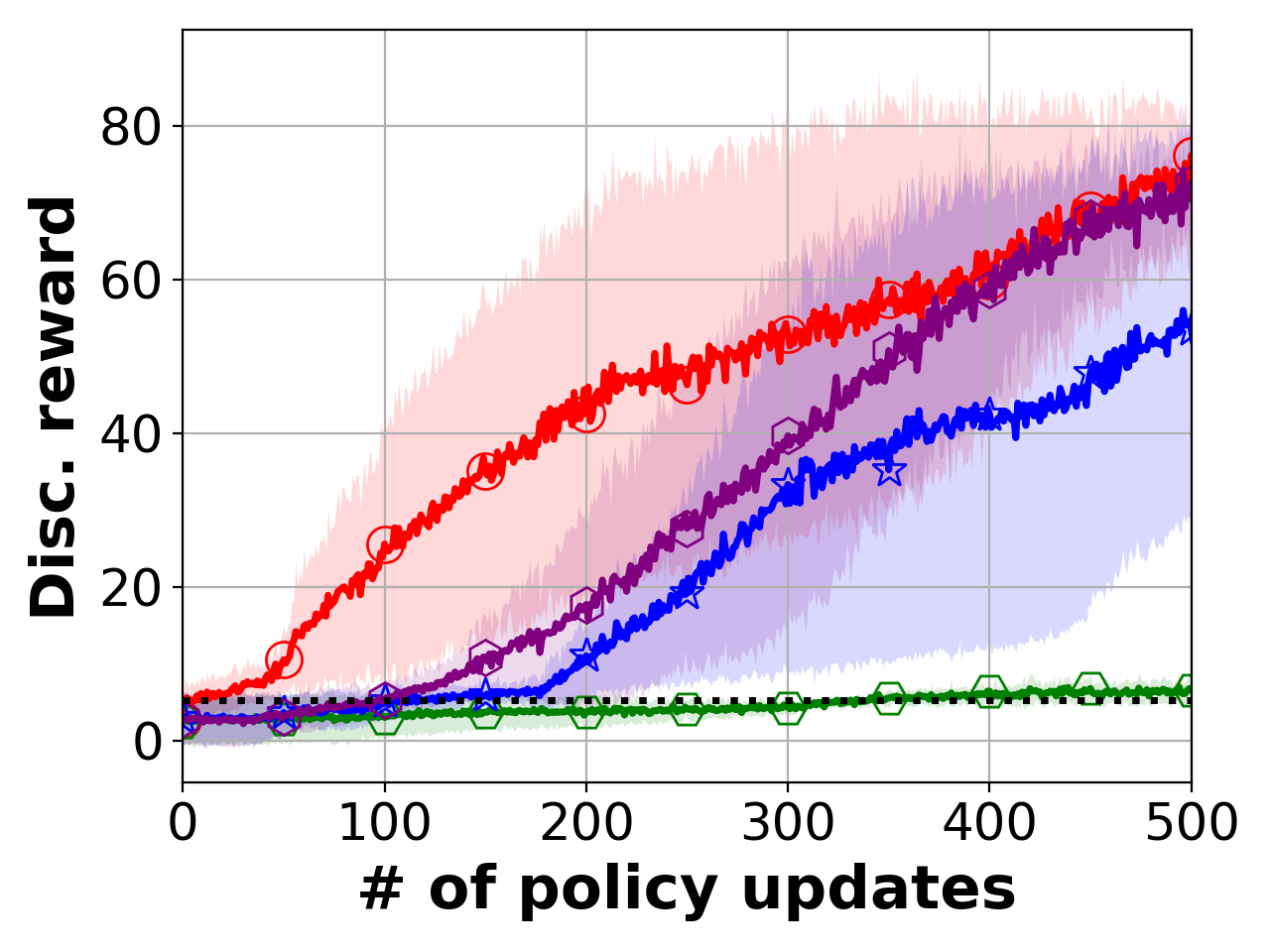}%
\includegraphics[width=0.33\linewidth]{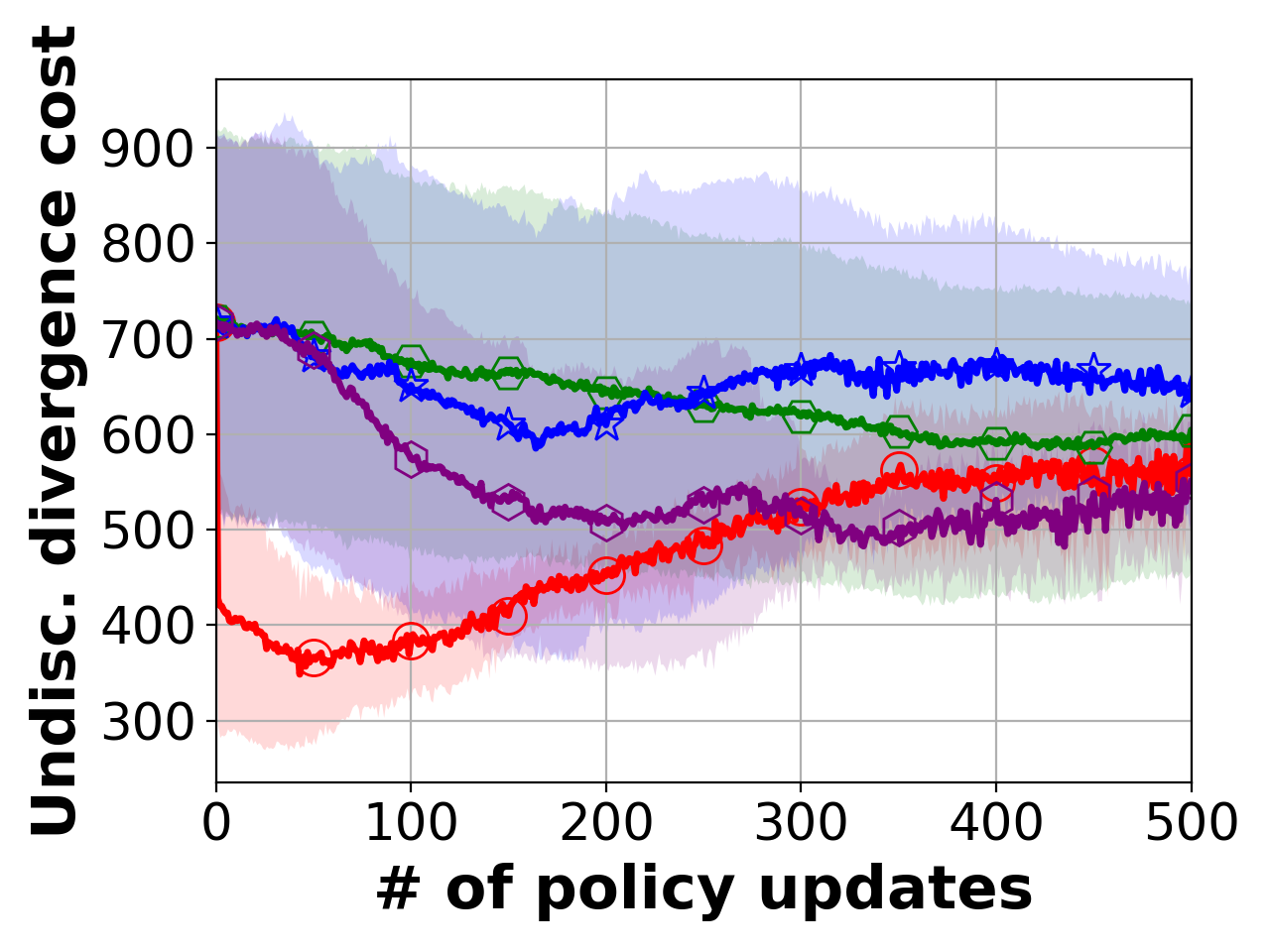}%
\end{tabular}}%

\vspace{-1mm}

\includegraphics[width=0.75\linewidth]{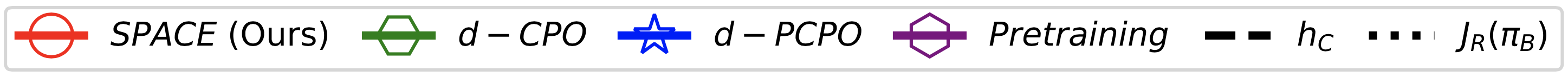}
\vspace{-4mm}
\caption{
The undiscounted constraint cost,
the discounted reward, and
the undiscounted divergence cost
over policy updates for the tested algorithms and tasks.
The solid line is the mean and the shaded area is the standard deviation over 5 runs.
\algname\ outperforms d-CPO, d-PCPO and the pre-training approach in terms of the efficiency of the reward improvement and cost constraint satisfaction. 
(Best viewed in color.)
}
\label{fig:appendix_pretrainingPrior}
\vspace{-3mm}
%\end{mdframed}
\end{figure*}

\paragraph{Comparison of \algname\ under the KL-divergence and the 2-norm Projections (see Fig.~\ref{fig:appendix_KLvsL2projections}).} 
Theorem \ref{theorem:P2CPO_converge} shows that under the KL-divergence and 2-norm projections, \algname\ converges to different stationary points.
To demonstrate the difference between these two projections, Fig.~\ref{fig:appendix_KLvsL2projections} shows the learning curves of the undiscounted constraint cost, the discounted reward, and the undiscounted divergence cost over policy updates for all tested algorithms and tasks.
In the Mujoco tasks, we observe that \algname\ under the KL-divergence projection achieves higher reward.
For instance, in the point gather task the final reward is 25\% higher under the same cost constraint satisfaction.  
In contrast, in the traffic management tasks, we observe that \algname\ under the 2-norm projection achieves better cost constraint satisfaction.
For instance, in the grid task \algname\ under the 2-norm projection achieves a lower reward but more cost constraint satisfaction. In addition, in the bottleneck task \algname\ under the 2-norm projection achieves more reward and cost constraint satisfaction.
These observations imply that \algname\ converges to different stationary points under two possible projections depending on tasks.

\begin{figure*}[t]
\vspace{-3mm}
\centering

\subfloat[Point gather\label{subfig:ac}]{\begin{tabular}[b]{@{}c@{}}%
\includegraphics[width=0.33\linewidth]{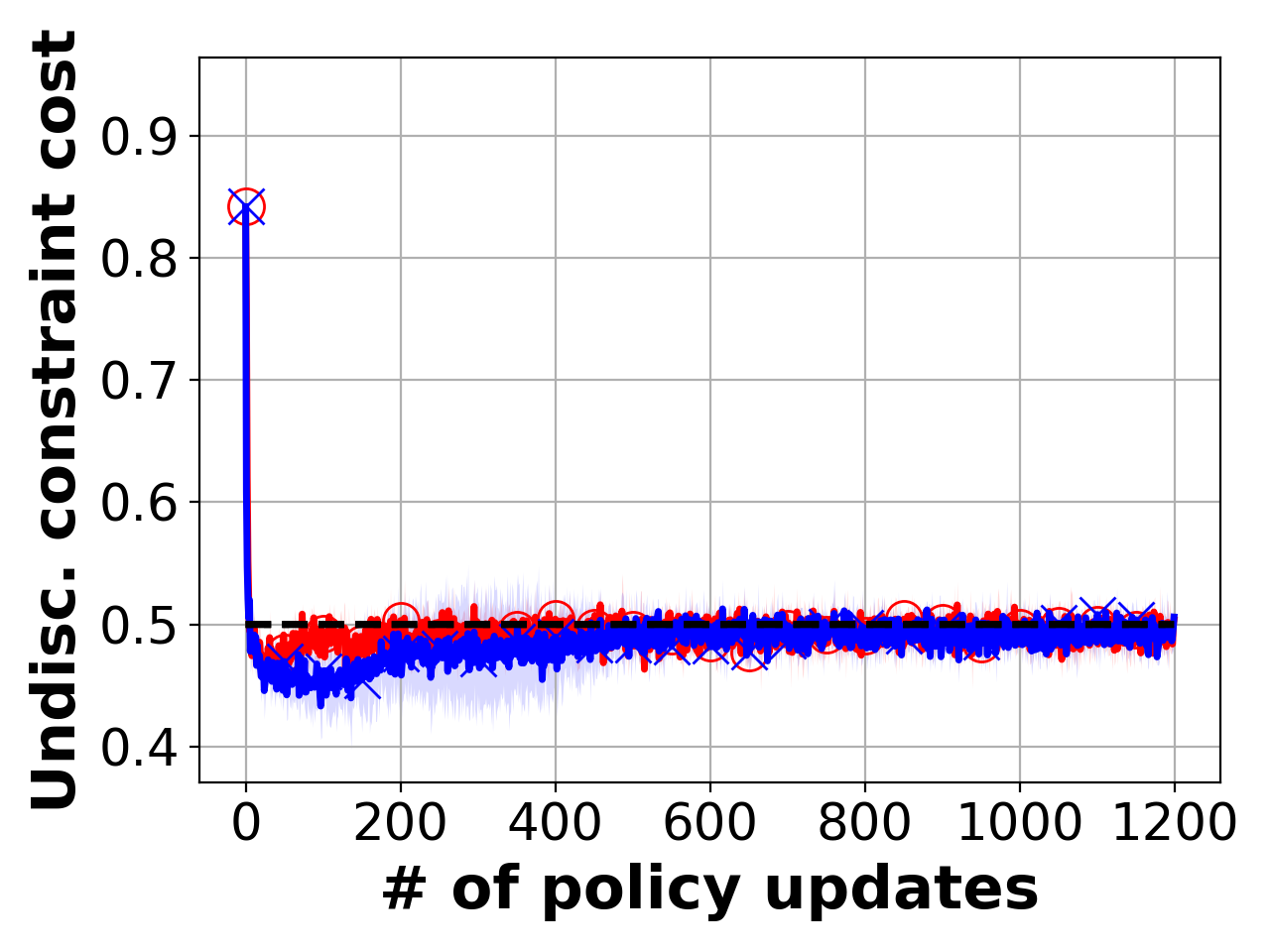}%
\includegraphics[width=0.33\linewidth]{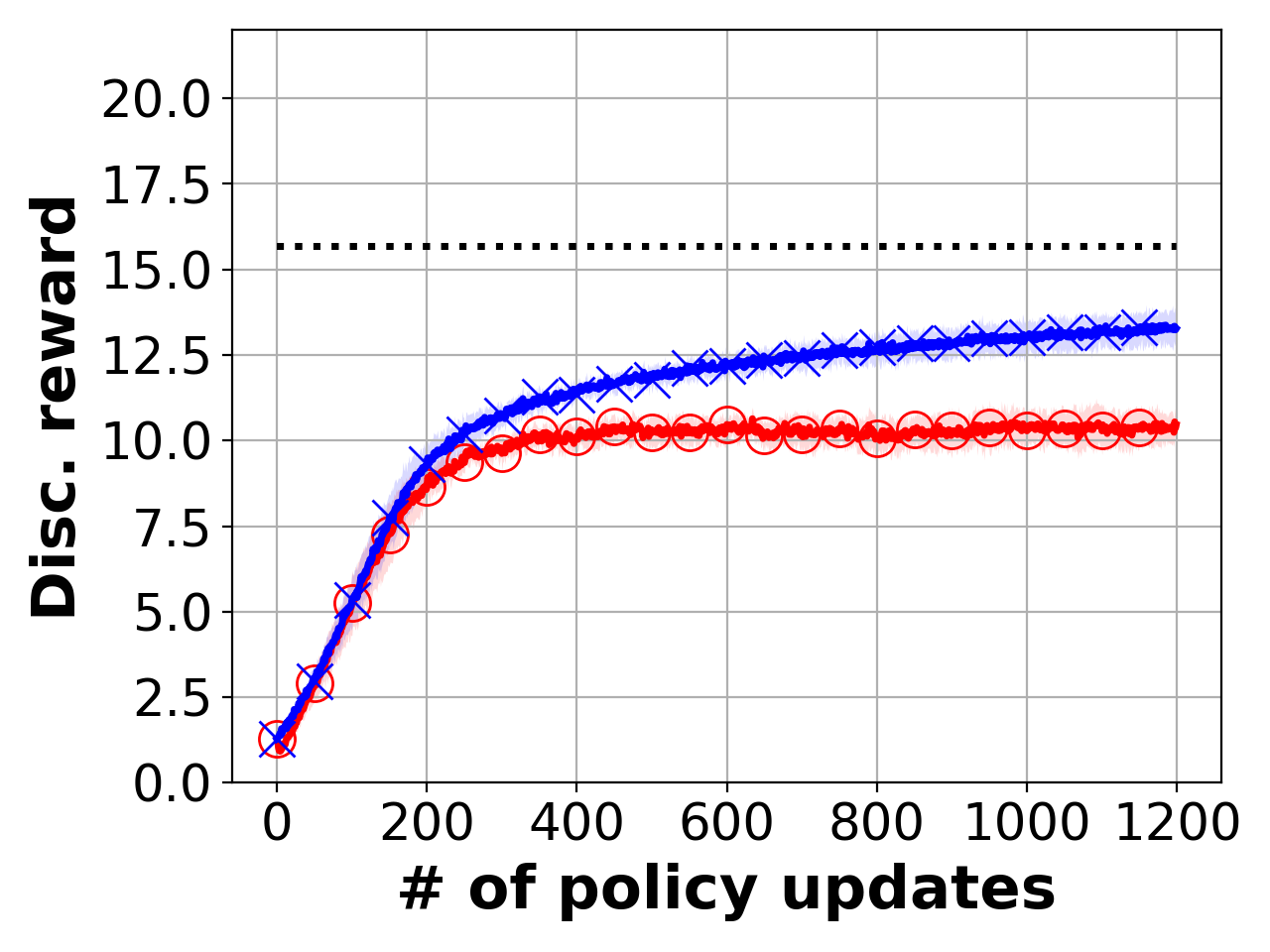}%
\includegraphics[width=0.33\linewidth]{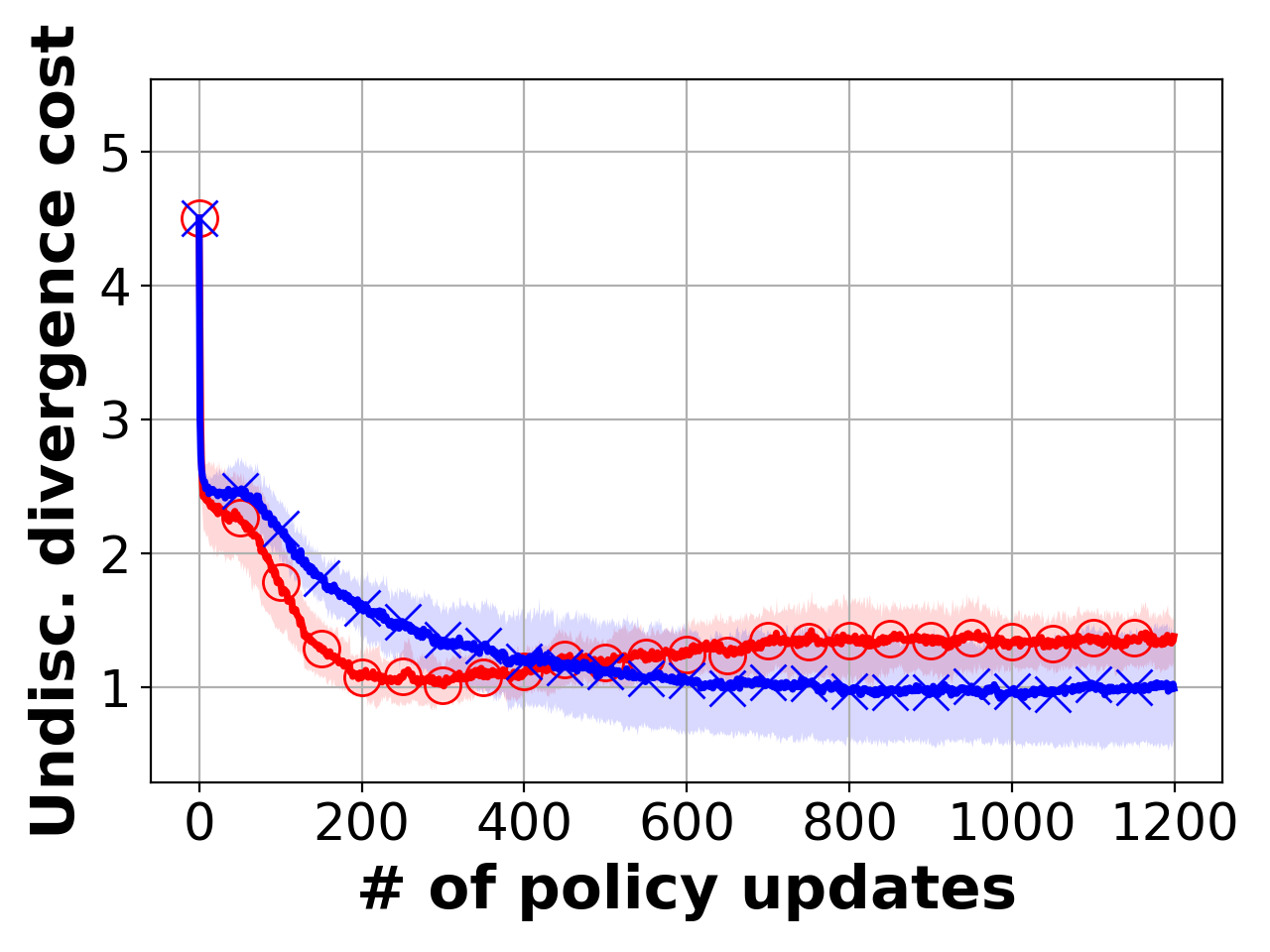}%
\end{tabular}}%

\subfloat[Point circle\label{subfig:ac}]{\begin{tabular}[b]{@{}c@{}}%
\includegraphics[width=0.33\linewidth]{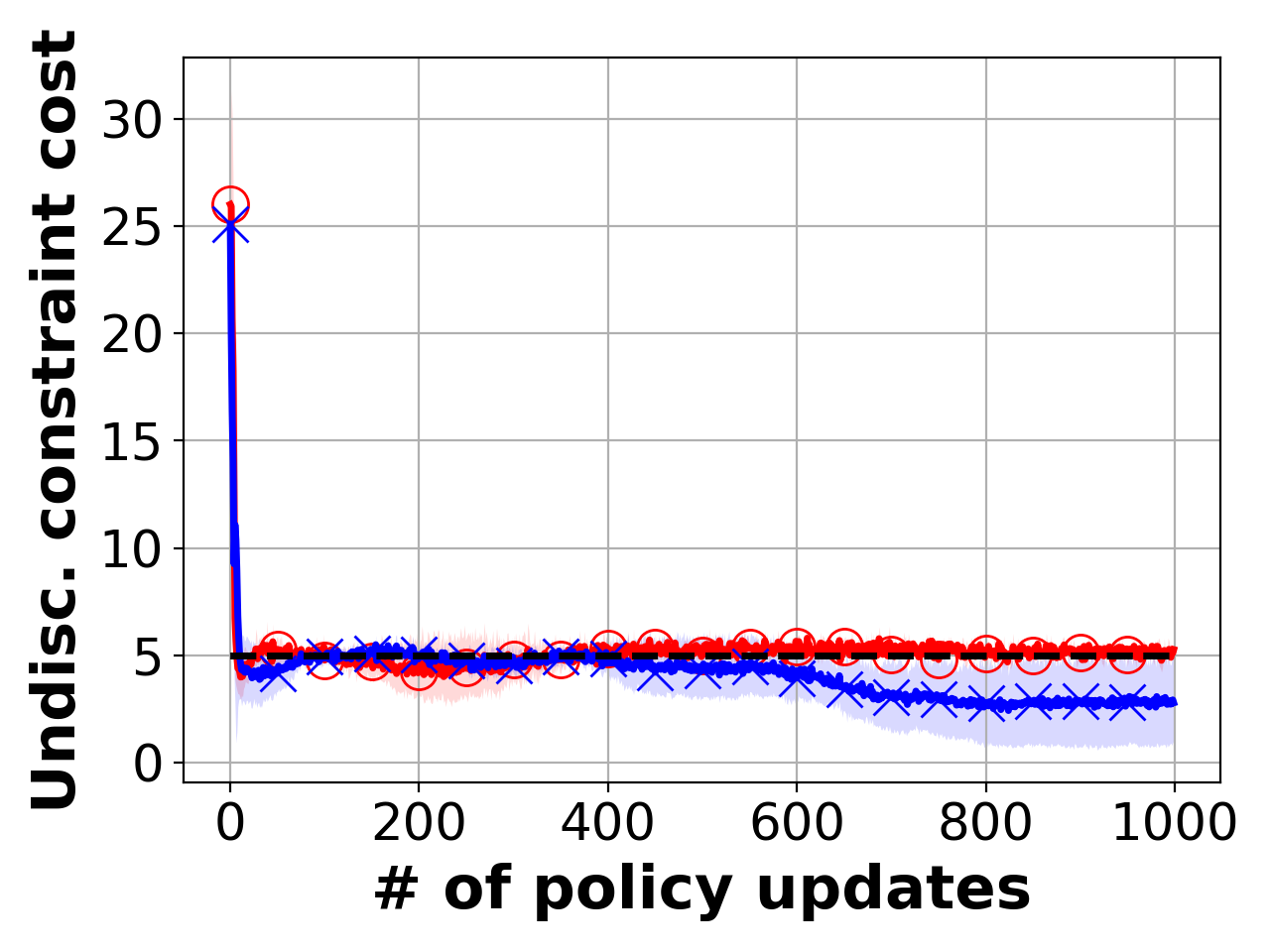}%
\includegraphics[width=0.33\linewidth]{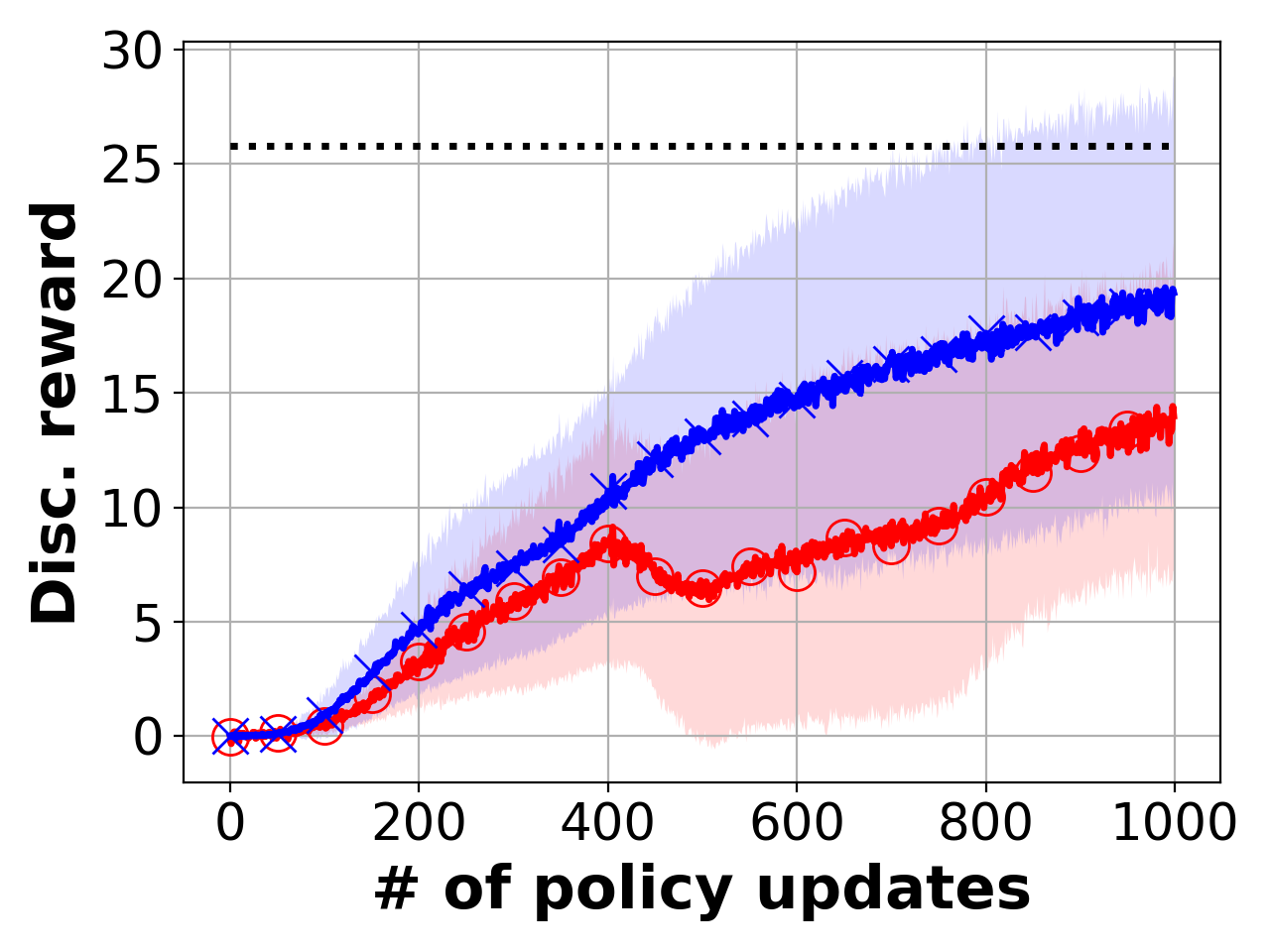}%
\includegraphics[width=0.33\linewidth]{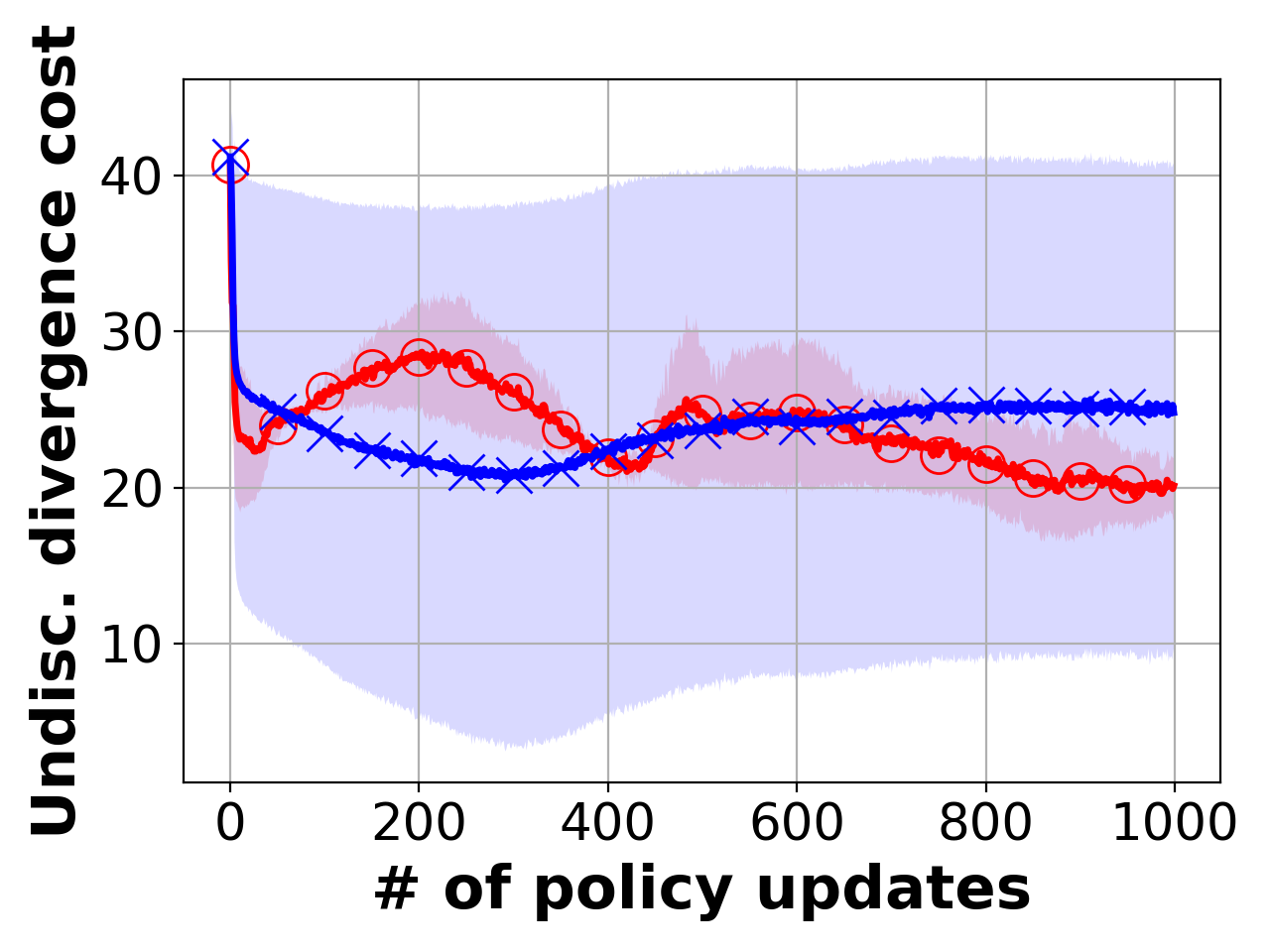}%
\end{tabular}}%

\subfloat[Grid\label{subfig:ac}]{\begin{tabular}[b]{@{}c@{}}%
\includegraphics[width=0.33\linewidth]{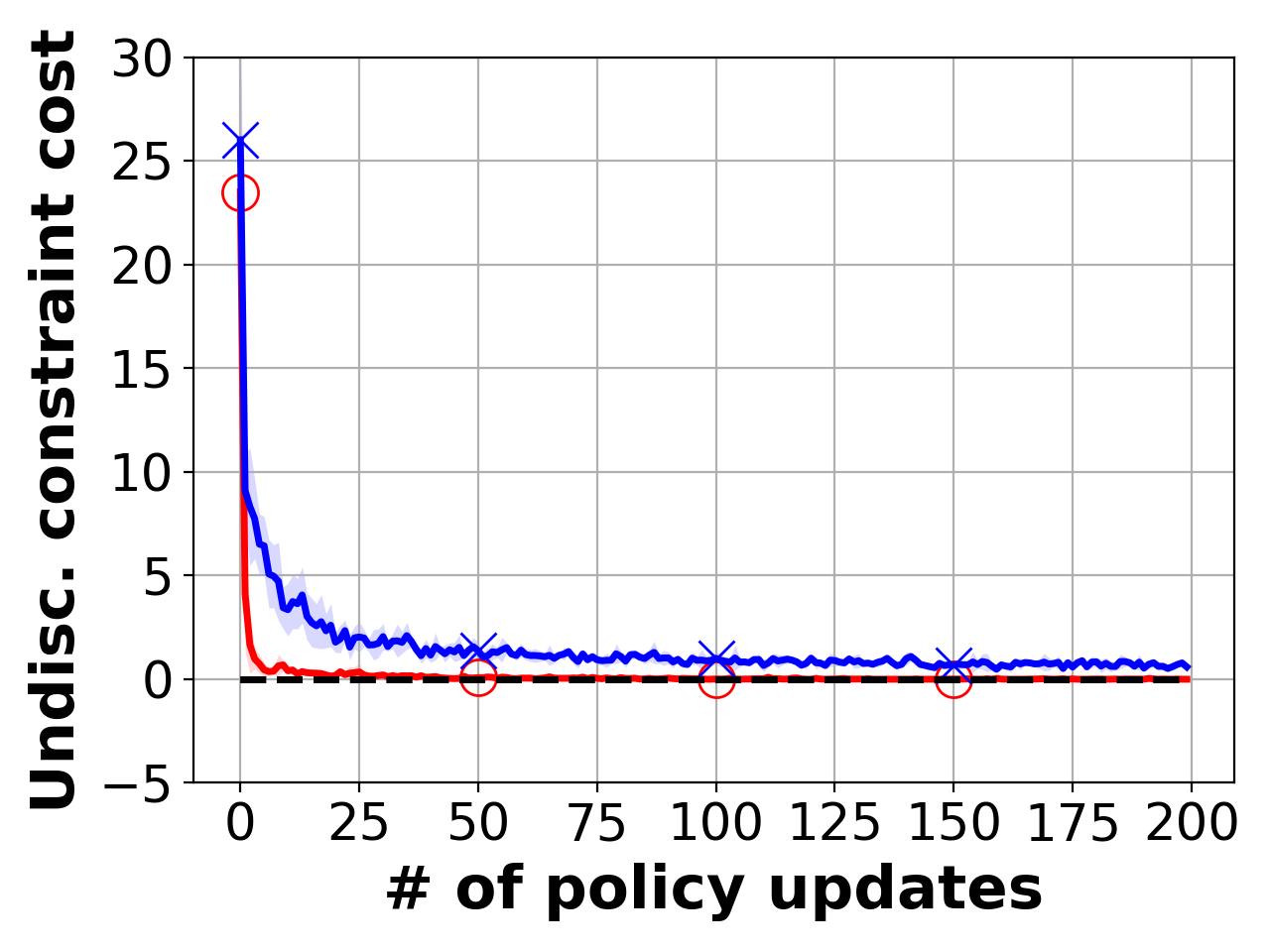}%
\includegraphics[width=0.33\linewidth]{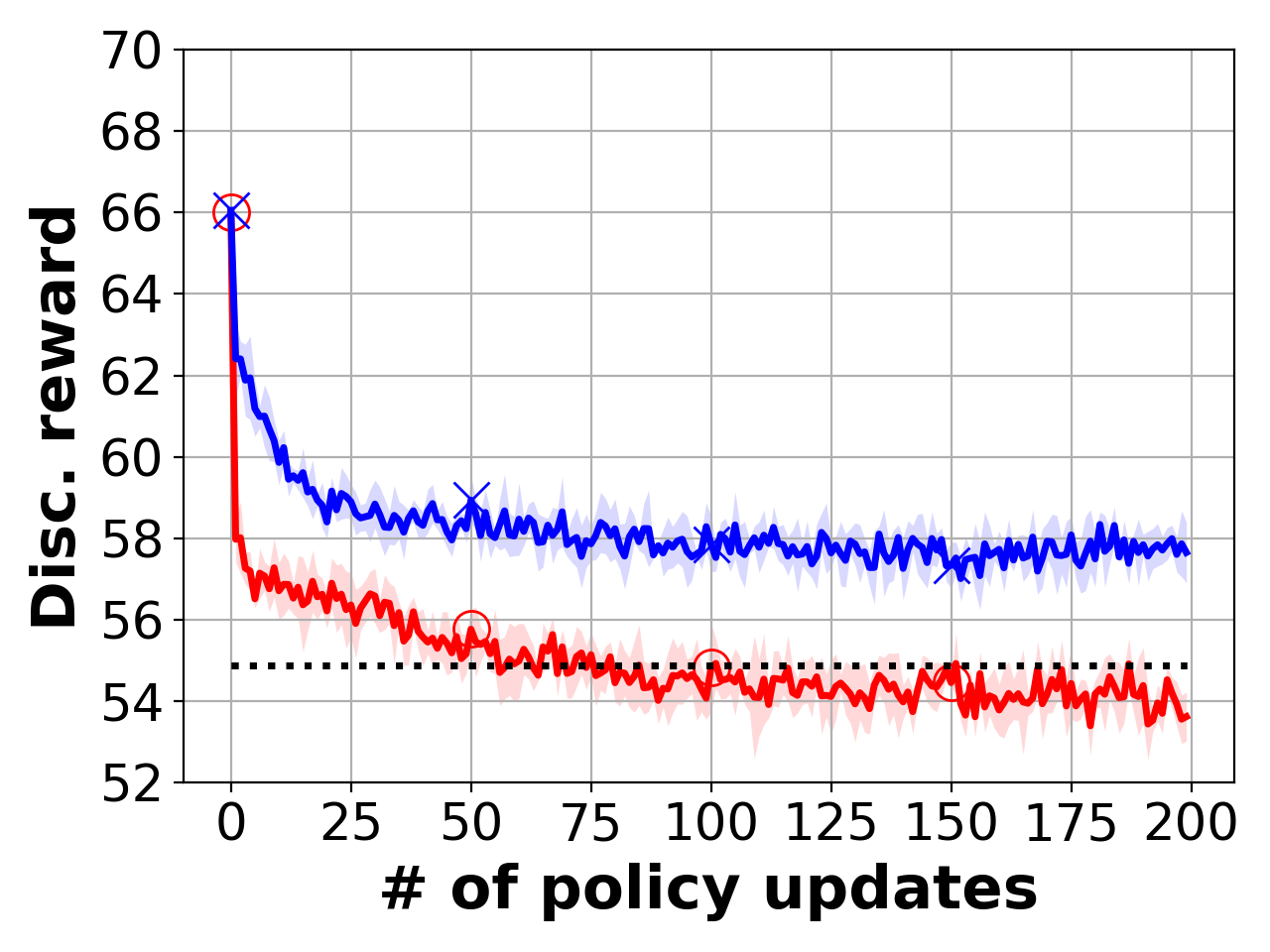}%
\includegraphics[width=0.33\linewidth]{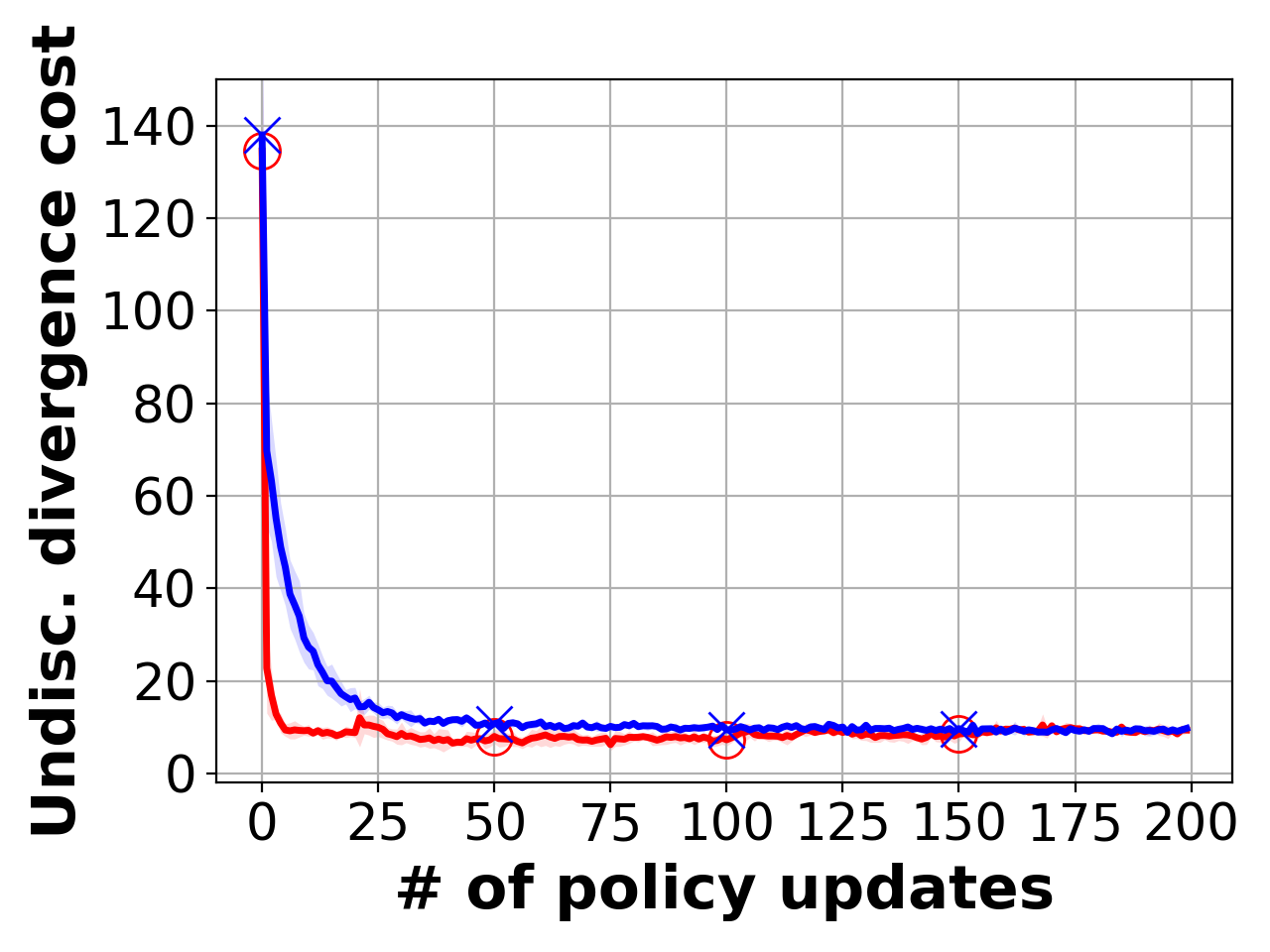}%
\end{tabular}}%

\subfloat[Bottleneck\label{subfig:cr}]{\begin{tabular}[b]{@{}c@{}}%
\includegraphics[width=0.33\linewidth]{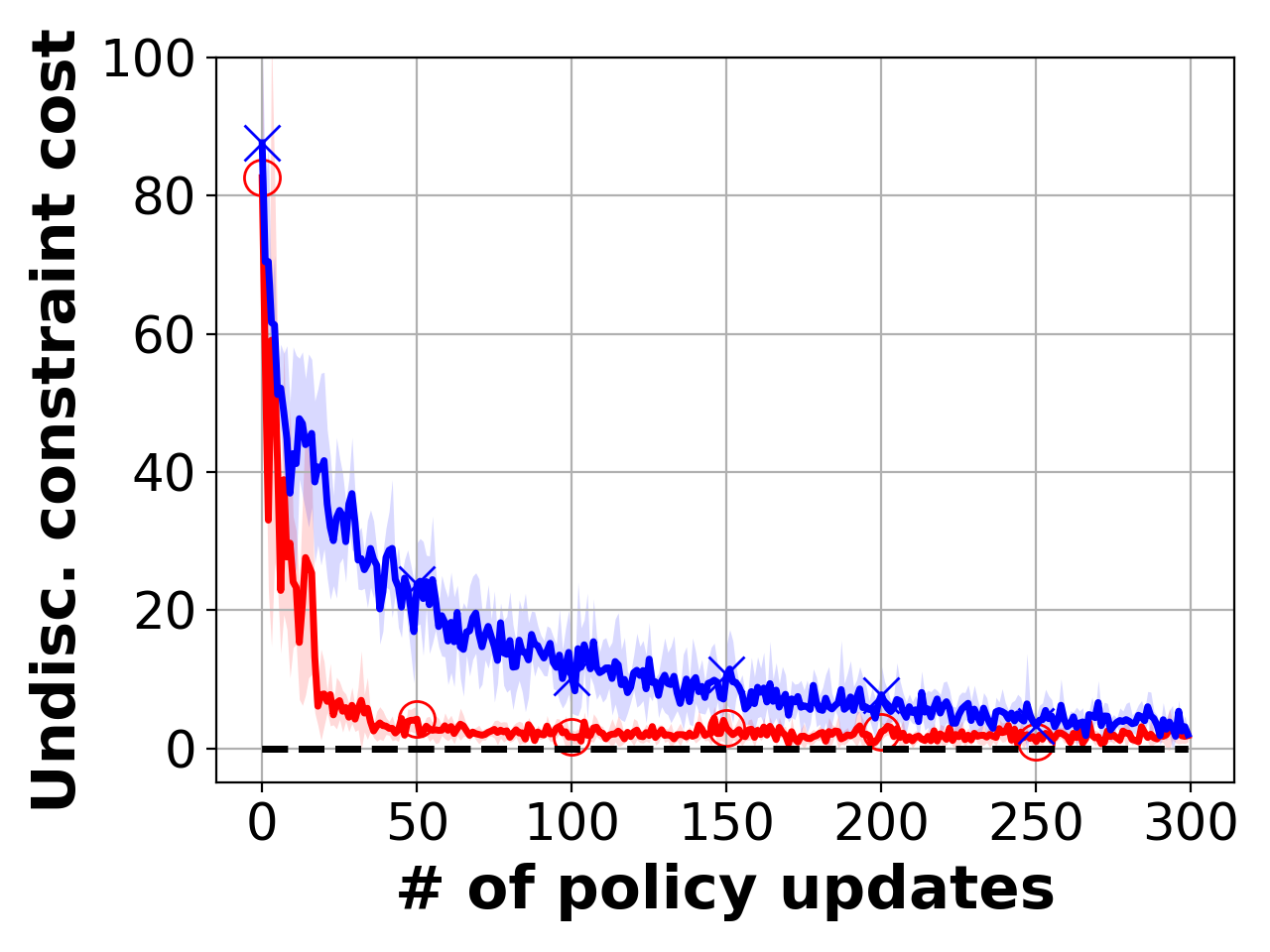}%
\includegraphics[width=0.33\linewidth]{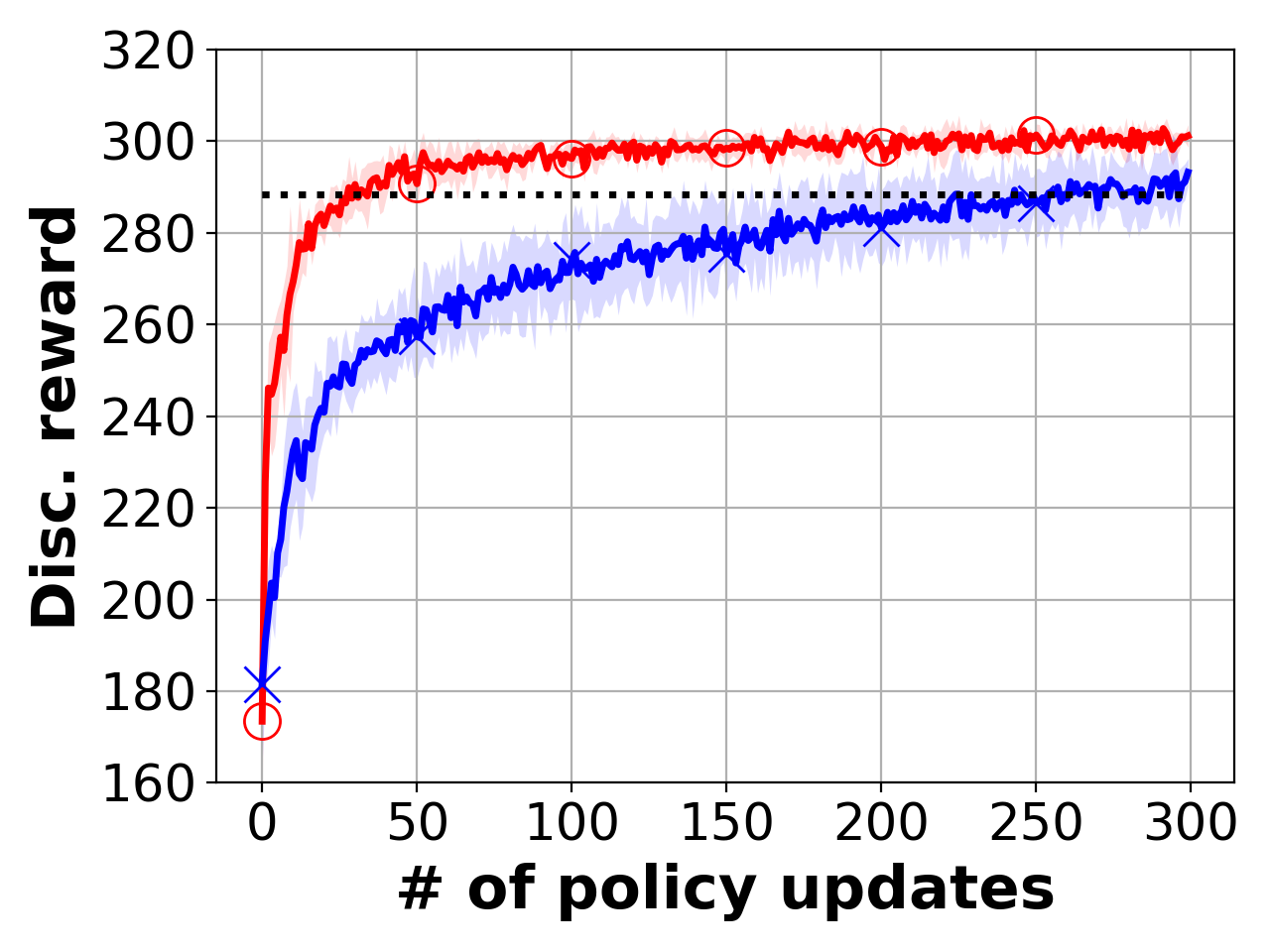}%
\includegraphics[width=0.33\linewidth]{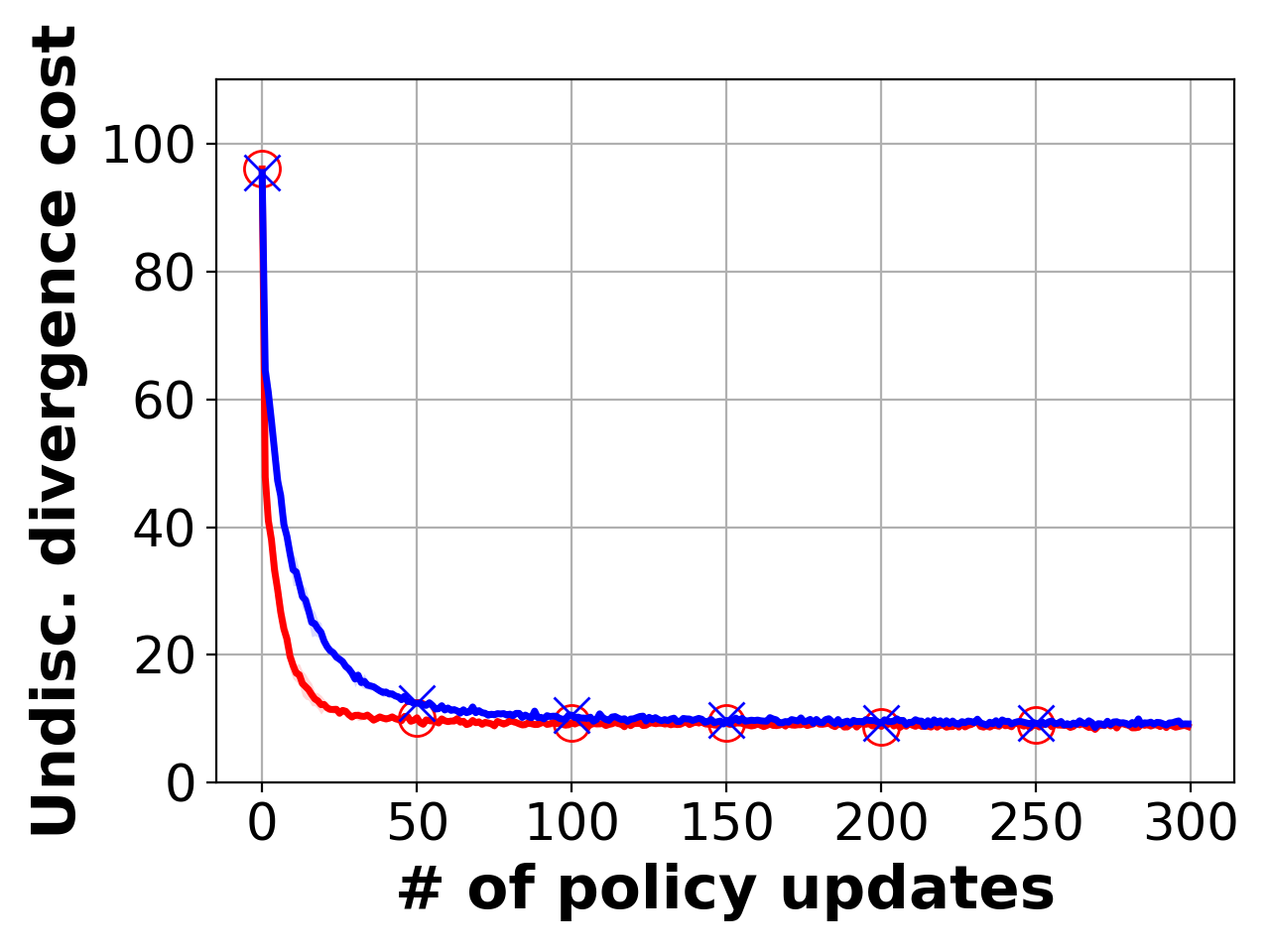}%
\end{tabular}}%

\vspace{+1mm}

\includegraphics[width=0.7\linewidth]{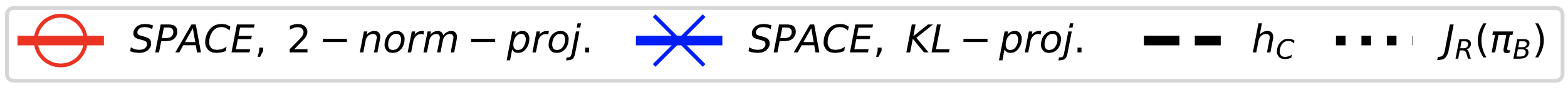}
\vspace{-2mm}

\caption{The undiscounted constraint cost,
the discounted reward, and
the undiscounted divergence cost
over policy updates for the tested algorithms and tasks.
The solid line is the mean and the shaded area is the standard deviation over 5 runs.
\algname\ converges to differently stationary points under two possible projections.
(Best viewed in color.)
}
\label{fig:appendix_KLvsL2projections}
\vspace{-3mm}
%\end{mdframed}
\end{figure*}

\paragraph{Initial $h^0_D$ (see Fig. \ref{fig:appendix_initialPriorConstraintThreshold}).} 
%
%\algname\ constructs the prior constraint set to safely learn from the prior.
%
%One question is that how we can select an appropriate initial value of $h_D^0.$
%
To understand the effect of the initial value of $h^0_D,$ we test \algname\ with three different initial values: $h_D^0=1, h_D^0=5,$ and $h_D^0=25$ in the ant circle and car-racing tasks. 
The learning curves of the undiscounted constraint cost, the discounted reward, and the undiscounted divergence cost over policy updates are shown for all tested algorithms and tasks in Fig. \ref{fig:appendix_initialPriorConstraintThreshold}.
In both tasks, we observe that the initial value of $h_D^0$ does not affect the reward and the cost performance significantly (\ie the mean of learning curves lies in roughly the same standard deviation over the initialization).
In addition, the value of the divergence cost over three $h_D^0$ are similar throughout the training.
These observations imply that the update scheme of $h_D^k$ in \algname\ is robust to the choice of the initial value of $h_D^0.$
However, in the car-racing task we observe that the learning curves of using a smaller $h^0_D$ tend to have higher variances.
For example, the standard deviation of $h_D^0=1$ in the reward plot is 6 times larger than the one with $h_D^0=25.$
This implies that \algname\ may have reward performance degradation when using a smaller initial value of $h_D^0.$
One possible reason is that when the distance between the learned and baseline policies is large, using a small value of $h_D^0$ results in an inaccurate projection (\ie due to approximation errors).
This causes the policy to follow a zigzag path.
We leave the improvement of this in future work.

\begin{figure*}[t]
\vspace{-3mm}
\centering

\subfloat[Ant circle\label{subfig:ac}]{\begin{tabular}[b]{@{}c@{}}%
\includegraphics[width=0.33\linewidth]{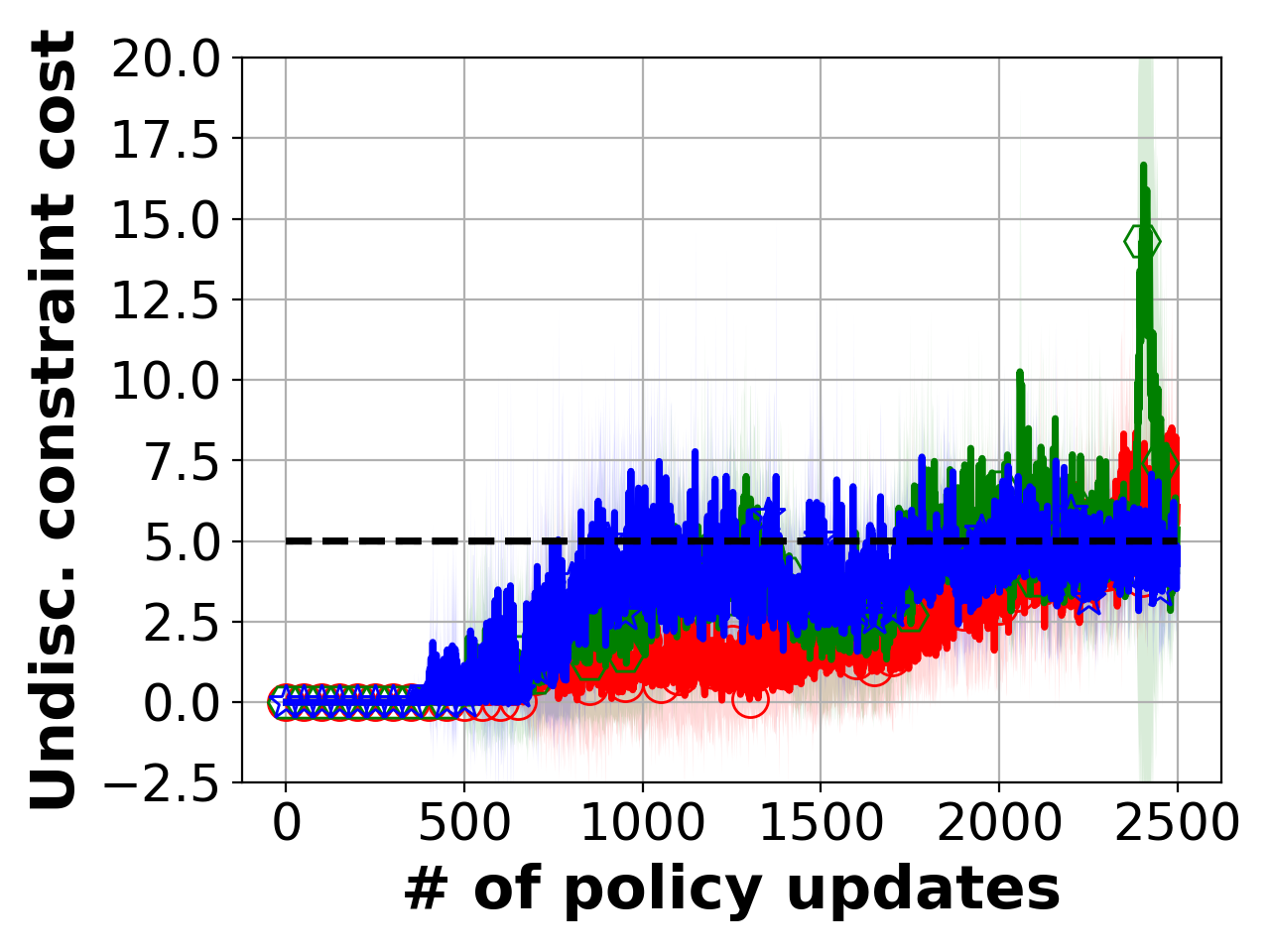}%
\includegraphics[width=0.33\linewidth]{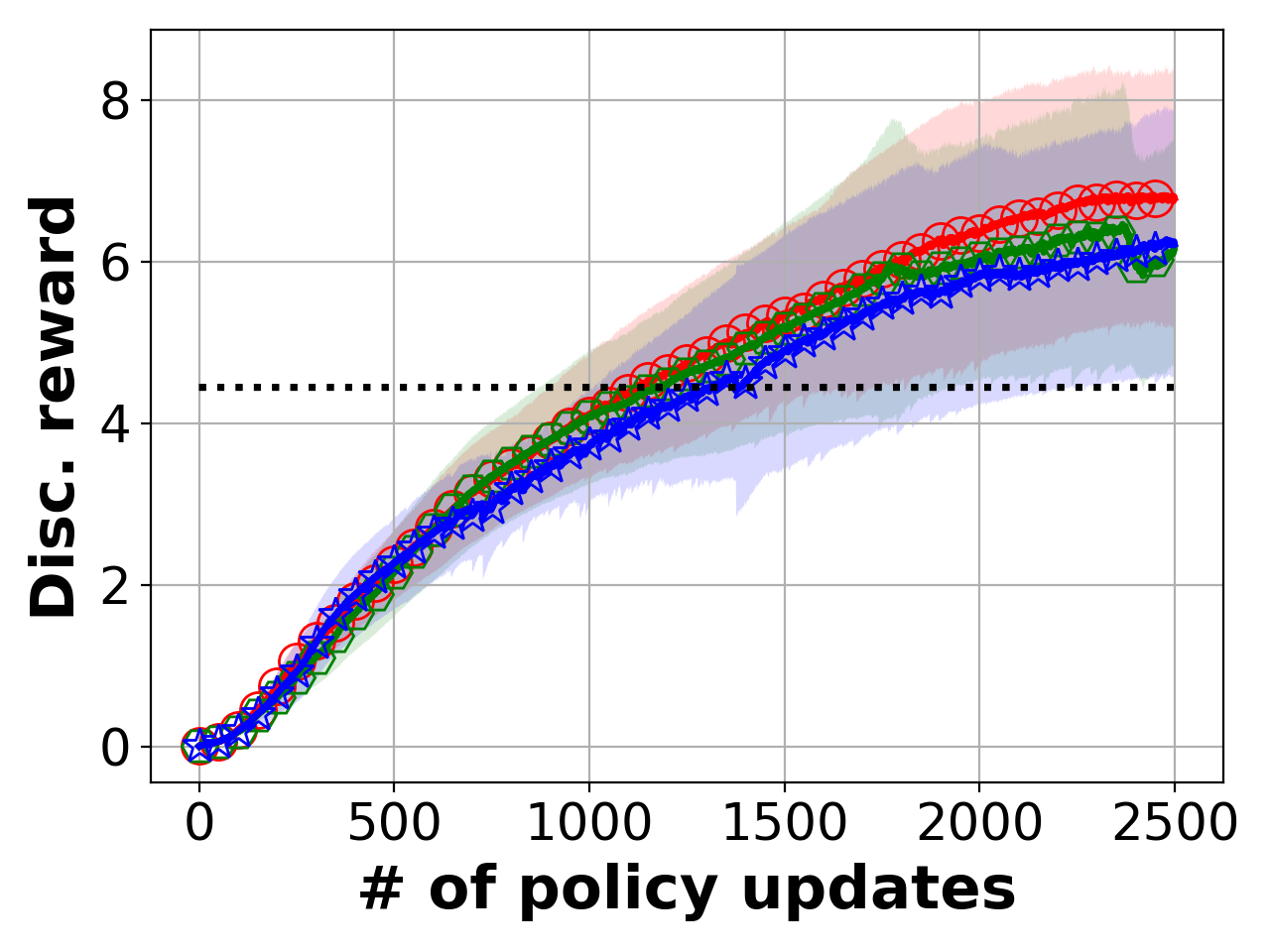}%
\includegraphics[width=0.33\linewidth]{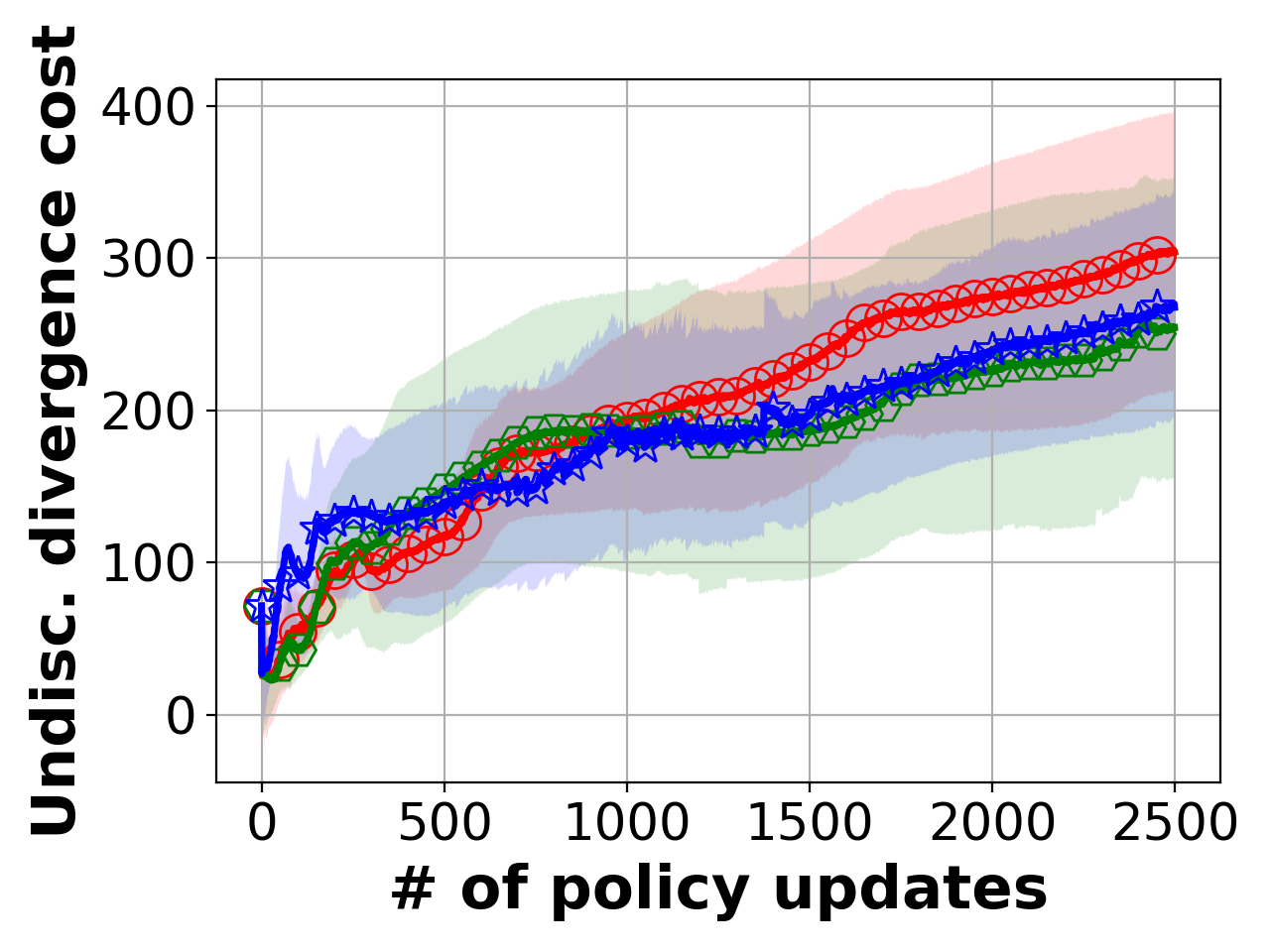}%
\end{tabular}}%

\subfloat[Car-racing\label{subfig:cr}]{\begin{tabular}[b]{@{}c@{}}%
\includegraphics[width=0.33\linewidth]{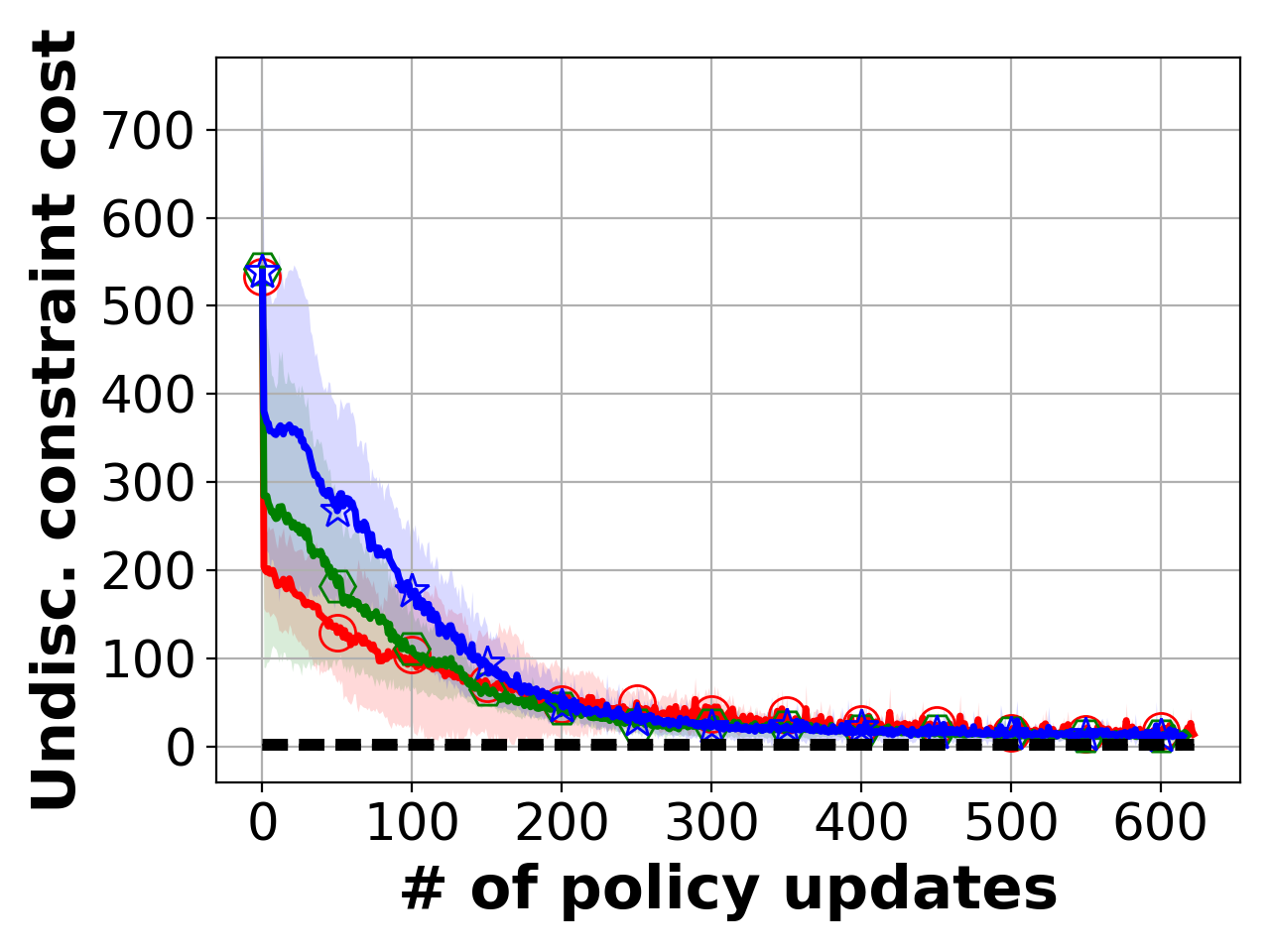}%
\includegraphics[width=0.33\linewidth]{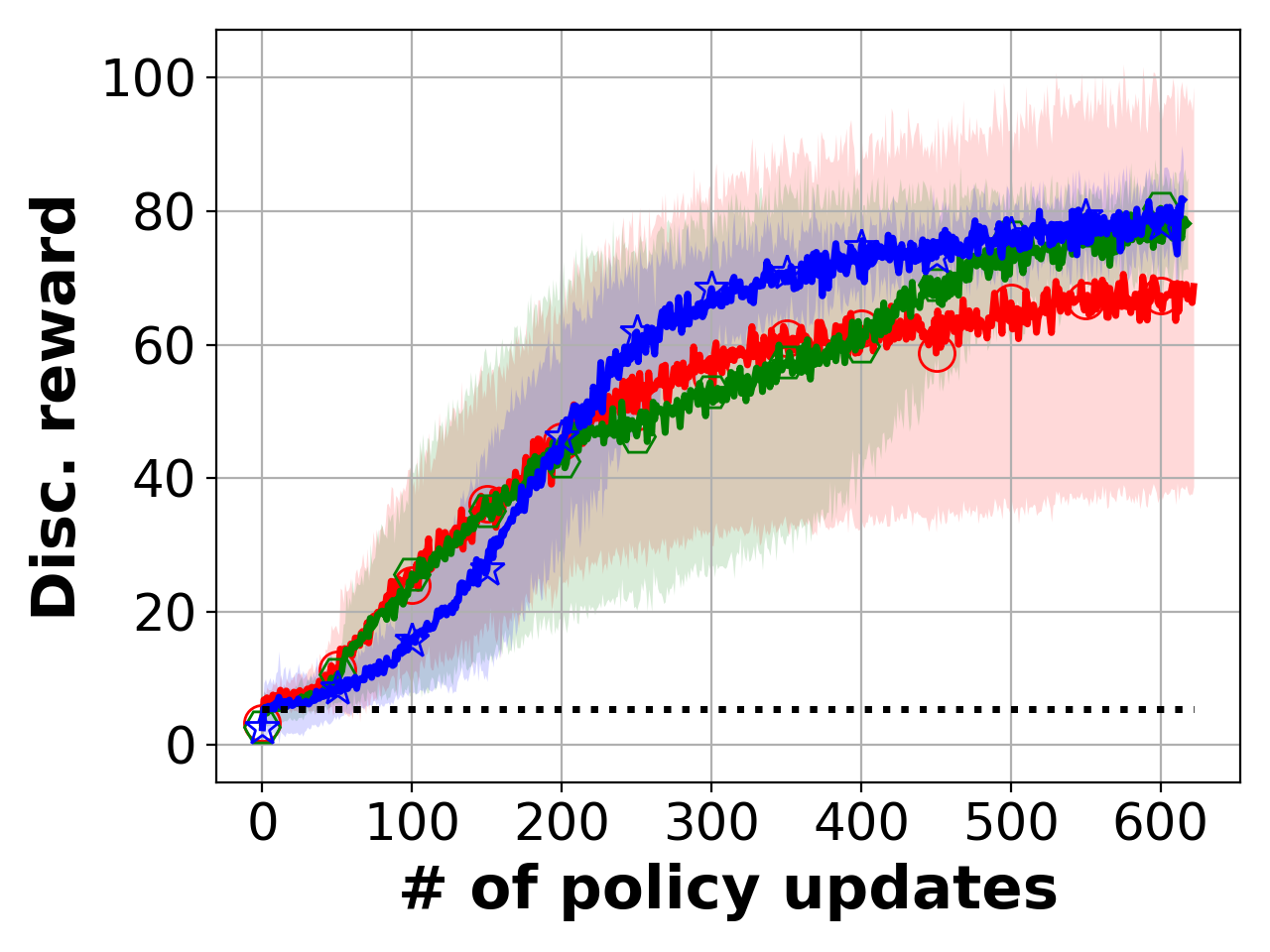}%
\includegraphics[width=0.33\linewidth]{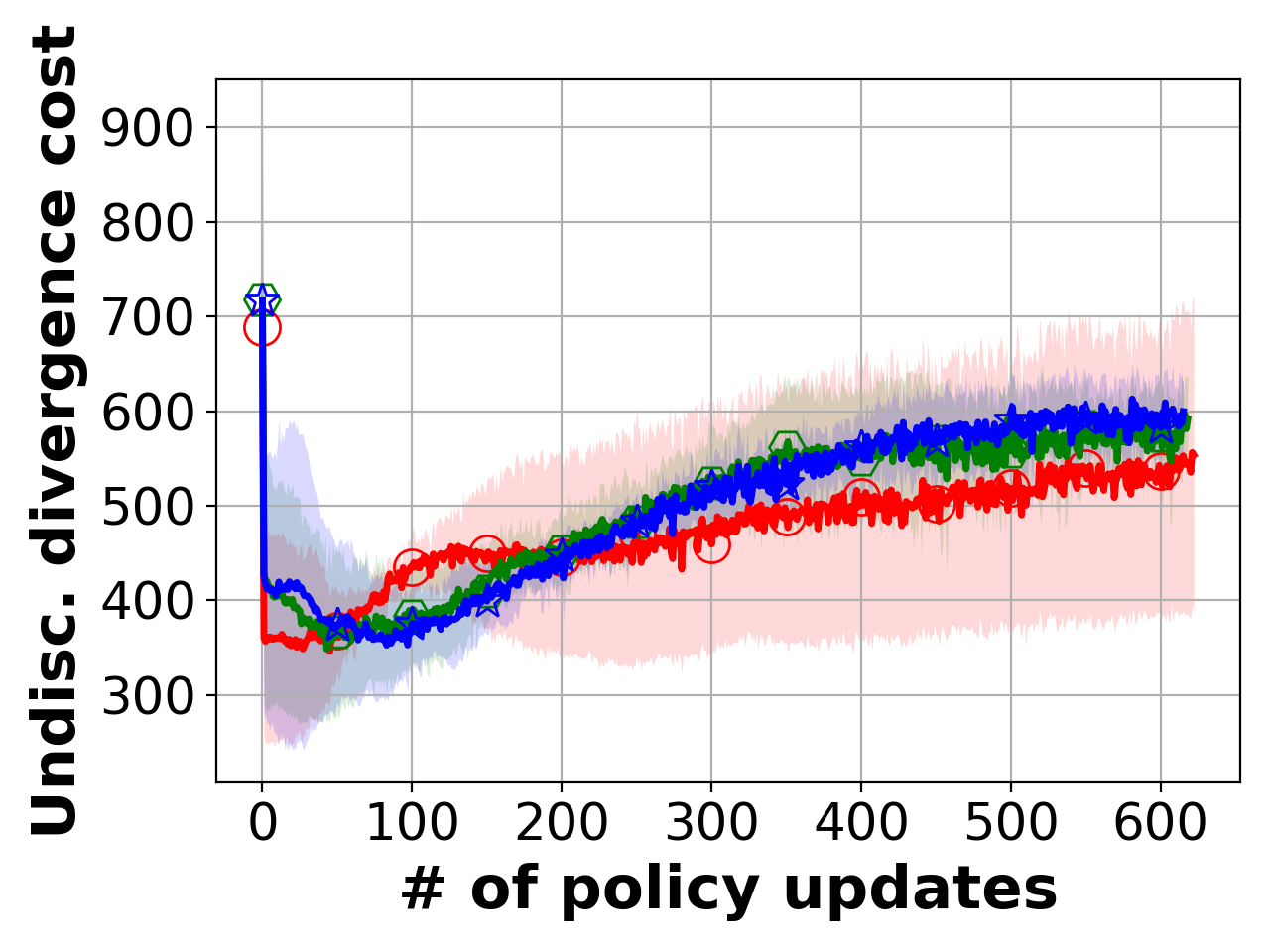}%
\end{tabular}}%

\vspace{+1mm}

\includegraphics[width=0.7\linewidth]{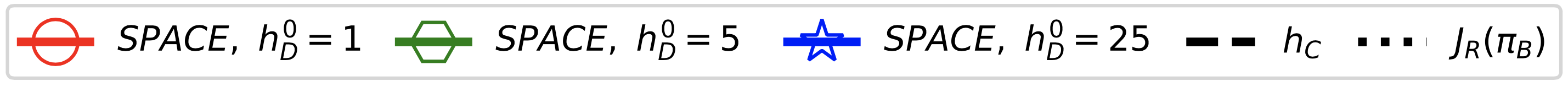}
\vspace{-2mm}

\caption{
The undiscounted constraint cost,
the discounted reward, and
the undiscounted divergence cost
over policy updates for the tested algorithms and tasks.
The solid line is the mean and the shaded area is the standard deviation over 5 runs.
We observe that the initial value of $h_D^0$ does not affect the reward and the cost performance significantly.
(Best viewed in color.)
}
\label{fig:appendix_initialPriorConstraintThreshold}
\vspace{-3mm}
%\end{mdframed}
\end{figure*}

\section{Human Policies}
\label{appendix:human_policy}
We now describe the procedure for collecting human demonstration data in the car-racing task.
A player uses the right key, left key, up key and down key to control the direction, acceleration, and brake of the car. 
The human demonstration data contain the display of the game (\ie the observed state), the actions, and the reward.
We collect 20 minutes of demonstration data.
A human player is instructed to stay in the lane but does not know the cost constraint.
This allows us to test whether \algname\ can safely learn from the baseline policy which need not satisfy the cost constraints.
We then use an off-policy algorithm (DDPG) trained on the demonstration data to get the baseline human policy.
Since the learned baseline human policy does not interact with the environment, its reward performance cannot be better than the human performance.
Fig. \ref{fig:cr_human} shows the procedure.

\paragraph{Implementation Details of DDPG.} 
We use DDPG as our off-policy algorithm. 
We use a convolutional neural network with two convolutional operators of size 24 and 12 followed by a dense layer of size (32, 16) to represent a Gaussian policy. 
A Q function shares the same architecture of the policy.
The learning rates of the policy and Q function are set to $10^{-4}$ and $10^{-3},$ respectively.
%
%The remaining hyperparameters can be found in ~\url{https://sites.google.com/view/spaceneurips}.

\begin{figure*}[t]
\centering
\includegraphics[scale=0.25]{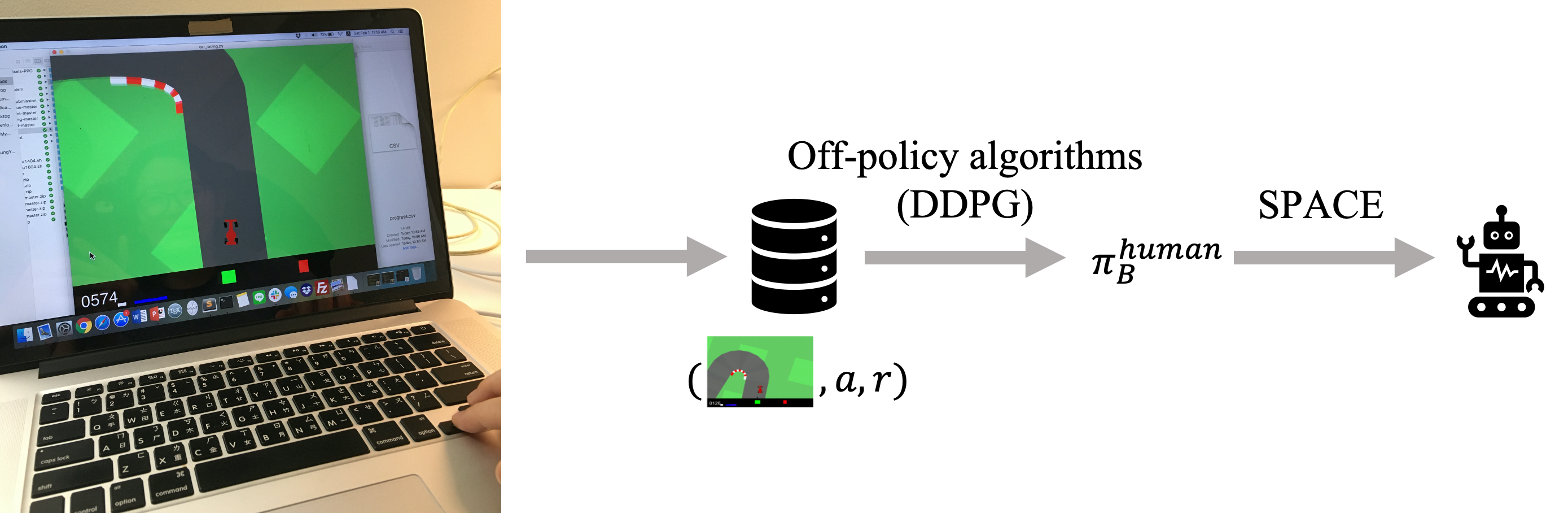}
\caption{
Procedure for getting a baseline human policy. 
We ask a human to play the car-racing game.
He/She does not know the cost constraint.
The trajectories (\ie display of the game, the action, and the reward) are then stored.
A human policy is obtained by using an off-policy algorithm (DDPG) trained on the trajectories.
}
\label{fig:cr_human}
\end{figure*}

\iffalse
\section{Comparison between PCPO \citep{yang2020projection} and CPO \citep{achiam2017constrained}} Fig. \ref{fig:pcpo_cpo}.

\begin{figure*}[t]
\vspace{-3mm}
\centering
\subfloat[PCPO \citep{yang2020projection}]{
\includegraphics[width=0.23\linewidth]{figure/pcpo.pdf}}
\subfloat[CPO \citep{achiam2017constrained}]{
\includegraphics[width=0.23\linewidth]{figure/cpo.pdf}}
\caption{Update procedures of for PCPO and CPO. 
}
\vspace{-3mm}
\label{fig:pcpo_cpo}
\vspace{-3mm}
%\end{mdframed}
\end{figure*}
\fi

\section{The Machine Learning Reproducibility Checklist (Version 1.2, Mar.27 2019)}
\label{appendix:sec:reproduce}
For all models and algorithms presented, indicate if you include\footnote{Here is a link to the list: \url{https://www.cs.mcgill.ca/~jpineau/ReproducibilityChecklist.pdf}.}:
\begin{itemize}
\item A clear description of the mathematical setting, algorithm, and/or model:
    \begin{itemize}
         \item \textbf{Yes}, please see the problem formulation in Section \ref{sec:preliminaries}, the update procedure for \algname\ in Section \ref{sec:implementation}, and the architecture of the policy in Section \ref{subsec:appendix_details}.
    \end{itemize}
\item An analysis of the complexity (time, space, sample size) of any algorithm:
    \begin{itemize}
        \item \textbf{Yes}, please see the implementation details in Section \ref{subsec:appendix_details}.
    \end{itemize}
\item A link to a downloadable source code, with specification of all dependencies, including external libraries:
    \begin{itemize}
        \item \textbf{Yes}, please see the implementation details in Section \ref{subsec:appendix_details}.
    \end{itemize}
\end{itemize}
For any theoretical claim, check if you include:
\begin{itemize}
\item A statement of the result:
    \begin{itemize}
        \item \textbf{Yes}, please see Section \ref{sec:model} and Section \ref{subsec:p2cpoConvergence}.
    \end{itemize}
\item A clear explanation of any assumptions:
    \begin{itemize}
        \item \textbf{Yes}, please see Section \ref{sec:model} and Section \ref{subsec:p2cpoConvergence}.
    \end{itemize}
\item A complete proof of the claim:
    \begin{itemize}
        \item \textbf{Yes}, please see Section~\ref{appendix:sec:theorem:h_D}, Section \ref{appendix:proof_update_rule_1},
        and Section \ref{appendix:sec:converge}.
    \end{itemize}
\end{itemize}
For all figures and tables that present empirical results, indicate if you include:
\begin{itemize}
\item A complete description of the data collection process, including sample size:
    \begin{itemize}
        \item \textbf{Yes}, please see Section \ref{subsec:appendix_details} for the implementation details.
    \end{itemize}    
\item A link to a downloadable version of the dataset or simulation environment:
    \begin{itemize}
        \item \textbf{Yes}, please see Section \ref{subsec:appendix_details} for the simulation environment.
    \end{itemize}  
\item An explanation of any data that were excluded, description of any pre-processing step:
    \begin{itemize}
        \item \textbf{It's not applicable.} This is because that data comes from simulated environments.
    \end{itemize}  

\item An explanation of how samples were allocated for training / validation / testing:
    \begin{itemize}
        \item \textbf{It's not applicable.} The complete trajectories (\ie data) is used for training. There is no validation set. Testing is performed in the form of online learning approaches.
    \end{itemize}  
    
\item The range of hyper-parameters considered, method to select the best hyper-parameter configuration, and specification of all hyper-parameters used to generate results:
    \begin{itemize}
        \item \textbf{Yes}, we randomly select five random seeds, and please see Section \ref{subsec:appendix_details} for the implementation details. 
    \end{itemize}  
    
\item The exact number of evaluation runs:
    \begin{itemize}
        \item \textbf{Yes}, please see Section \ref{subsec:appendix_details} for the implementation details. 
    \end{itemize}  

\item A description of how experiments were run:
    \begin{itemize}
        \item \textbf{Yes}, please see Section \ref{subsec:appendix_details} for the implementation details. 
    \end{itemize}  
    
\item A clear definition of the specific measure or statistics used to report results:
    \begin{itemize}
        \item \textbf{Yes}, please see Section \ref{sec:experiments}. 
    \end{itemize}

\item Clearly defined error bars:
    \begin{itemize}
        \item \textbf{Yes}, please see Section \ref{sec:experiments}. 
    \end{itemize}  

\item A description of results with central tendency (\eg mean) variation (\eg stddev):
    \begin{itemize}
        \item \textbf{Yes}, please see Section \ref{sec:experiments}. 
    \end{itemize}  

\item A description of the computing infrastructure used:
    \begin{itemize}
        \item \textbf{Yes}, please see Section \ref{subsec:appendix_details} for the implementation details. 
    \end{itemize}
\end{itemize}